\documentclass{article}
\usepackage[preprint,nonatbib]{neurips_2020}

\usepackage[utf8]{inputenc} % allow utf-8 input
\usepackage[T1]{fontenc}    % use 8-bit T1 fonts
\usepackage{hyperref}       % hyperlinks
\usepackage{url}            % simple URL typesetting
\usepackage{booktabs}       % professional-quality tables
\usepackage{amsfonts}       % blackboard math symbols
\usepackage{nicefrac}       % compact symbols for 1/2, etc.
\usepackage{microtype}      % microtypography

\usepackage{amsmath,amsthm}
\usepackage{cleveref}
\usepackage{caption,subcaption}
\usepackage{algorithm,algpseudocode}
\usepackage{multirow}
\usepackage{graphicx}
\usepackage{xr-hyper}

\newtheorem{theorem}{Theorem}
\newcommand{\R}{\mathbb{R}}

\title{Toward Understanding Why Adam Converges Faster Than SGD for Transformers}

\author{
  Yan Pan\\
  Carnegie Mellon University\\
  \texttt{ypan2@andrew.cmu.edu}\\
  \And
  Yuanzhi Li\\
  Carnegie Mellon University\\
  \texttt{yuanzhil@andrew.cmu.edu}\\
}

\begin{document}

\maketitle

\begin{abstract}
    While stochastic gradient descent (SGD) is still the most popular optimization algorithm in deep learning, adaptive algorithms such as Adam have established empirical advantages over SGD in some deep learning applications such as training transformers.
    However, it remains a question that why Adam converges significantly faster than SGD in these scenarios.
    In this paper, we propose one explanation of why Adam converges faster than SGD using a new concept \emph{directional sharpness}.
    We argue that the performance of optimization algorithms is closely related to the directional sharpness of the update steps, and show SGD has much worse directional sharpness compared to adaptive algorithms.
    We further observe that only a small fraction of the coordinates causes the bad sharpness and slow convergence of SGD, and propose to use coordinate-wise clipping as a solution to SGD and other optimization algorithms.
    We demonstrate the effect of coordinate-wise clipping on sharpness reduction and speeding up the convergence of optimization algorithms under various settings.
    We show that coordinate-wise clipping improves the local loss reduction when only a small fraction of the coordinates has bad sharpness.
    We conclude that the sharpness reduction effect of adaptive coordinate-wise scaling is the reason for Adam's success in practice and suggest the use of coordinate-wise clipping as a universal technique to speed up deep learning optimization.
\end{abstract}

\section{Introduction\label{sec:intro}}

Stochastic gradient descent (SGD)~\cite{robbins1951stochastic,bottou1991stochastic} is one of the most widely used optimization algorithms for deep learning, due to its simplicity and efficiency on various large-scale neural networks. However, in some tasks, such as training transformers~\cite{vaswani2017attention,devlin2019bert}, which are powerful models for natural language processing and other domains, SGD often performs poorly compared to adaptive variants of stochastic gradient methods. Adaptive algorithms, such as Adagrad~\cite{duchi2011adaptive}, Adam~\cite{kingma2015adam}, and AMSGrad~\cite{reddi2018convergence}, adjust the learning rate for each parameter based on the magnitude and history of the gradients, which can help them exploit the local geometry of the objective function and escape from saddle points or plateaus. While adaptive algorithms have shown empirical advantages over SGD in many applications~\cite{goodfellow2016deep,zhang2020adaptive,duchi2013estimation}, the theoretical understanding of their superior performance in these tasks is limited~\cite{zhang2020adaptive,danilova2022recent}. The best known non-convex convergence rate for AMSGrad~\cite{reddi2018convergence} only matches the best convergence rate of SGD but does not improve upon it~\cite{zhou2018convergence, ghadimi2013stochastic}. While pursuing a better general convergence rate for adaptive algorithms is a possible but challenging direction, a more realistic and relevant question is what makes Adam so effective and SGD so ineffective on certain architectures and tasks, such as transformers on language tasks. We aim to identify some properties of transformers that give rise to this phenomenon, and to find some quantities that can indicate the performance of different optimization algorithms in practice. Such insights could then be used to guide the selection and design of faster and more robust optimization algorithms for deep learning.

In this paper, we propose one possible explanation for why Adam converges faster than SGD in practice, especially for transformers. We begin by revisiting a classic simple example. Consider minimizing the diagonal quadratic function $f : \R^d \to \R$ given by $f(x) = x^\top A x$, where $A_{11} = 100$ and $A_{ii} = 1$ for all $i > 1$. The gradient is given by $\nabla f(x) = (200x_1, 2x_2, \dots, 2x_d)$ and the Hessian has spectral norm $200$. If we run gradient descent, then by standard convex optimization analysis, we can choose a learning rate at most $\frac{1}{100}$ for any initial point, which will result in slow convergence. However, if we run adaptive algorithms, signSGD, or simply clip the first coordinate, we can use a much larger learning rate and converge in a few steps. Although this example is much simpler than practical applications of adaptive algorithms, it illustrates the key idea that coordinate-wise scaling can help adaptive algorithms to adjust their step size on different coordinates and exploit the curvature of the function. We wonder if there are similar phenomena in real-world neural networks.

Inspired by this example, we study the local geometry of transformers in~\Cref{sec:directional_sharpness}. Instead of analyzing the global convergence and trajectory of optimization algorithms, we focus on the simpler question of finding a good update direction in a fixed local geometry. We decompose the goal of locally minimizing the objective function into two components: \emph{gradient correlation}, which measures the alignment of the update direction with the negative gradient, and \emph{directional sharpness}, which measures the curvature of the function along the update direction. We argue that the directional sharpness of the update direction is a more useful indicator of the performance of optimization algorithms, as high sharpness usually implies low performance.
%, and is more important than the gradient correlation in optimization. 
Empirically, we observe through experiments that the update directions of SGD have much higher directional sharpness compared to adaptive algorithms. By studying more algorithms, we observe that in general, algorithms with high directional sharpness converge much slower than adaptive algorithms, which typically have low directional sharpness. We also visualize the corresponding landscape along the update directions, and our results show that algorithms with low directional sharpness can generally achieve a better local loss reduction if optimal step sizes are chosen.

We investigate the cause of SGD's high directional sharpness and find that it is mainly due to the imbalanced distribution of gradient across coordinates. We observe that only a small fraction of the coordinates account for most of SGD's directional sharpness and we infer that it is because of the positive correlation between the Hessian and gradient coordinates. To address this issue, we propose to use coordinate-wise clipping as a simple and effective technique to improve the convergence and directional sharpness of optimization algorithms. The intuition behind clipping is that when a few coordinates have large gradients and bad smoothness, clipping prevents them from dominating the update direction and inflating the directional sharpness. Theoretically, we show that clipping improves the worst-case directional sharpness and enables a better local loss reduction with a larger step size. Empirically, we show that clipping can consistently reduce the directional sharpness, which often leads to a better local function reduction and improves the convergence speed of various optimization algorithms including adaptive algorithms. We demonstrate our findings through two experiments under different settings and show that our observations are robust across different tasks, models, and iterations. Based on the experiments, we argue that the landscape of optimization algorithms in local geometry is a useful proxy for the global convergence speed. We conclude that the \textbf{adaptive coordinate-wise scaling} of Adam can effectively balance the trade-off between optimizing gradient correlation and directional sharpness, and that this ability is the key to Adam's fast convergence in deep learning training. 

Our main contributions can be summarized as follows:
\vspace{-0.5em}
\begin{enumerate}
\item We identify directional sharpness as a key indicator of the performance of optimization algorithms in local geometry, and show that adaptive algorithms have low directional sharpness compared to SGD, especially when training transformers.
\item We propose coordinate-wise clipping as a simple, effective and \textbf{universal} technique to improve the directional sharpness and convergence speed of various optimization algorithms, and provide theoretical and empirical support for its benefits.
\end{enumerate}

\section{Related Work}
\textbf{General Convergence Rates of Adaptive Algorithms.}
Adaptive algorithms have long been studied and applied in deep learning~\cite{armijo1966minimization,polyakintroduction,duchi2011adaptive,kingma2015adam,tieleman2012lecture,reddi2018convergence}.
Several previous work has proved convex and non-convex convergence rates for Adagrad~\cite{duchi2011adaptive,li2019convergence,defossez2020simple,ward2019adagrad} and Adam or AMSGrad~\cite{de2018convergence,reddi2018convergence,fang2019convergence,chen2019convergence,zhou2018convergence,phuong2019convergence,zou2018convergence}.
The best known non-convex convergence rate for Adagrad is $O(\frac{\log T}{\sqrt{T}})$~\cite{li2019convergence,defossez2020simple} and $O(\frac{1}{\sqrt{T}})$ for AMSGrad~\cite{zhou2018convergence}. While the result by~\cite{zhou2018convergence} matches the non-convex convergence rate $O(\frac{1}{\sqrt{T}})$ of SGD~\cite{ghadimi2013stochastic}, there is no theoretical proof that Adam can converge asymptotically faster than SGD for general functions~\cite{danilova2022recent}.
Therefore, there is still a significant gap of work between the theoretical understanding of Adam and its empirical fast performance.

\textbf{Faster Convergence Rates Under Certain Settings.}
Another line of work focused on specific settings that Adam might work better than SGD.
Adaptive algorithms can work asymptotically better when the stochastic gradients are sparse~\cite{duchi2011adaptive,zhou2018convergence} or when there is a sparse set of noise~\cite{bernstein2018signsgd}.
\cite{zhang2020adaptive} proved that global clipping methods outperforms SGD when the stochastic gradients have heavy-tail noise, argued that Adam can also deal with heavy-tail noise effectively, and designed a new algorithm based on coordinate-wise clipping.

\textbf{Coordinate-Wise Clipping.}
Both global clipping~\cite{pascanu2013difficulty,zhang2020adaptive} and coordinate-wise clipping~\cite{goodfellow2016deep} are commonly used in practice with SGD.
While global norm clipping and normalization has been studied both theoretically and empirically~\cite{pascanu2013difficulty,levy2016power,hazan2015beyond,zhang2020adaptive}, there has been very little research on coordinate-wise clipping methods.
The most relevant work is \cite{zhang2020adaptive}, where the authors use coordinate-wise clipping to propose algorithms CClip and ACClip that works well on transformers in practice.
They use adaptive thresholds updated as momentum parameters and clip the coordinates to the corresponding thresholds.
\cite{zhang2020adaptive} shows that ACClip can perform empirically better than Adam on various transformers.

The coordinate-wise properties of the gradient and Hessian is often used in coordinate descent methods~\cite{wright2015coordinate,shi2016primer,richtarik2014iteration}.
Recently, due to its ability to deal with heavy-tailed noise~\cite{zhang2020adaptive}, coordinate-wise clipping has been applied in differentially private coordinate descent methods as it adapts to the coordinate-wise imbalance of the objective~\cite{mangold2022differentially,mangold2023high,pasande2022stochastic,pichapati2019adaclip}.
In particular, \cite{pasande2022stochastic} designs a strategy to choose an adaptive clipping threshold based on the mean of the gradients, while we use the distribution of the gradients to select a threshold that clips exactly a constant fraction of the gradients.

Our work is inspired by the use of coordinate-wise clipping in algorithm design in~\cite{zhang2020adaptive}, but we propose different explanations of the effectiveness of coordinate-wise clipping with new empirical evidence.
We highlight important differences between our work and the analysis of CClip and ACClip algorithms in~\cite{zhang2020adaptive}.
First, we propose different explanations for the performance of clipping. \cite{zhang2020adaptive} claims that clipping can deal with heavy-tailed noise in transformers, while we discover directional sharpness as a quantitative metric that directly relates to loss minimization and whose properties can be verified easily.
Second, while CClip and typical coordinate-wise clipping methods choose thresholds independent to the gradient, we choose an adaptive clipping threshold based on the distribution of the gradient.
Most importantly, while~\cite{zhang2020adaptive} focus on designing a new algorithm that can outperform Adam, we aim to propose coordinate-wise clipping as a meta algorithm, such that every optimization can use and improve its performance. Then, every algorithm can beat itself if clipping is added as a new unit, similar to the role of momentum in deep learning.
% TODO: performance of clipping algorithms

\section{Directional Sharpness of Optimization Algorithms\label{sec:directional_sharpness}}
In this section, we introduce a new measurement \textbf{directional sharpness} that indicates the performance of optimization algorithms.
We show that minimizing the term is extremely important to fast convergence of optimization algorithms and argue that it is closely related to the slow convergence of SGD.

\subsection{From Quadratic Taylor Expansion to Directional Sharpness}
In convex and non-convex optimization, a typical proof strategy is to consider the quadratic Taylor expansion of the objective function
\begin{equation}
    f(x_{t+1}) = f(x_t) + \underbrace{\nabla f(x_t)^\top (x_{t+1} - x_t)}_{\text{gradient correlation}} + \frac{1}{2}\underbrace{(x_{t+1} - x_t)^\top \nabla^2 f(x_t)(x_{t+1} - x_t)}_{\text{directional sharpness}} + O(\eta^3)
    \label{eq:taylor}
\end{equation}
where $x_{t+1} - x_t$ is the update step of the optimization algorithm and $\eta$ is the step size.
In order to get $f(x_{t+1}) \le f(x_t)$ in expectation, the optimization algorithm should minimize the two terms that depends on the update step, which we respectively denote \emph{gradient correlation}, which measures the alignment of the update direction with the negative gradient, and \emph{directional sharpness}, which measures the curvature of the function along the update direction.
To bound the second-order term, the default method in convex and non-convex optimization is to assume that the objective function is $L$-smooth~, which equivalently says $\|\nabla^2 f(x)\|_2 \le L$ for every $x$~\cite{bubeck2015convex}, where $\|\cdot\|_2$ is the spectral norm. 
The local Hessian spectral norm is often called the \emph{sharpness} of the function in deep learning~\cite{cohen2021gradient}.
If we have $L$ as the global upper bound on the spectral norm of the Hessian, we would have
\begin{equation}
    \frac{1}{2}(x_{t+1} - x_t)^\top \nabla^2 f(x_t)(x_{t+1} - x_t) \le \frac{1}{2}\|\nabla^2 f(x_t)\|_2\|x_{t+1} - x_t\|_2^2 \le \frac{L}{2} \|x_{t+1} - x_t\|_2^2.
\end{equation}
Then we have the following inequality, which is one of the most frequently used lemma in optimization proofs~\cite{bubeck2015convex,ghadimi2013stochastic,reddi2018convergence,zhou2018convergence}
\begin{equation}\label{eq:smooth}
    f(x_{t+1}) \le f(x_t) + \nabla f(x_t)^\top (x_{t+1} - x_t) + \frac{L}{2} \|x_{t+1} - x_t\|_2^2.
\end{equation}
If the function is $L$-smooth, the loss can decrease when the first-order term is negative and the norm of the update step is sufficiently small, since the second-order term is quadratic in the step size and the first-order term is linear. This can be guaranteed by using a small learning rate, and this leads to the convergence proofs of many optimization algorithms.
However, there are disadvantages of the smoothness assumption in theoretical proofs.
For example, the Hessian can adapt to the geometry of the trajectory and can vary significantly for different algorithms~\cite{cohen2021gradient,cohen2022adaptive}, so using a global upper bound in the convergence proof might not be fair for some algorithms.
Furthermore, even if the local geometry and Hessian are fixed, the update direction $x_{t+1} - x_t$ is also extremely important to minimizing the second-order term.
The current bound assumes that we are choosing the worst direction possible, but typically optimization algorithm might find better directions in probability.
We could probably believe that if a good direction is chosen, the second-order term can be much lower than the global upper bound, so the bound need not be tight.

Motivated by the definition of sharpness and the above observations, we define the \emph{directional sharpness} of a function $f$ at $x$ in the direction $v \in \mathbb{R}^d, \|v\|_2 = 1$ as $v^\top \nabla^2 f(x) v$.
The directional sharpness at $x_t$ in the update direction is extremely important to minimizing $f(x_{t+1})$.
Since directional sharpness is quadratic in the step size $\eta$ and gradient correlation is linear, if we consider \Cref{eq:taylor} as a quadratic function of $\eta$, a lower directional sharpness implies the potential to take a larger step size and possibly lead to a larger local reduction of the objective function.
In contrast, if the directional sharpness is large, we have no choice but to take a tiny step, as otherwise the loss would blow up due to the second-order term.
This implies that having a low directional sharpness can sometimes be a more desirable property for update directions than having a high gradient correlation.

Although our definition is motivated by the sharpness definition in deep learning, we highlight important differences between them.
Sharpness describes the \textbf{worst-case directional sharpness} and is the supremum of directional sharpness over all directions.
However, directional sharpness consider the sharpness in the specific \textbf{update direction} of an iterative optimization algorithm, and can be much lower than the sharpness if the direction is ``good''.
The concept of sharpness is typically associated with the landscape and generalization of neural networks, such as in Sharpness-Aware Minimization~\cite{foret2021sharpness} and Edge of Stability~\cite{cohen2021gradient,cohen2022adaptive}.
We are only interested in optimization of the objective function in the empirical risk minimization problem, or the loss on the training set.

\subsection{Directional Sharpness and Update Directions}
We study the update step of different optimization algorithms under the same trajectory and local geometry using pseudo-update steps to compute the momentum in order to rule out the impact of trajectory.
We compute the directional sharpness of different optimization algorithms and visualize the optimization landscape in the update direction of a variety of optimization algorithms in~\Cref{fig:landscape_example_sgd_machine_translation,fig:landscape_example_adam_machine_translation,fig:landscape_example_sgd_autoregressive} and~\Cref{tab:sharpness}.
The details of the experiment is described in~\Cref{sec:exp_main} and~\Cref{sec:exp_appendix}.
Empirically, we observe that there can be a significant gap between the directional sharpness in the update direction of different optimization algorithms.
In particular, the directional sharpness is \textbf{much lower for adaptive algorithms} than for SGD.

Based on the observation, we argue that minimizing the directional sharpness is more important for fast convergence of optimization algorithms as compared to minimizing the gradient correlation.
The update step of SGD has the best correlation with the actual gradient, so the loss decrease faster when the step size is very small, since in this case the linear term dominates the quadratic term in~\Cref{eq:taylor}.
However, because of the large directional sharpness, when the step size increases the quadratic term grows faster than the linear term, so the loss reaches the local minima in the direction after a very small step size.
For adaptive algorithms, the directional sharpness is much lower than SGD, so they have the potential to use a much larger step size and the optimal step could give a much lower loss compared to SGD.

\begin{figure}[h]
    \begin{subfigure}{0.32\textwidth}
        \includegraphics[width=\textwidth]{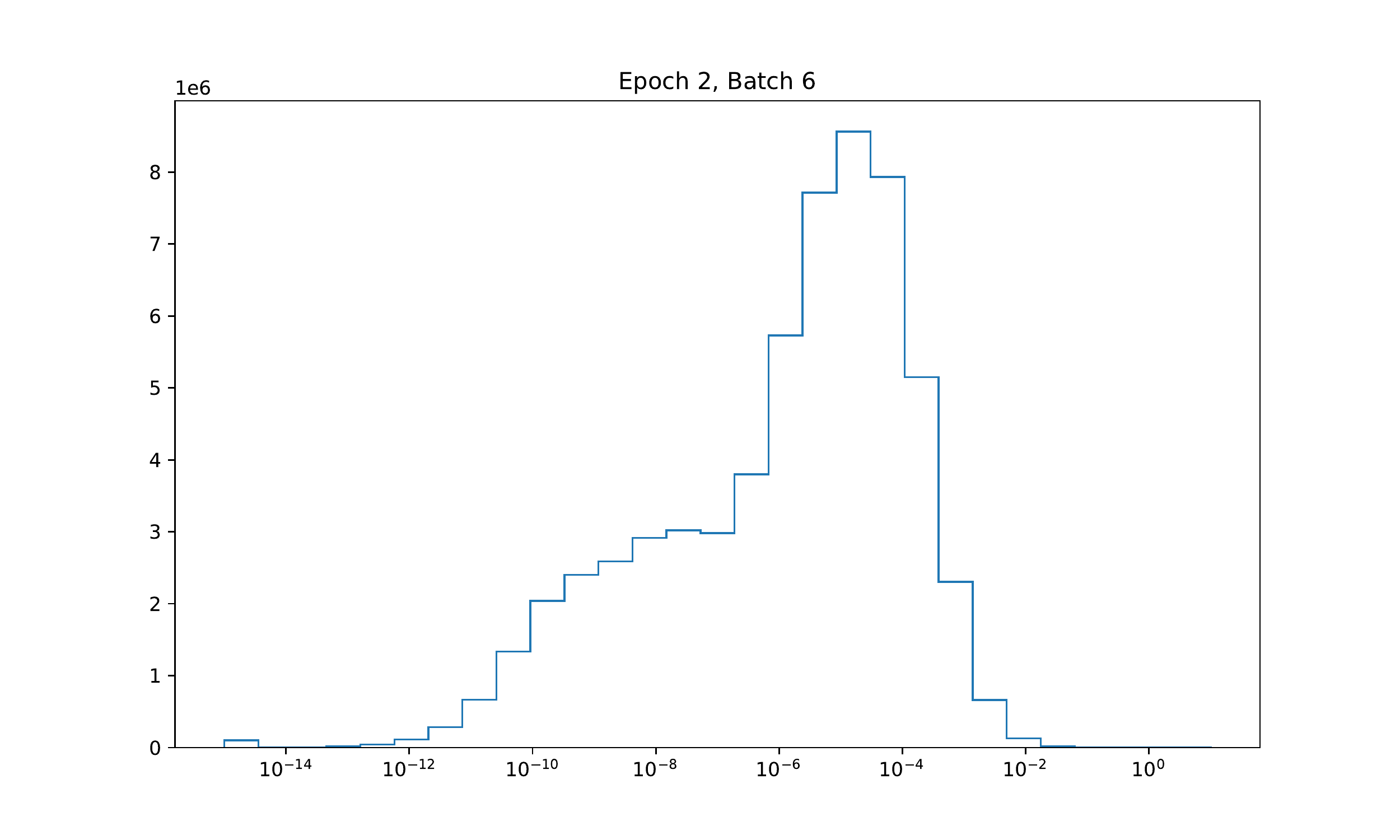}
        \caption{SGD}
    \end{subfigure}
    \begin{subfigure}{0.32\textwidth}
        \includegraphics[width=\textwidth]{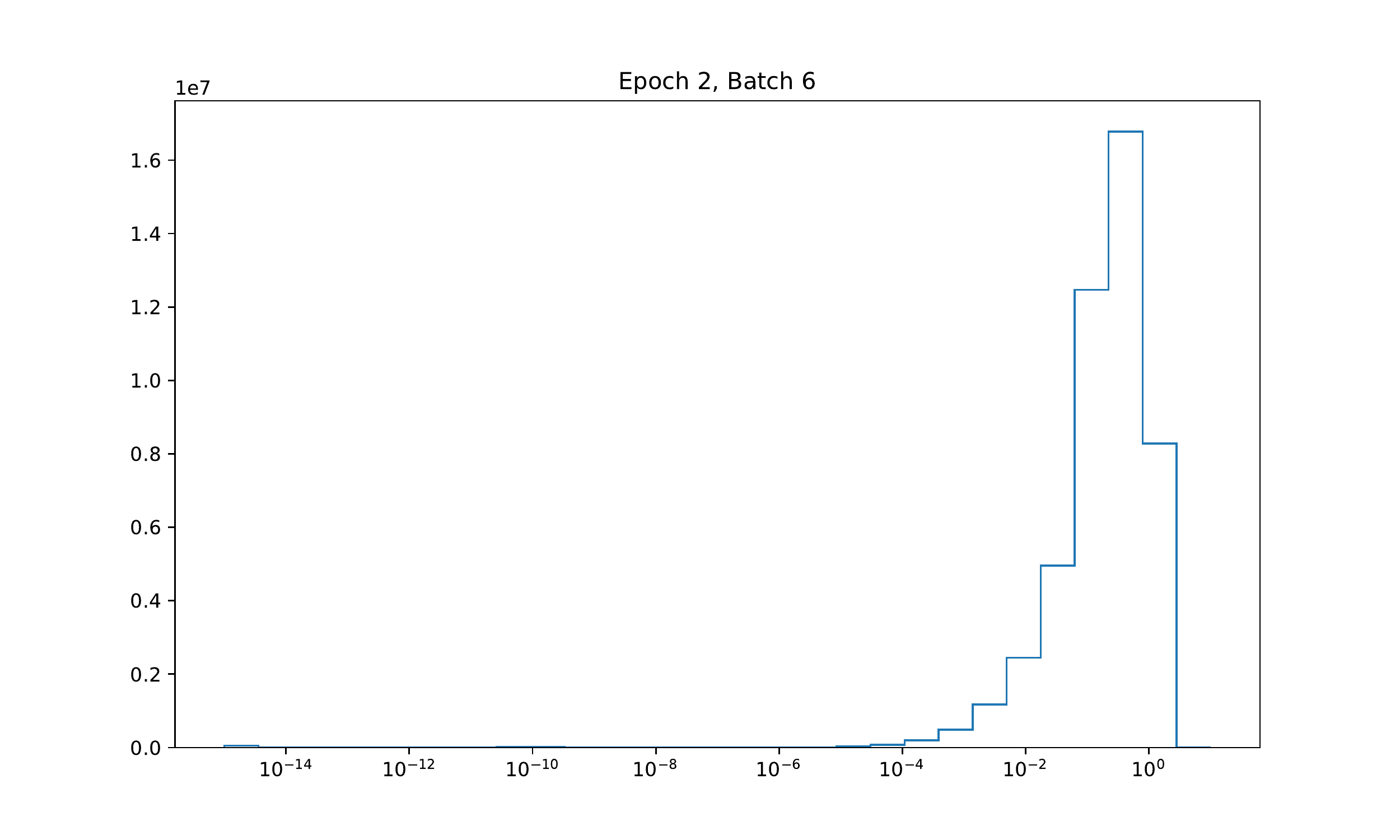}
        \caption{Adam}
    \end{subfigure}
    \begin{subfigure}{0.32\textwidth}
        \includegraphics[width=\textwidth]{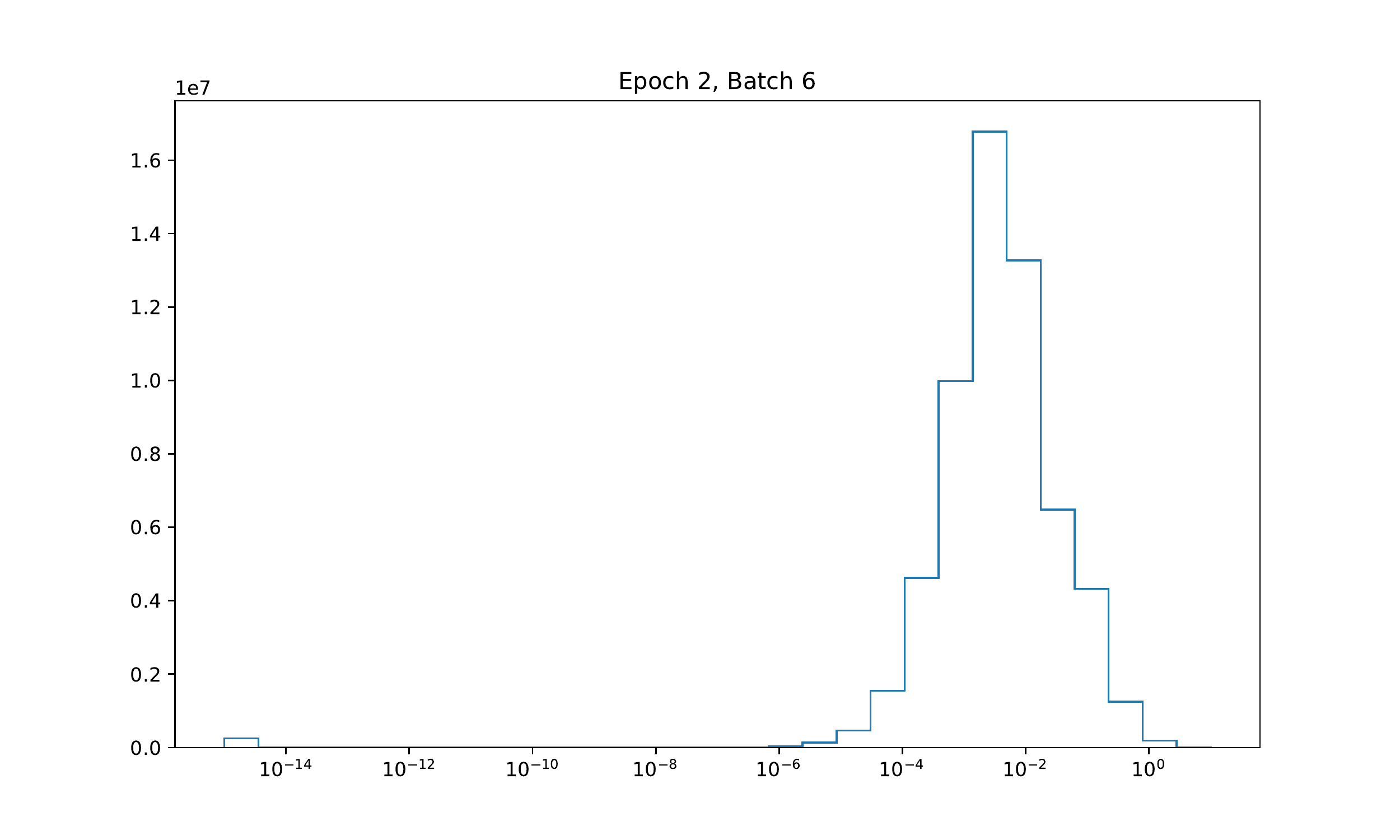}
        \caption{Adafactor}
    \end{subfigure}
    \caption{Histogram of update step distribution over coordinates for SGD, Adam, and Adafactor on machine translation.\label{fig:hist}}
\end{figure}
\begin{table}[h]
    \centering
    \begin{tabular}{|l|l|l|}
        \hline
        \textbf{Algorithm} & \textbf{Sharpness} & \textbf{Ratio to SGD}\\\hline
        SGD & $8.674583$ & $1$\\\hline
        SGD Clip 10\% & $0.527104$ & $0.060764$ \\\hline
        Adam & $0.252707$ & $0.029131$ \\\hline
        Adam Clip 50\% & $0.000574$ & $6.617 \times 10^{-5}$\\\hline
        Adafactor & $5.999 \times 10^{-5}$ & $6.916 \times 10^{-6}$ \\\hline
        Adafactor Clip 50\% & $2.051 \times 10^{-7}$ & $2.364 \times 10^{-8}$\\\hline
        Lion & $0.118202$ & $0.013626$ \\\hline
        Normalized SGD & $0.722253$ & $0.083261$ \\\hline
        Normalized SGD Clipping & $0.179141$ & $0.020651$ \\\hline
    \end{tabular}
    \caption{The sharpness of different optimization algorithms when trained on machine translation, in the same experiment and iteration as~\Cref{fig:landscape_example_sgd_machine_translation}.
    The directional sharpness of different optimization algorithms varies significantly.
    For example, the directional sharpness of SGD can be more than $10^7$ times the directional sharpness of Adafactor with clipping.
    Furthermore, clipping almost always improve the directional sharpness of optimization algorithms.
    \label{tab:sharpness}}
\end{table}

\begin{figure}[h]
    \centering
    \begin{subfigure}{\textwidth}
    \includegraphics[width=0.32\textwidth]{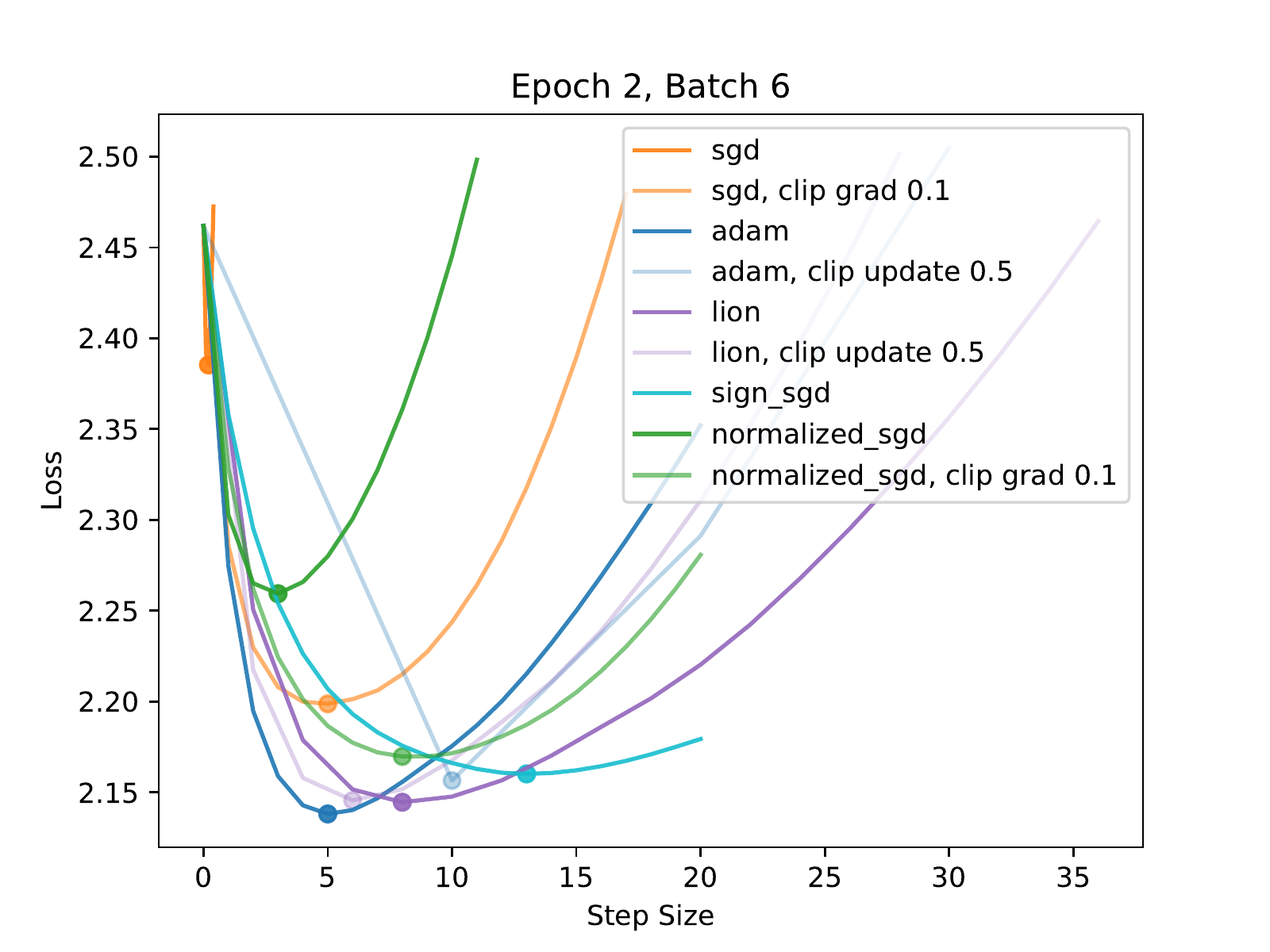}
    \includegraphics[width=0.32\textwidth]{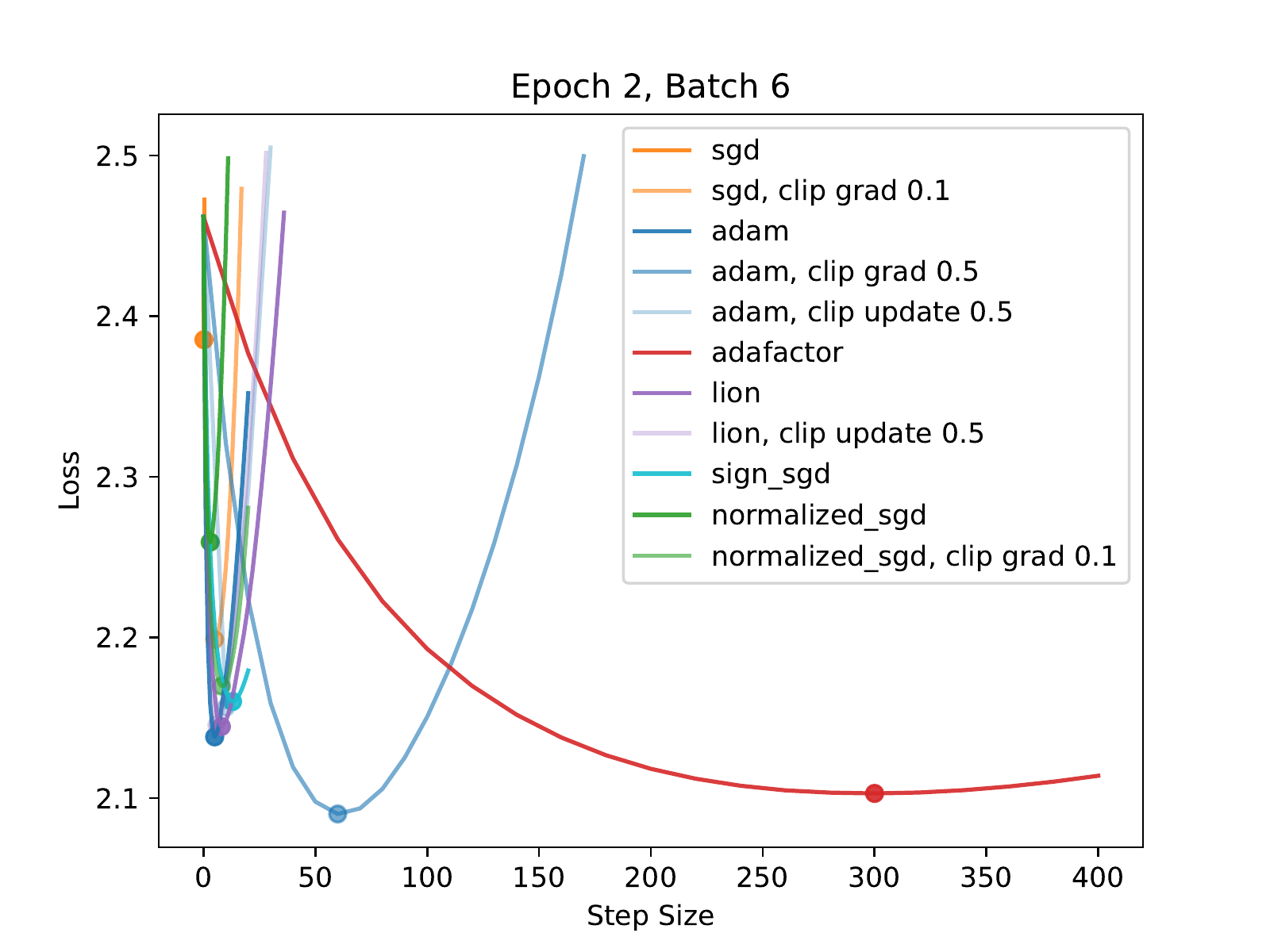}
    \includegraphics[width=0.32\textwidth]{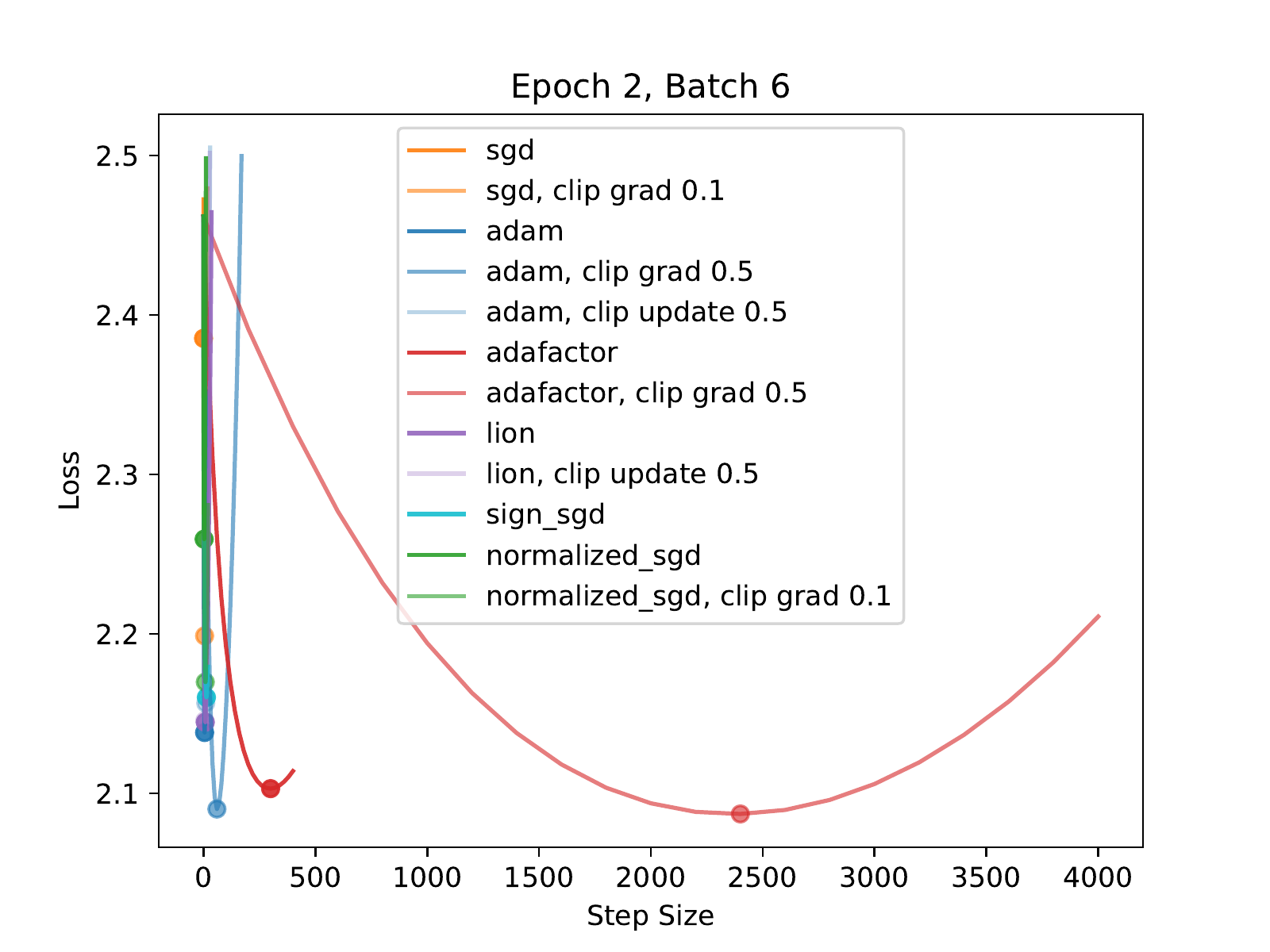}
    \end{subfigure}
    \begin{subfigure}{\textwidth}
    \includegraphics[width=0.32\textwidth]{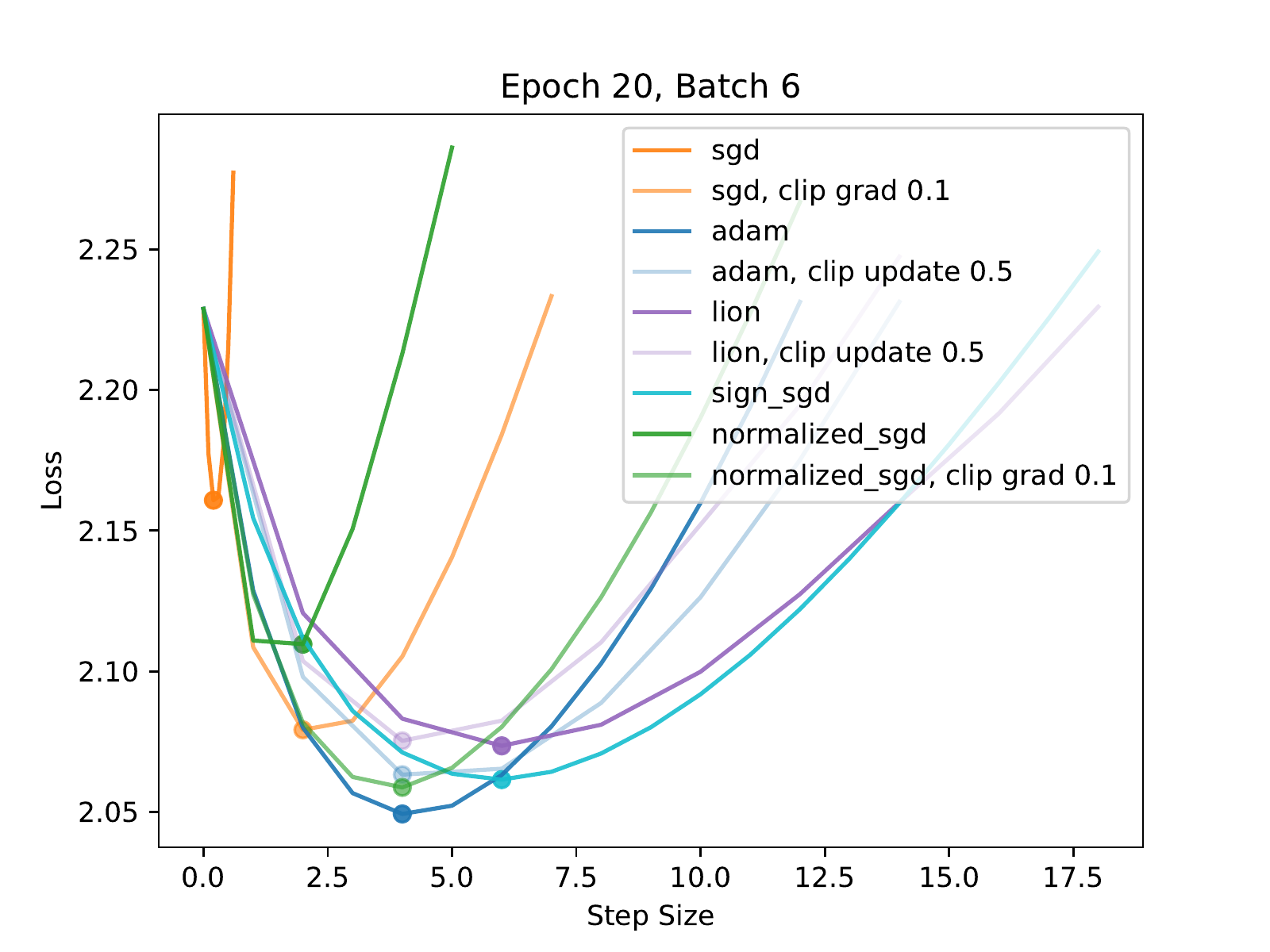}
    \includegraphics[width=0.32\textwidth]{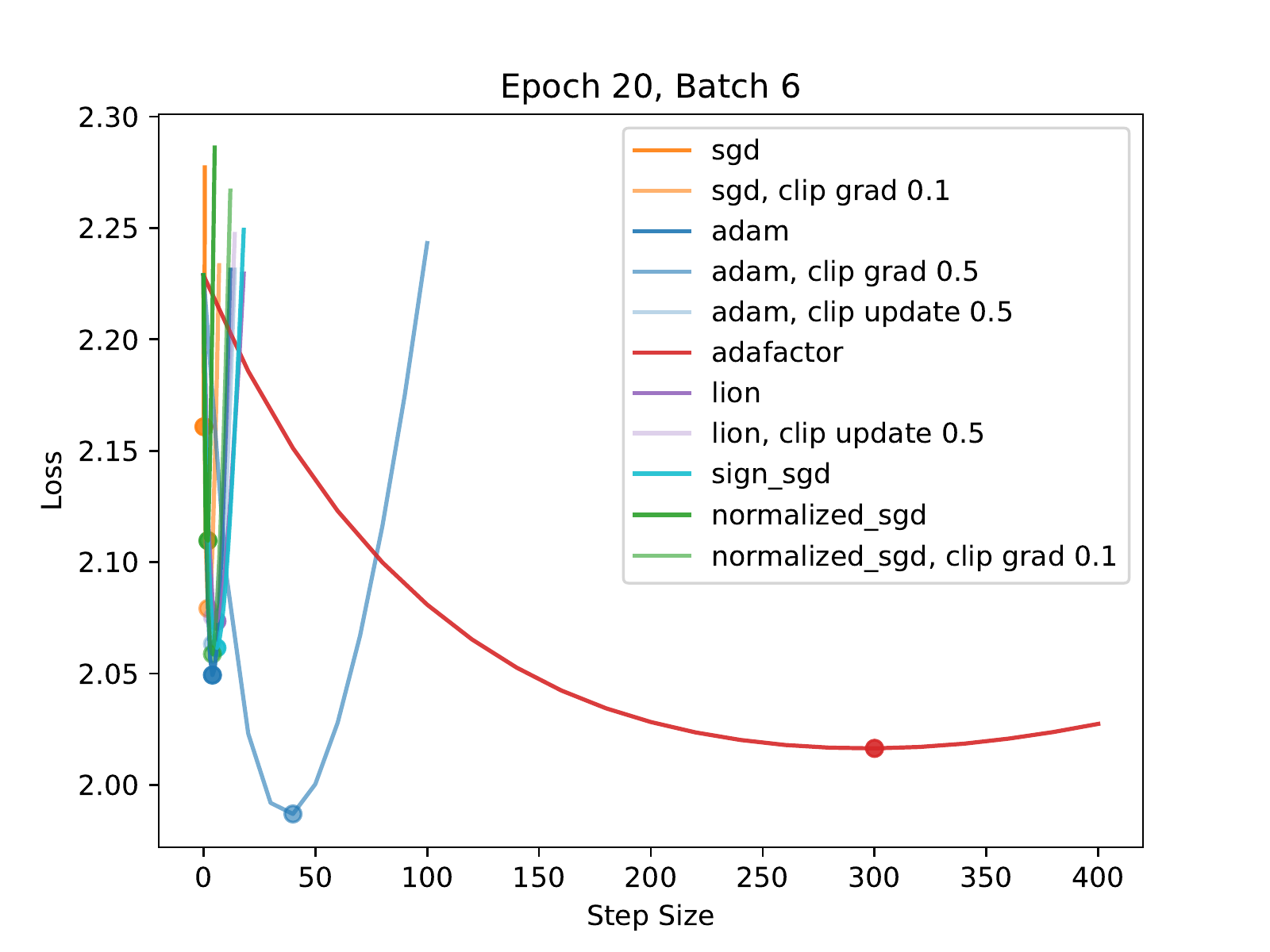}
    \includegraphics[width=0.32\textwidth]{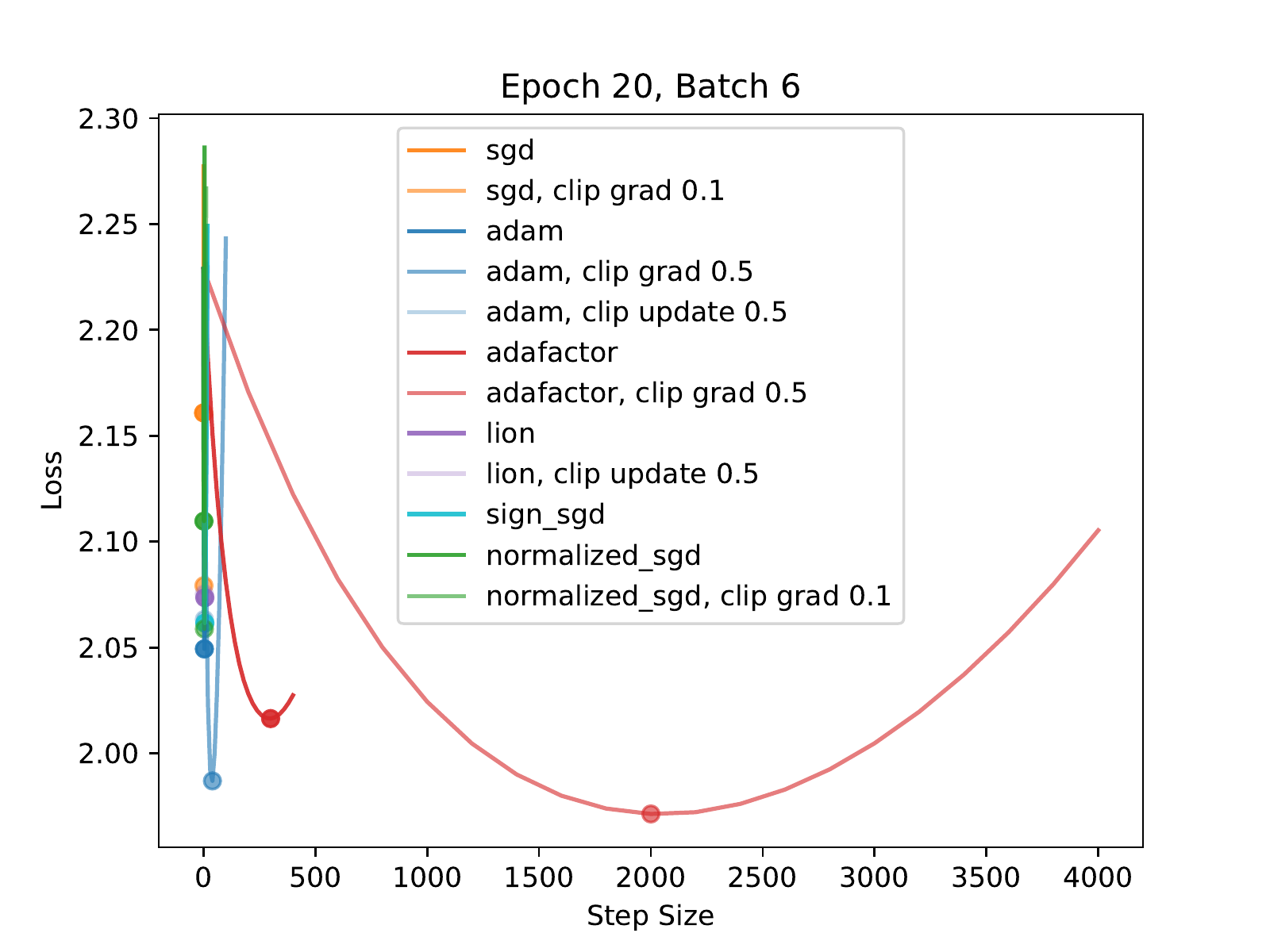}
    \end{subfigure}
    \caption{The loss landscape in different update directions on machine translation in SGD geometry. The step size is the learning rate normalized by the update step $\ell_2$ norm. The plots of clipped and unclipped variants of the same algorithm have the same color with different opacity.\label{fig:landscape_example_sgd_machine_translation}}
\end{figure}

\begin{figure}[h]
    \centering
    \begin{subfigure}{\textwidth}
    \includegraphics[width=0.32\textwidth]{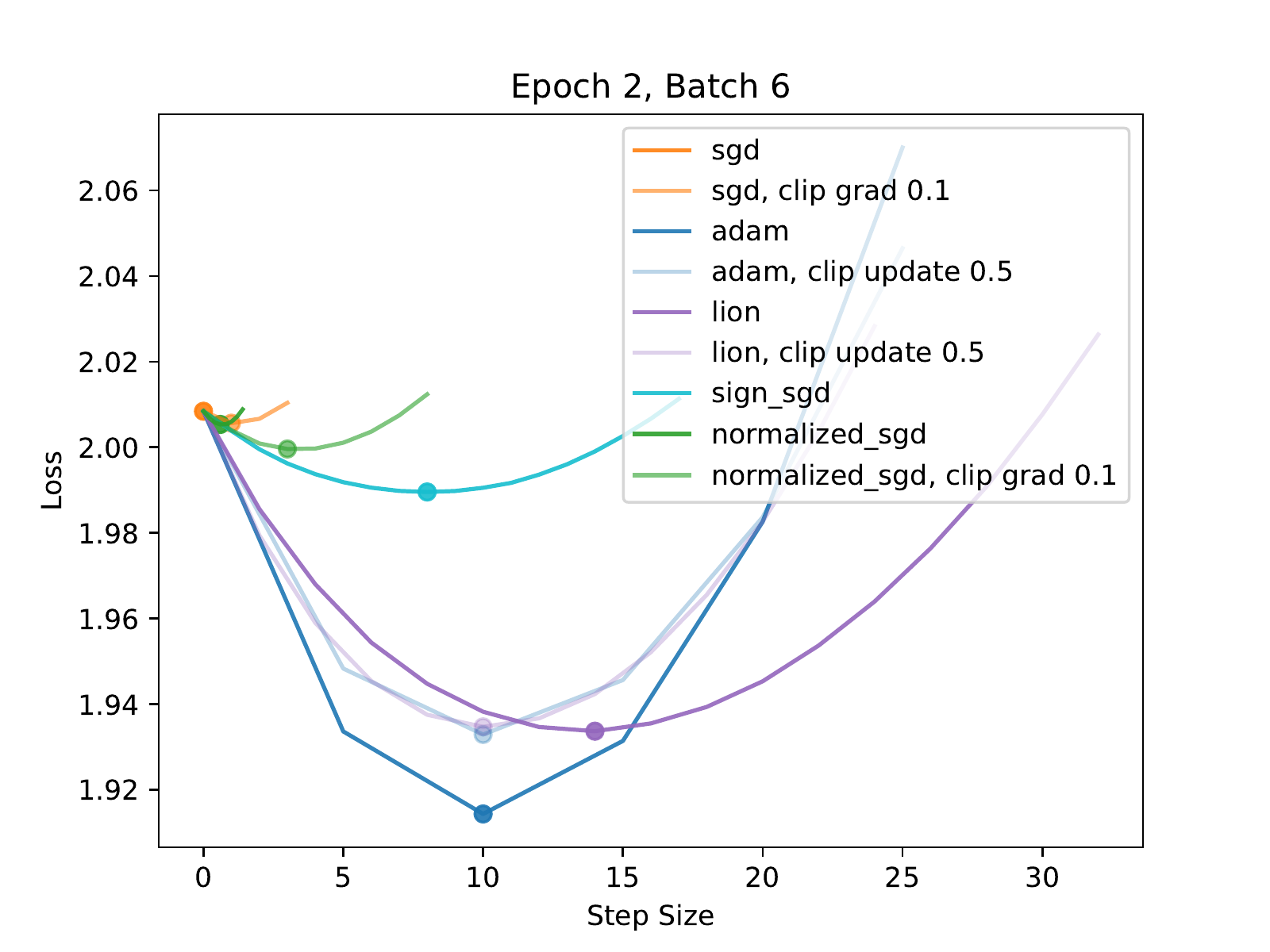}
    \includegraphics[width=0.32\textwidth]{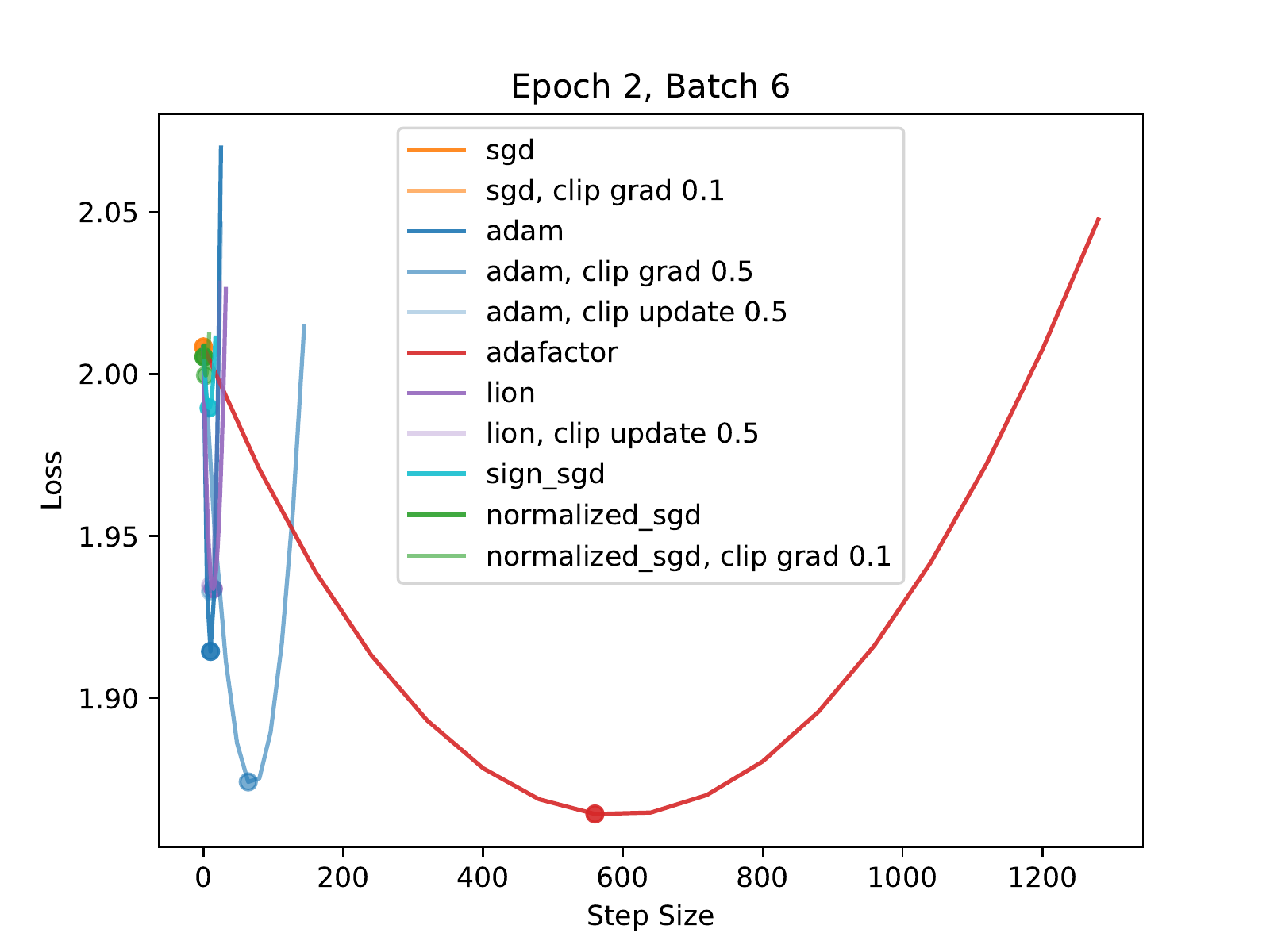}
    \includegraphics[width=0.32\textwidth]{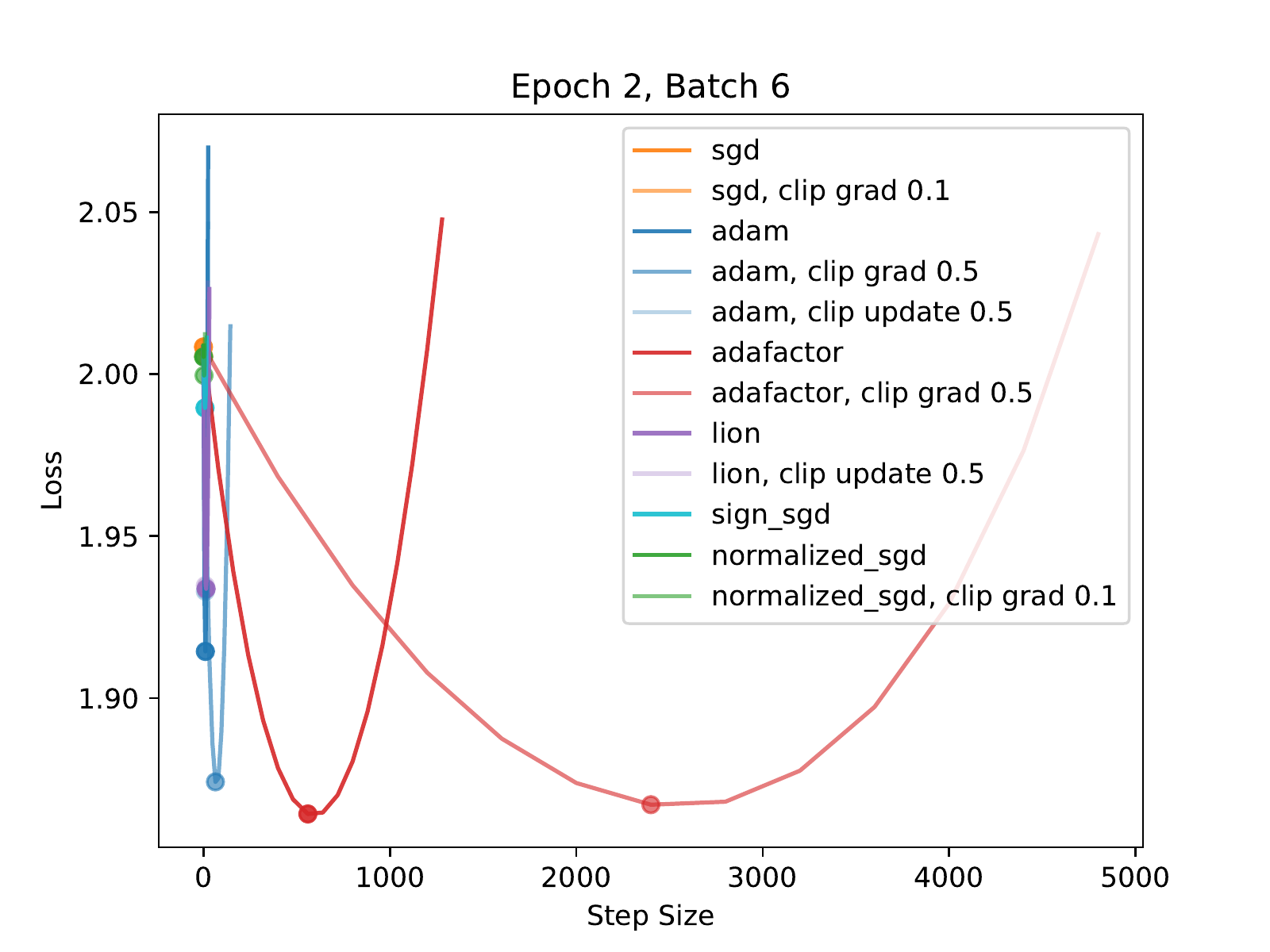}
    \end{subfigure}
    \caption{The loss landscape in different update directions on machine translation in Adam geometry.\label{fig:landscape_example_adam_machine_translation}}
\end{figure}

\begin{figure}[h]
    \centering
    \begin{subfigure}{\textwidth}
    \includegraphics[width=0.32\textwidth]{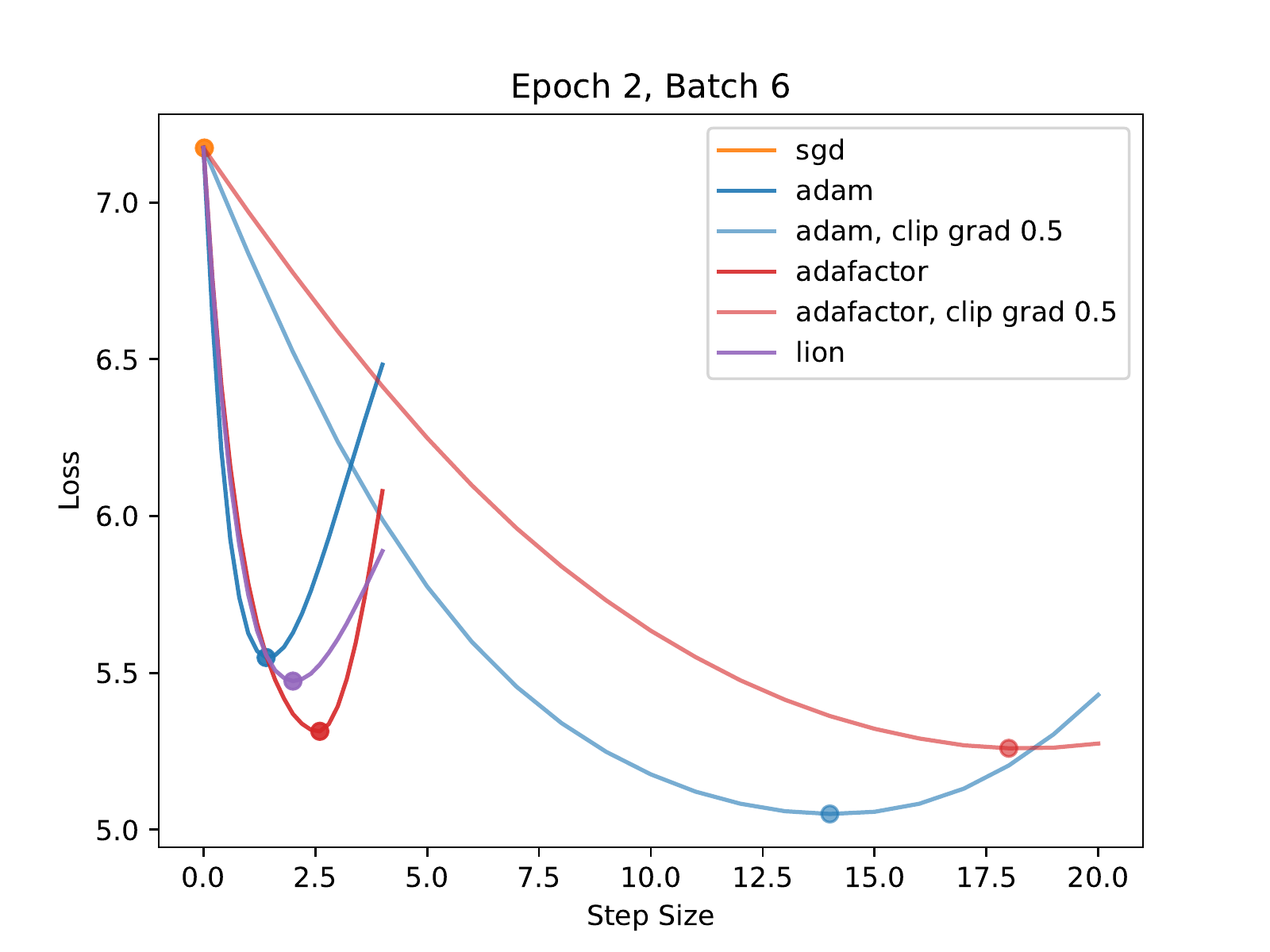}
    \includegraphics[width=0.32\textwidth]{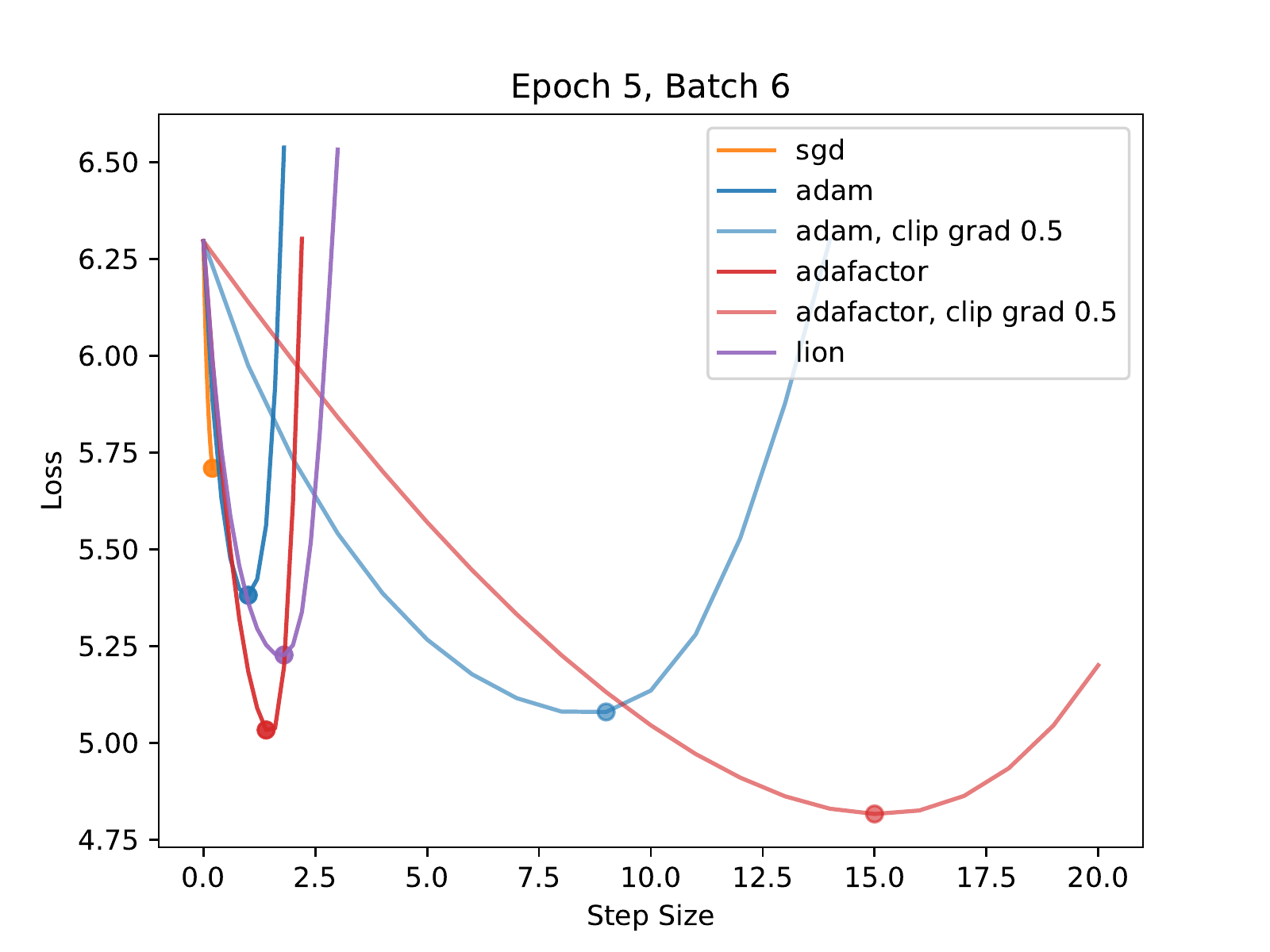}
    \includegraphics[width=0.32\textwidth]{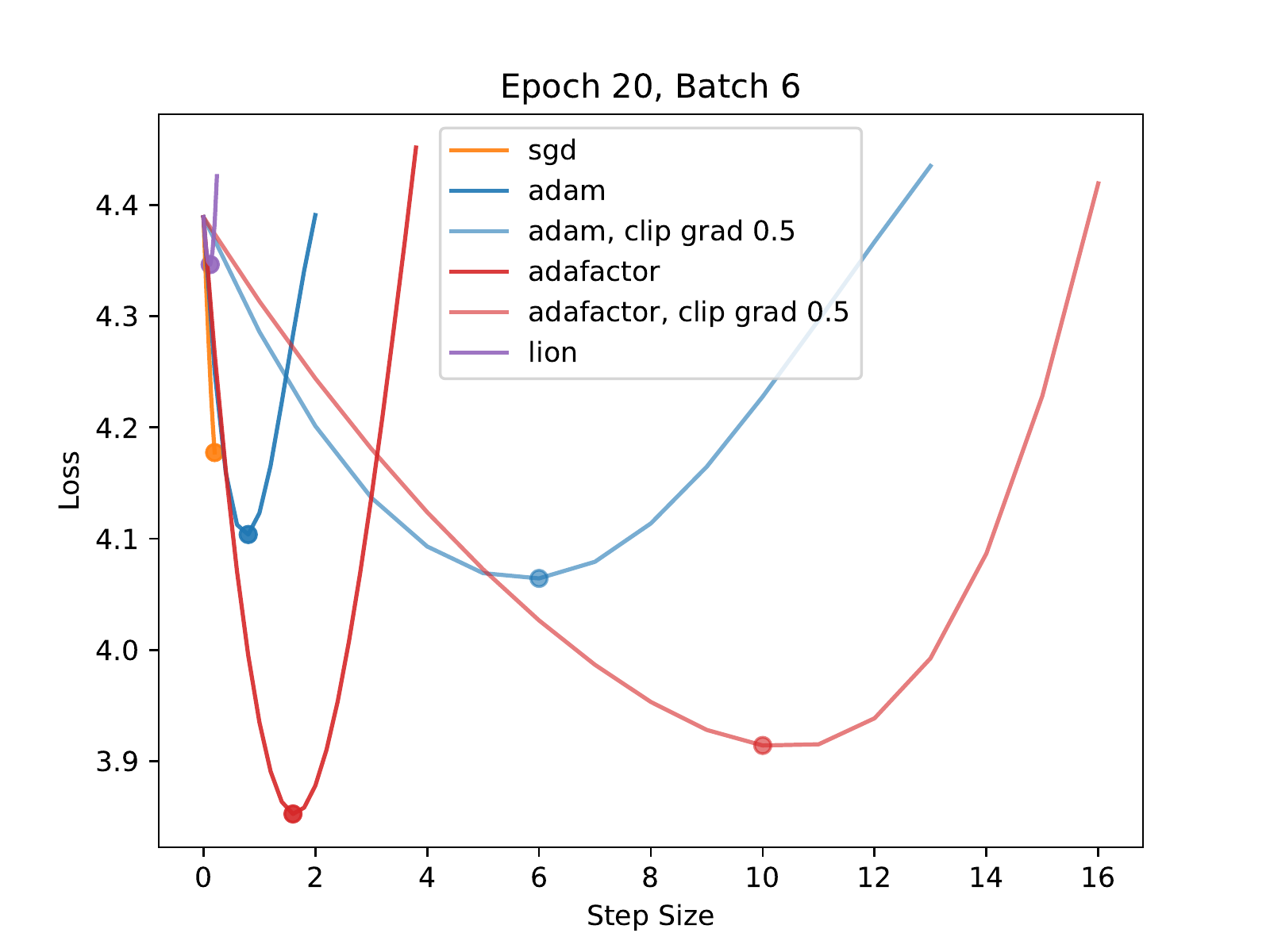}
    \end{subfigure}
    \caption{The loss landscape in different update directions on autoregressive language modeling in SGD geometry.\label{fig:landscape_example_sgd_autoregressive}}
\end{figure}

In order to explain the sharpness reduction effect of adaptive algorithms, since the strategy for adaptive algorithms is to find a coordinate-wise scaling of the gradient, we investigate the distribution of gradient norm across different coordinates. We visualize a histogram of the absolute value of SGD momentum coordinates in \Cref{fig:hist}.
We observe that the gradients are distributed unevenly across the coordinates, with half of the coordinates have absolute value ranging from $10^{-12}$ to $10^{-6}$, but also exists an innegligible portion of coordinates that can be as high as $10^{-4}$ to $10^{-2}$, contributing to most of the gradient norm.
The histogram suggests that the gradients are concentrated on a small fraction of the coordinates, and this small fraction of coordinates can contribute to a large portion of sharpness, making optimization hard.
For adaptive algorithms, since they already used some forms of scaling, the imbalanced gradient distribution will not be as severe as SGD.
As a result, they would have better convergence rate.

In~\Cref{sec:resnet}, we do a simple experiment with ResNet~\cite{he2016deep} on image classification that shows the property might be related to the transformer architecture.
In particular, the directional sharpness of adaptive algorithms might be worse than SGD for ResNets.
This is consistent with empirical observations of the performance of adaptive algorithms in vision tasks, that it is often slower than SGD.
% TODO: simple experiment on ResNet, show that this is not the case for ResNet

\section{Coordinate-wise Clipping\label{sec:clip}}
\subsection{Coordinate-wise Clipping Improves Directional Sharpness}
We propose to use \emph{coordinate-wise clipping} as a solution to the aforementioned imbalanced distribution of gradient based on our experimental findings.
We observe that the sharpness is also concentrated in the large coordinates in the gradient, and clipping those coordinates can significantly decrease directional sharpness.
Although clipping can decrease gradient correlation, since the dependence on the clipped entry is quadratic for the second-order term and linear for the first-order term, it might not be beneficial to use these coordinates.
The use of clipping in optimization algorithms is a trade-off between improving gradient correlation and reducing directional sharpness.
By clipping the top coordinates in the gradient, although gradient correlation decreases, the directional sharpness can decrease even more to make up the loss.

We consider using clipping on a variety of optimization algorithms, including SGD, normalized SGD, signSGD, Adam~\cite{kingma2015adam}, Adafactor~\cite{shazeer2018adafactor}.
We demonstrate that coordinate-wise clipping significantly reduces the sharpness of adaptive algorithms and speeds up the optimization process.
Specifically, at every iteration $t$, we compute the threshold $\tau_t$ for the top $k\%$ gradients in terms of the absolute value, and clip the gradient coordinates $g_{t,i}$ to $\hat{g}_{t,i} = \mathrm{sgn}(g_{t,i})\min\{|g_{t,i}|, \tau_t\}$ based on their sign.
Then, the clipped gradient $\hat{g}_t$ is used to update the momentum term.
For adaptive algorithms, we make a slight modification of the use of clipped gradient $\hat{g}_t$, that we only update the momentum in the numerator, that is proportional to the update step, using the clipped gradient. The momentum in the denominator is still updated using the original gradient $g_t$.
This is because if we update both terms with the clipped gradient, the normalization effect of adaptive algorithms will cancel out with the clipping of the denominator, so the scaling of the update step will be insufficient.
Examples of SGD momentum and Adam with coordinate-wise clipping are shown in~\Cref{fig:clip_alg}.
We also considered clipping the update step for adaptive algorithms, but since the update steps are already scaled based on the gradient, clipping the update step does not appear to be beneficial.

For clipping threshold, we use a small clipping fraction of 10\% for SGD and normalized SGD since they do not have coordinate-wise scaling in their algorithms. Hence, we can observe a significant improvement with a small clipping fraction.
For Adam and Adafactor, since they already did coordinate-wise scaling in the original algorithm, we use a large clipping fraction of 50\%.
From \Cref{tab:sharpness}, we can see that clipping the top the directional sharpness decrease significantly.
Since we normalize the update step when we compute the directional sharpness, the sharpness reduction effect of coordinate-wise clipping is not due to significant reduction of the norm of the update step, but the improved flatness of the direction.
The landscape visualization in~\Cref{fig:landscape_example_sgd_machine_translation} gives a consistent message, that clipped algorithms can find a direction that has better local reduction of the loss in the local geometry.
\begin{figure}[htb]
\begin{minipage}{0.48\textwidth}
    \begin{algorithm}[H]
        \caption{SGD momentum with clipping}
        \label{alg:sgd_main}
        \begin{algorithmic}
        \Require{initial point $x_0$, learning rate $\eta$, momentum term $\beta$}
        \For{$t \gets 1, \dots, T$}
            \State{$g_t \gets \nabla f(x_t)$ or stochastic gradient}
            \State{$\hat{g}_t \gets \mathbf{clip}(g_t)$}
            \State{$m_t \gets \beta m_{t-1} + (1 - \beta) \hat{g}_t$}
            \State{$x_t \gets x_{t-1} - \eta m_t$}
        \EndFor
        \end{algorithmic}
    \end{algorithm}
\end{minipage}
\hfill
\begin{minipage}{0.48\textwidth}
    \begin{algorithm}[H]
        \caption{Adam with clipping}
        \label{alg:adam_main}
        \begin{algorithmic}
        \Require{initial point $x_0$, learning rate $\eta$, momentum term $\beta_1, \beta_2$, regularization constant $\epsilon$}
        \For{$t \gets 1, \dots, T$}
            \State{$g_t \gets \nabla f(x_t)$ or stochastic gradient}
            \State{$\hat{g}_t \gets \mathbf{clip}(g_t)$}
            \State{$m_t \gets \beta_1 m_{t-1} + (1 - \beta_1) \hat{g}_t$}
            \State{$v_t \gets \beta_2 v_{t-1} + (1 - \beta_2) g_t^2$}
            \State{$x_t \gets x_{t-1} - \eta m_t/\sqrt{v_t + \epsilon}$}
        \EndFor
        \end{algorithmic}
    \end{algorithm}
\end{minipage}
\caption{Example of optimization algorithms with coordinate-wise clipping. Note that for Adam, the clipped gradient is only used in the first order momentum.}
\label{fig:clip_alg}
\end{figure}

Finally we demonstrate that clipping algorithms can converge faster than the original counterpart by directly training transformers with the clipping algorithms, with the loss curve shown in~\Cref{fig:loss}.
According to the result, clipping algorithms can speedup training significantly.
For coordinate-wise scaling algorithms such as Adam, it is possible to consider larger clipping thresholds to improve the convergence of the algorithms.
Our result suggests that clipping can be used as an universal technique in any non-coordinate-wise-scaling algorithms and speed up training.
The new finding can provide insight into designing new optimization algorithms.

\begin{figure}[h]
    \centering
    \begin{subfigure}{0.48\textwidth}
        \centering
        \includegraphics[width=\textwidth]{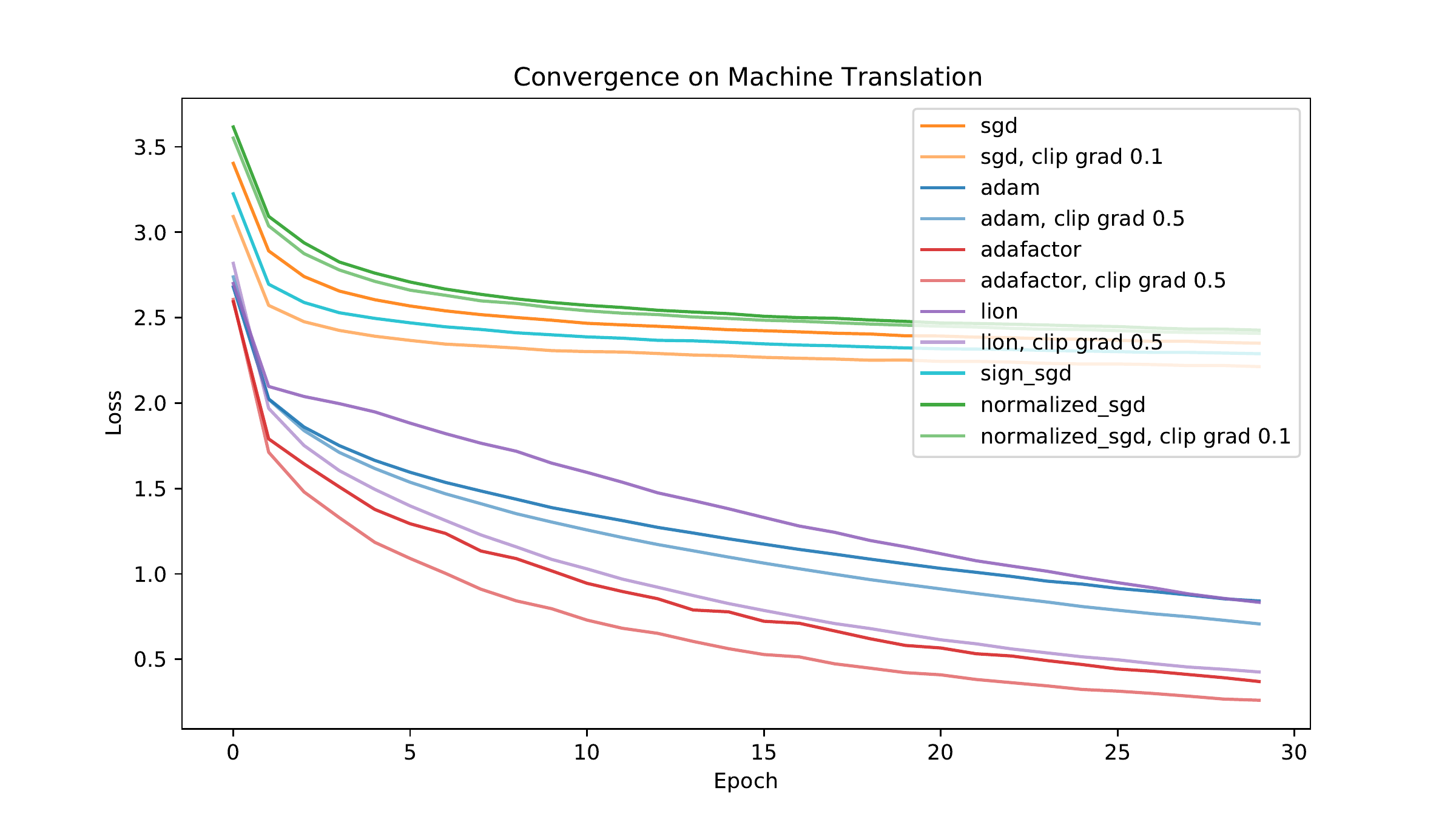}
        \caption{Machine Translation}
    \end{subfigure}
    \begin{subfigure}{0.48\textwidth}
        \centering
        \includegraphics[width=\textwidth]{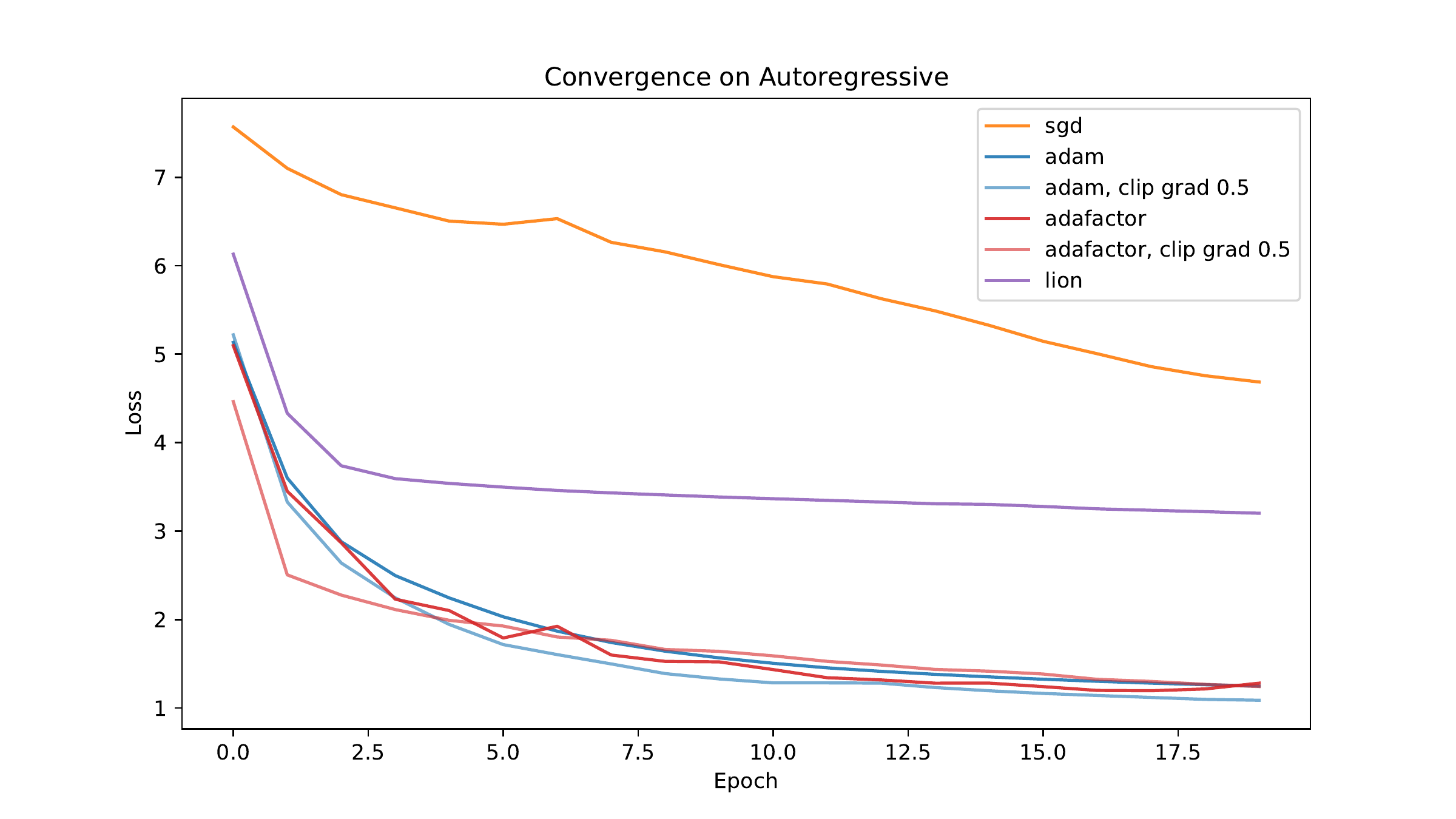}
        \caption{Autoregressive}
    \end{subfigure}
    \caption{Clipped optimization algorithms generally converge faster than the original algorithms. Furthermore, the result is consistent with the landscape analysis in~\Cref{fig:landscape_example_sgd_machine_translation,fig:landscape_example_adam_machine_translation,fig:landscape_example_sgd_autoregressive} and~\Cref{sec:exp_results}, that performance in local geometry is a good indicator of global convergence speed.\label{fig:loss}}
\end{figure}

\subsection{Connection with Coordinate-wise Smoothness}
Based on our experimental findings, we conjecture that there is a \textbf{positive correlation} between the absolute value of Hessian coordinates and gradient coordinates.
The positive correlations is also mentioned in~\cite{zhang2020gradient}, but their proposed correlation is between the norm of Hessian and norm of gradient. We further suggest that there is a positive correlation between the \textbf{coordinates} of gradient and Hessian, and the success of Adam is due to the ability to scale down the bad coordinates and reduce the sharpness through coordinate-wise scaling of the gradient.

We revisit the example given in~\Cref{sec:intro}, that $f(x) = x^\top A x$ and $A_{11} = 100$, $A_{ii} = 1$ for all $i > 1$.
For SGD, the convergence depends on the worst coordinate with smoothness 100, and the gradient is also large in the first coordinate at most of the points since the formula is given as $200x_1$.
This gives us a bad sharpness on the first coordinate.
But if we use clipping, the gradient could not be too large on the first coordinate, so we could choose a larger learning rate even if the Hessian is still unchanged.

A closedly related concept in optimization is the coordinate-wise version of the $L$-smooth assumption in convex and non-convex optimization, typically used in analysis of coordinate descent methods~\cite{wright2015coordinate,shi2016primer,richtarik2014iteration,bernstein2018signsgd,lu2018accelerating}.
Instead of bounding the Hessian with a constant $L$, each coordinates were bounded using different constants $L_1, \dots, L_d$ such that $L_i \le L$ and $\max L_i = L$.
If the gradient has a balanced distribution, the convergence depends on the \textbf{average} of the constants.
Hence, the bound could be better since most of $L_1, \dots, L_d$ could be much less than $L$.
However, if the gradient has an imbalanced distribution, where gradient is concentrated in a small fraction of the coordinates, then the convergence mostly depends on the smoothness of that fraction of coordinates.
Then, clipping works well since it removes the imbalanced distribution of the gradients, ensuring ``uniformity'' of the gradient coordinates.
When only an $\varepsilon$-fraction of coordinates have bad smoothness $L$, with clipping threshold $c_t$, the norm of clipped gradient on the $\varepsilon$-fraction of coordinates is at most $\sqrt{\varepsilon d} c_t$, so the dependence on $L$ is at most $O(\sqrt{\varepsilon} L)$.
Similarly, adaptive algorithm enforce the same constraint on the gradients, removing the correlation between the Hessian and gradient.

In \Cref{sec:hessian}, we justify with an additional simple experiment that suggests only a small fraction of the coordinates has large smoothness.
We approximate the Hessian of the neural network with the Gauss-Newton matrix~\cite{martens2016second,bottou2018optimization,cohen2021gradient} and study the smoothness of the Hessian if we could remove a small fraction of the coordinates.
The result shows that by removing $\le 4\%$ of the coordinates, the smoothness of the neural network improve by a constant factor.
This provides intuition into why coordinate-wise clipping improves the directional sharpness.
Then, under the assumption that we can remove a small fraction of coordinates and achieve a better smoothness, we can formally study the local loss reduction of SGD with clipping, as described by the following informal theorem.
\begin{theorem}[informal]\label{thm:informal}
    Suppose $f$ is non-convex and $L$-smooth, and there exists $0 < \varepsilon < 1$ and $\ell \ll L$ such that for every $x$, after removing $\varepsilon$-fraction of the coordinates, the remaining Hessian has spectral norm at most $\ell$.
    Then, in the worst case, if we run SGD clipping with some optimal step size $\eta \ge \frac{2}{L}$, it achieves better loss reduction than SGD with step size $\eta \le \frac{2}{L}$. In particular, the upper bound on the directional sharpness is at most $O(\sqrt{\varepsilon} L + \ell) \ll L$ compared to $L$ of SGD.
\end{theorem}
The formal statement and proof are given in \Cref{sec:theory}.
The theorem shed light onto how gradient clipping can improve the loss locally.
Understanding of this phenomenon could be essential in proving convergence rates for Adam or clipping algorithms faster than SGD.

% \begin{assumption}
    % For every $x \in \R^d$, there exists a constant $\varepsilon > 0$ such that there exists we can remove $\varepsilon$-fraction of the coordinates, such that the remaining Hessian matrix has spectral norm at most $\ell$.
% \end{assumption}

% \begin{theorem}
    % If Assumptions hold, then 
% \end{theorem}

\section{Experiment Setups\label{sec:exp_main}}
In this section, we describe the setting of our full experiments.
We demonstrate our findings with two types of experiments, as described in~\Cref{sec:directional_sharpness,sec:clip}.
We explore several different tasks and settings and show our results hold in various setting.
Further discussions of the results are in~\Cref{sec:discussion}

\textbf{Optimization algorithms.} We select a variety of optimization algorithms.
The algorithms all uses momentum in their update steps for a fair comparison.
The baseline algorithm is SGD momentum, which we compare the sharpness of other algorithms with.
For the class of adaptive algorithms, we choose Adam~\cite{kingma2015adam}, Adafactor~\cite{shazeer2018adafactor}, and Lion~\cite{chen2023symbolic}.
Adam is the most popular adaptive algorithm, and Adafactor and Lion both claim to be the state-of-the-art optimization algorithm on some specific tasks~\cite{shazeer2018adafactor,chen2023symbolic}.
We also include signSGD due to its similarity with the Lion optimizer~\cite{chen2023symbolic} and having probably the simplist form of adaptive algorithm.
Note that signSGD is just SGD with clipping threshold 100\%.
To show that the improvement in directional sharpness and convergence speed is more related to coordinate-wise scaling than weight-matrix-wise scaling, we also design an algorithm which we call normalized SGD, that normalizes the square of Frobenius norm of each weight matrix to be proportional to the size of the matrix.
By comparing normalized SGD with SGD clipping, we can see the importance of \textbf{coordinate-wise} scaling in adaptive algorithms and clipping.

\textbf{Tasks.}
We run our experiments on two tasks, including machine translation and autoregressive language modeling, which are two popular tasks in language processing and can be solved efficiently with transformers.
For machine translation, we train a small t5~\cite{raffel2020exploring} model on the opus books English-French dataset~\cite{tiedemann2012parallel}.
For autoregressive, we train a GPT-Neo~\cite{black2021gpt,gao2020pile} model on the stack dataset~\cite{kocetkov2022stack} for Python code generation.
The code generation task is slightly different from natural language tasks such as machine translation since it deals with programming languages.
We will show that most of our results still holds in the setting, suggesting that the observation is more related to properties of the transformer architectures.

\textbf{Directional Sharpness and Landscape.}
We compute the directional sharpness of a variety of optimization algorithms, including SGD, normalized SGD, signSGD, Adam~\cite{kingma2015adam}, Adafactor~\cite{shazeer2018adafactor}, and Lion~\cite{chen2023symbolic}, and visualize the corresponding loss landscape direction, under different local geometry.
We show that SGD has bad sharpness under all of the settings, regardless of the task, model, or local geometry.
In addition, we demonstrate \textbf{clipping can always improve the directional sharpness of optimization algorithms}, and often result in better local loss reduction in the update direction.

\textbf{Global Convergence.}
We also implement clipping algorithms and use them to train different models, and demonstrate that clipping algorithms converge faster in practice.
The result matches the goodness of the direction as measured by the landscape visualization and directional sharpness, that algorithms with better directional sharpness and better local loss reduction in the update direction in the SGD geometry generally converges faster.
We conclude that the \textbf{performance of optimization algorithms in local geometry can be a good indicator of speed of global convergence}.

\section{Conclusion}

In summary, our work provides a new insight of why Adam converges faster than SGD in practice.
In contrast to assumptions on properties of the gradient, we propose to study directional sharpness as an important indicator for the performance of optimization algorithms in deep learning.
We show that adaptive algorithms and clipped optimization algorithms can generally achieve significantly better directional sharpness compared to SGD.
We argue that the slow convergence of SGD is related to the high directional sharpness, caused by a positive coordinate-wise gradient-Hessian correlation.
We propose to use coordinate-wise clipping as a solution to the problem of high sharpness.
We demonstrate the sharpness reduction effect of coordinate-wise clipping and show that it is possible to step into a lower loss in the update direction of clipping algorithms compared to the original algorithms.
We further demonstrate the effectiveness of coordinate-wise clipping in a wide range of optimization algorithms without coordinate-wise scaling, including SGD, normalized SGD, and Adafactor.
We suggest the use of coordinate-wise clipping as a universal technique to speed up any deep learning optimization algorithms.
Our work provide useful explanations and conjectures about the superior performance of Adam and further understanding of the results could be useful in theoretical understanding of the empirical advantage of Adam over SGD.

\bibliographystyle{plain}
\bibliography{references}

\clearpage
\appendix

\clearpage
\section{Convergence of Clipping with Coordinate-wise Smoothness\label{sec:theory}}
We prove the formal statement of~\Cref{thm:informal}.
We assume that $f$ is non-convex and $L$-smooth.
In addition, by result of the experiment in~\Cref{sec:hessian}, we assume that there exists $0 < \varepsilon < 1$ and $\ell \ll L$ such that for every $x$, after removing $\varepsilon$-fraction of the coordinates, the remaining Hessian has spectral norm at most $\ell$.
Finally, we will need an additional assumption that the gradient is some sort of ``uniform'', that it could not be too imbalanced, such that more than $(1 - \varepsilon)$-fraction of the coordinates are 0 or approximately 0.
This assumption is natural since the gradient is from an neural network, so the edge case should not occur.
The clipped part is not too large, that is $\|g_t\|_2 \ge C_1 \|\nabla f(x_t)\|_2$ for some constant $C_1 > 0$.
Without this assumption, the clipping algorithm could not operate on the edge case that all gradient in the $(1-\delta)$-fraction of unclipped coordinates are 0, so $c_t$ will be 0. Then, the remaining gradient will be 0.
We also want the gradient norm to be large compared to $c_t$, that the norm of the remaining part is also comparable to $\sqrt{d}c_t$.
So we assume that $C_2 \|g_t\|_2 \ge \sqrt{d}c_t$ for some constant $C_2$.
In practice, this is controlled by the clipping threshold $c_t$, but simply assuming the clipping fraction does not suffice, since the aforementioned counterexample could always work if the gradient is given by an adversarial.
Since our results in~\Cref{sec:resnet} show that the properties are transformer-specific, removal of these assumptions requires theoretical analysis of the transformer architecture, which we will not discuss in this work.

Then, we show the following version of the gradient descent lemma that establishes the expected loss decrement with respect to the norm of the gradient.

\begin{theorem}[Gradient descent lemma]
    Suppose $f$ is non-convex and $L$-smooth, and there exists $0 < \varepsilon < 1$ and $\ell \ll L$ such that for every $x$, there is a submatrix of $\nabla^2 f(x)$ with size $(1-\varepsilon)d \times (1-\varepsilon)d$ that has spectral norm at most $\ell$.
    Assuming that $\|g_t\|_2 \ge C_1 \|\nabla f(x_t)\|_2$ and $C_2 \|g_t\|_2 \ge \sqrt{d}c_t$.
    Then, in the worst case, if we run SGD that clips the top $\delta$-fraction such that $\delta > \varepsilon$, with step size $\eta \ge \frac{2}{L}$, it achieves loss decrement of
    \[
        f(x_{t+1}) \le f(x_t) - \frac{C_1^2}{(4\sqrt{\varepsilon}L + 2\ell)C_2} \|\nabla f(x_t)\|_2^2
    \]
    which is asymptotically better than SGD with loss decrement $\frac{1}{2L}\|\nabla f(x_t)\|_2^2$.
    In particular, the upper bound on the directional sharpness is at most $O(\sqrt{\varepsilon} L + \ell) \ll L$ compared to $L$ of SGD.
\end{theorem}
\begin{proof}
    Without loss of generality, assume the first $\varepsilon d$ coordinates can be clipped.
    Since the Hessian is always symmetric, we can define
    \[
        \nabla^2 f(x_t) = \begin{bmatrix}
            A_t & B_t\\
            B_t^\top & H_t
        \end{bmatrix}
    \]
    where $A_t \in \R^{\varepsilon d \times \varepsilon d}$, $H_t \in \R^{(1-\varepsilon)d \times (1-\varepsilon)d}$, $B_t \in \R^{\varepsilon d \times (1 - \varepsilon)d}$.
    We define
    \[
        P_1 := \begin{bmatrix}
            A_t & C_t\\
            0 & 0
        \end{bmatrix}
        \qquad
        P_2 := \begin{bmatrix}
            0 & 0\\
            C_t^\top & 0
        \end{bmatrix}
        \qquad
        P_3 := \begin{bmatrix}
            0 & 0\\
            0 & B_t
        \end{bmatrix}
    \]
    so $\nabla^2 f(x_t) = P_1 + P_2 + P_3$.
    Then, we can bound the directional sharpness as
    \begin{align*}
        |g_t^\top \nabla^2 f(x_t) g_t|
        &= |g_t^\top P_1g_t + g_t^\top P_2g_t + g_t^\top P_3g_t|\\
        &\le \|g_t I_{i \le \varepsilon d}\|_2 \|P_1 g_t\|_2 + \|g_t I_{i \le \varepsilon d}\|_2 \|P_2^\top g_t\|_2 + \|P_3\|_2 \|g_t\|_2^2\\
        &\le 2\sqrt{\varepsilon d} c_t \cdot L \|g_t\|_2 + \ell \|g_t\|_2^2 \\
        &\le 2\sqrt{\varepsilon d} c_t \cdot L \|g_t\|_2 + \ell \|g_t\|_2^2\\
        &\le (2\sqrt{\varepsilon} L + \ell) \sqrt{d}c_t\|g_t\|_2.
    \end{align*}
    Then, if we normalize according $\|g_t\|_2$, we would have the directional sharpness is at most $C_2 (2 \sqrt{\varepsilon} L + \ell)$.
    In this case, the directional sharpness is $O(\sqrt{\varepsilon}L + \ell) \ll L$.

    Then, we work on the gradient descent lemma.
    \begin{align*}
        f(x_{t+1})
        &\le f(x_t) - \eta \nabla f(x_t)^\top g_t + \frac{1}{2} \eta^2 g_t^\top \nabla^2 f(\xi_t) g_t\\
        &= f(x_t) - \eta \sum_{i=1}^d |\nabla f(x_t)_i| |g_{t,i|} + \eta^2 (\sqrt{\varepsilon}L + \ell/2) \sqrt{d} c_t\|g_t\|_2\\
        &= f(x_t) - \eta \sum_{i=1}^d (|g_{t,i|} + |h_{t,i|}) |g_{t,i|} + \eta^2 (\sqrt{\varepsilon}L + \ell/2) \sqrt{d} c_t\|g_t\|_2\\
        &= f(x_t) - \eta \|g_t\|_2^2 - \eta \sum_{i=1}^d |h_{t,i|} c_t + \eta^2 (\sqrt{\varepsilon}L + \ell/2) \sqrt{d} c_t\|g_t\|_2\\
        &\le f(x_t) - \eta \|g_t\|_2^2 + \eta^2 (\sqrt{\varepsilon}L + \ell/2) \sqrt{d} c_t\|g_t\|_2
    \end{align*}
    Then, we use assumptions that $g_t$ is uniform, so
    \begin{align*}
        f(x_{t+1}) \le f(x_t) - \eta \|g_t\|_2^2 + \eta^2 (\sqrt{\varepsilon}L + \ell/2) C_2 \|g_t\|_2^2.
    \end{align*}
    By choosing $\eta = \frac{1}{(2\sqrt{\varepsilon}L + \ell)C_2}$, we have
    \begin{align*}
        f(x_{t+1}) &\le f(x_t) - \frac{1}{(4\sqrt{\varepsilon}L + 2\ell)C_2} \|g_t\|_2^2\\
        &\le f(x_t) - \frac{C_1^2}{(4\sqrt{\varepsilon}L + 2\ell)C_2} \|\nabla f(x_t)\|_2^2.
    \end{align*}
    We know that for gradient descent, the optimal learning rate is obtained by choosing $\eta = \frac{1}{L}$~\cite{bubeck2015convex}, in which case we would have
    \begin{align*}
        f(x_{t+1}) &\le f(x_t) - \eta \|\nabla f(x_t)\|_2^2 + \frac{L\eta^2}{2}\|\nabla f(x_t)\|_2^2\\
        &\le f(x_t) - \frac{1}{2L} \|\nabla f(x_t)\|_2^2.
    \end{align*}
    This finishes the proof for the gradient descent lemma for SGD clipping.
\end{proof}

\clearpage
\section{Experimental Details\label{sec:exp_appendix}}

\subsection{Tasks, Datasets, and Models}
The details of the dataset, training set size, and model we use are in \Cref{tab:tasks}.
For machine translation, we use a batch size of 1024 and we randomly select a subset of 10240 data as our training set, so we have 10 batches each epoch.
Since we're mainly interested in minimizing the training loss, we do not use any test or validation sets, nor any evaluation metrics other than the cross-entropy loss.
For machine translation, we use the English to French opus books dataset~\cite{tiedemann2012parallel} and t5 model~\cite{raffel2020exploring}.
For autoregressive, we use the GPT-Neo model~\cite{black2021gpt,gao2020pile} pretrained on Code Clippy dataset.
We use ``the-stack-smol'' version of the stack dataset~\cite{kocetkov2022stack}.
In order to evaluate the function in a offline setting, we generate fixed masks with probability 0.15 at the beginning of the training and does not generate new masks whenever we collate the data.

\begin{table}[h]
    \centering
    \begin{tabular}[h]{|l|l|l|l|}
        \hline
        \textbf{Task} & \textbf{Dataset} & \textbf{Batch Size} & \textbf{Model} \\\hline
        Machine Translation & opus books~\cite{tiedemann2012parallel} & 1024 & t5-small~\cite{raffel2020exploring}\\\hline
        Autoregressive & the-stack-smol~\cite{kocetkov2022stack} & 1000 & GPT-Neo~\cite{black2021gpt,gao2020pile}\\\hline
    \end{tabular}
    \caption{Details of the tasks, datasets, training set sizes, and models we use for the two different experiments.\label{tab:tasks}}
\end{table}

\subsection{Optimization Algorithms and Clipping Methods}

We use 6 optimization algorithms, including Adam~\cite{kingma2015adam}, SGD, signSGD, normalized SGD, Adafactor~\cite{shazeer2018adafactor}, and Lion~\cite{chen2023symbolic}.
The reason for selecting these algorithms are described in~\Cref{sec:exp_main}.

We use momentum for all of the optimization algorithms to rule out any potential effect of momentum. The clipped optimization algorithms are described in~\Cref{alg:sgd,alg:normalized_sgd,alg:sign_sgd,alg:adam,alg:adafactor_matrix,alg:adafactor_vector,alg:lion}.
Notice that for Adam and Adafactor, we only clip the gradient in the nominator of the final update step, since otherwise the scaling effect could cancel out or even increase the norm.
Adafactor is originally used with the relative step sizes $\alpha_t$, but in certain cases we use a fixed learning rate in place of $\alpha_t$.
In the algorithms, we assume $\mathbf{clip}(g)$ calculates the clipping threshold $\tau$ for the top $k\%$ coordinates and returns $\tilde{g}$ where $\tilde{g_i} = \mathrm{sgn}(g_i)\min\{|g_i|, \tau\}$.
We use a large fraction $10\%$ and $50\%$ respectively for non-coordinate-wise scaling algorithms and adaptive algorithms, to better demonstrate the effectiveness of clipping.
However, significant but weaker effects can also be observed by setting a very small value such as $0.1\%$.

We also test clipping the update step instead of the gradient for Adam and Lion. The results are also shown in the landscape visualization. However, since the update steps are already scaled based on the gradient, clipping the update step does not improve the result significantly.

\begin{algorithm}
    \caption{SGD with momentum}
    \label{alg:sgd}
    \begin{algorithmic}
    \Require{initial point $x_0$, learning rate $\eta$, momentum term $\beta$}
    \For{$t \gets 1, \dots, T$}
        \State{$g_t \gets \nabla f(x_t)$}
        \State{$\hat{g}_t \gets \mathbf{clip}(g_t)$}
        \State{$m_t \gets \beta m_{t-1} + (1 - \beta) \hat{g}_t$}
        \State{$x_t \gets x_{t-1} - \eta m_t$}
    \EndFor
    \end{algorithmic}
\end{algorithm}

\begin{algorithm}
    \caption{Normalized SGD with momentum for weight matrices and vectors}
    \label{alg:normalized_sgd}
    \begin{algorithmic}
    \Require{initial point $x_0 \in \mathbb{R}^{m \times n}$, learning rate $\eta$, momentum term $\beta$}
    \For{$t \gets 1, \dots, T$}
        \State{$g_t \gets \nabla f(x_t)$}
        \State{$\hat{g_t} \gets \mathbf{clip}(g_t)$}
        \State{$m_t \gets \beta m_{t-1} + (1 - \beta) \hat{g_t}$}
        \State{$v_t \gets \frac{m_t}{\|m_t\|_2} \cdot \sqrt{mn}$}
        \State{$x_t \gets x_{t-1} - \eta v_t$}
    \EndFor
    \end{algorithmic}
\end{algorithm}

\begin{algorithm}
    \caption{Sign SGD with momentum}
    \label{alg:sign_sgd}
    \begin{algorithmic}
        \Require{}
        \For{$t \gets 1, \dots, T$}
            \State{$g_t \gets \nabla f(x_t)$}
            \State{$\hat{g}_t \gets \mathbf{clip}(g_t)$}
            \State{$m_t \gets \beta m_{t-1} + (1 - \beta) \hat{g}_t$}
            \State{$x_t \gets x_{t-1} - \eta \cdot \mathrm{sgn}(m_t)$}
        \EndFor
    \end{algorithmic}
\end{algorithm}

\begin{algorithm}
    \caption{Adam~\cite{kingma2015adam}}
    \label{alg:adam}
    \begin{algorithmic}
    \Require{initial point $x_0$, learning rate $\eta$, momentum term $\beta_1, \beta_2$, regularization constant $\epsilon$}
    \For{$t \gets 1, \dots, T$}
        \State{$g_t \gets \nabla f(x_t)$ or stochastic gradient}
        \State{$\hat{g}_t \gets \mathbf{clip}(g_t)$}
        \State{$m_t \gets \beta_1 m_{t-1} + (1 - \beta_1) \hat{g}_t$}
        \State{$v_t \gets \beta_2 v_{t-1} + (1 - \beta_2) g_t^2$}
        \State{$x_t \gets x_{t-1} - \eta m_t/\sqrt{v_t + \epsilon}$}
    \EndFor
    \end{algorithmic}
\end{algorithm}

\begin{algorithm}
    \caption{Adafactor for weight matrices~\cite{shazeer2018adafactor}}
    \label{alg:adafactor_matrix}
    \begin{algorithmic}
    \Require{initial point $x_0 \in \mathbb{R}^{m \times n}$, relative step sizes $\rho_t = \min\{10^{-2}, \frac{1}{\sqrt{t}}\}$, second moment decay $\hat{\beta}_{2t} = 1 - t^{-0.8}$, regularization constants $\epsilon_1 = 10^{-30}$ and $\epsilon_2 = 10^{-3}$, clipping threshold $d = 1$, $\mathrm{RMS}(x) := \frac{\|x\|_F}{\sqrt{mn}}$}
    \For{$t \gets 1, \dots, T$}
        \State{$ \alpha_t \gets \max\{\epsilon_2, \mathrm{RMS}(x_{t-1})\} \rho_t$}
        \State{$ G_t \gets \nabla f(x_{t-1})$ or stochastic gradient}
        \State{$ \hat{G_t} \gets \mathbf{clip}(G_t)$}
        \State{$ R_t \gets \hat{\beta}_{2t} R_{t-1} + (1 - \hat{\beta}_{2t})(G_t^2 + \epsilon_1)\mathbf{1}_m$}
        \State{$ C_t \gets \hat{\beta}_{2t} C_{t-1} + (1 - \hat{\beta}_{2t})\mathbf{1}_n^\top (G_t^2 + \epsilon_1)$}
        \State{$ \hat{V_t} \gets R_t C_t / \mathbf{1}_n^\top R_t$}
        \State{$ U_t \gets \hat{G_t} / \sqrt{\hat{V_t}}$}
        \State{$ \hat{U_t} \gets U_t / \max\{1, \mathrm{RMS}(U_t)/d\}$}
        \State{$ x_t \gets x_{t-1} - \alpha_t \hat{U_t}$}
    \EndFor
    \end{algorithmic}
\end{algorithm}

\begin{algorithm}
    \caption{Adafactor for weight vectors~\cite{shazeer2018adafactor}}
    \label{alg:adafactor_vector}
    \begin{algorithmic}
    \Require{initial point $x_0 \in \mathbb{R}^n$, relative step sizes $\rho_t = \min\{10^{-2}, \frac{1}{\sqrt{t}}\}$, second moment decay $\hat{\beta}_{2t} = 1 - t^{-0.8}$, regularization constants $\epsilon_1 = 10^{-30}$ and $\epsilon_2 = 10^{-3}$, clipping threshold $d = 1$, $\mathrm{RMS}(x) := \frac{\|x\|_2}{\sqrt{n}}$}
    \For{$t \gets 1, \dots, T$}
        \State{$ \alpha_t \gets \max\{\epsilon_2, \mathrm{RMS}(x_{t-1})\}\rho_t$}
        \State{$ G_t \gets \nabla f(x_{t-1})$ or stochastic gradient}
        \State{$ \hat{G_t} \gets \mathbf{clip}(G_t)$}
        \State{$ \hat{V_t} \gets \hat{\beta}_{2t}\hat{V}_{t-1} + (1 - \hat{\beta}_{2t})(G_t^2 + \epsilon_1)$}
        \State{$ U_t = \hat{G}_t / \sqrt{\hat{V}_t}$}
        \State{$ \hat{U_t} \gets U_t / \max\{1, \mathrm{RMS}(U_t)/d\}$}
        \State{$ x_t \gets x_{t-1} - \alpha_t \hat{U_t}$}
    \EndFor
    \end{algorithmic}
\end{algorithm}

\begin{algorithm}
    \caption{Lion~\cite{chen2023symbolic}}
    \label{alg:lion}
    \begin{algorithmic}
        \Require{}
        \For{$t \gets 1, \dots, T$}
            \State{$g_t \gets \nabla f(x_t)$ or stochastic gradient}
            \State{$u_t \gets \beta_1 m_{t-1} + (1 - \beta_t) g_t$}
            \State{$u_t \gets \mathrm{sgn}(u_t)$}
            \State{$m_t \gets \beta_2 m_{t-1} + (1 - \beta_2)g_t$}
            \State{$x_t \gets x_{t-1} - \eta u_t$}
        \EndFor
    \end{algorithmic}
\end{algorithm}

\begin{algorithm}

\end{algorithm}

\clearpage
\subsection{Experiment for Directional Sharpness of Optimization Algorithms}

\textbf{Pseudo-Update Step.} Since all algorithms we use has momentum part, we need to compute the momentum term in a different trajectory using ``pseudo-update step.'' Specifically, we compute the momentum term for all the optimization algorithms at time $t$ using the past values of $x_1, \dots, x_{t-1}$, regardless of the optimization algorithm we use to perform the actual update step. The values we computed for the algorithms were only used to visualize the landscape and compare the sharpness, but not used for training. The momentum parameters are set to the default values~\cite{kingma2015adam,shazeer2018adafactor}.

\textbf{Training Optimizer.} We use different training optimizers to compare our results across different local geometry and optimization trajectory. We use SGD momentum with learning rate $2 \times 10^{-4}$ and Adam with learning rate $2 \times 10^{-4}$ as training optimizers. The momentum parameters are set to the default values~\cite{kingma2015adam}.

\textbf{Test Batch.} Since computation on the full-batch objective function is very computationally expensive, we sample a fixed random subset of size 1024 as the test dataset at the beginning of the training, and fix it during all epochs and batches, in order to speed up the landscape visualization process. The losses in all the plots are the losses on the test batch.

\textbf{Landscape Visualization.} To visualize the landscape, we update the weight with the desired update step and compute the loss.
Afterwards, we reset the weight back to the original value before the update, and repeat the above step with a new step size.

\textbf{Directional Sharpness.}
We utilize PyTorch's Hessian-vector product to efficiently compute directional sharpness.
Note that if we compute the directional sharpness as $v^\top \nabla f(x_t) v$, then the sharpness can be negative sometimes.
This is because the second-order Taylor expansion is given as
\[
    f(x_{t+1}) \le f(x_t) - \eta \nabla f(x_t)^\top g_t + \frac{1}{2} \eta^2 g_t^\top \nabla^2 f(\xi_t) g_t
\]
for some $\xi_t$ a linear combination of $x_t$ and $x_{t+1}$.
In general, we could approximate the directional sharpness using $x_t$ instead of $\xi_t$, but in some \textbf{very rare} cases of getting a negative sharpness, we use the following formula to compute a more robust version of directional sharpness
\begin{equation}\label{eq:robust_sharpness}
    v^\top \nabla^2 f(x_t + \delta v) v
\end{equation}
for some small $\delta$ where we choose $\delta = 0.01$.
Then the sharpness becomes positive.
In the experiment results in~\Cref{sec:exp_results}, we guarantee that the SGD sharpness are all positive and large.
We will mark the epochs where SGD sharpness is negative and we use~\Cref{eq:robust_sharpness} to compute the directional sharpness.

\subsection{Experiment for Convergence of Clipped Optimization Algorithms}

We demonstrate the convergence of clipped optimization algorithms.
We manually tune the learning rate to find the best learning rate for the experiments.
The learning rate configuration of our experiment is shown in~\Cref{tab:lr}.

\begin{table}[h]
    \centering
    \begin{tabular}[h]{|l|l|l|}
        \hline
        \textbf{Task} & \textbf{Algorithm} & \textbf{Learning Rate}\\\hline
        \multirow{10}{*}{Machine Translation} 
        & sgd & $1 \times 10^{-3}$ \\\cline{2-3}
        & sgd, clip grad 0.1& $1 \times 10^{-3}$ \\\cline{2-3}
        & adam & $2 \times 10^{-3}$ \\\cline{2-3}
        & adam, clip grad 0.5 & $3 \times 10^{-3}$ \\\cline{2-3}
        & adafactor & relative \\\cline{2-3}
        & adafactor, clip grad 0.5 & $3 \times 10^{-2}$ \\\cline{2-3}
        & lion & $2 \times 10^{-3}$ \\\cline{2-3}
        & sign sgd & $2 \time 10^{-3}$ \\\cline{2-3}
        & normalized sgd & $5 \times 10^{-4}$ \\\cline{2-3}
        & normalized sgd, clip grad 0.1 & $6 \times 10^{-4}$ \\\cline{1-3}
        \multirow{6}{*}{Autoregressive}
        & sgd & $3 \times 10^{-5}$ \\\cline{2-3}
        & adam & $1 \times 10^{-4}$ \\\cline{2-3}
        & adam, clip grad 0.5 & $1.5 \times 10^{-4}$ \\\cline{2-3}
        & adafactor & $5 \times 10^{-4}$ \\\cline{2-3}
        & adafactor, clip grad 0.5 & $5 \times 10^{-3}$ \\\cline{2-3}
        & lion & $2 \times 10^{-4}$ \\\cline{1-3}
    \end{tabular}
    \caption{Learning rate configuration of our experiments. The relative learning rate for Adafactor is defined in~\Cref{alg:adafactor_matrix,alg:adafactor_vector} and~\cite{shazeer2018adafactor}.\label{tab:lr}}
\end{table}

\clearpage
\section{Directional Sharpness Results\label{sec:exp_results}}
In this section we show our experimental result for the directional sharpness of optimization algorithms.
For each of the landscape visualization, we show two plots, where one of them has Adafactor and the other does not. The rest of the plots are the same with different scales.
We repeat each experiment with 3 different random seeds.

\subsection{SGD Trajectory\label{sec:sgd_geometry}}

\begin{figure}[h]
    \centering
    \begin{subfigure}{\textwidth}\centering
        \includegraphics[width=0.32\textwidth]{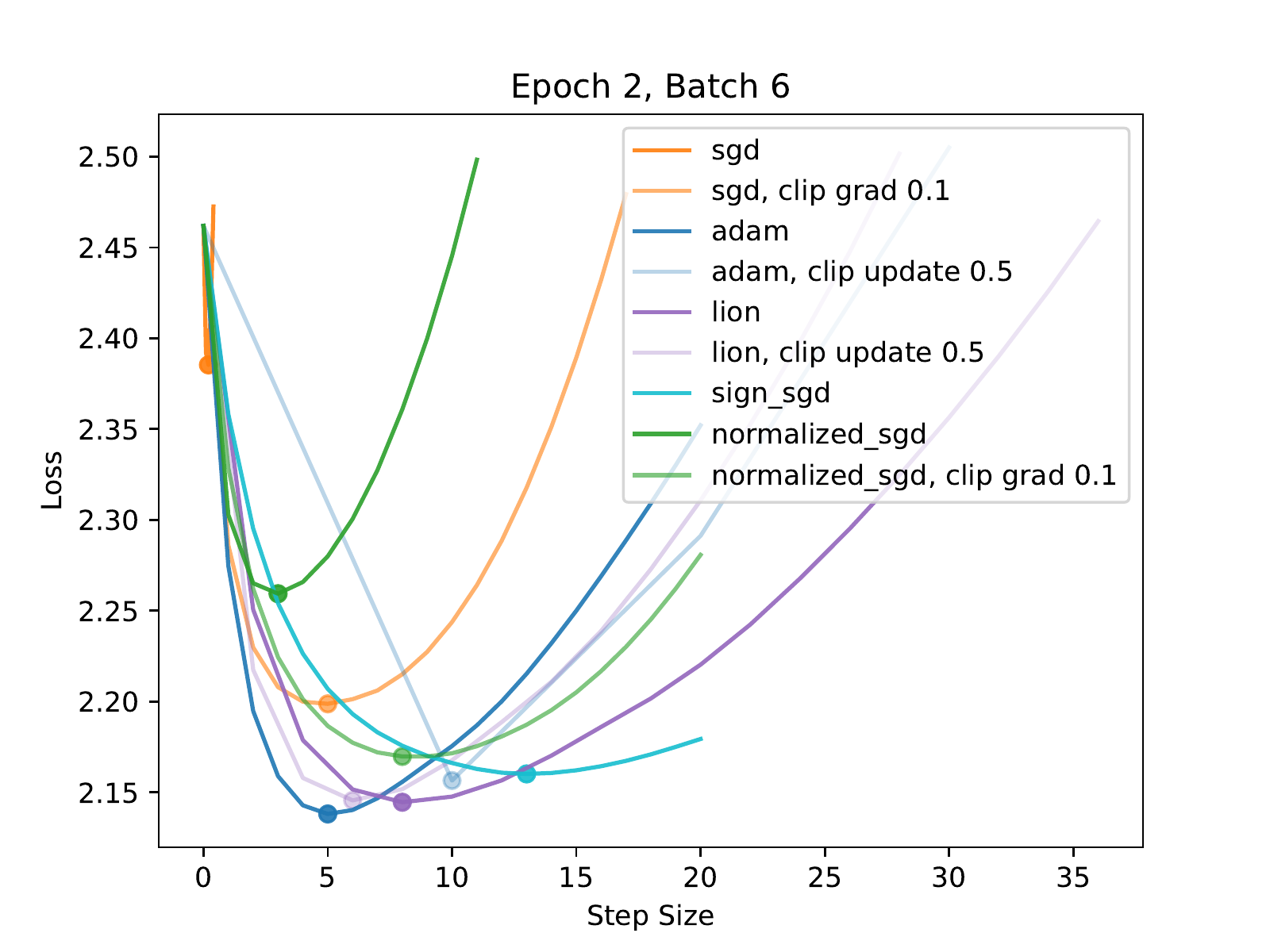}
        \includegraphics[width=0.32\textwidth]{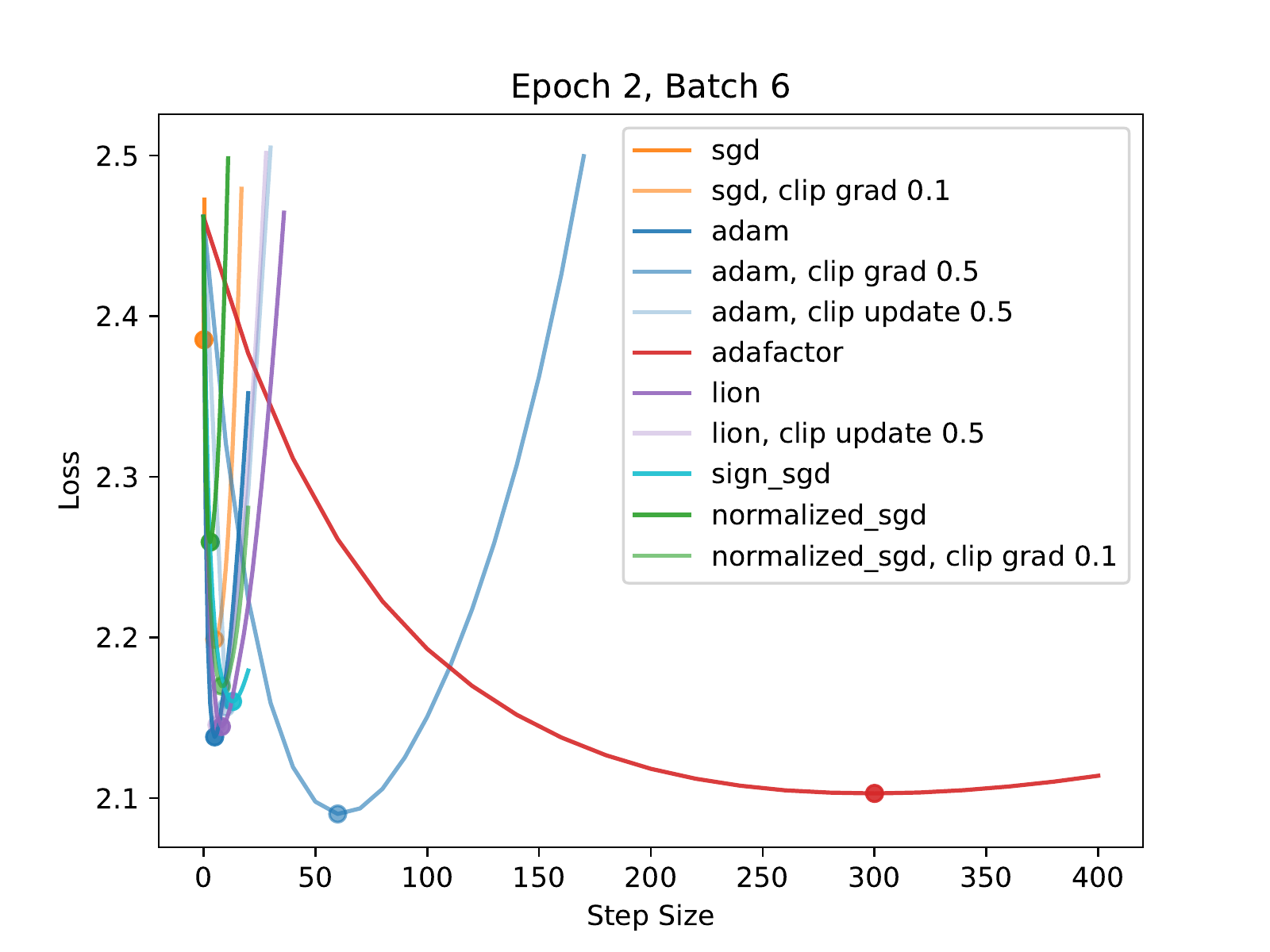}
        \includegraphics[width=0.32\textwidth]{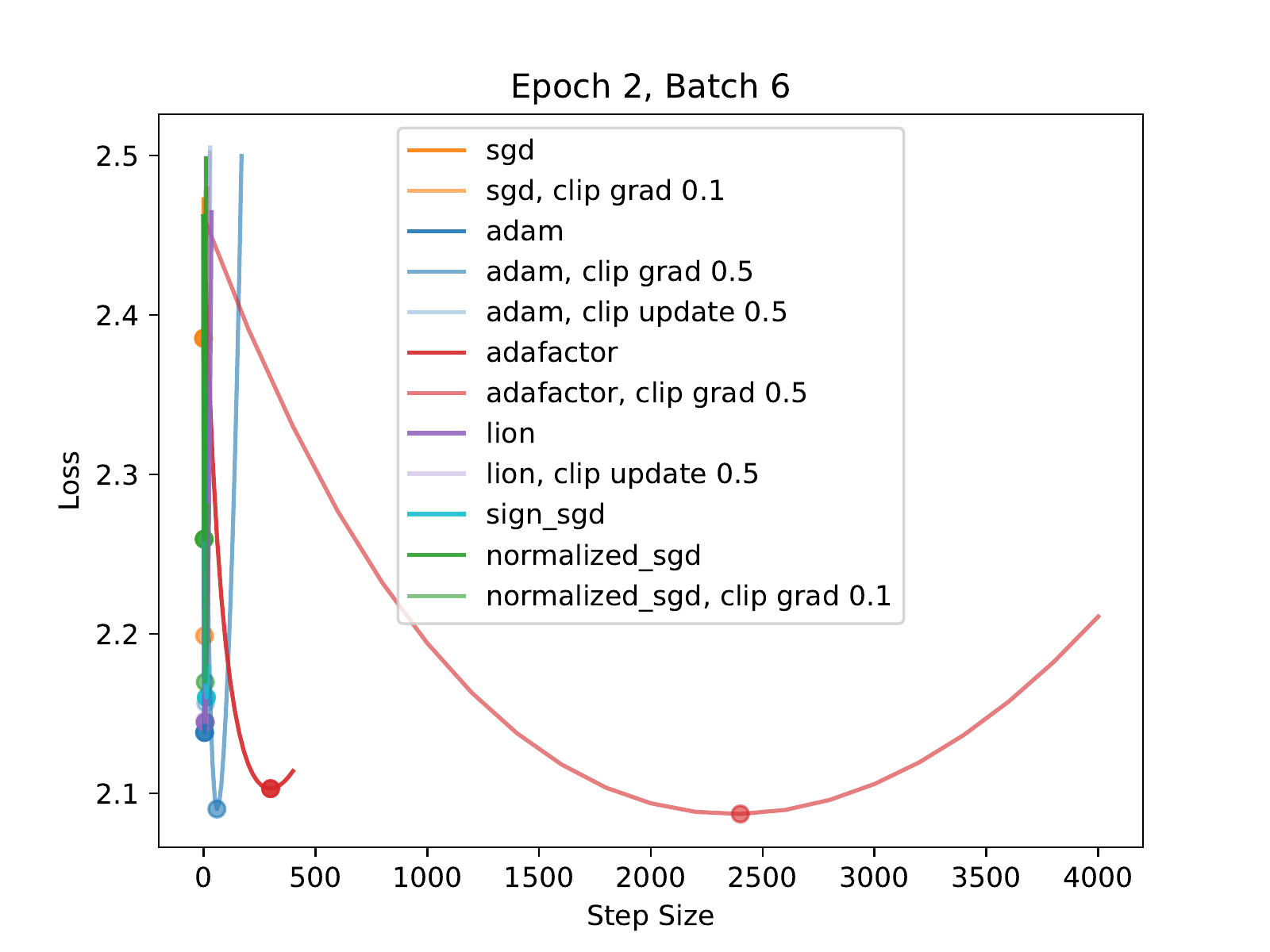}
        \caption{Experiment 1}
    \end{subfigure}
    \begin{subfigure}{\textwidth}\centering
        \includegraphics[width=0.32\textwidth]{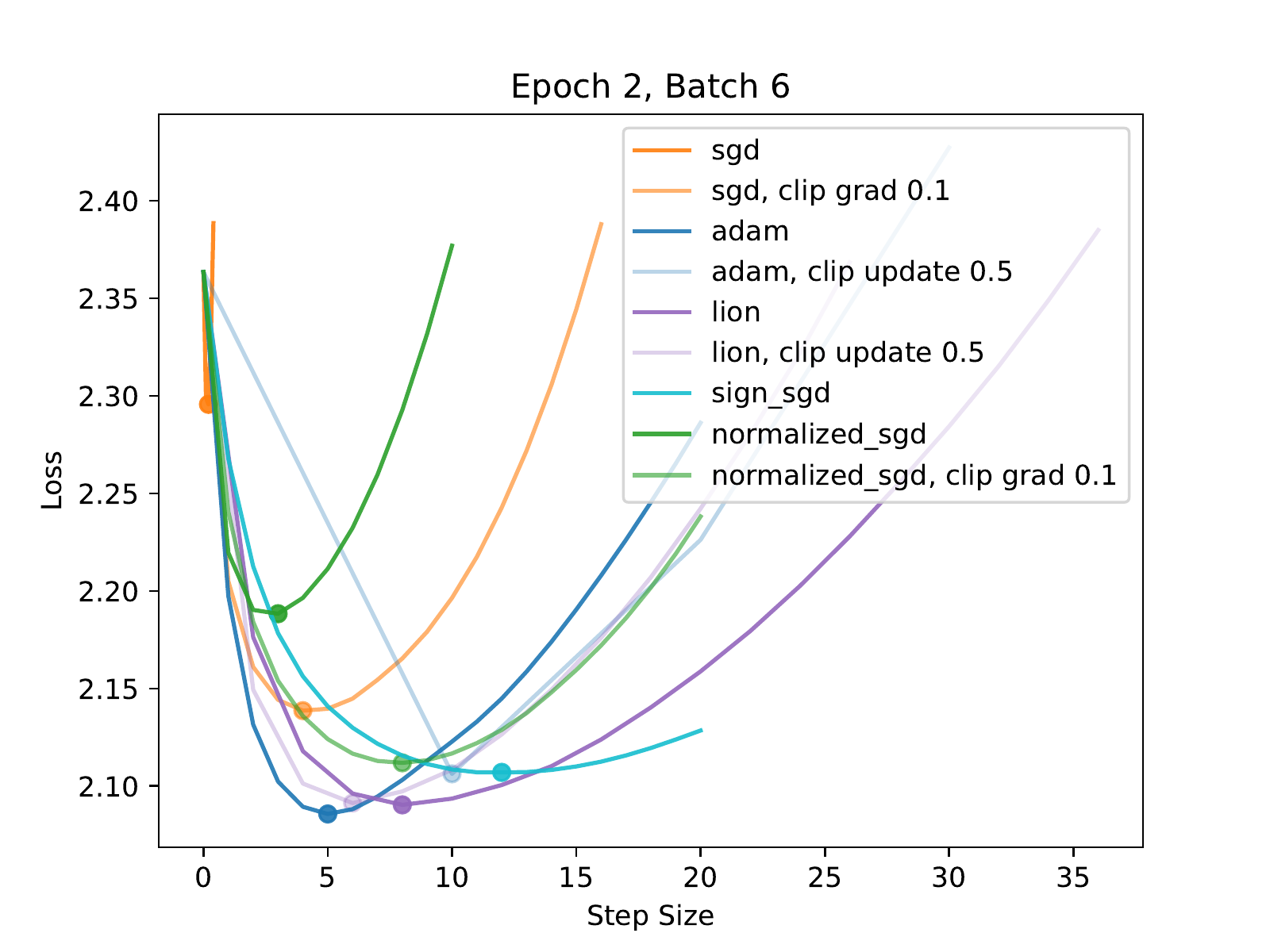}
        \includegraphics[width=0.32\textwidth]{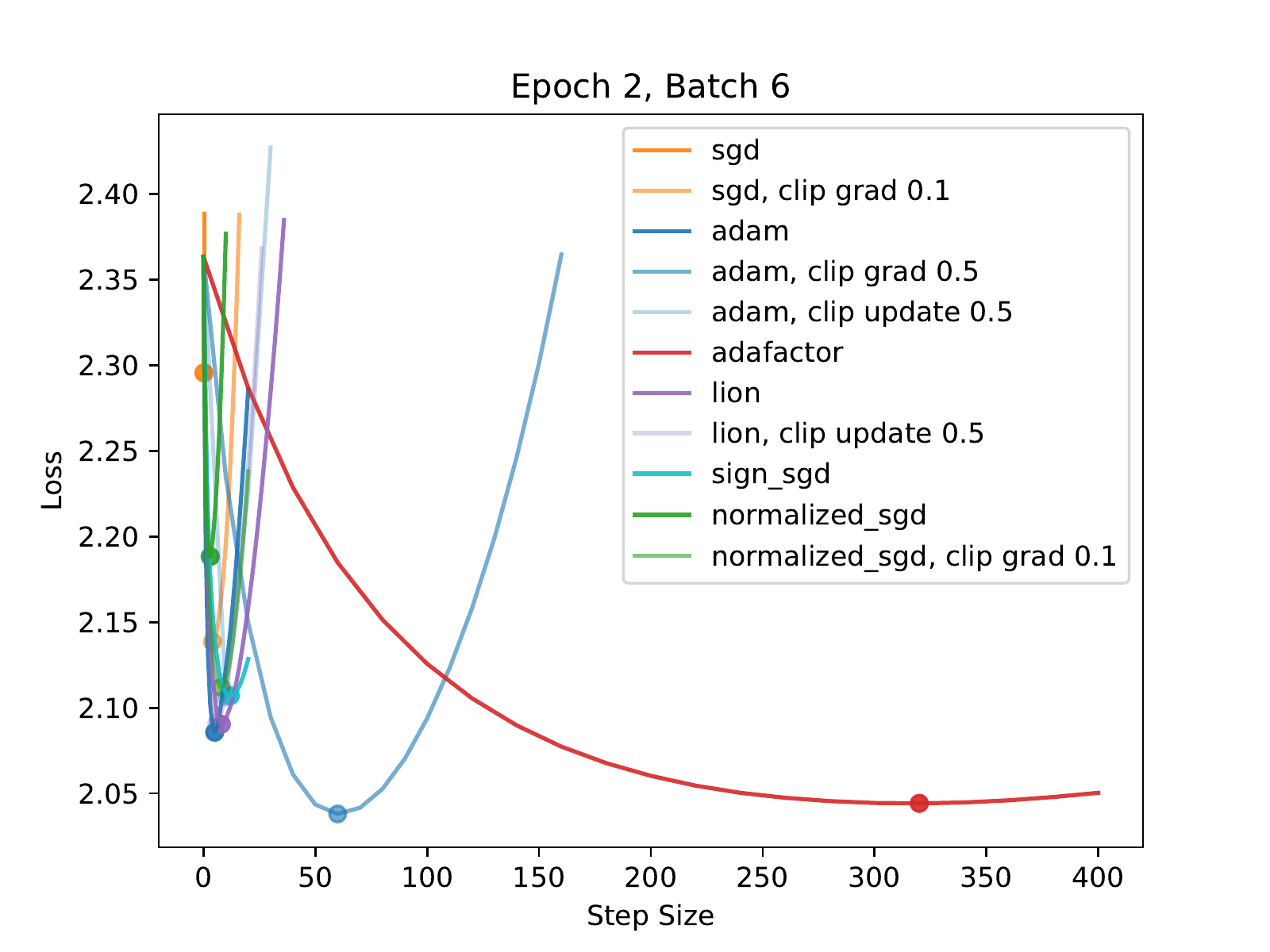}
        \includegraphics[width=0.32\textwidth]{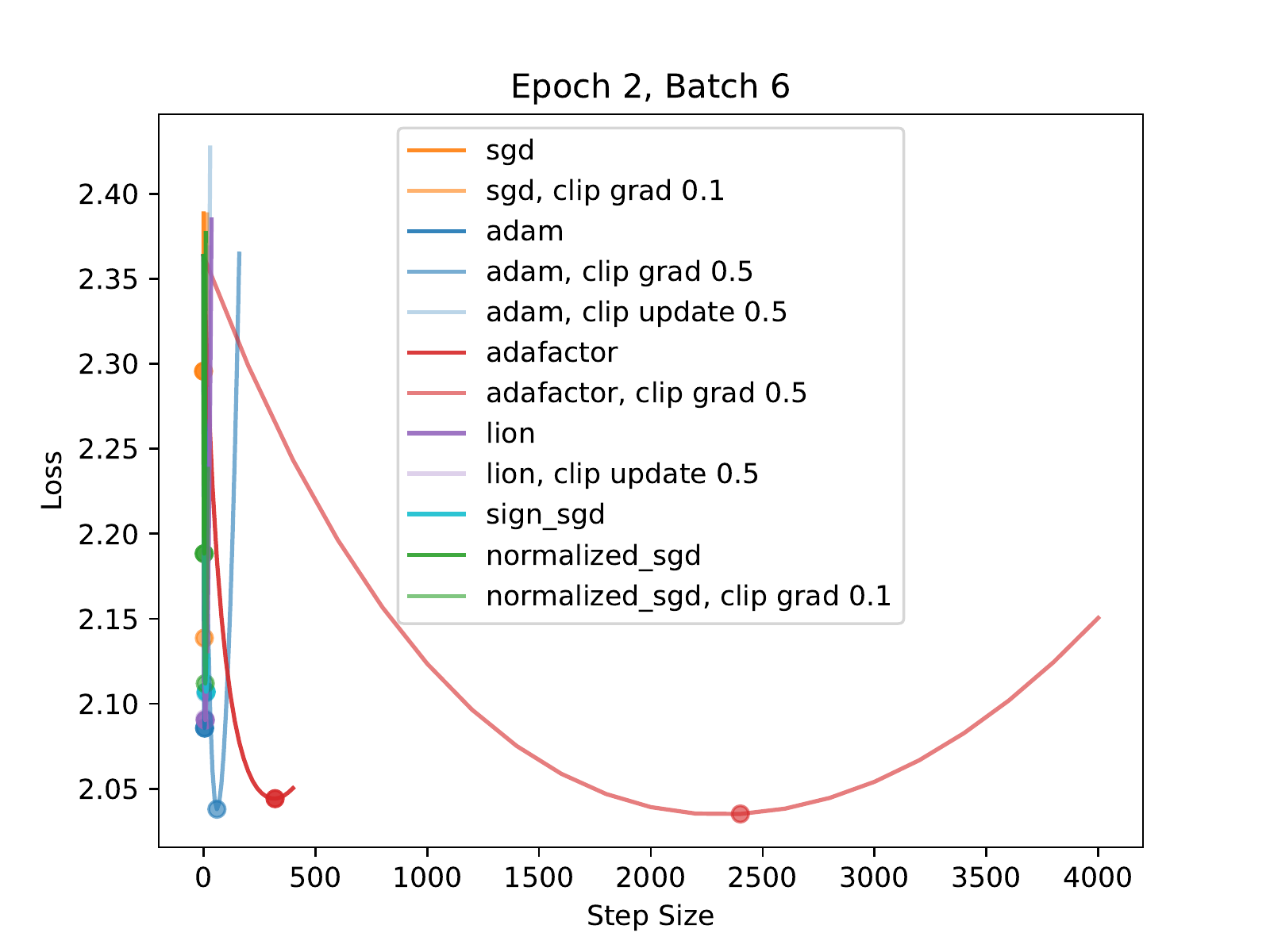}
        \caption{Experiment 2}
    \end{subfigure}
    \begin{subfigure}{\textwidth}\centering
        \includegraphics[width=0.32\textwidth]{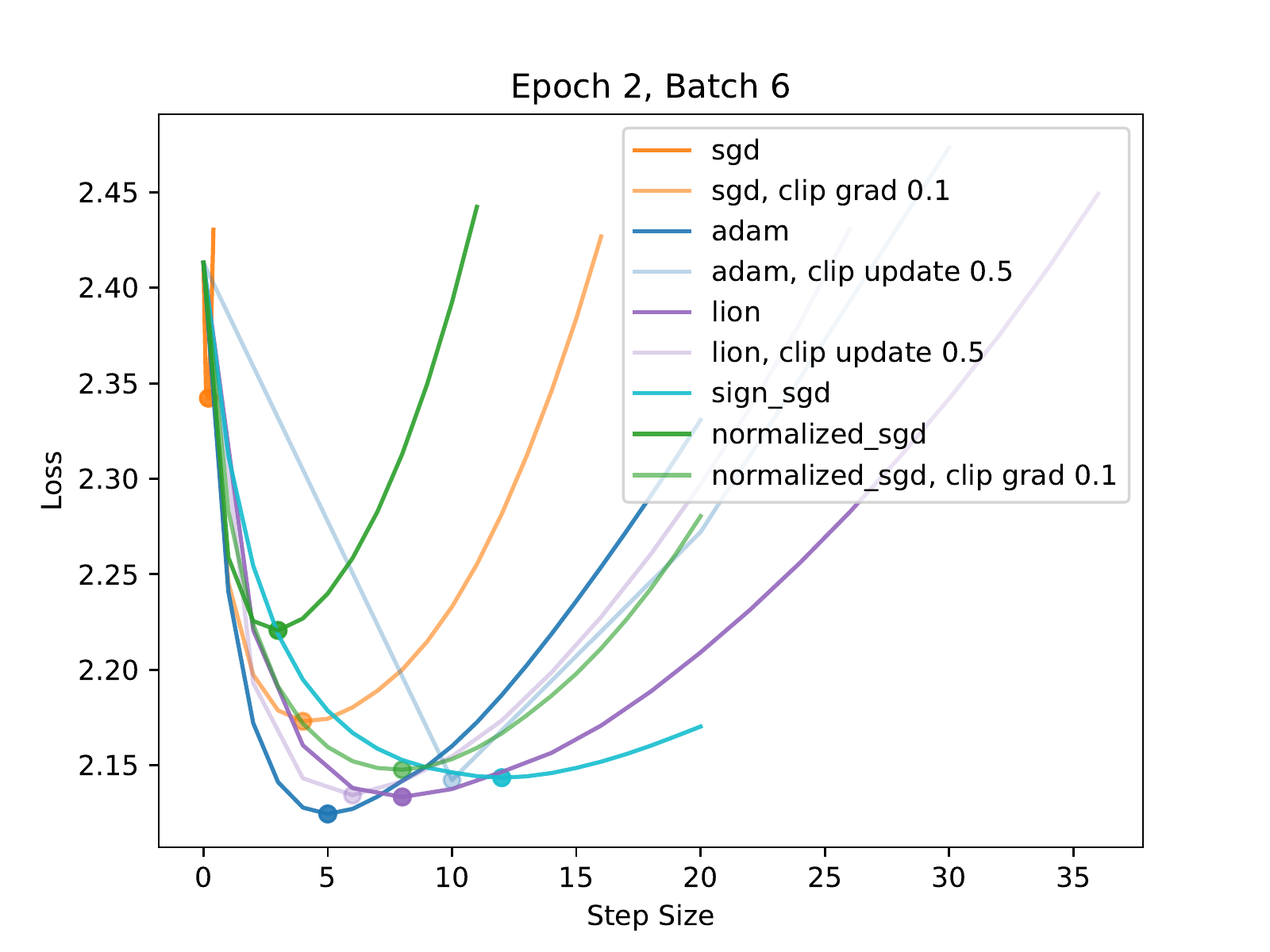}
        \includegraphics[width=0.32\textwidth]{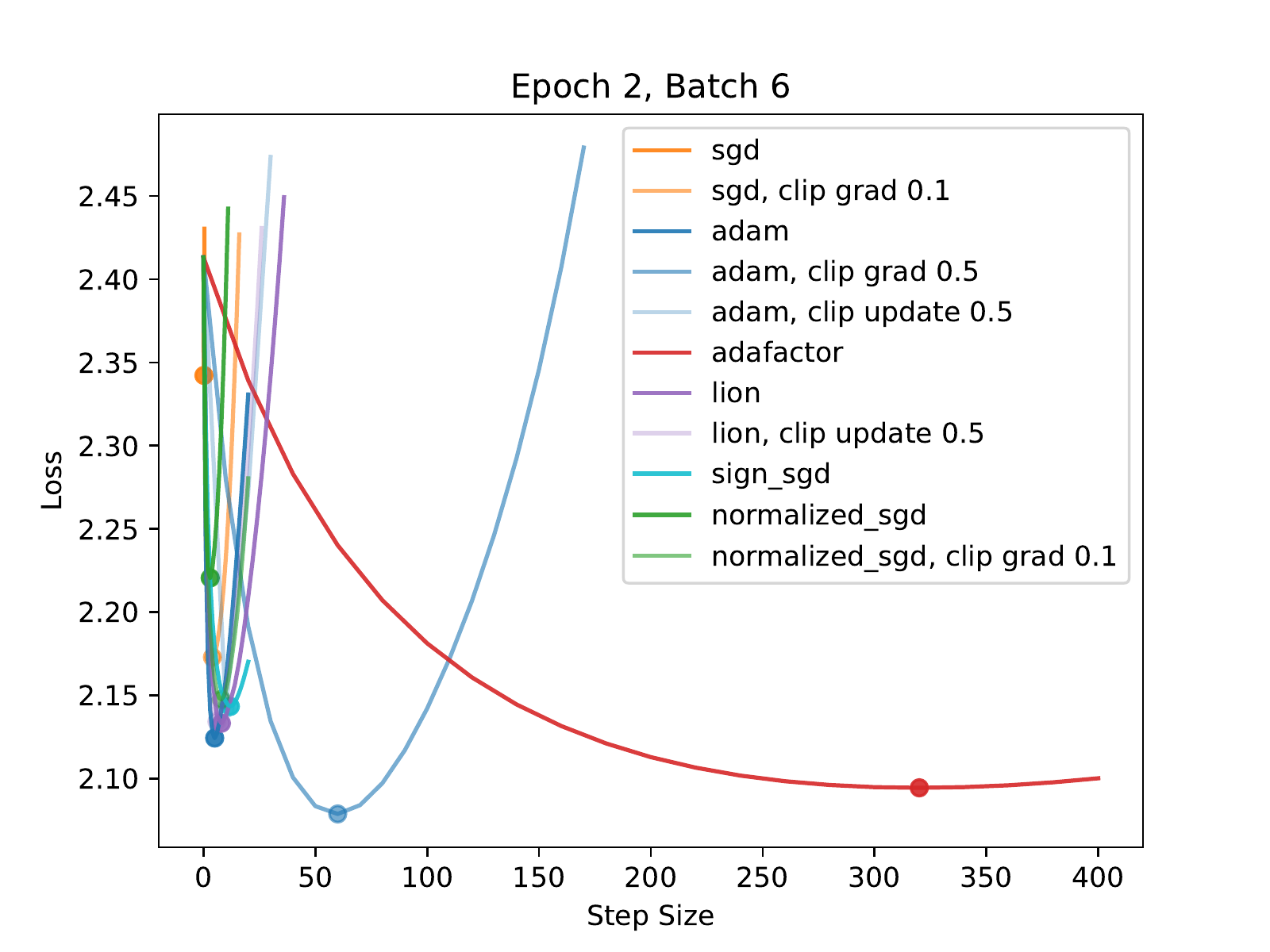}
        \includegraphics[width=0.32\textwidth]{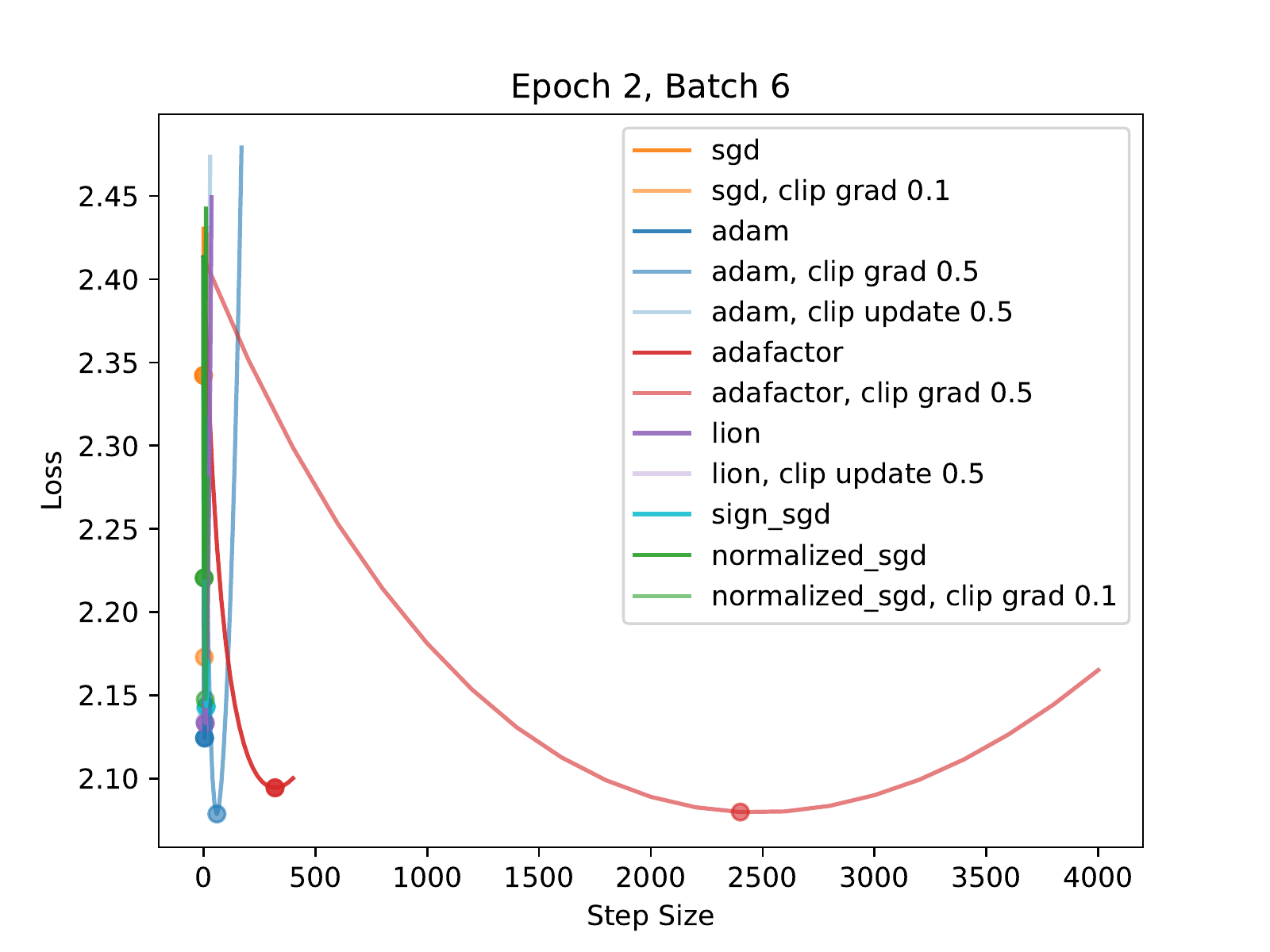}
        \caption{Experiment 3}
    \end{subfigure}
    \caption{Landscape visualization of machine translation in SGD trajectory at Epoch 2.}
\end{figure}

\begin{figure}[h]
    \centering
    \begin{subfigure}{\textwidth}\centering
        \includegraphics[width=0.32\textwidth]{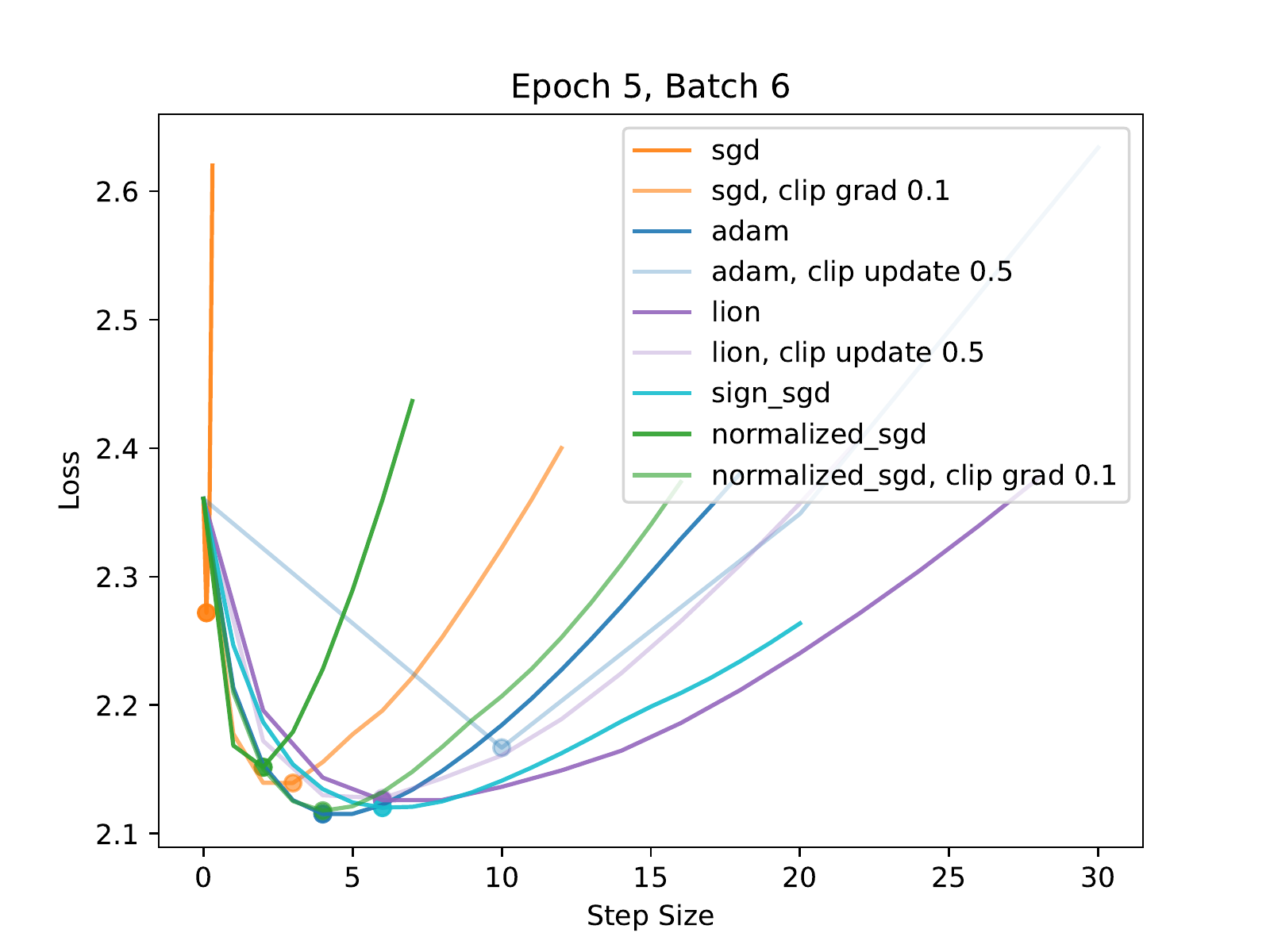}
        \includegraphics[width=0.32\textwidth]{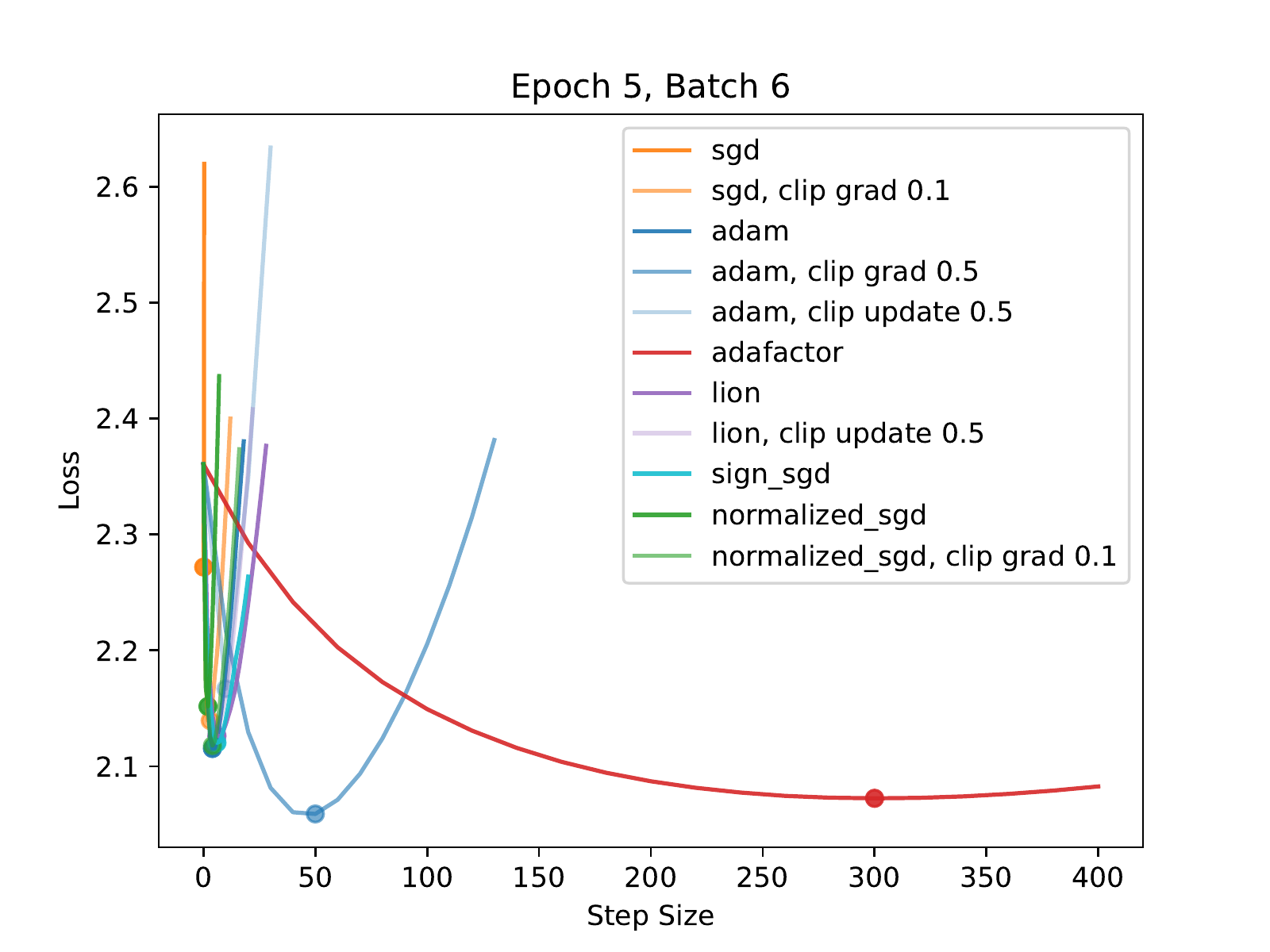}
        \includegraphics[width=0.32\textwidth]{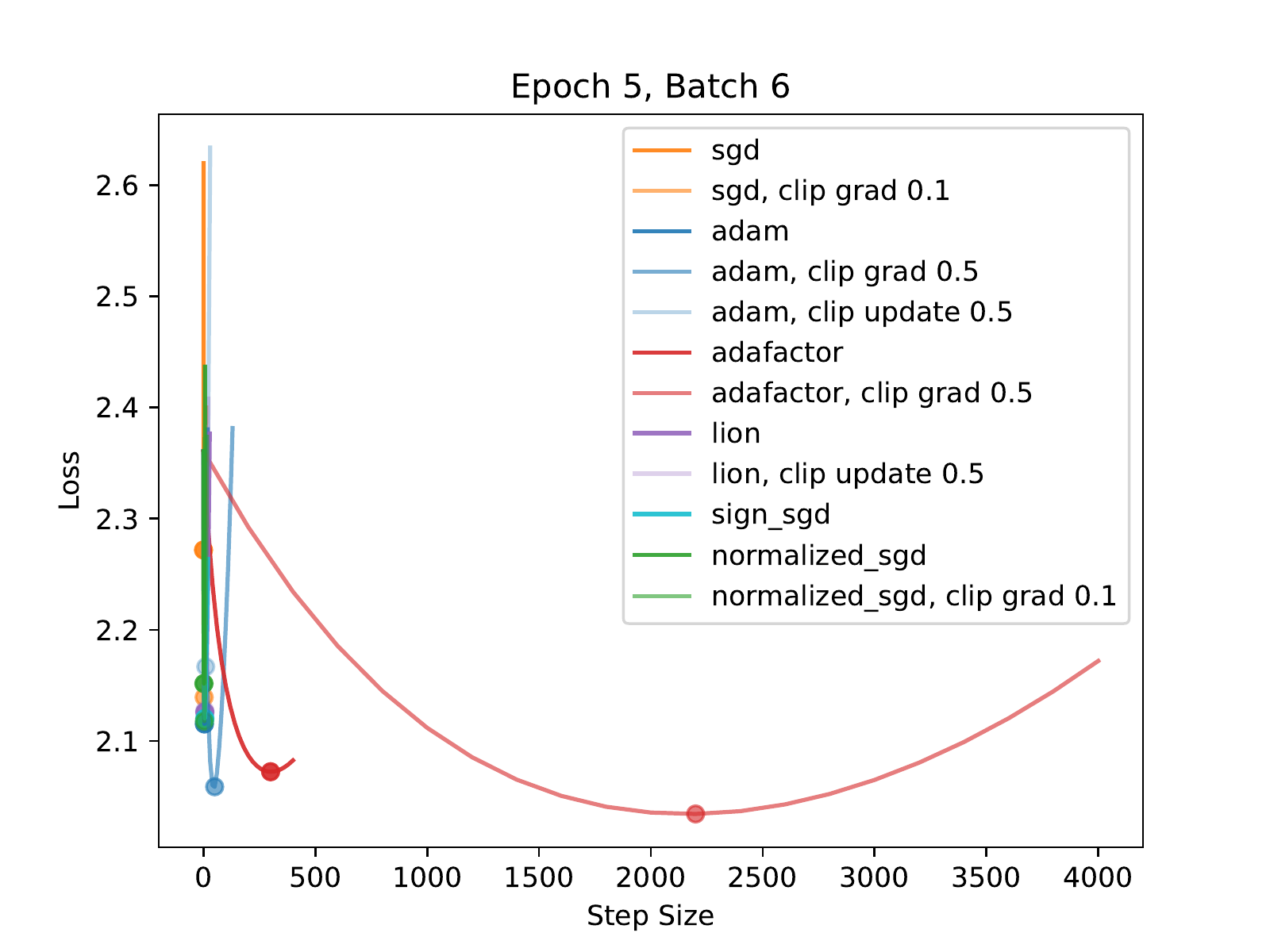}
        \caption{Experiment 1}
    \end{subfigure}
    \begin{subfigure}{\textwidth}\centering
        \includegraphics[width=0.32\textwidth]{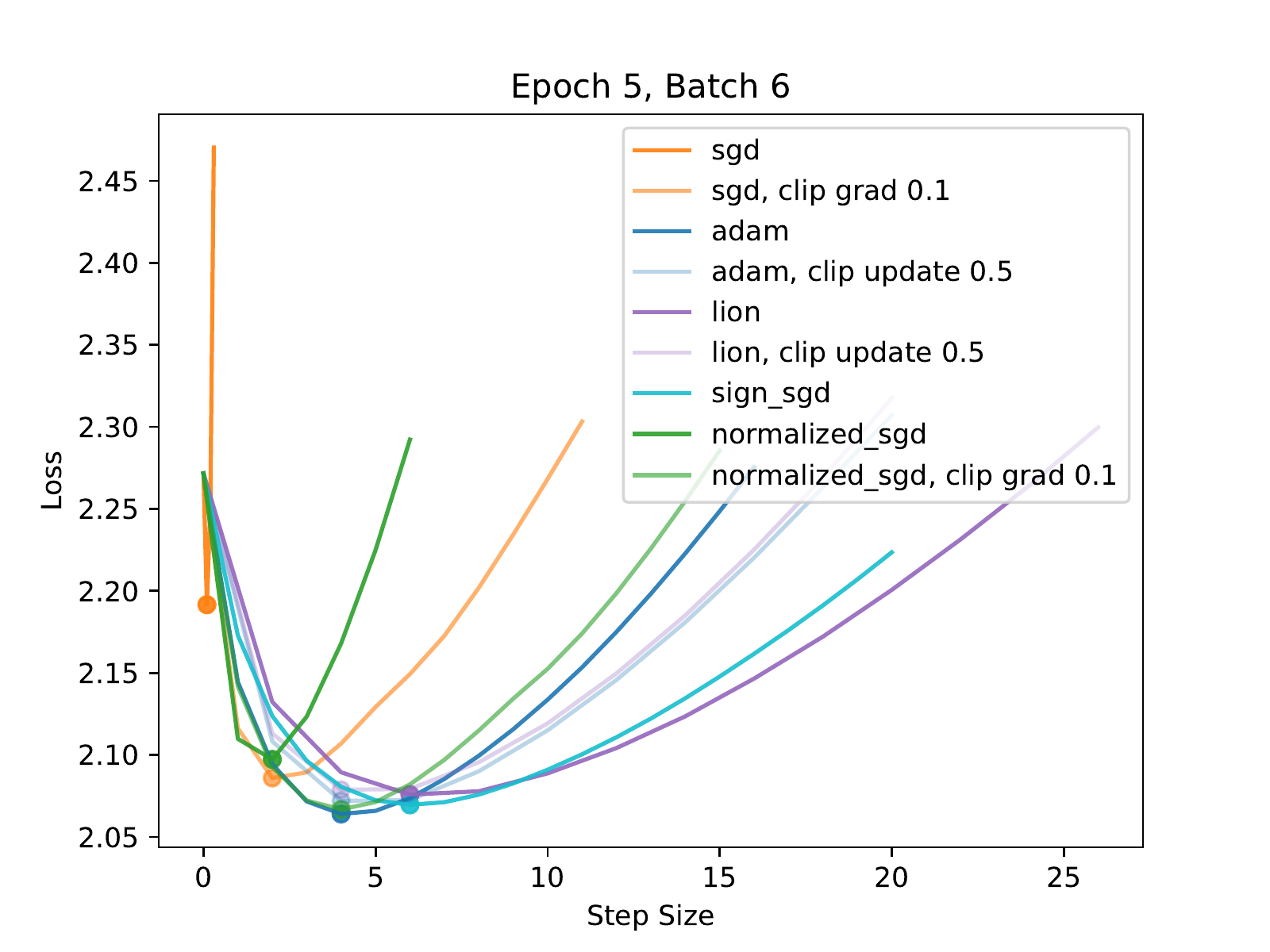}
        \includegraphics[width=0.32\textwidth]{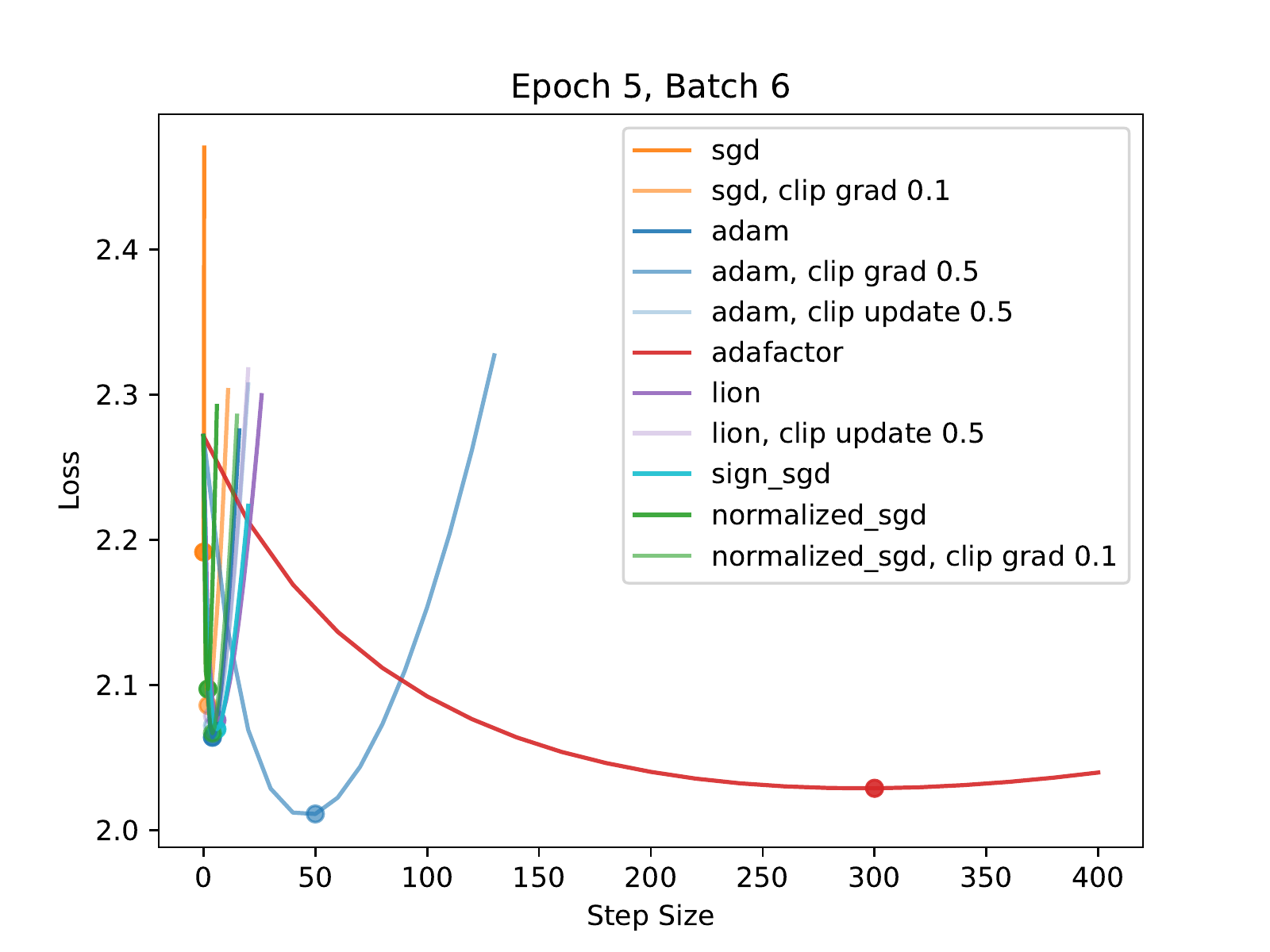}
        \includegraphics[width=0.32\textwidth]{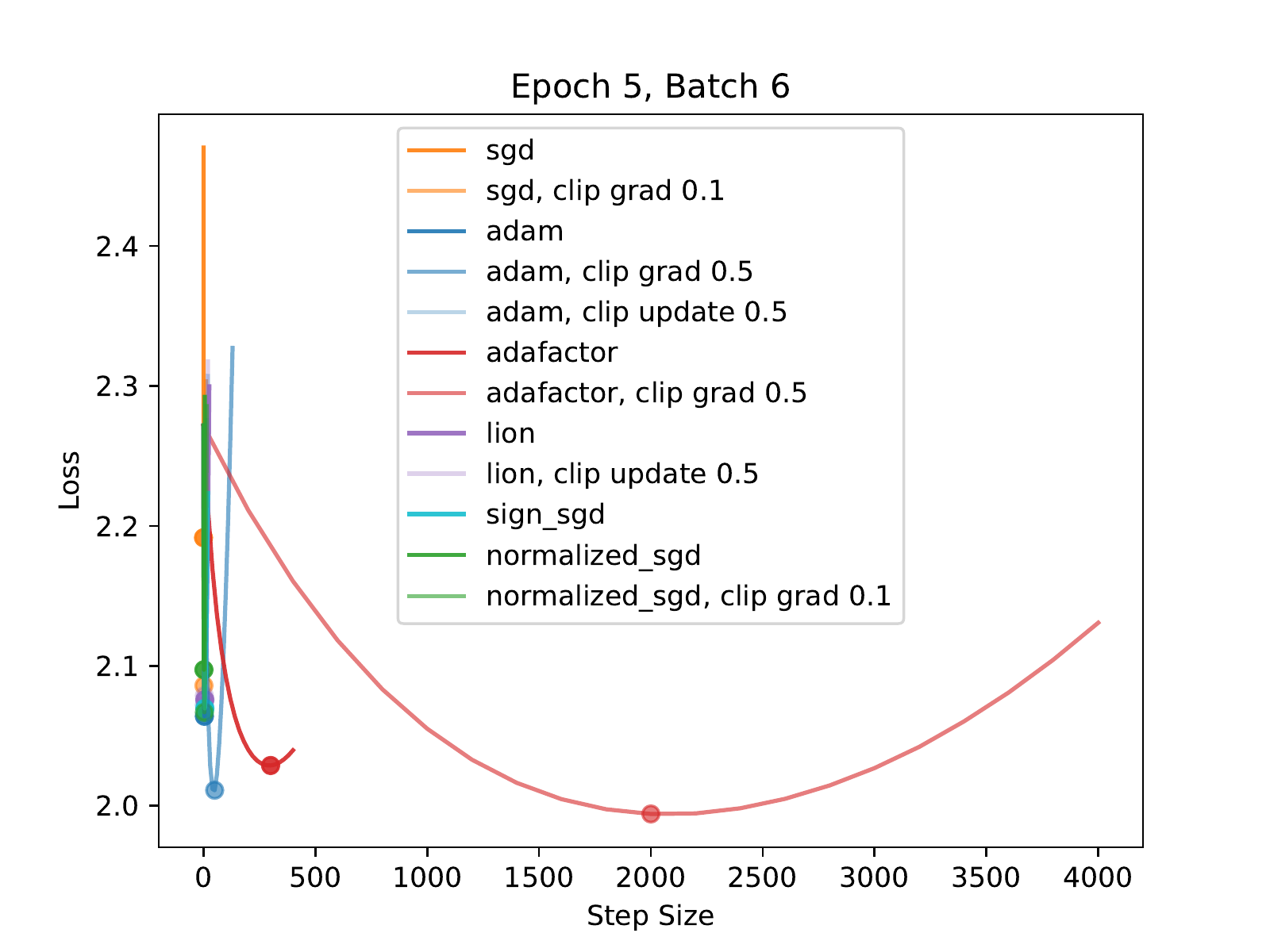}
        \caption{Experiment 2}
    \end{subfigure}
    \begin{subfigure}{\textwidth}\centering
        \includegraphics[width=0.32\textwidth]{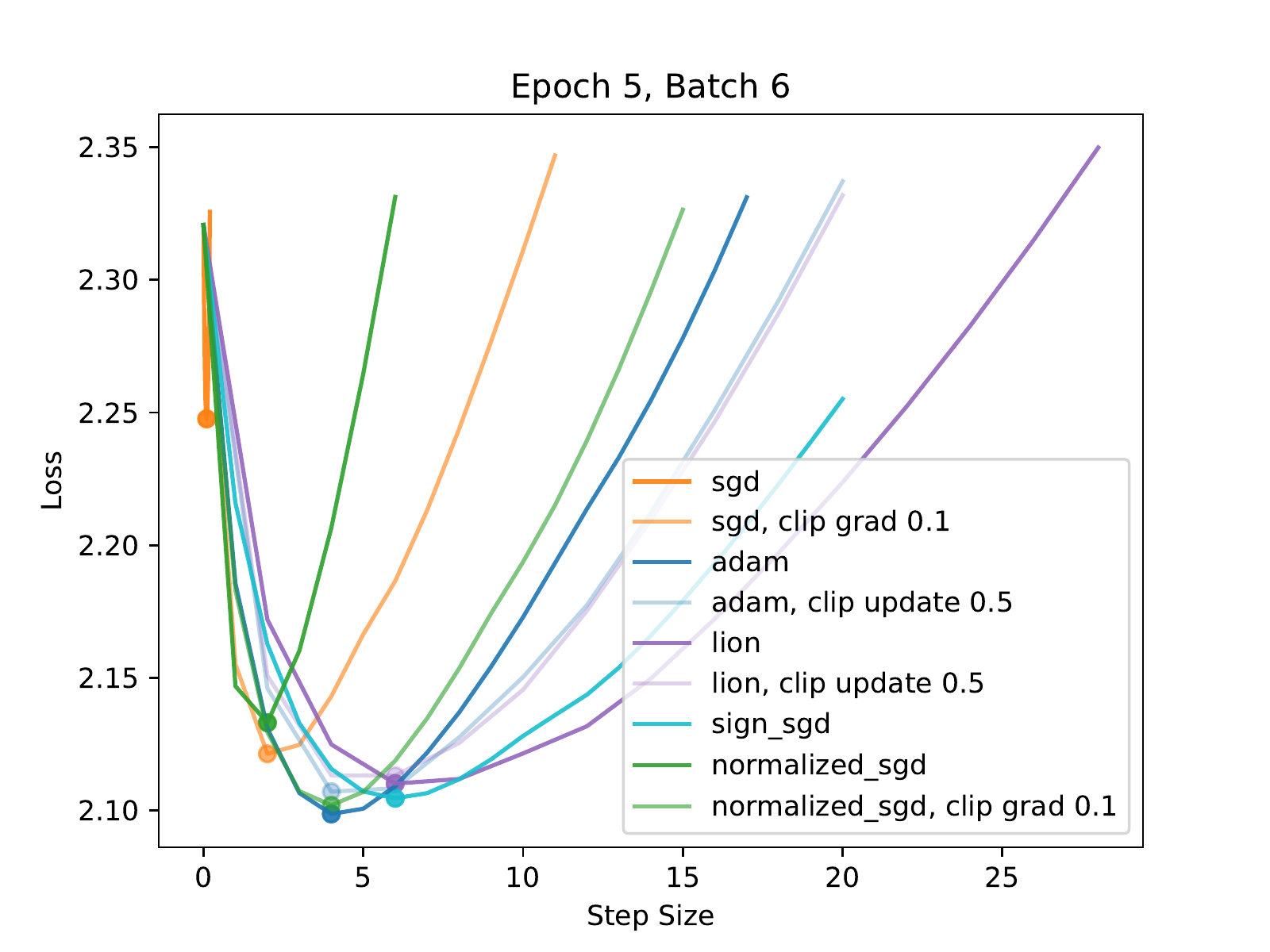}
        \includegraphics[width=0.32\textwidth]{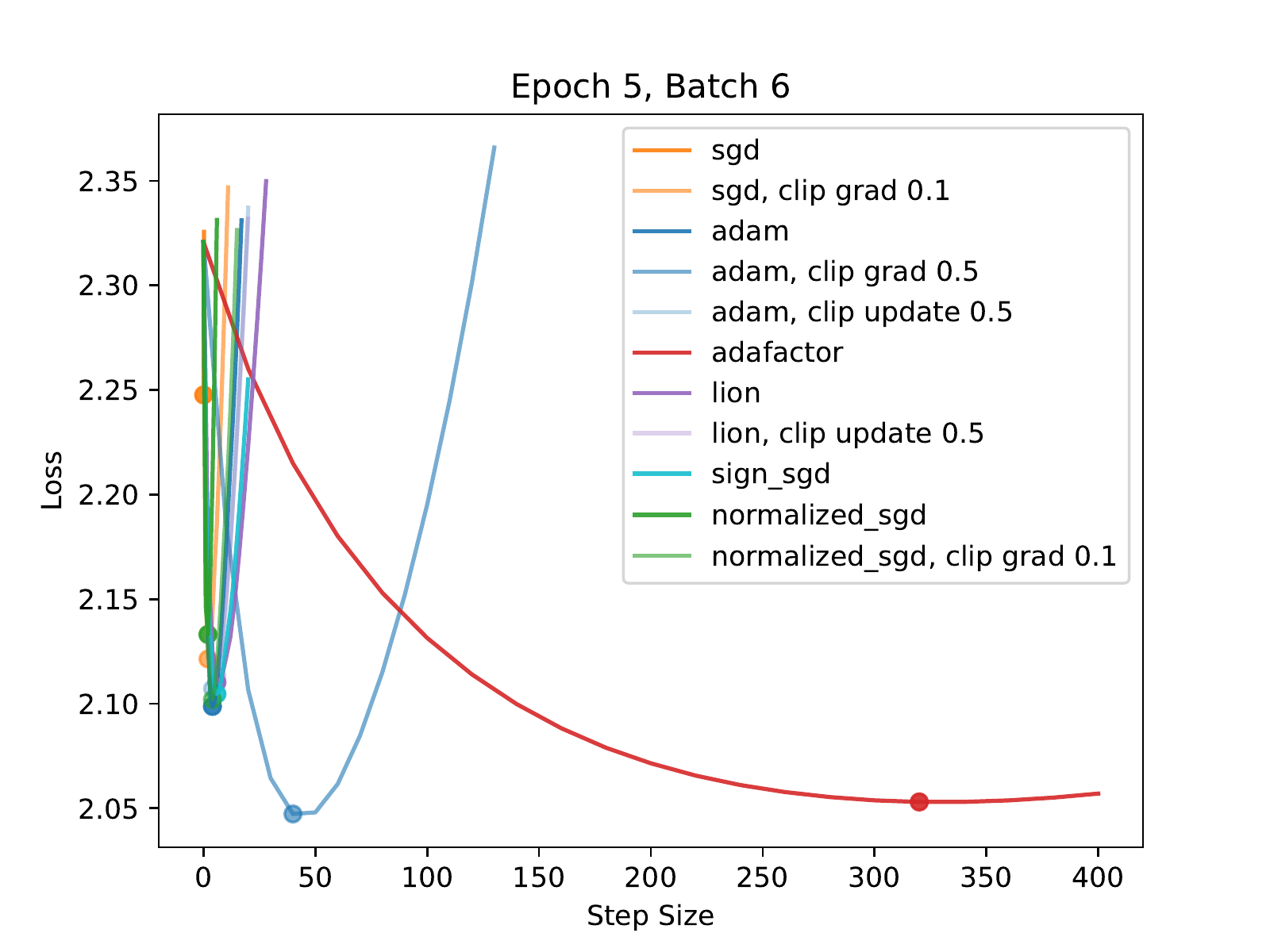}
        \includegraphics[width=0.32\textwidth]{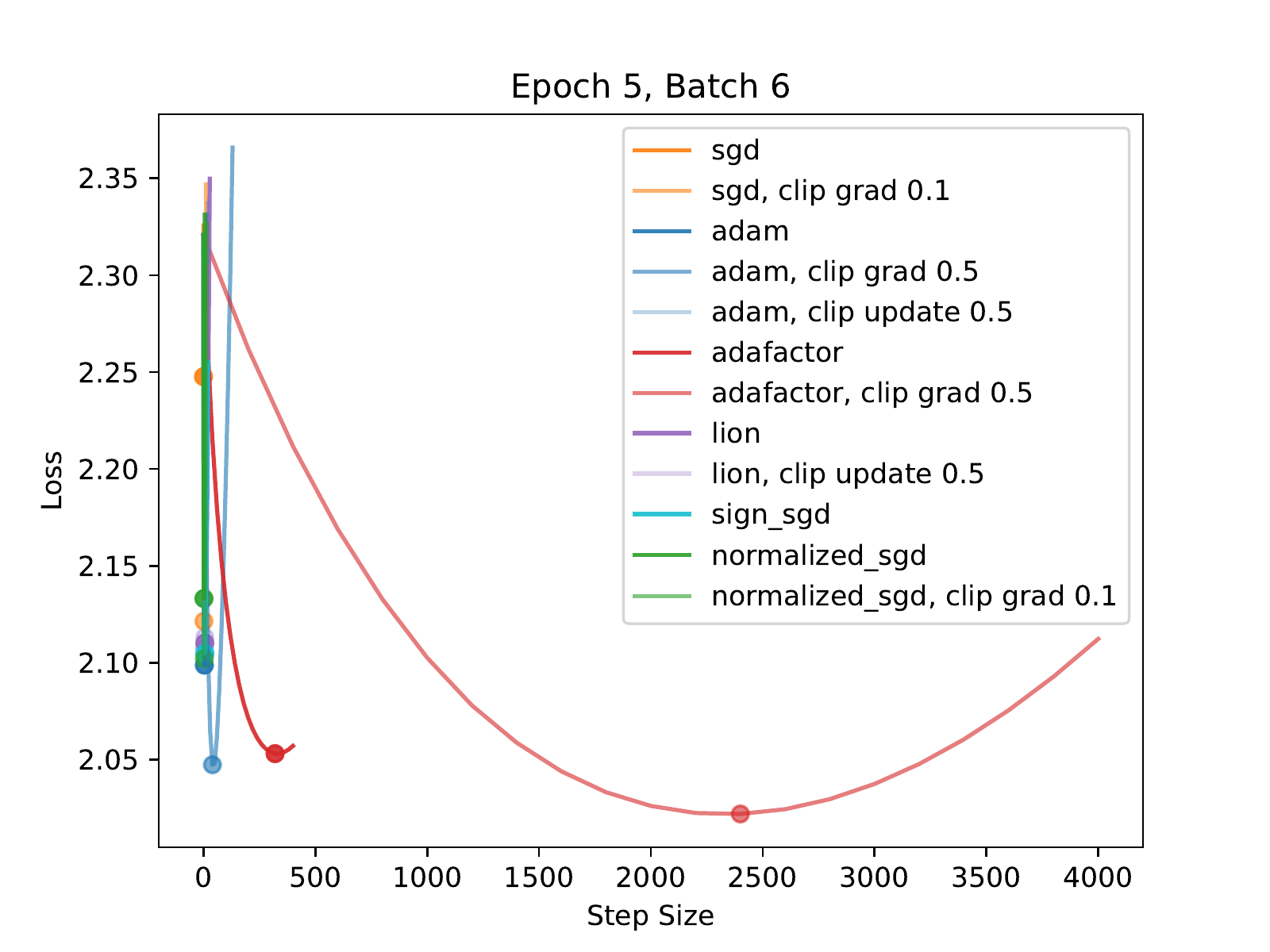}
        \caption{Experiment 3}
    \end{subfigure}
    \caption{Landscape visualization of machine translation in SGD trajectory at Epoch 5.}
\end{figure}

\begin{figure}[h]
    \centering
    \begin{subfigure}{\textwidth}\centering
        \includegraphics[width=0.32\textwidth]{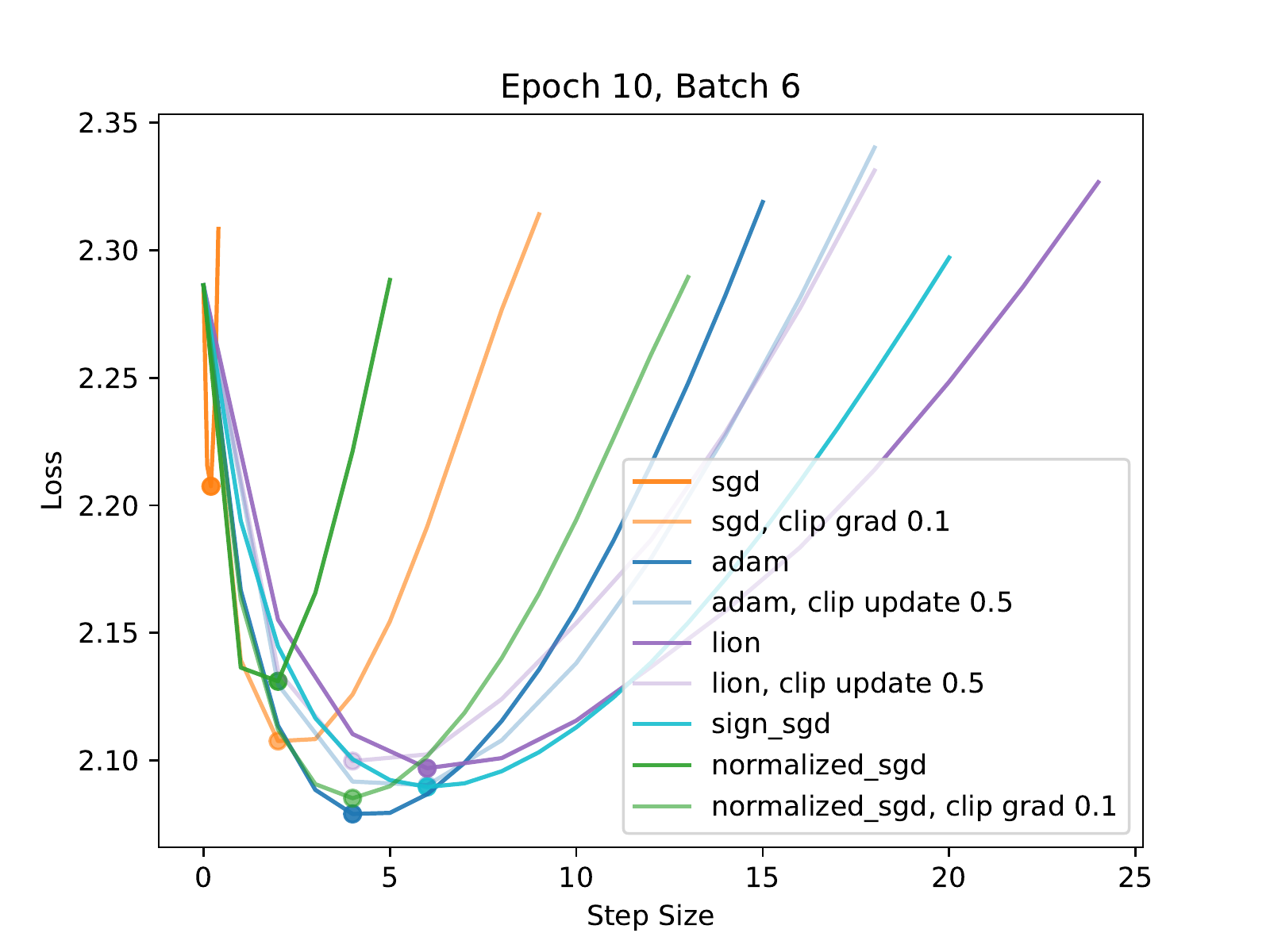}
        \includegraphics[width=0.32\textwidth]{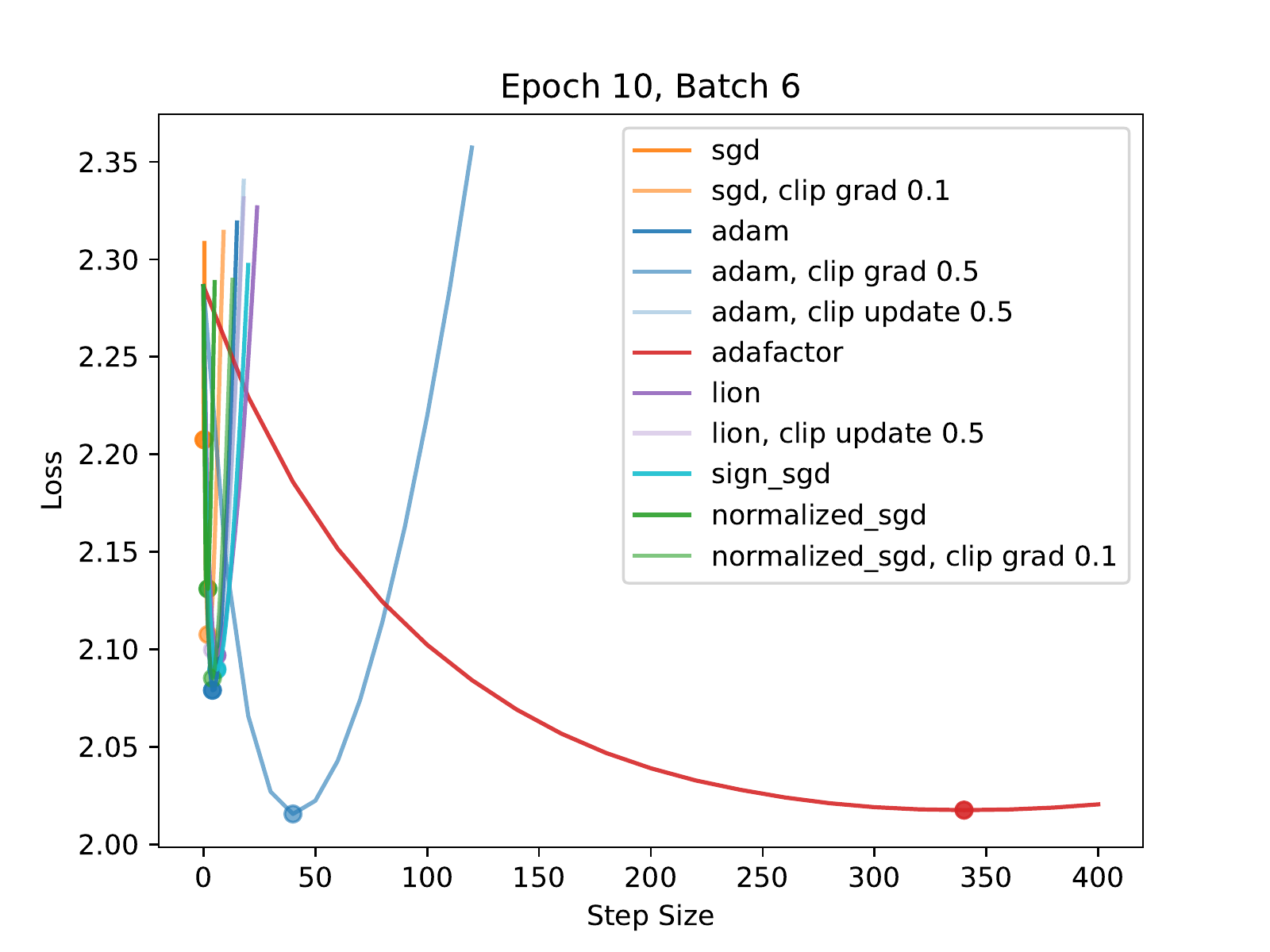}
        \includegraphics[width=0.32\textwidth]{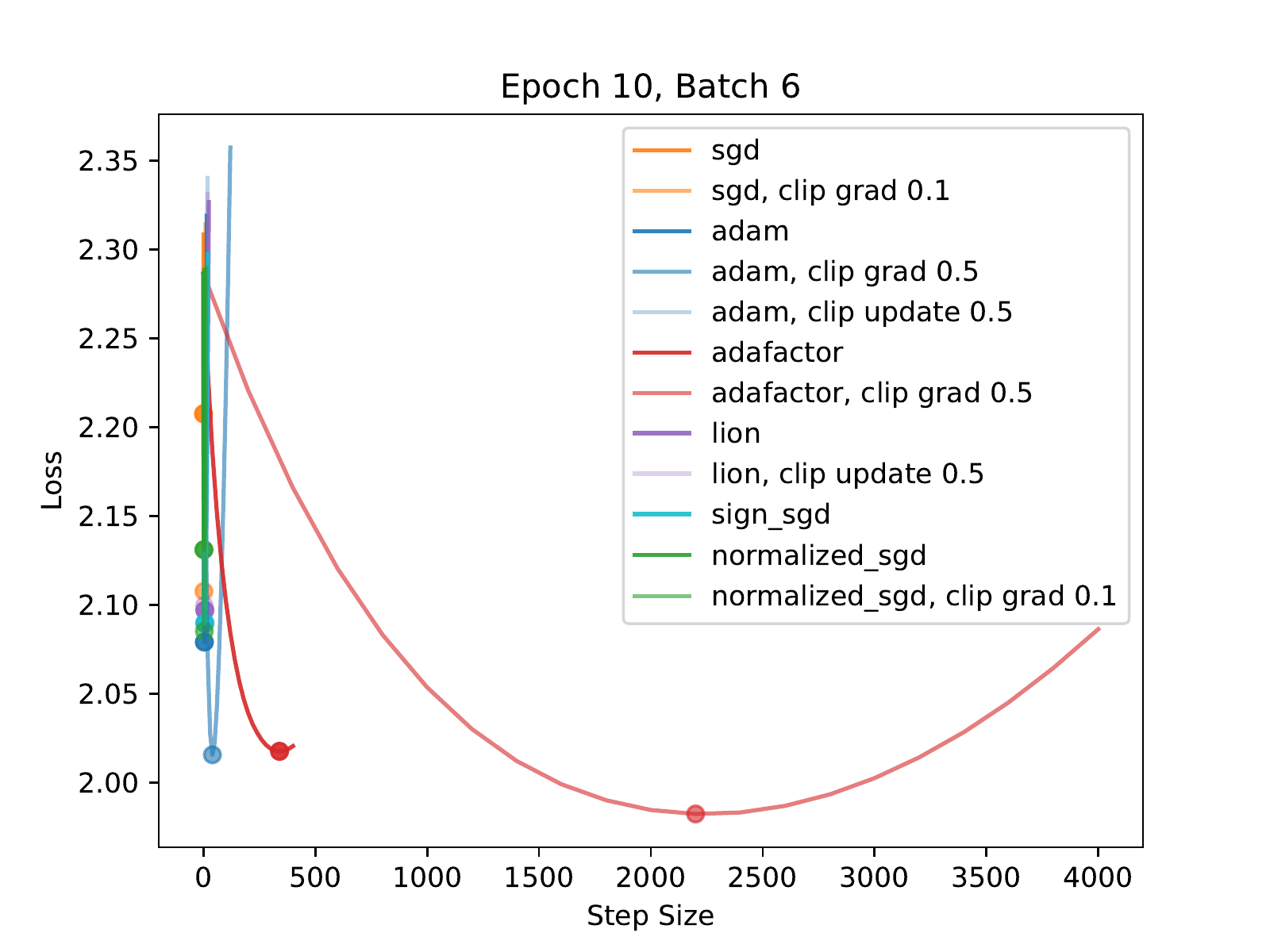}
        \caption{Experiment 1}
    \end{subfigure}
    \begin{subfigure}{\textwidth}\centering
        \includegraphics[width=0.32\textwidth]{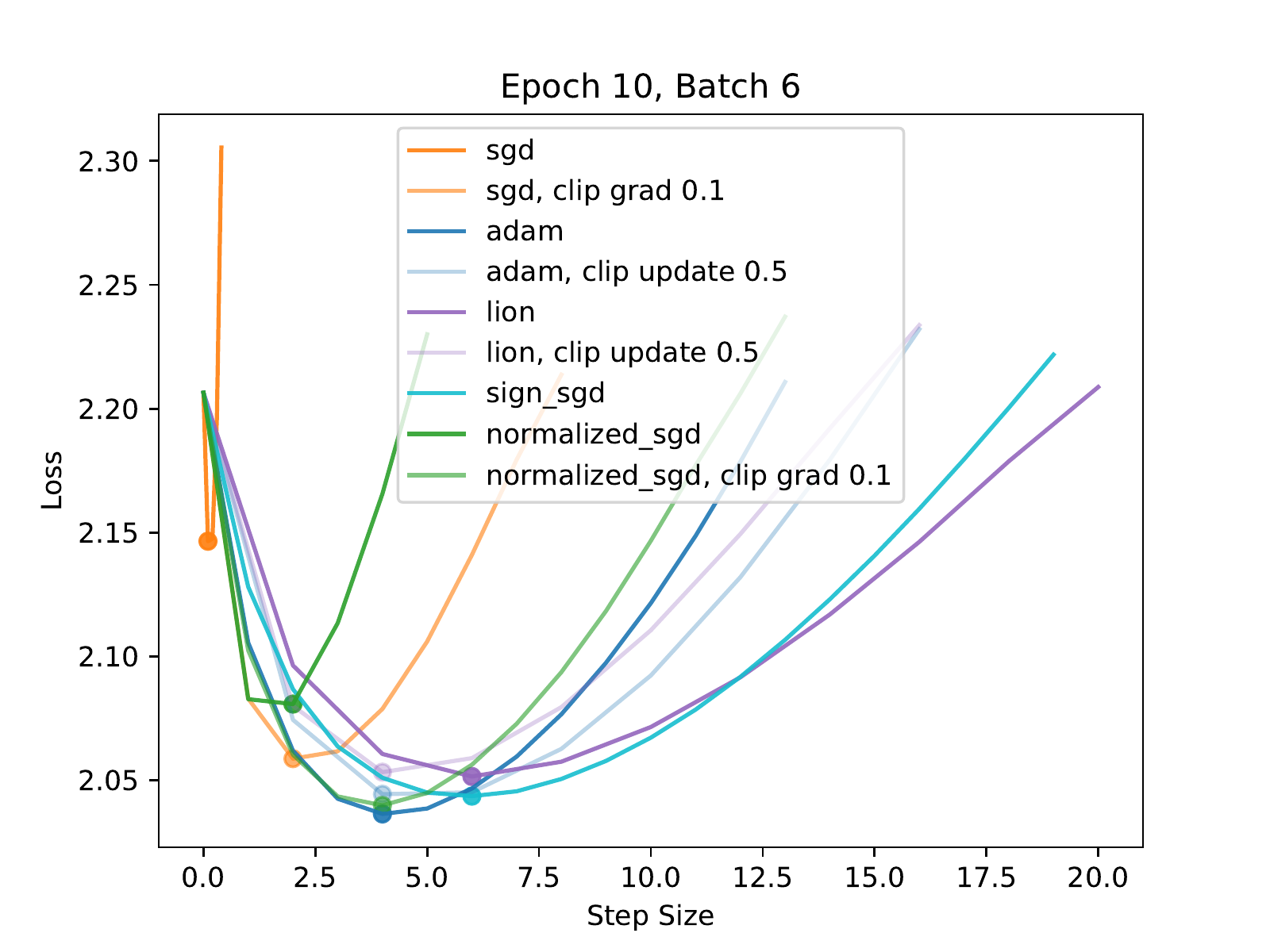}
        \includegraphics[width=0.32\textwidth]{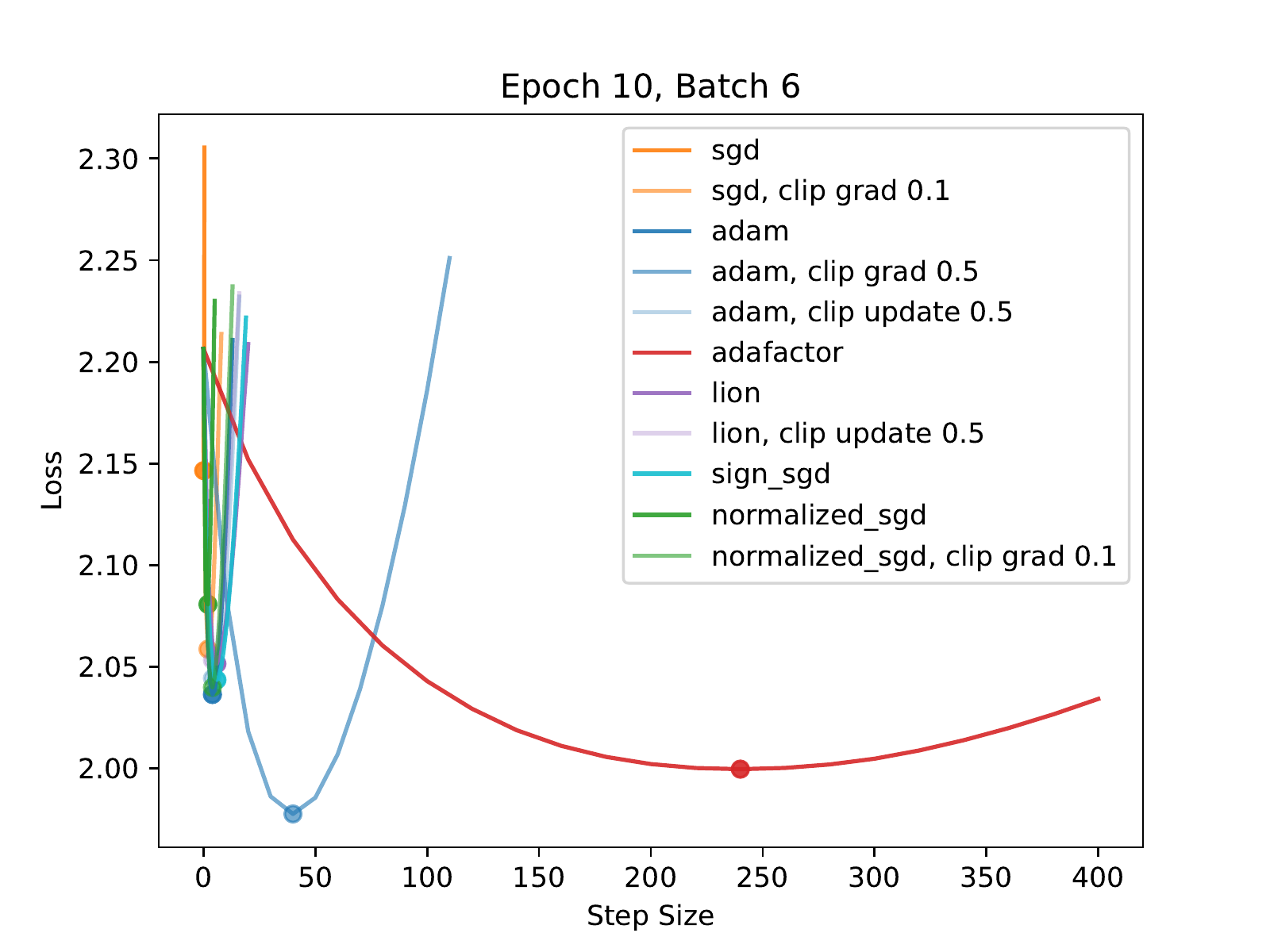}
        \includegraphics[width=0.32\textwidth]{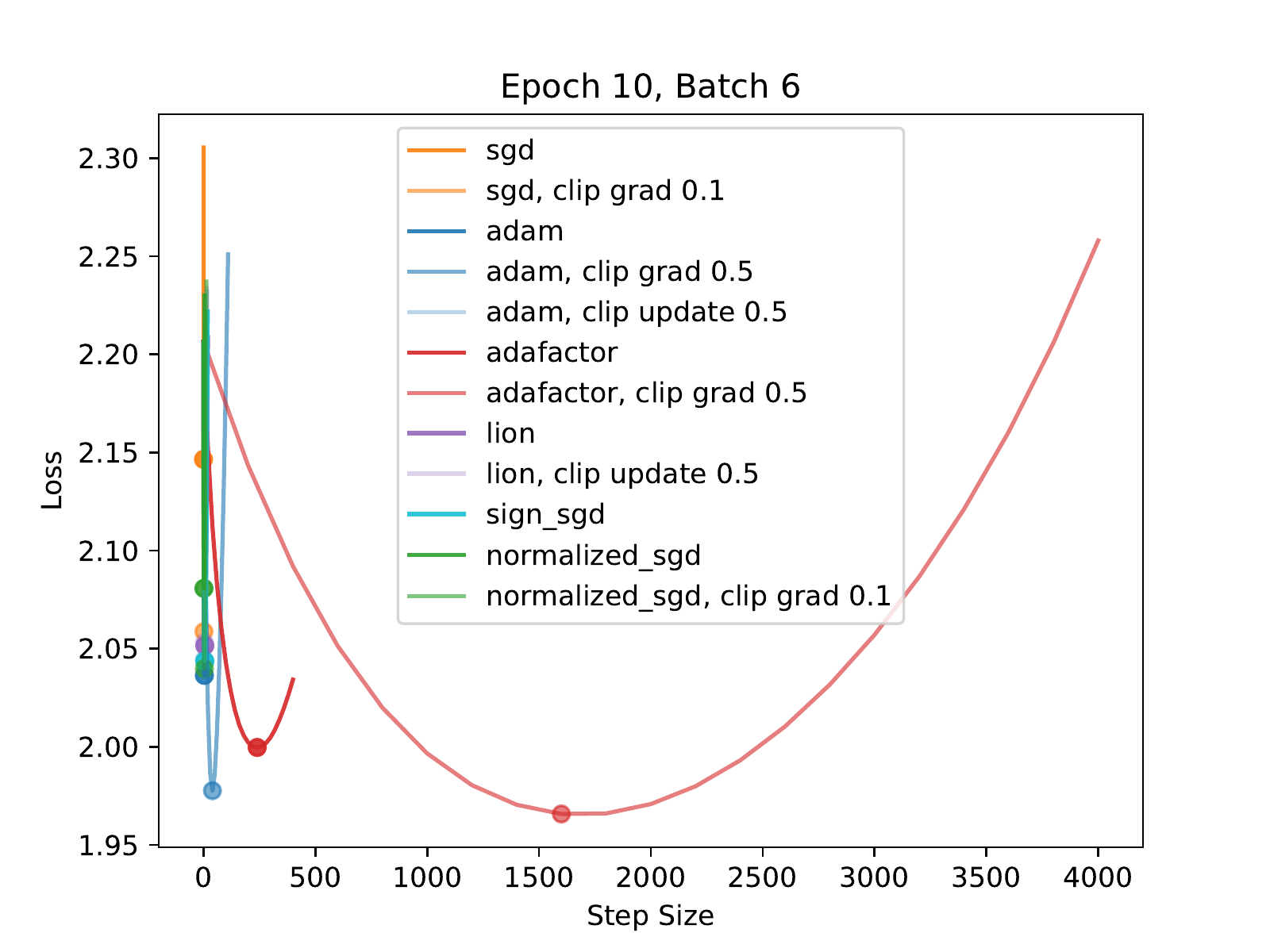}
        \caption{Experiment 2}
    \end{subfigure}
    \begin{subfigure}{\textwidth}\centering
        \includegraphics[width=0.32\textwidth]{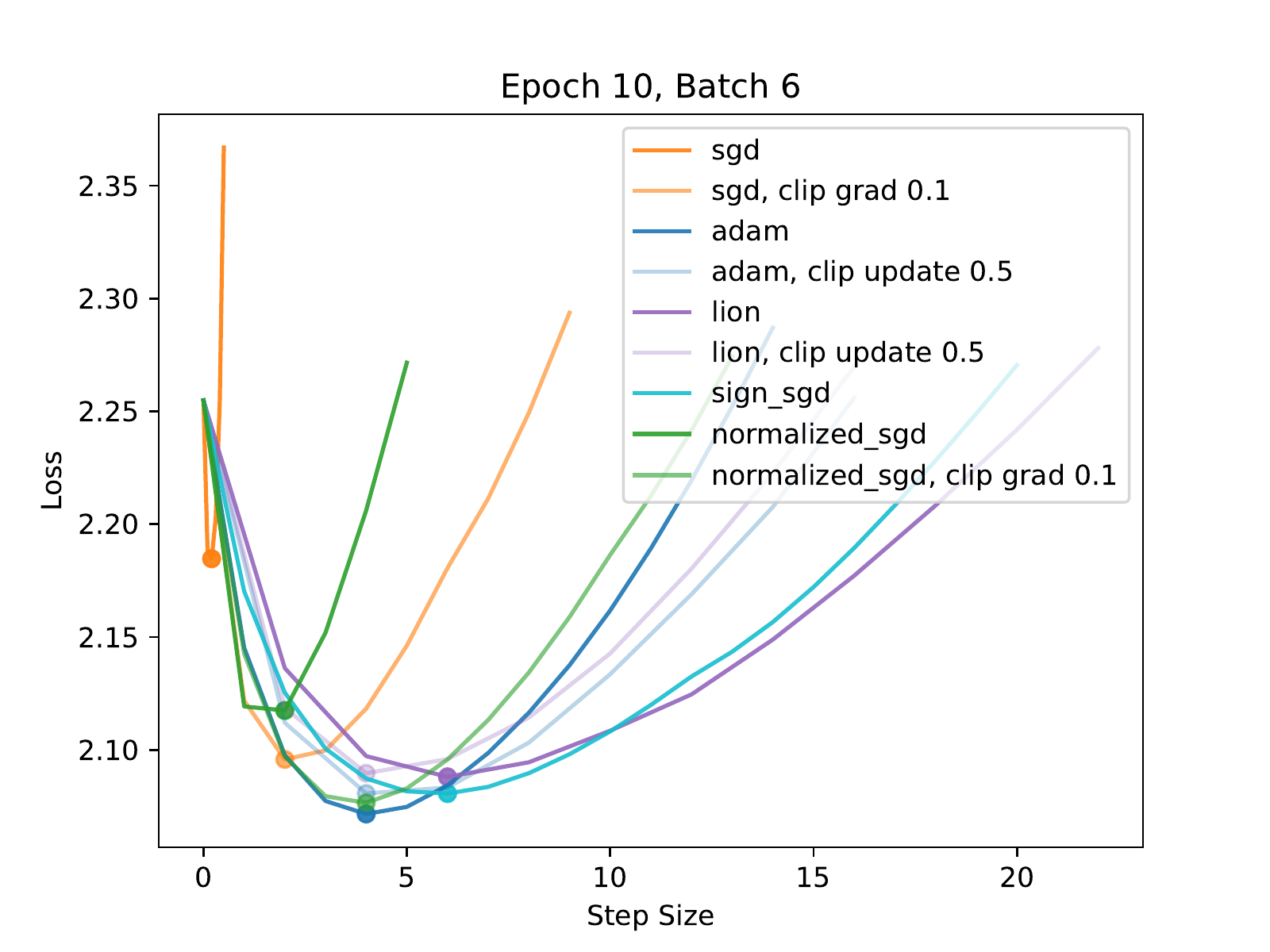}
        \includegraphics[width=0.32\textwidth]{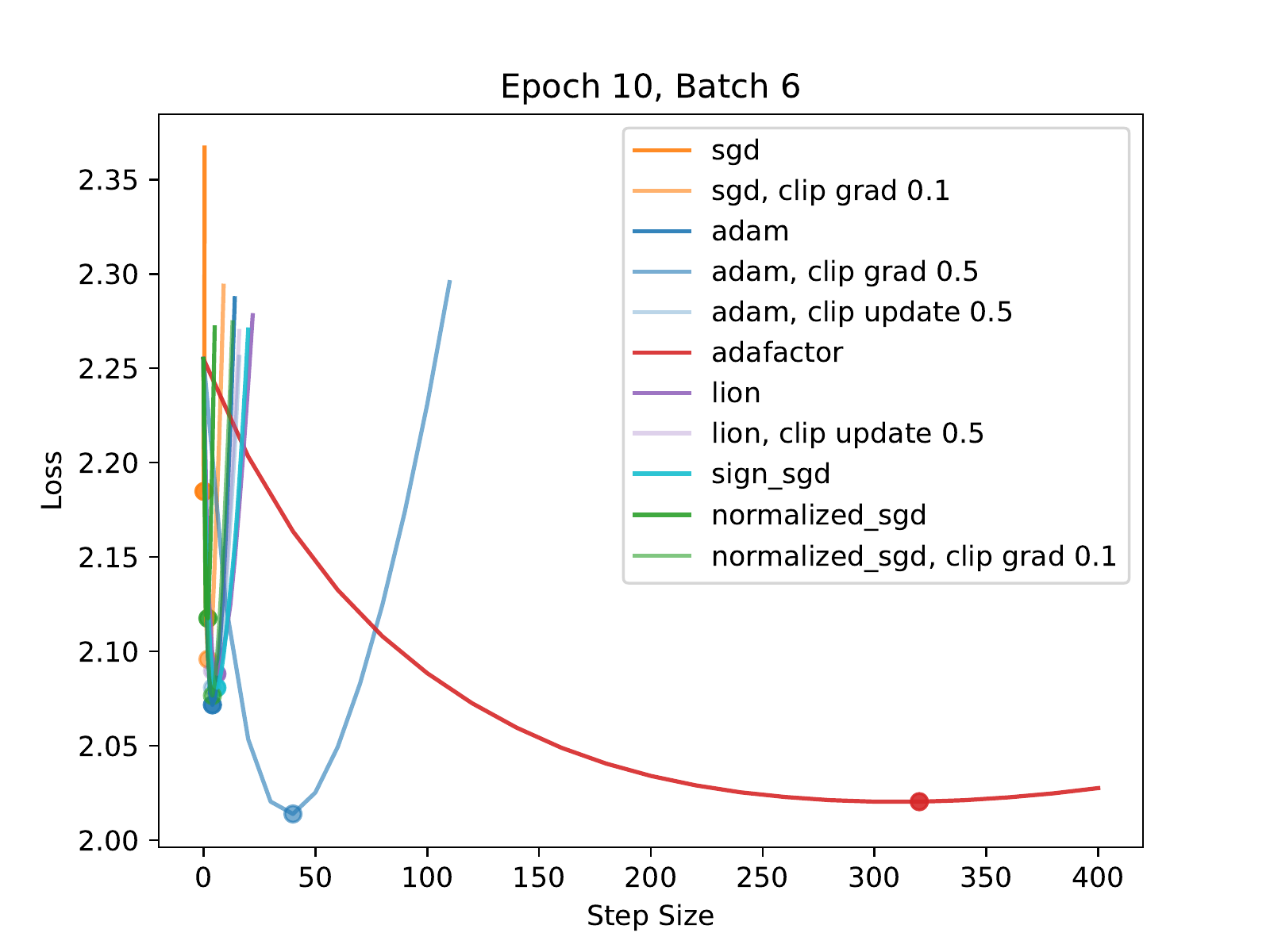}
        \includegraphics[width=0.32\textwidth]{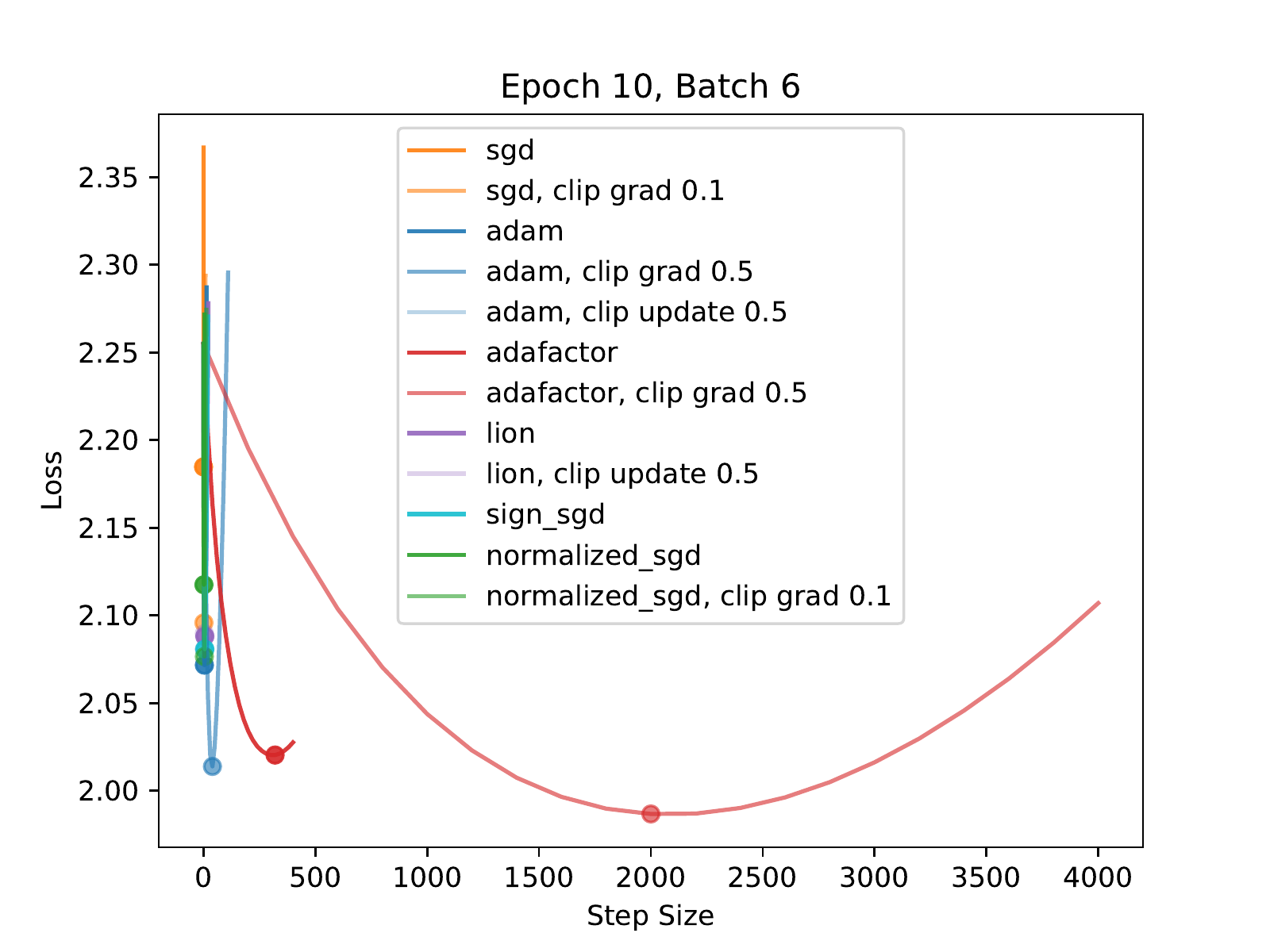}
        \caption{Experiment 3}
    \end{subfigure}
    \caption{Landscape visualization of machine translation in SGD trajectory at Epoch 10.}
\end{figure}

\begin{figure}[h]
    \centering
    \begin{subfigure}{\textwidth}\centering
        \includegraphics[width=0.32\textwidth]{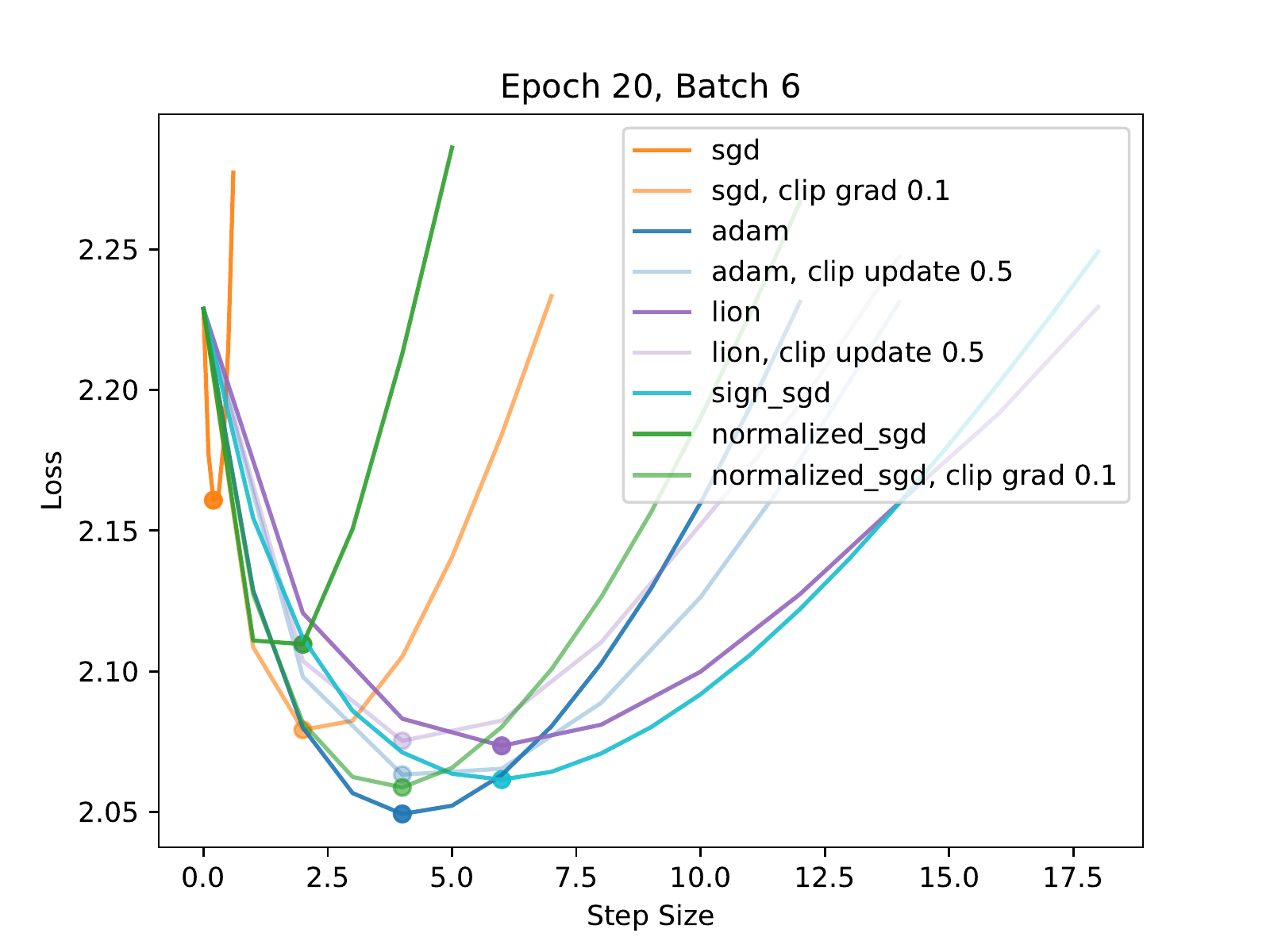}
        \includegraphics[width=0.32\textwidth]{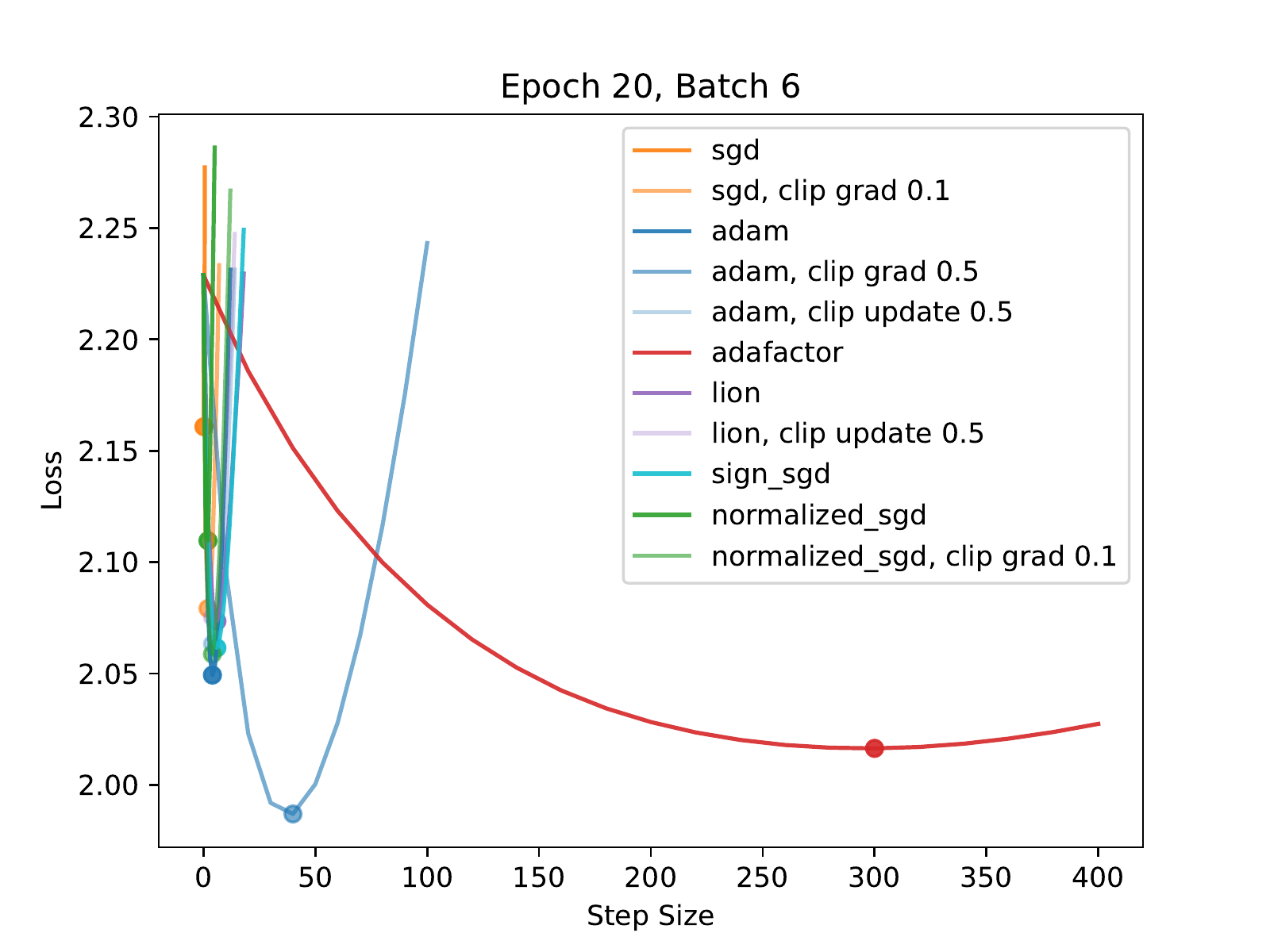}
        \includegraphics[width=0.32\textwidth]{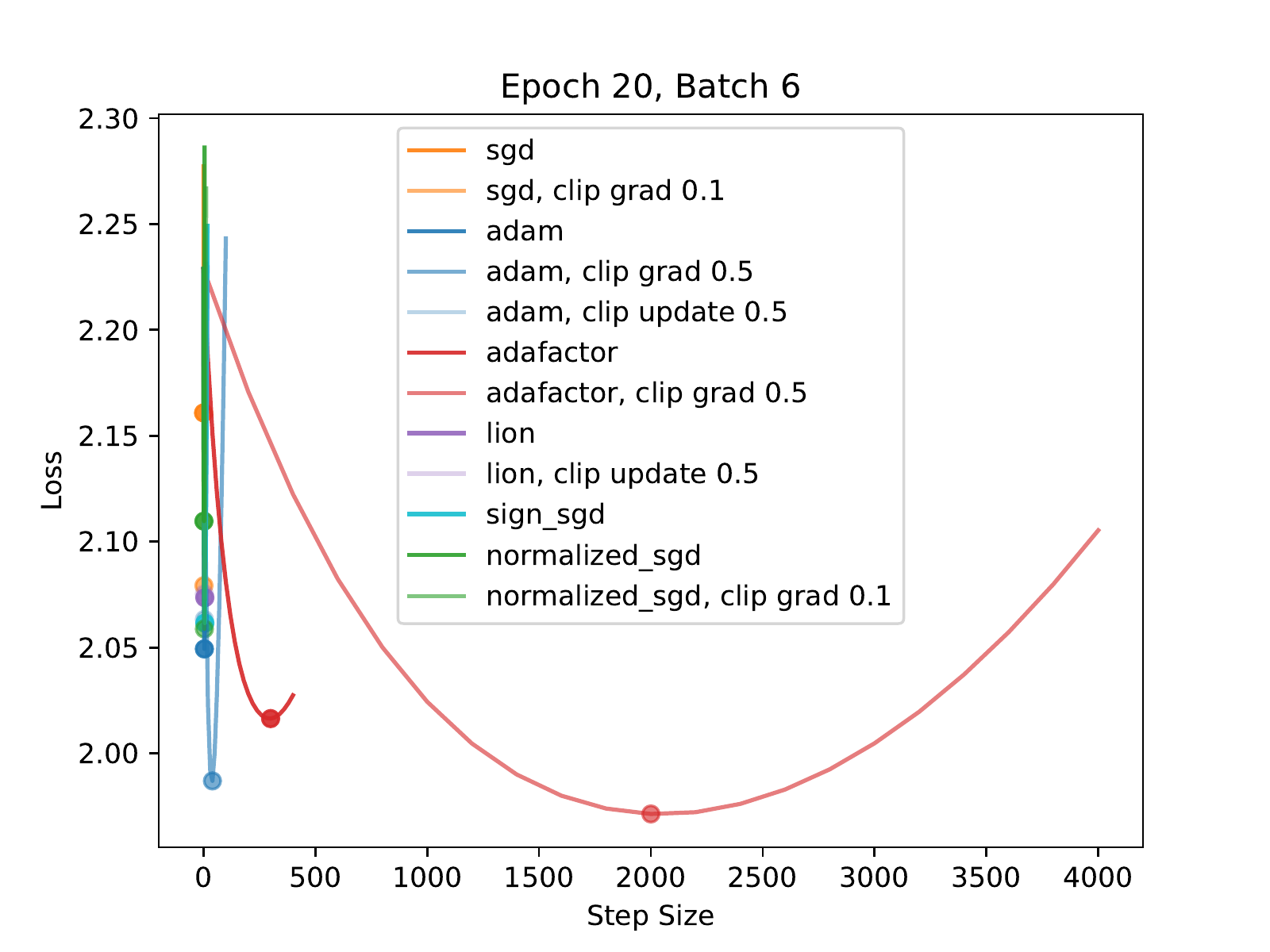}
        \caption{Experiment 1}
    \end{subfigure}
    \begin{subfigure}{\textwidth}\centering
        \includegraphics[width=0.32\textwidth]{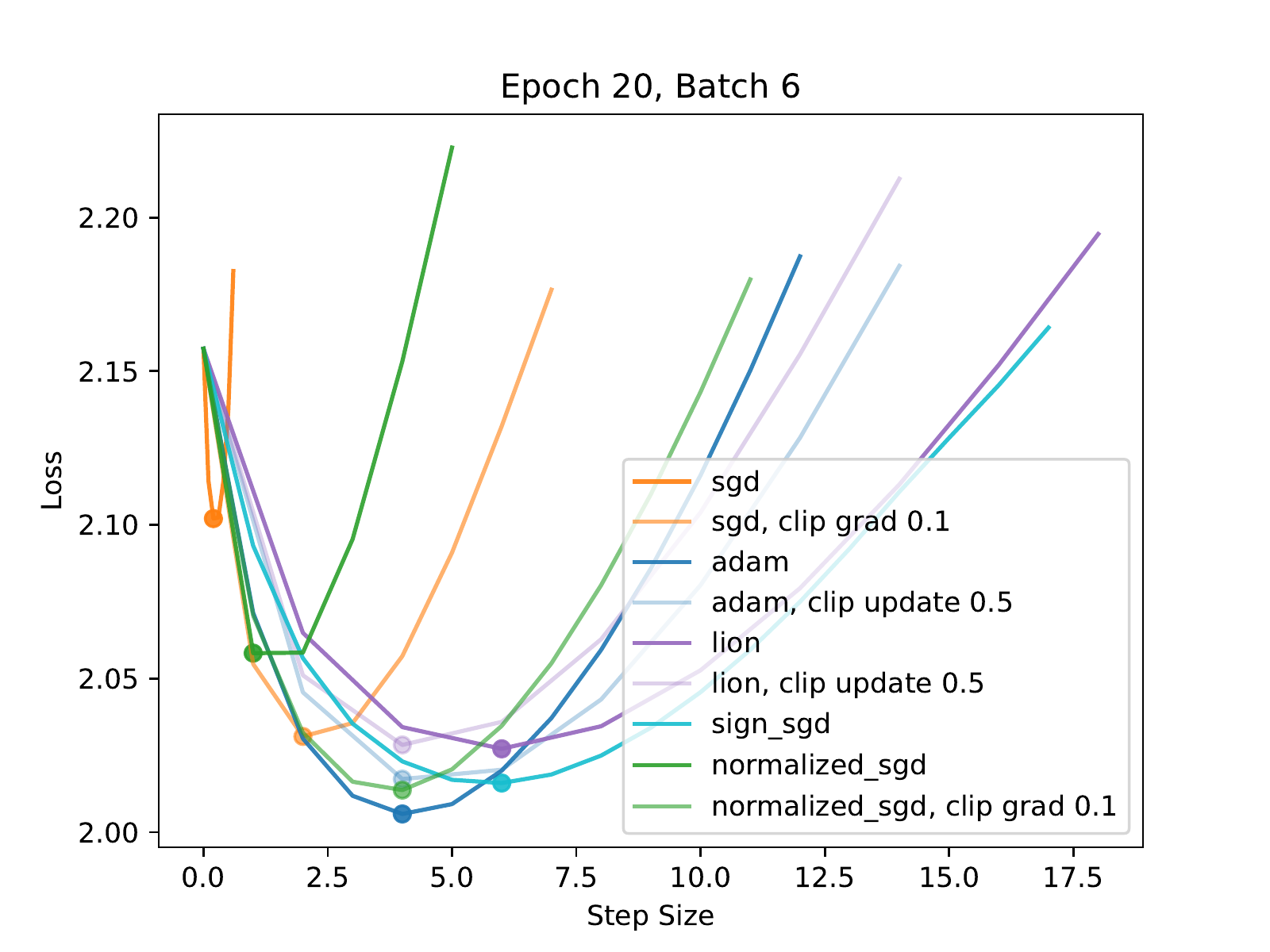}
        \includegraphics[width=0.32\textwidth]{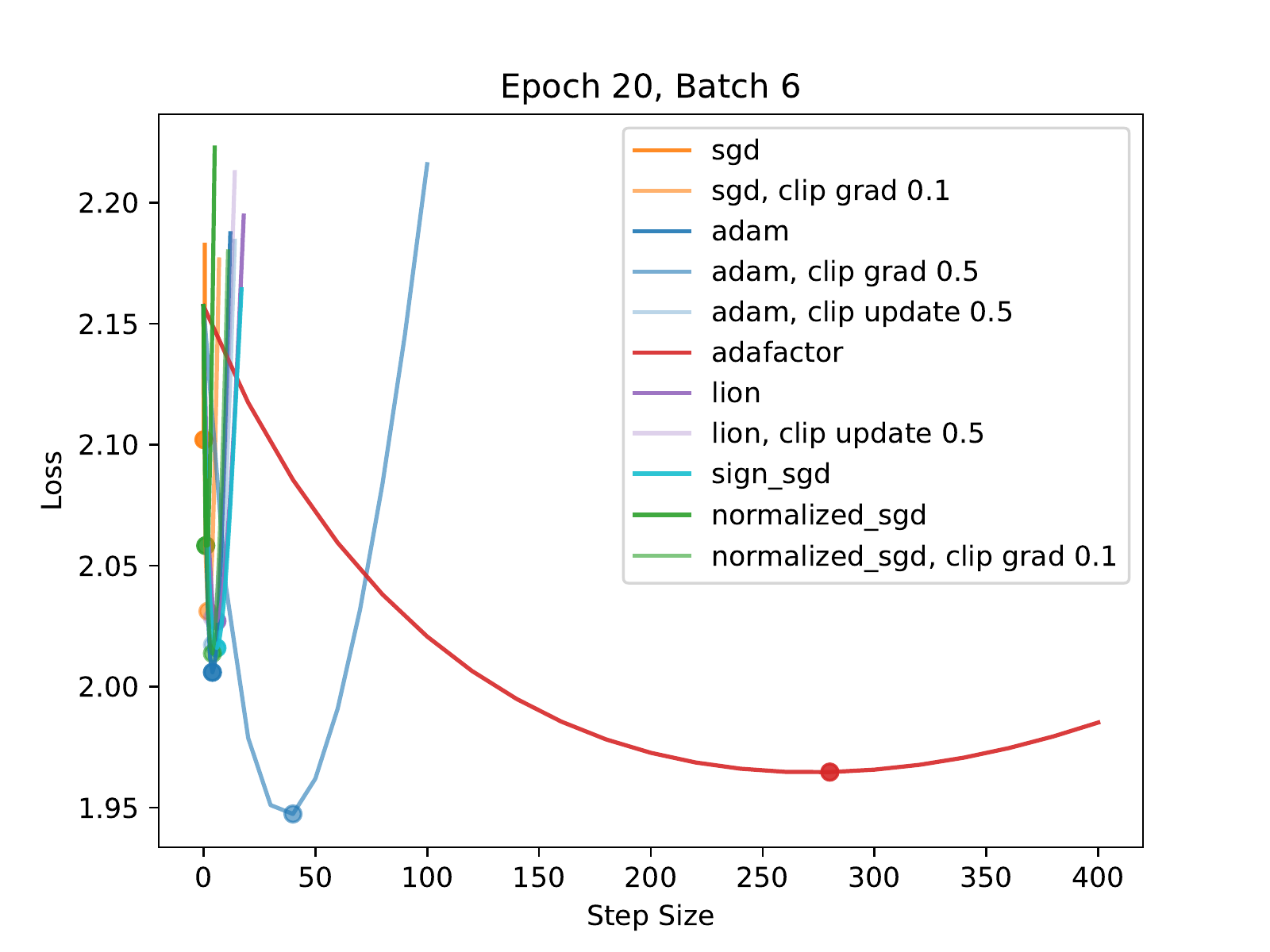}
        \includegraphics[width=0.32\textwidth]{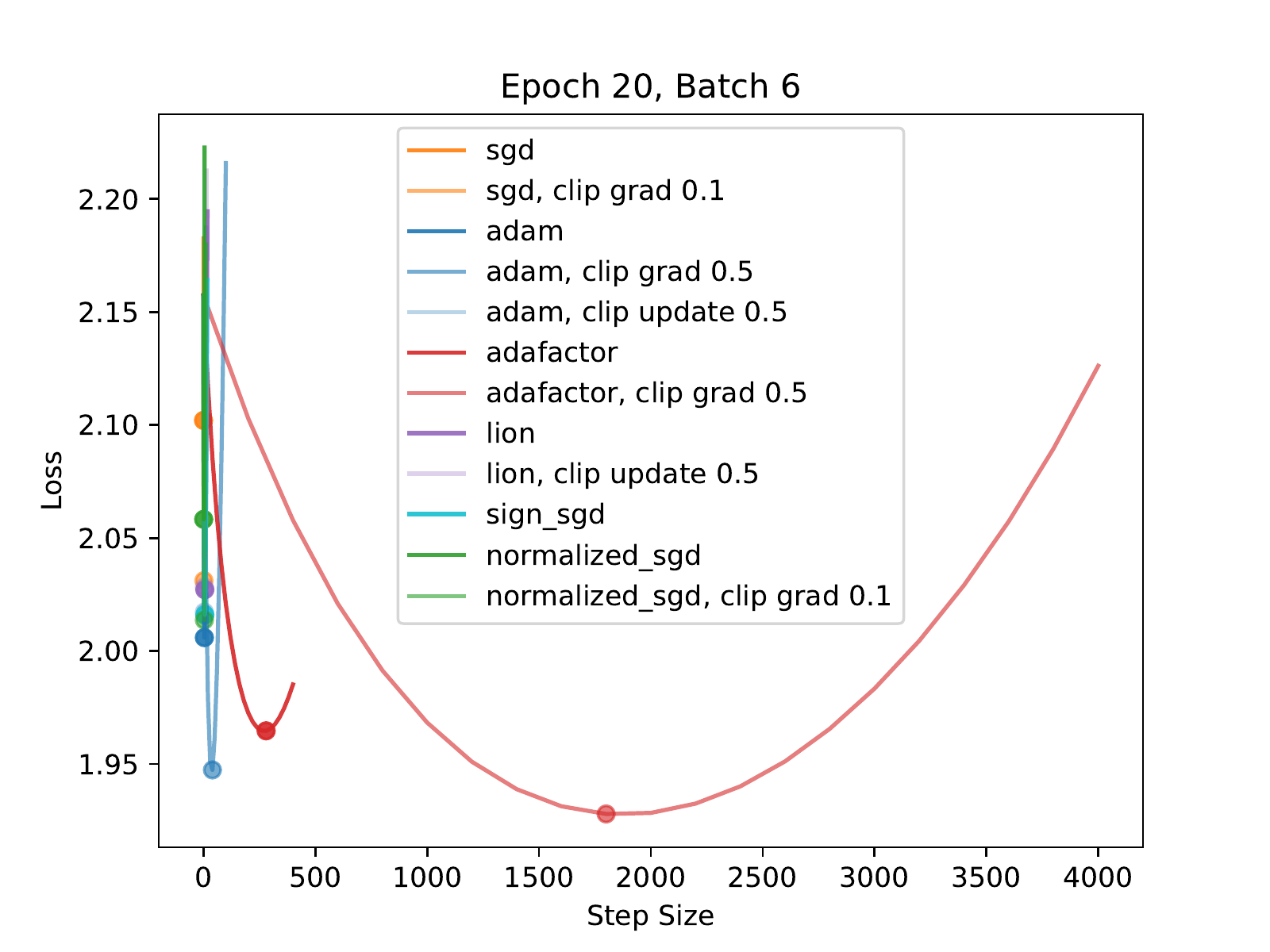}
        \caption{Experiment 2}
    \end{subfigure}
    \begin{subfigure}{\textwidth}\centering
        \includegraphics[width=0.32\textwidth]{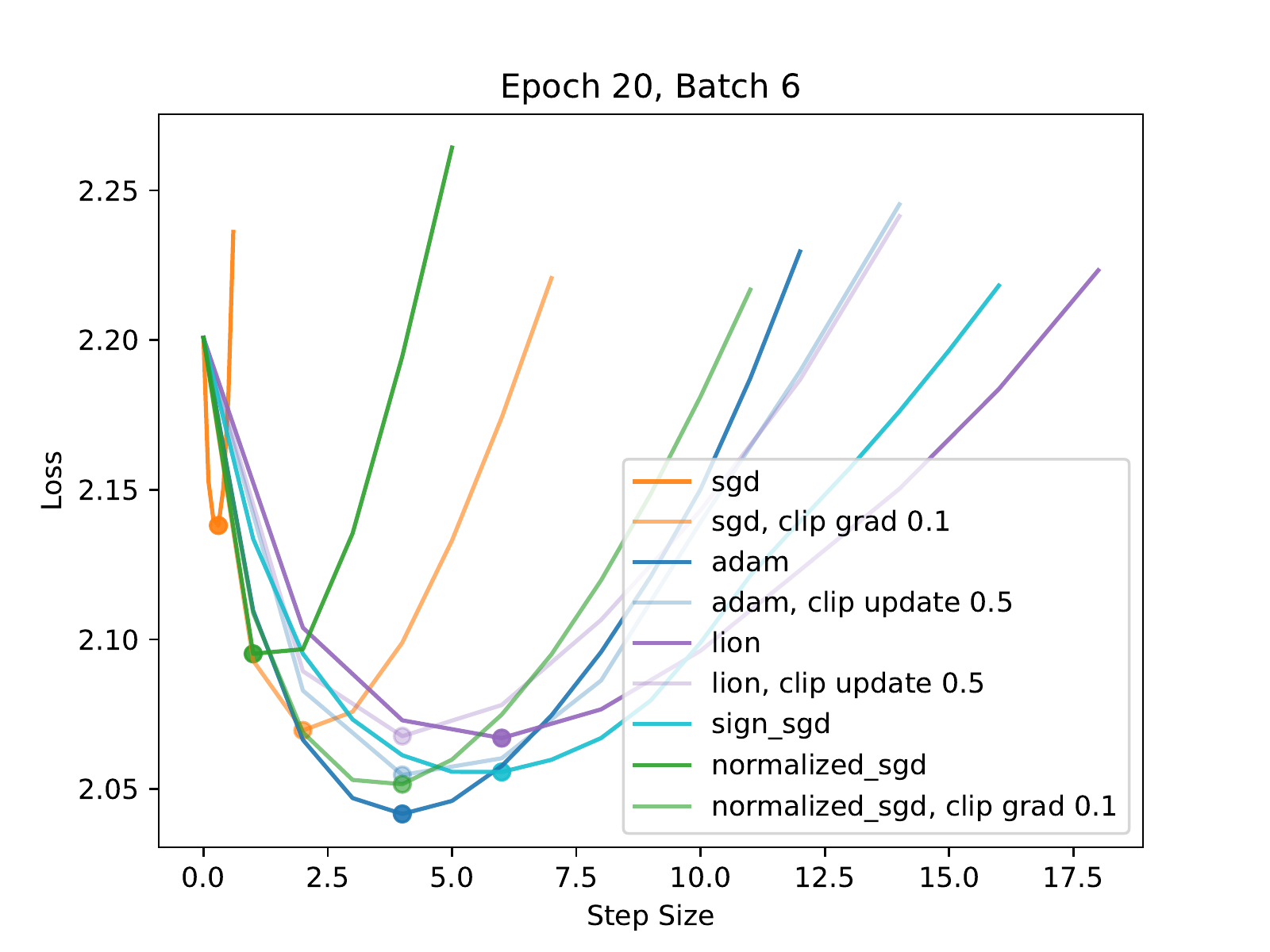}
        \includegraphics[width=0.32\textwidth]{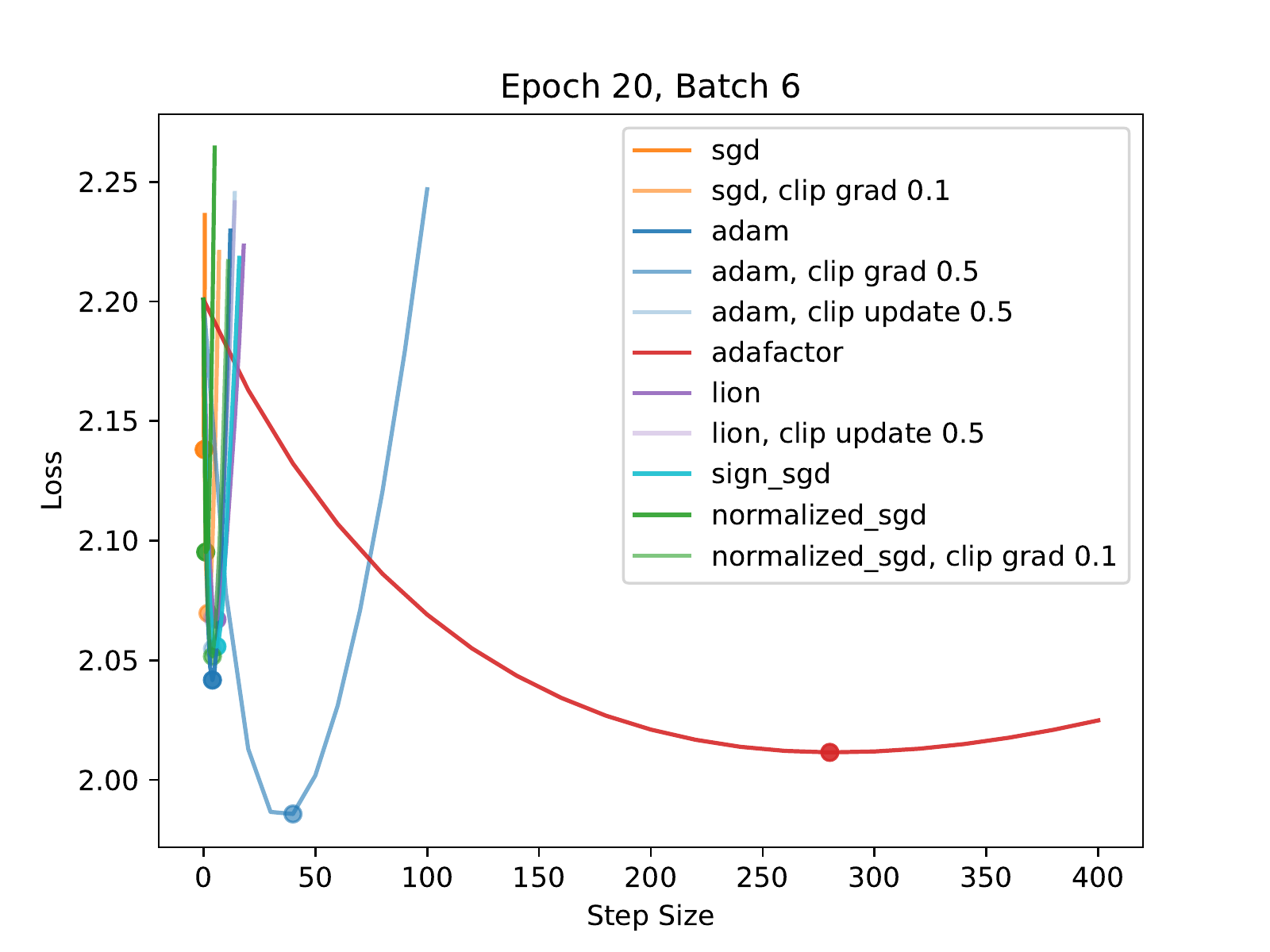}
        \includegraphics[width=0.32\textwidth]{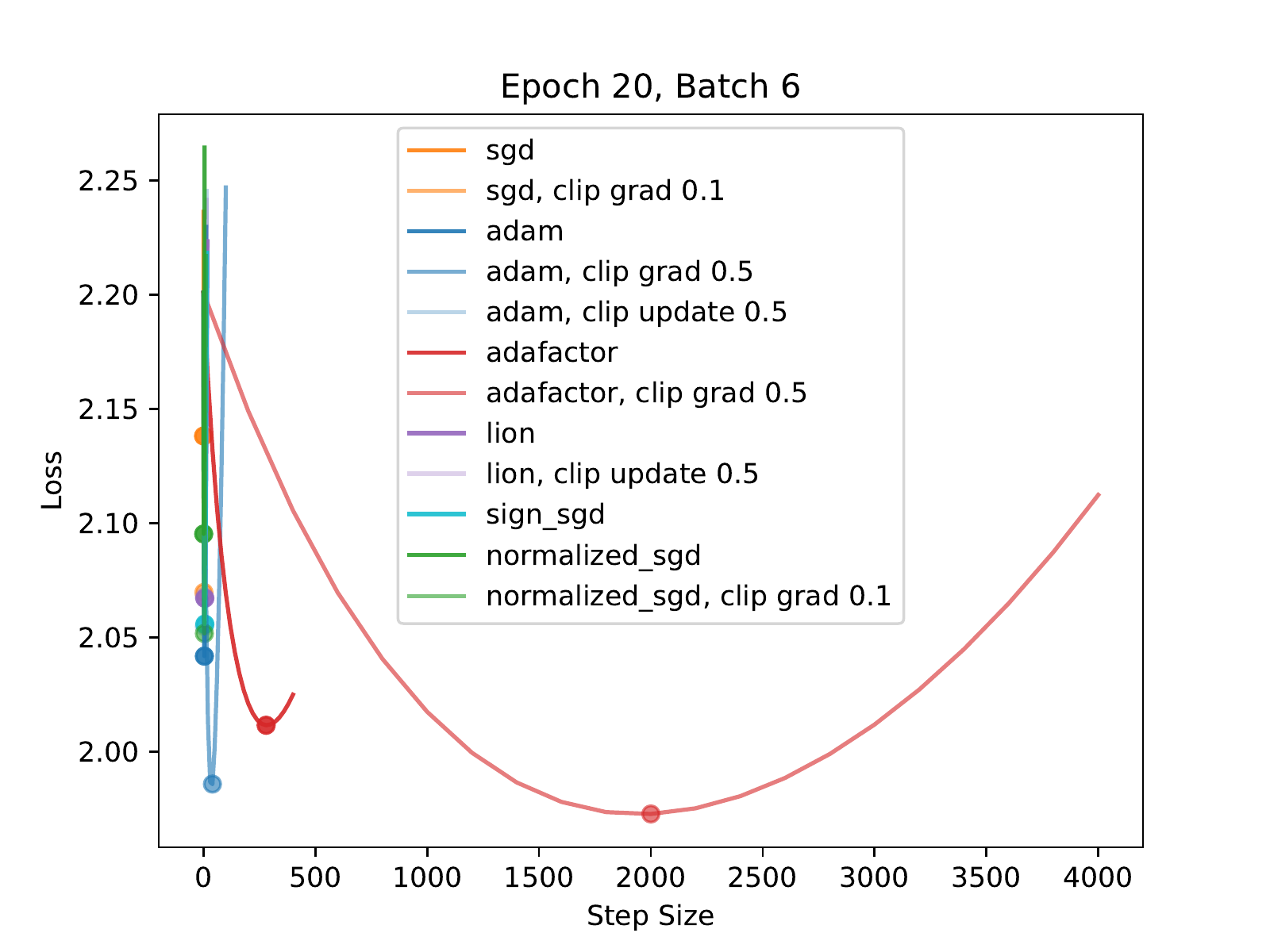}
        \caption{Experiment 3}
    \end{subfigure}
    \caption{Landscape visualization of machine translation in SGD trajectory at Epoch 20.}
\end{figure}

\begin{table}[h]
    \centering
    \begin{tabular}[h]{|l|l|l|l|l|l|}
        \hline
        \textbf{Epoch} & \textbf{Algorithm} & \textbf{Ratio 1} & \textbf{Ratio 2} & \textbf{Ratio 3} & \textbf{Mean}\\\hline
        \multirow{12}{*}{2} & sgd & $1.0$ & $1.0$ & $1.0$ & $1.0$ \\\cline{2-6}
        & sgd, clip grad 0.1 & $0.060764$ & $0.04182$ & $0.041375$ & $0.047986$ \\\cline{2-6}
        & adam & $0.029132$ & $0.020919$ & $0.020288$ & $0.023446$ \\\cline{2-6}
        & adam, clip grad 0.5 & $6.62\times 10^{-5}$ & $4.62\times 10^{-5}$ & $4.56\times 10^{-5}$ & $5.27\times 10^{-5}$ \\\cline{2-6}
        & adam, clip update 0.5 & $0.021666$ & $0.014873$ & $0.014498$ & $0.017012$ \\\cline{2-6}
        & adafactor & $6.91\times 10^{-6}$ & $4.47\times 10^{-6}$ & $3.91\times 10^{-6}$ & $5.1\times 10^{-6}$ \\\cline{2-6}
        & adafactor, clip grad 0.5 & $2.36\times 10^{-8}$ & $1.49\times 10^{-8}$ & $1.25\times 10^{-8}$ & $1.7\times 10^{-8}$ \\\cline{2-6}
        & lion & $0.013626$ & $0.009318$ & $0.00908$ & $0.010675$ \\\cline{2-6}
        & lion, clip update 0.5 & $0.024812$ & $0.016958$ & $0.016522$ & $0.019431$ \\\cline{2-6}
        & sign sgd & $0.009136$ & $0.006336$ & $0.006284$ & $0.007252$ \\\cline{2-6}
        & normalized sgd & $0.083261$ & $0.056086$ & $0.055766$ & $0.065037$ \\\cline{2-6}
        & normalized sgd, clip grad 0.1 & $0.020651$ & $0.01418$ & $0.013987$ & $0.016273$ \\\cline{1-6}
        \multirow{12}{*}{5} & sgd & $1.0$ & $1.0$ & $1.0$ & $1.0$ \\\cline{2-6}
        & sgd, clip grad 0.1 & $0.010348$ & $0.015692$ & $-0.00493$ & $0.007037$ \\\cline{2-6}
        & adam & $0.003186$ & $0.004729$ & $-0.000738$ & $0.002392$ \\\cline{2-6}
        & adam, clip grad 0.5 & $1.04\times 10^{-5}$ & $1.75\times 10^{-5}$ & $8.79\times 10^{-6}$ & $1.22\times 10^{-5}$ \\\cline{2-6}
        & adam, clip update 0.5 & $0.002292$ & $0.003452$ & $-0.000687$ & $0.001686$ \\\cline{2-6}
        & adafactor & $7.61\times 10^{-7}$ & $1.13\times 10^{-6}$ & $5.07\times 10^{-7}$ & $8.0\times 10^{-7}$ \\\cline{2-6}
        & adafactor, clip grad 0.5 & $3.07\times 10^{-9}$ & $5.01\times 10^{-9}$ & $4.24\times 10^{-9}$ & $4.11\times 10^{-9}$ \\\cline{2-6}
        & lion & $0.001448$ & $0.00218$ & $-0.000473$ & $0.001052$ \\\cline{2-6}
        & lion, clip update 0.5 & $0.002629$ & $0.003959$ & $-0.000863$ & $0.001908$ \\\cline{2-6}
        & sign sgd & $0.001542$ & $0.002334$ & $-0.000362$ & $0.001172$ \\\cline{2-6}
        & normalized sgd & $0.013939$ & $0.021917$ & $0.005331$ & $0.013729$ \\\cline{2-6}
        & normalized sgd, clip grad 0.1 & $0.003351$ & $0.005242$ & $-0.000223$ & $0.00279$ \\\cline{1-6}
        \multirow{12}{*}{10} & sgd & $1.0$ & $1.0$ & $1.0$ & $1.0$ \\\cline{2-6}
        & sgd, clip grad 0.1 & $0.038175$ & $0.033345$ & $0.044562$ & $0.038694$ \\\cline{2-6}
        & adam & $0.010894$ & $0.009658$ & $0.013078$ & $0.01121$ \\\cline{2-6}
        & adam, clip grad 0.5 & $5.4\times 10^{-5}$ & $4.83\times 10^{-5}$ & $6.47\times 10^{-5}$ & $5.57\times 10^{-5}$ \\\cline{2-6}
        & adam, clip update 0.5 & $0.008573$ & $0.007475$ & $0.010194$ & $0.008748$ \\\cline{2-6}
        & adafactor & $2.59\times 10^{-6}$ & $3.31\times 10^{-6}$ & $3.09\times 10^{-6}$ & $3.0\times 10^{-6}$ \\\cline{2-6}
        & adafactor, clip grad 0.5 & $1.4\times 10^{-8}$ & $1.69\times 10^{-8}$ & $1.7\times 10^{-8}$ & $1.6\times 10^{-8}$ \\\cline{2-6}
        & lion & $0.005113$ & $0.004561$ & $0.006098$ & $0.005257$ \\\cline{2-6}
        & lion, clip update 0.5 & $0.009256$ & $0.008259$ & $0.011046$ & $0.00952$ \\\cline{2-6}
        & sign sgd & $0.005738$ & $0.00497$ & $0.006808$ & $0.005839$ \\\cline{2-6}
        & normalized sgd & $0.053926$ & $0.046155$ & $0.061938$ & $0.054007$ \\\cline{2-6}
        & normalized sgd, clip grad 0.1 & $0.012715$ & $0.01102$ & $0.014784$ & $0.01284$ \\\cline{1-6}
        \multirow{12}{*}{20} & sgd & $1.0$ & $1.0$ & $1.0$ & $1.0$ \\\cline{2-6}
        & sgd, clip grad 0.1 & $0.046561$ & $-0.009589$ & $0.043289$ & $0.026754$ \\\cline{2-6}
        & adam & $0.012866$ & $-0.002554$ & $0.011917$ & $0.00741$ \\\cline{2-6}
        & adam, clip grad 0.5 & $9.15\times 10^{-5}$ & $-1.97\times 10^{-5}$ & $8.71\times 10^{-5}$ & $5.3\times 10^{-5}$ \\\cline{2-6}
        & adam, clip update 0.5 & $0.010421$ & $-0.002174$ & $0.009732$ & $0.005993$ \\\cline{2-6}
        & adafactor & $3.02\times 10^{-6}$ & $-6.96\times 10^{-7}$ & $2.7\times 10^{-6}$ & $1.68\times 10^{-6}$ \\\cline{2-6}
        & adafactor, clip grad 0.5 & $2.46\times 10^{-8}$ & $-6.03\times 10^{-9}$ & $2.23\times 10^{-8}$ & $1.36\times 10^{-8}$ \\\cline{2-6}
        & lion & $0.006787$ & $-0.001373$ & $0.006381$ & $0.003932$ \\\cline{2-6}
        & lion, clip update 0.5 & $0.012249$ & $-0.002472$ & $0.011518$ & $0.007098$ \\\cline{2-6}
        & sign sgd & $0.006936$ & $-0.001495$ & $0.006453$ & $0.003965$ \\\cline{2-6}
        & normalized sgd & $0.066658$ & $-0.014244$ & $0.06107$ & $0.037828$ \\\cline{2-6}
        & normalized sgd, clip grad 0.1 & $0.015486$ & $-0.00331$ & $0.014405$ & $0.00886$ \\\cline{1-6}
    \end{tabular}
    \caption{Ratio of directional sharpness of optimization algorithms with respect to SGD on the machine translation task in SGD trajectory in 3 experiments.}
\end{table}

\begin{figure}[h]
    \centering
    \begin{subfigure}{0.45\textwidth}
        \centering
        \includegraphics[width=\textwidth]{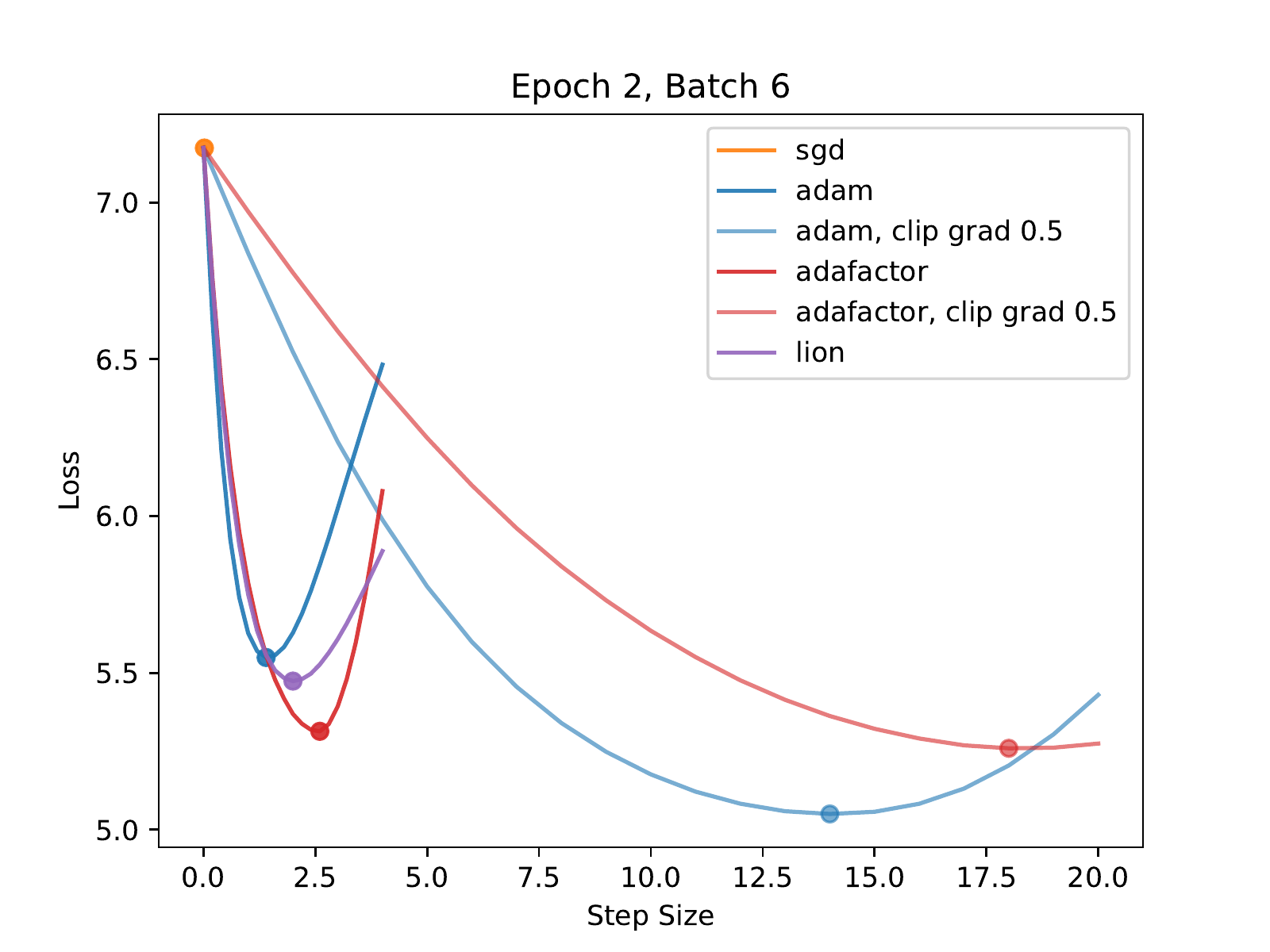}
        \caption{Epoch 2}
    \end{subfigure}
    \begin{subfigure}{0.45\textwidth}
        \centering
        \includegraphics[width=\textwidth]{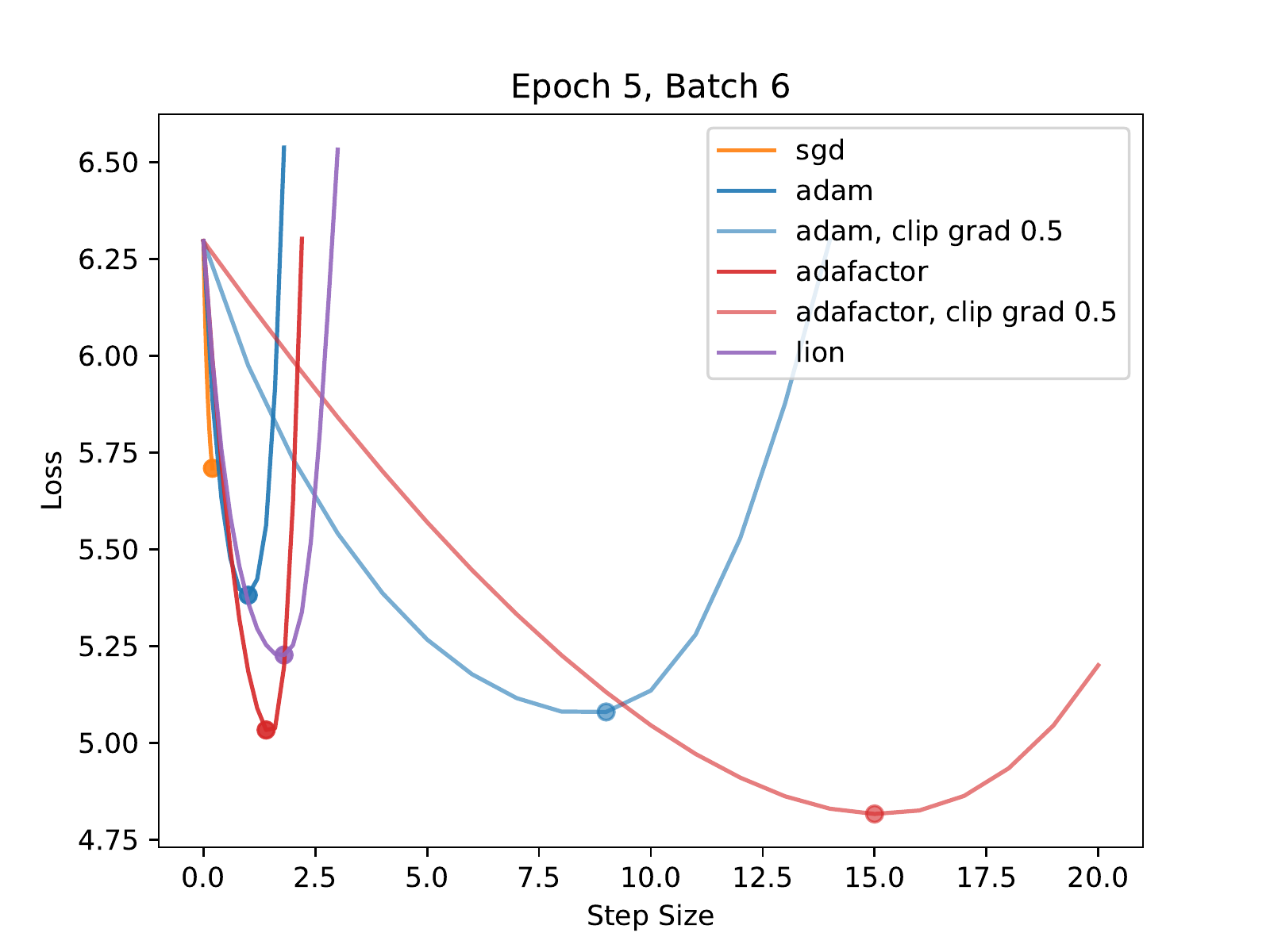}
        \caption{Epoch 5}
    \end{subfigure}
    \begin{subfigure}{0.45\textwidth}
        \centering
        \includegraphics[width=\textwidth]{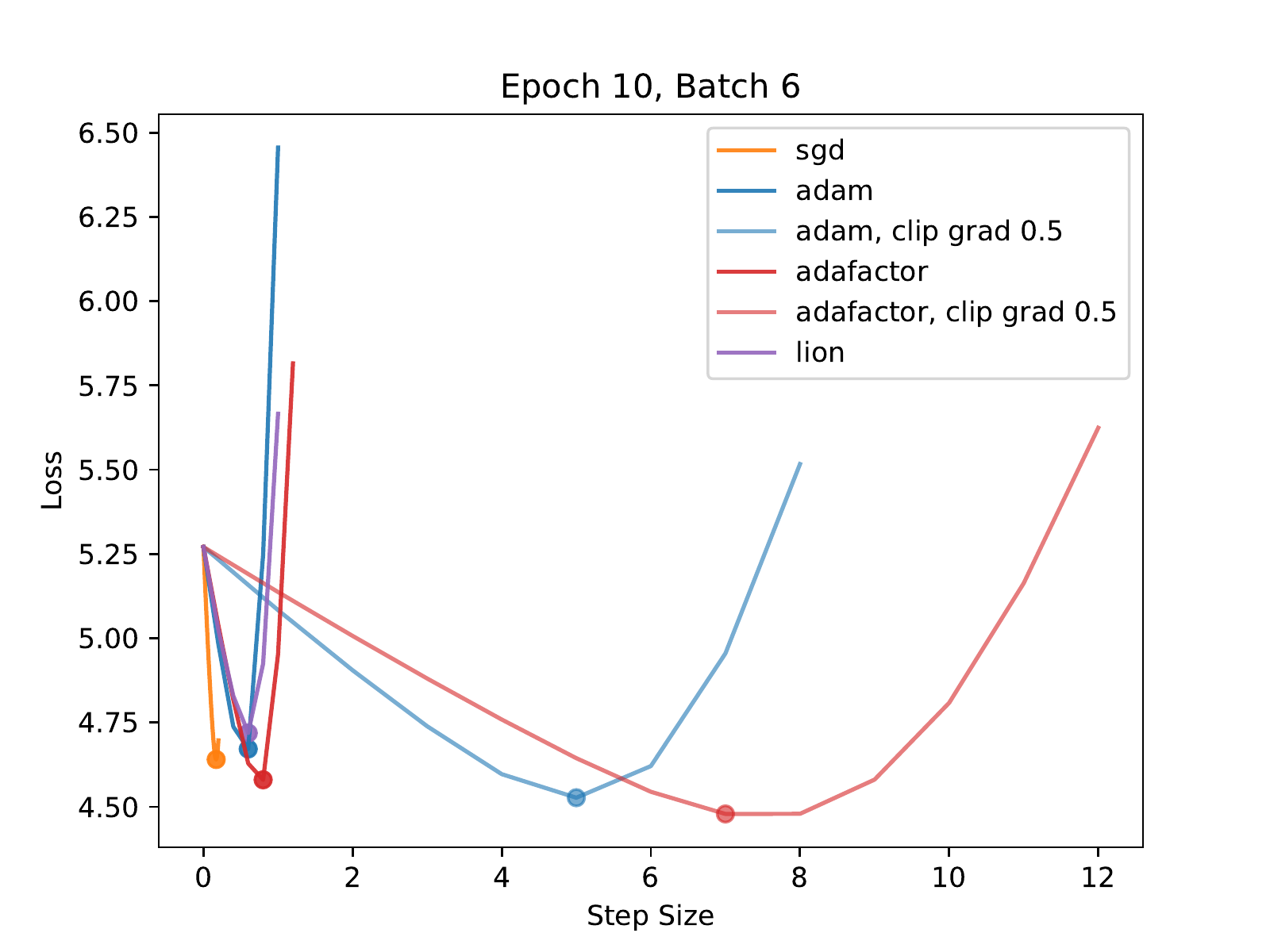}
        \caption{Epoch 10}
    \end{subfigure}
    \begin{subfigure}{0.45\textwidth}
        \centering
        \includegraphics[width=\textwidth]{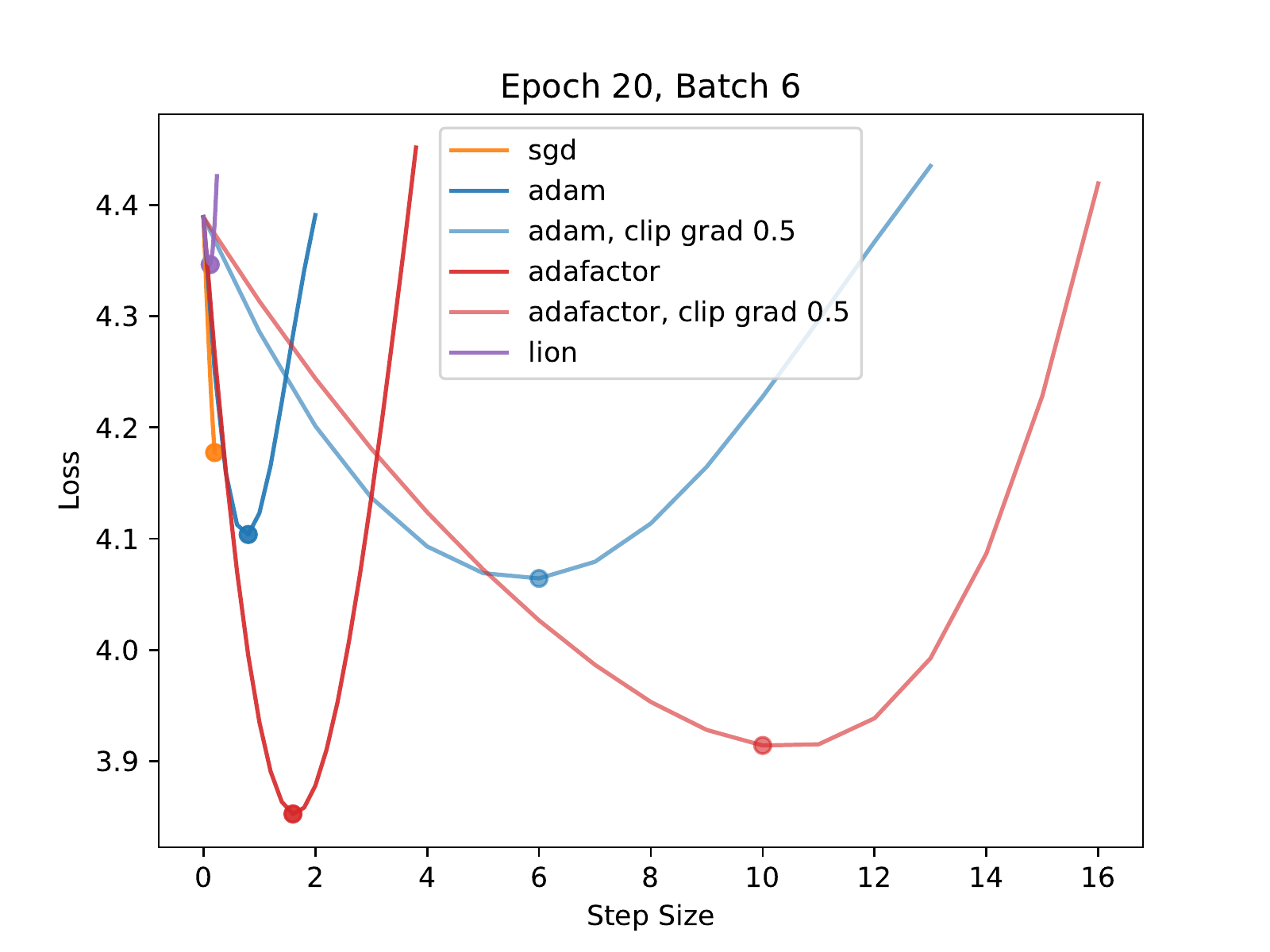}
        \caption{Epoch 20}
    \end{subfigure}
    \caption{Landscape visualization of autoregressive in SGD trajectory.}
\end{figure}

\clearpage
\begin{table}[h]
    \centering
    \begin{tabular}[h]{|l|l|l|}
        \hline
        \textbf{Epoch} & \textbf{Algorithm} & \textbf{Ratio}\\\hline
        \multirow{6}{*}{2} & sgd & $1.0$ \\\cline{2-3}
        & adam & $-0.00031$ \\\cline{2-3}
        & adam, clip grad 0.5 & $2.07\times 10^{-7}$ \\\cline{2-3}
        & adafactor & $0.015206$ \\\cline{2-3}
        & adafactor, clip grad 0.5 & $-8.57\times 10^{-7}$ \\\cline{2-3}
        & lion & $-0.013808$ \\\cline{1-3}
        \multirow{6}{*}{5} & sgd & $1.0$ \\\cline{2-3}
        & adam & $7.3\times 10^{-5}$ \\\cline{2-3}
        & adam, clip grad 0.5 & $5.32\times 10^{-7}$ \\\cline{2-3}
        & adafactor & $0.00028$ \\\cline{2-3}
        & adafactor, clip grad 0.5 & $-5.59\times 10^{-8}$ \\\cline{2-3}
        & lion & $0.002276$ \\\cline{1-3}
        \multirow{6}{*}{10} & sgd & $1.0$ \\\cline{2-3}
        & adam & $0.000661$ \\\cline{2-3}
        & adam, clip grad 0.5 & $4.15\times 10^{-8}$ \\\cline{2-3}
        & adafactor & $0.023778$ \\\cline{2-3}
        & adafactor, clip grad 0.5 & $6.84\times 10^{-7}$ \\\cline{2-3}
        & lion & $0.005283$ \\\cline{1-3}
        \multirow{6}{*}{20} & sgd & $1.0$ \\\cline{2-3}
        & adam & $0.000605$ \\\cline{2-3}
        & adam, clip grad 0.5 & $2.79\times 10^{-7}$ \\\cline{2-3}
        & adafactor & $-0.001095$ \\\cline{2-3}
        & adafactor, clip grad 0.5 & $-2.82\times 10^{-6}$ \\\cline{2-3}
        & lion & $0.03734$ \\\cline{1-3}
    \end{tabular}
    \caption{Ratio of directional sharpness of optimization algorithms with respect to SGD on the autoregressive task in SGD trajectory.}
\end{table}

\clearpage
\subsection{Adam Trajectory\label{sec:adam_geometry}}

\begin{figure}[h]
    \centering
    \begin{subfigure}{\textwidth}\centering
        \includegraphics[width=0.32\textwidth]{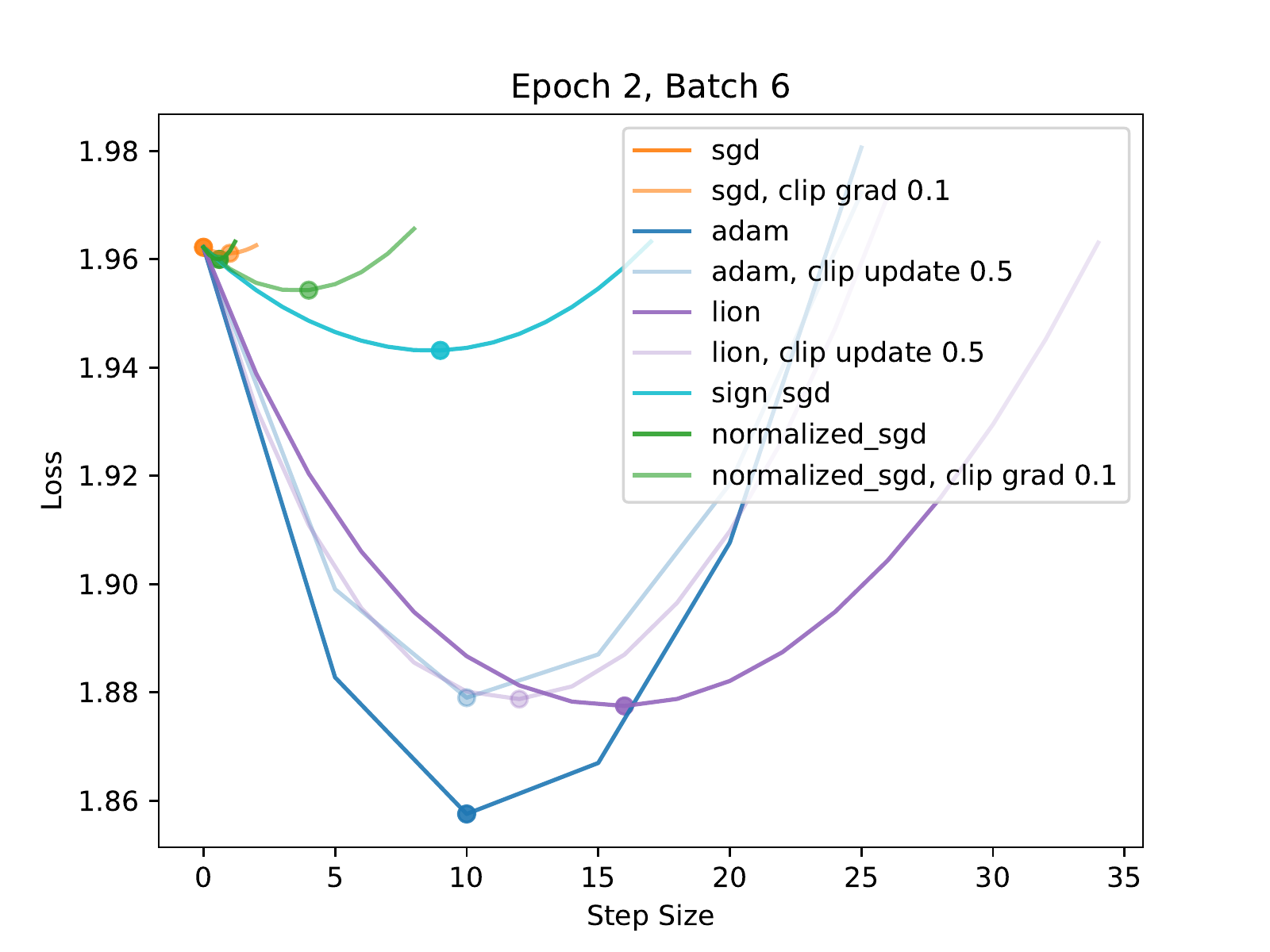}
        \includegraphics[width=0.32\textwidth]{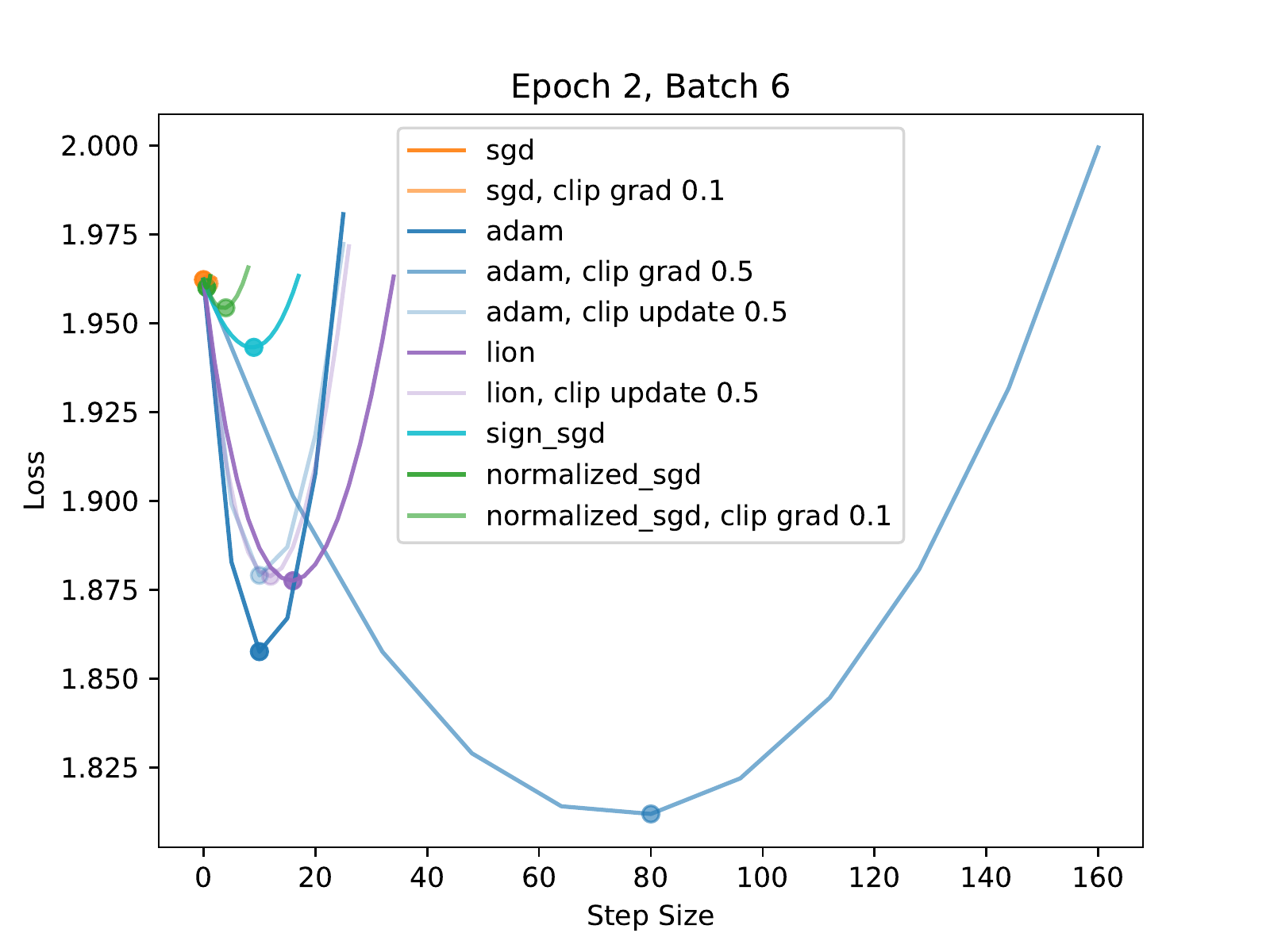}
        \includegraphics[width=0.32\textwidth]{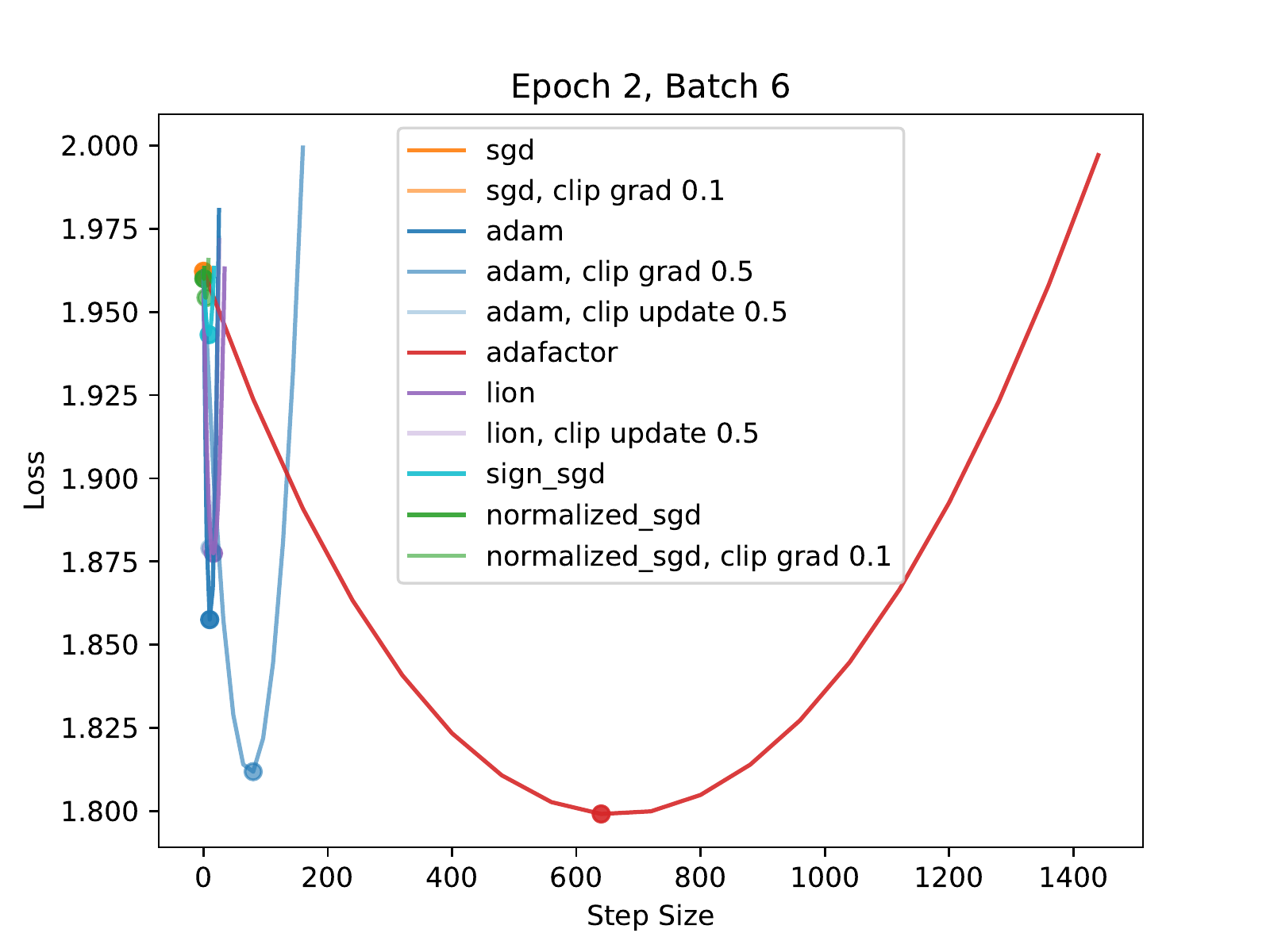}
        \caption{Experiment 1}
    \end{subfigure}
    \begin{subfigure}{\textwidth}\centering
        \includegraphics[width=0.32\textwidth]{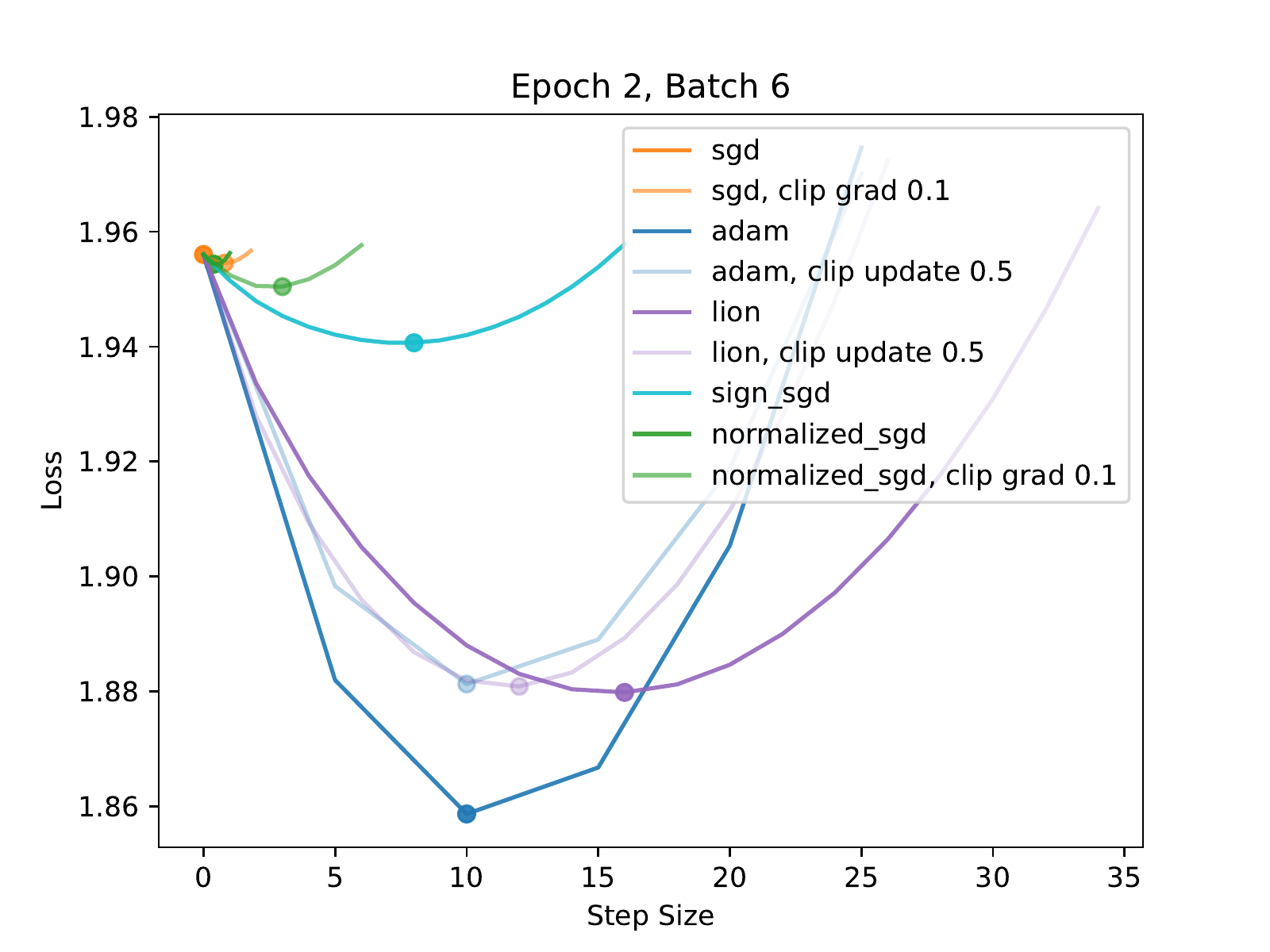}
        \includegraphics[width=0.32\textwidth]{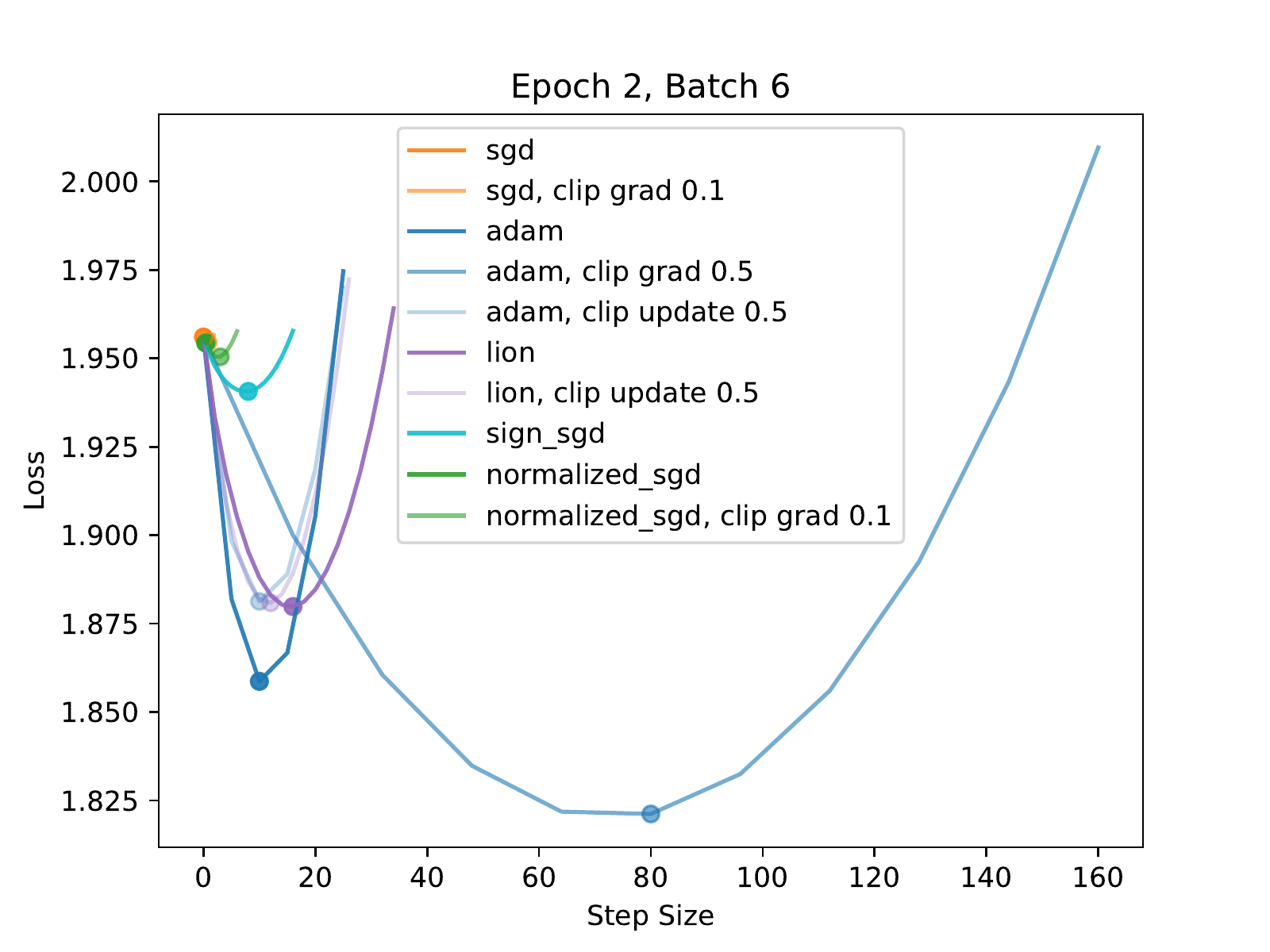}
        \includegraphics[width=0.32\textwidth]{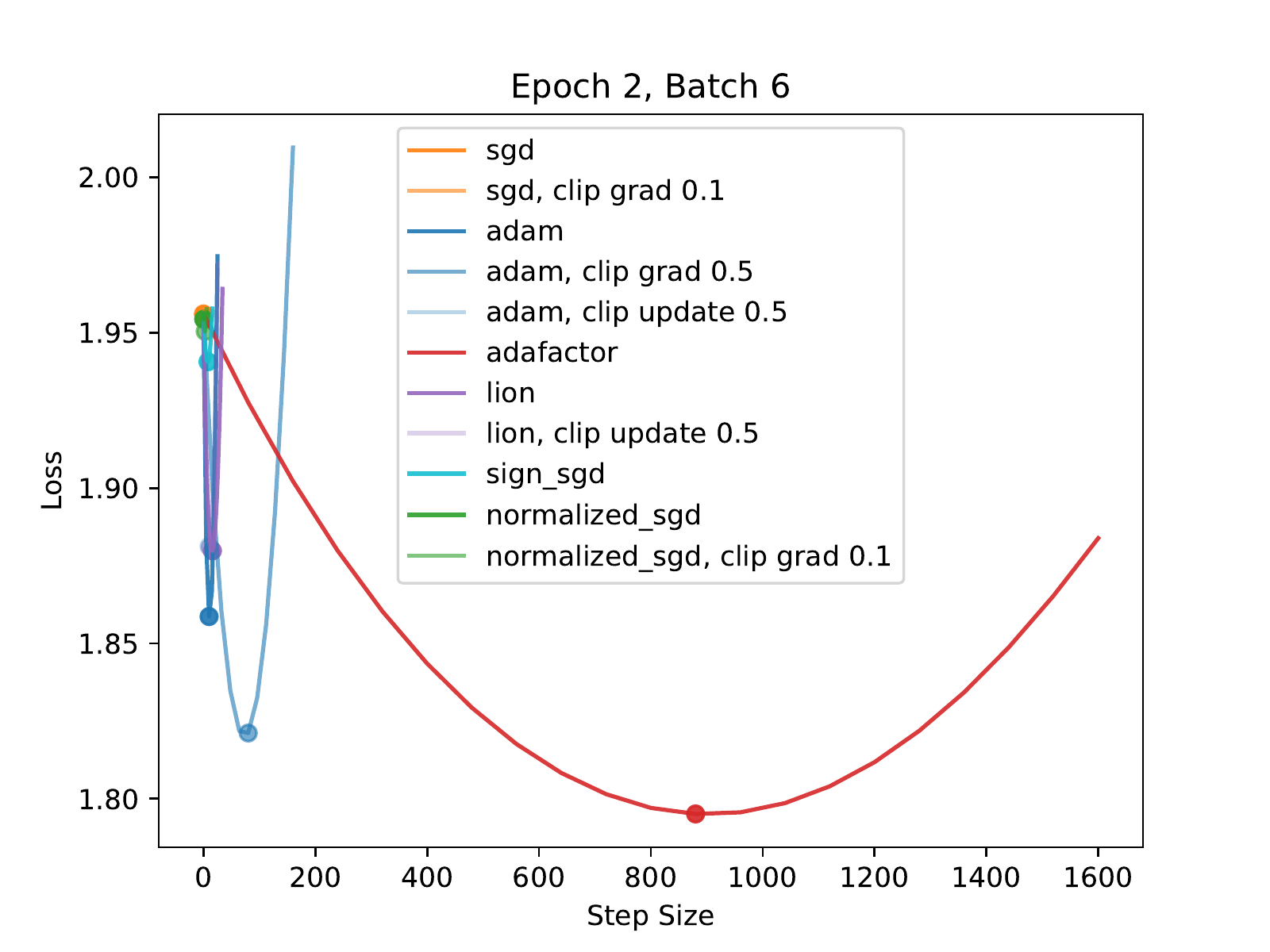}
        \caption{Experiment 2}
    \end{subfigure}
    \begin{subfigure}{\textwidth}\centering
        \includegraphics[width=0.32\textwidth]{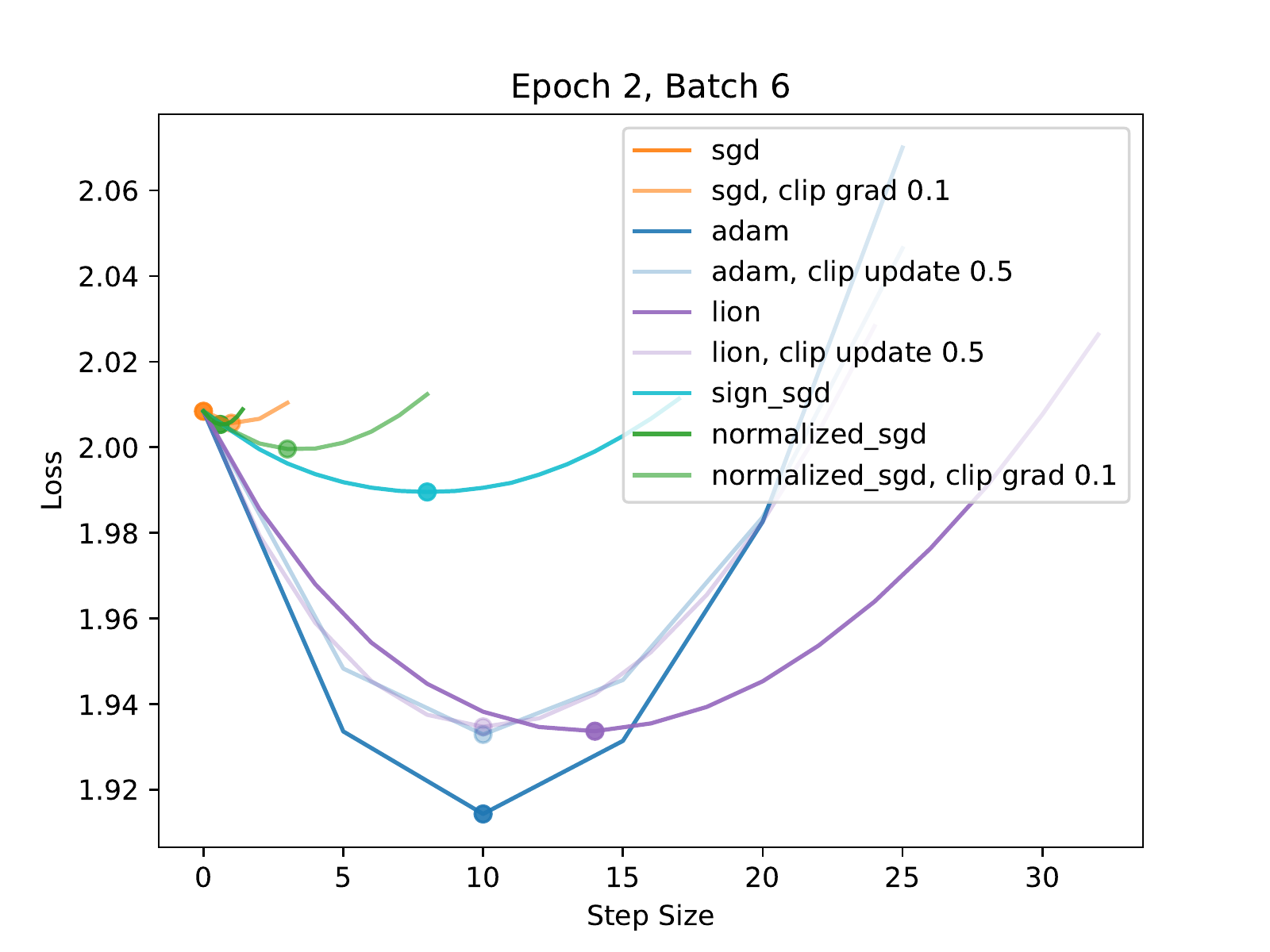}
        \includegraphics[width=0.32\textwidth]{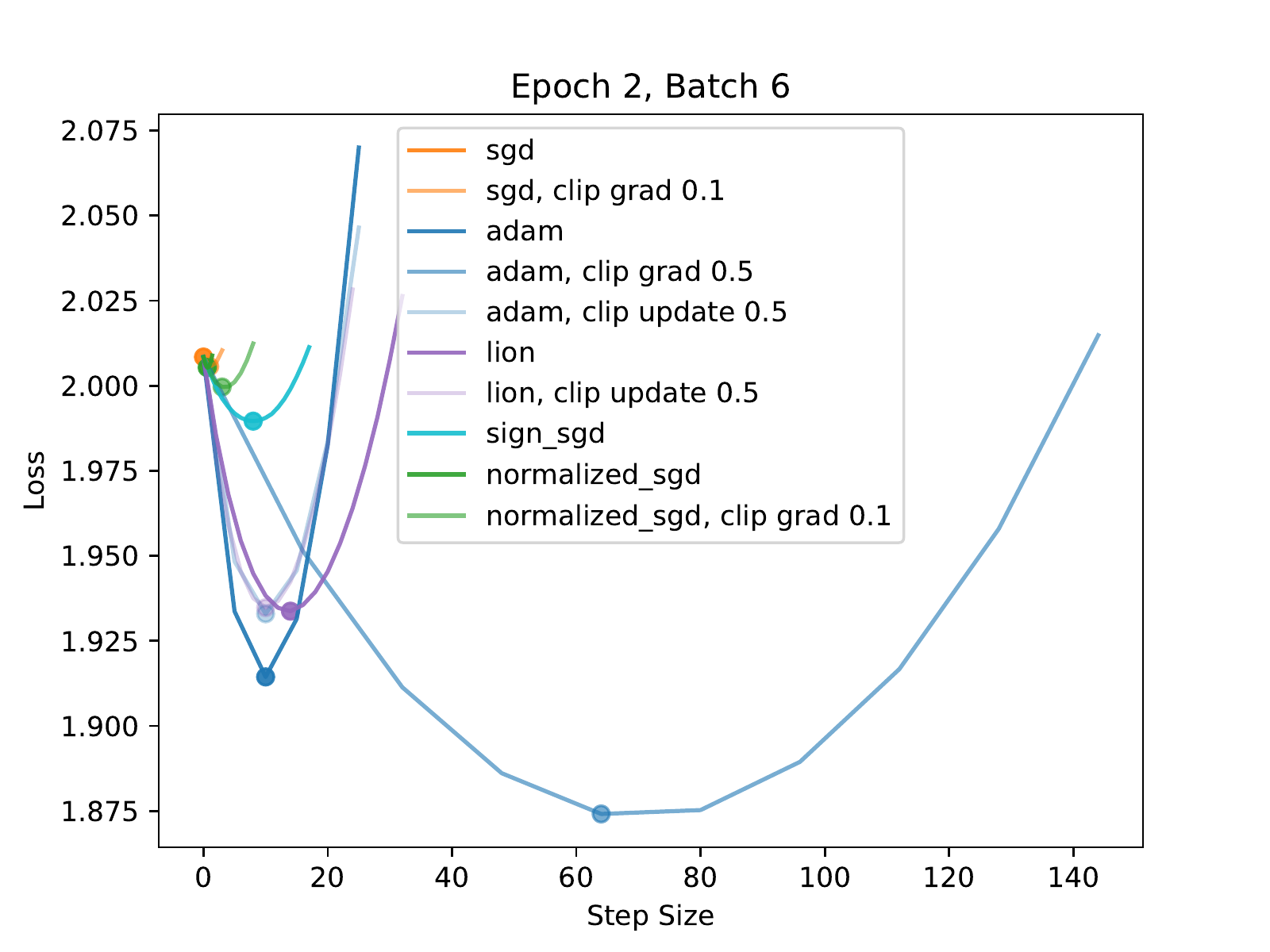}
        \includegraphics[width=0.32\textwidth]{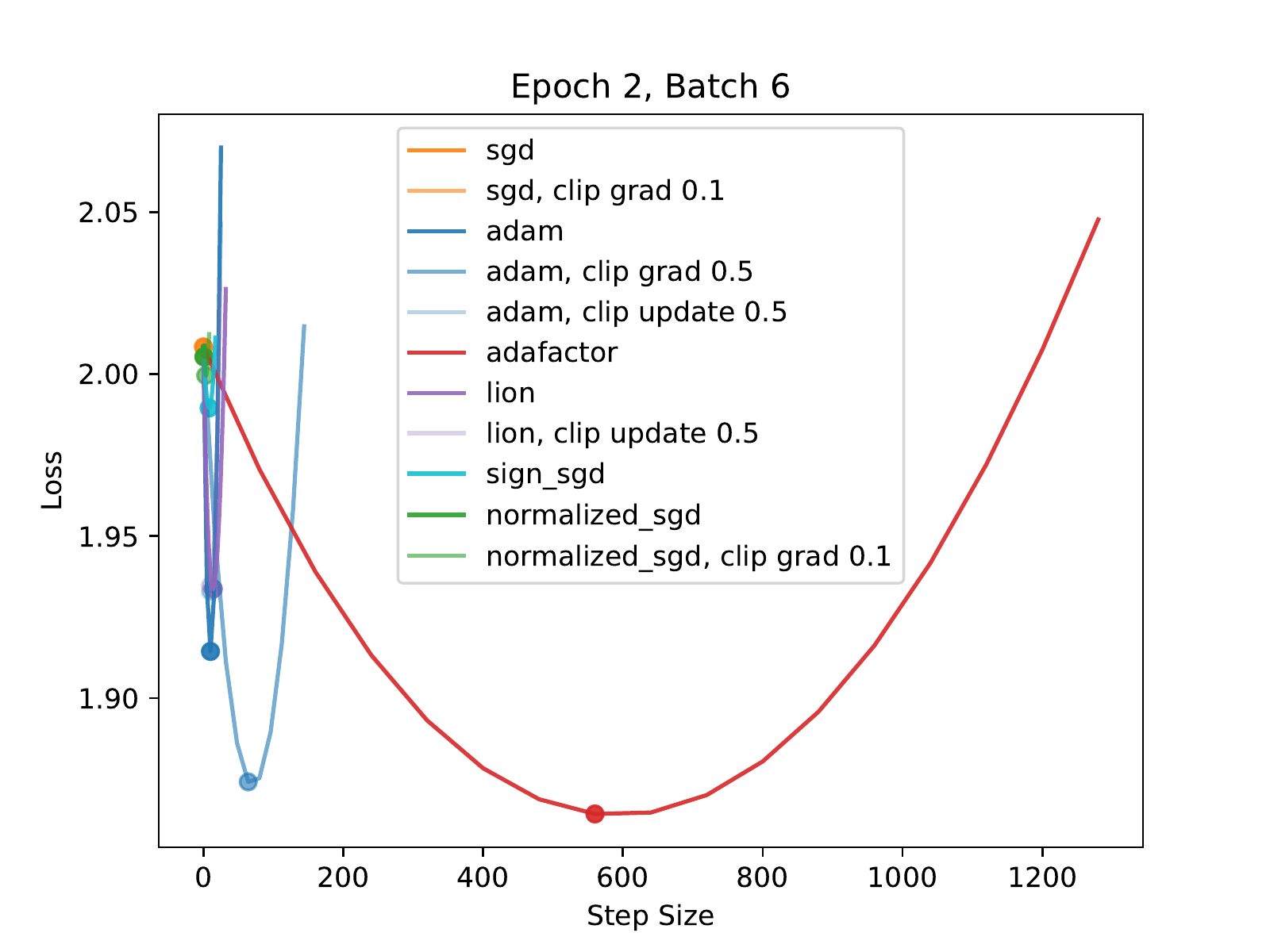}
        \caption{Experiment 3}
    \end{subfigure}
    \caption{Landscape visualization of machine translation in Adam trajectory at Epoch 2.}
\end{figure}

\begin{figure}[h]
    \centering
    \begin{subfigure}{\textwidth}\centering
        \includegraphics[width=0.32\textwidth]{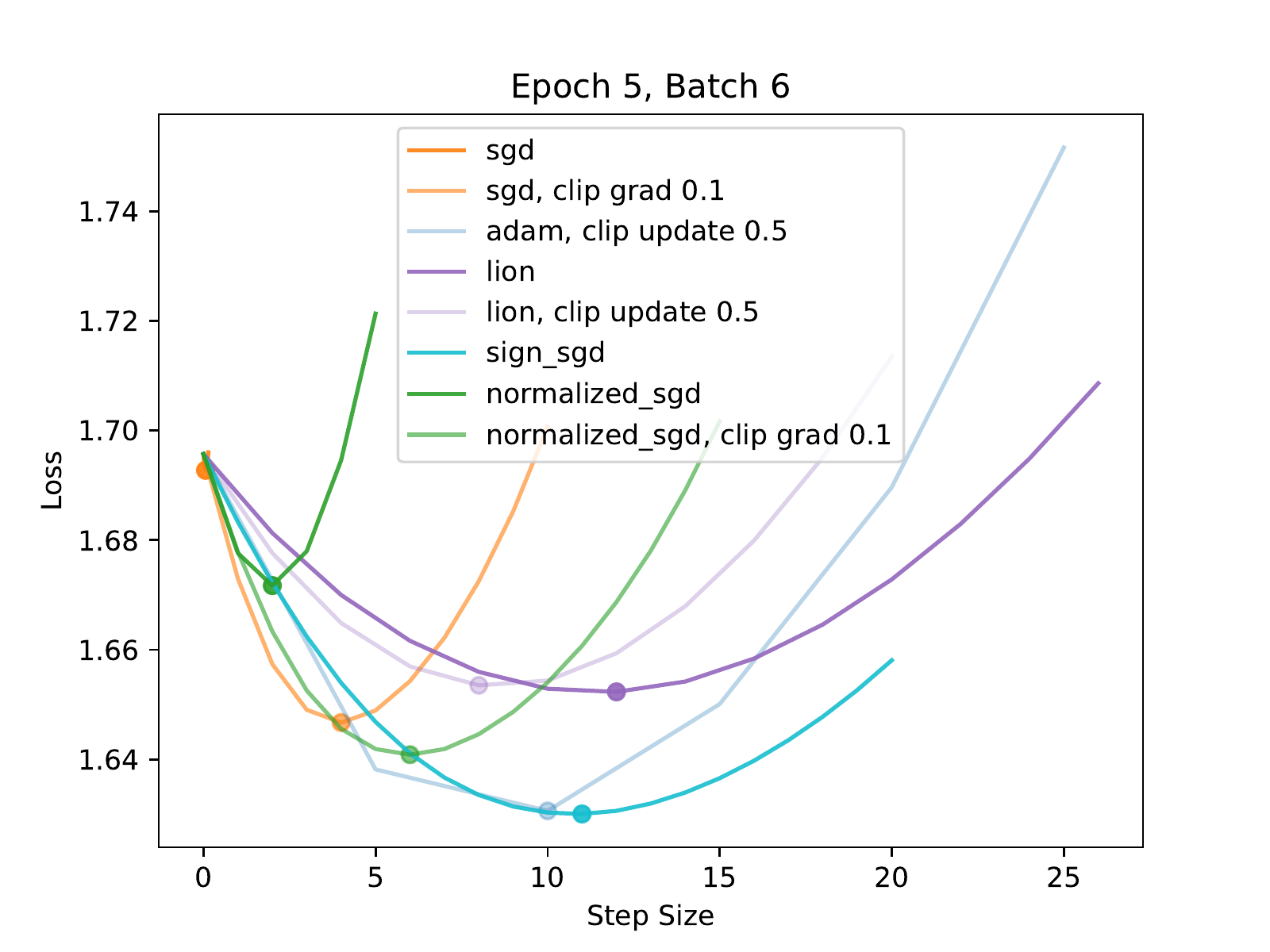}
        \includegraphics[width=0.32\textwidth]{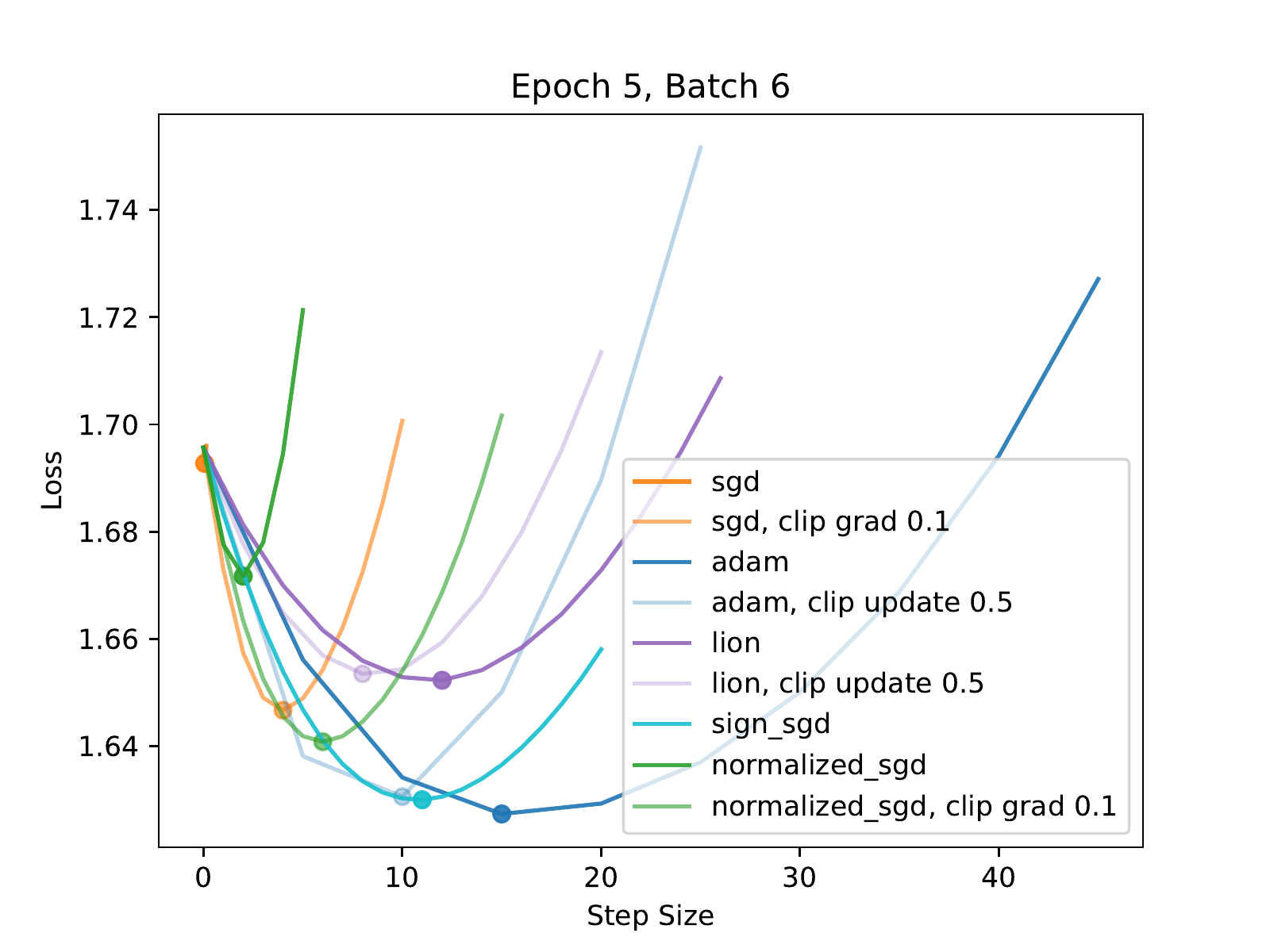}
        \includegraphics[width=0.32\textwidth]{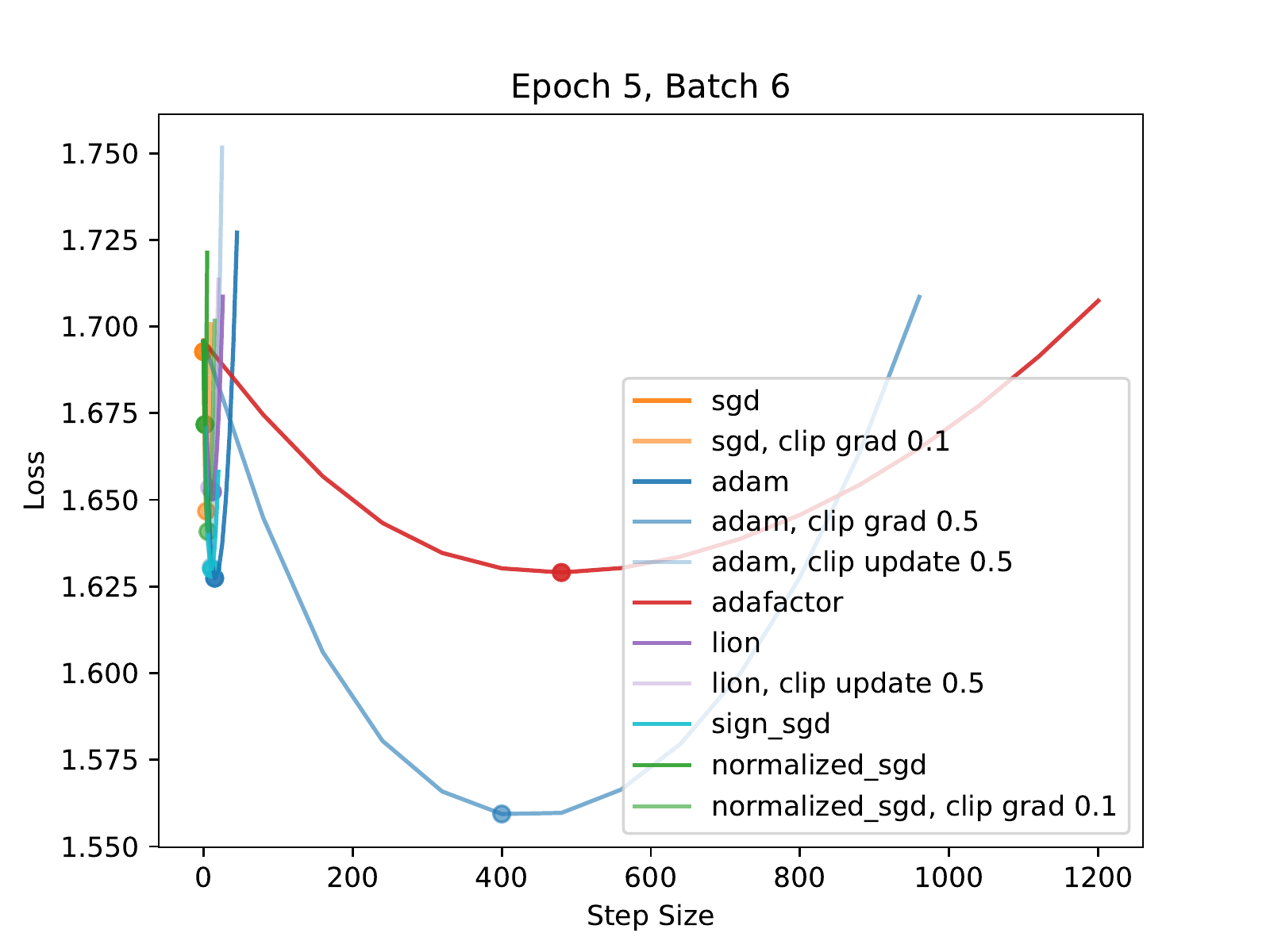}
        \caption{Experiment 1}
    \end{subfigure}
    \begin{subfigure}{\textwidth}\centering
        \includegraphics[width=0.32\textwidth]{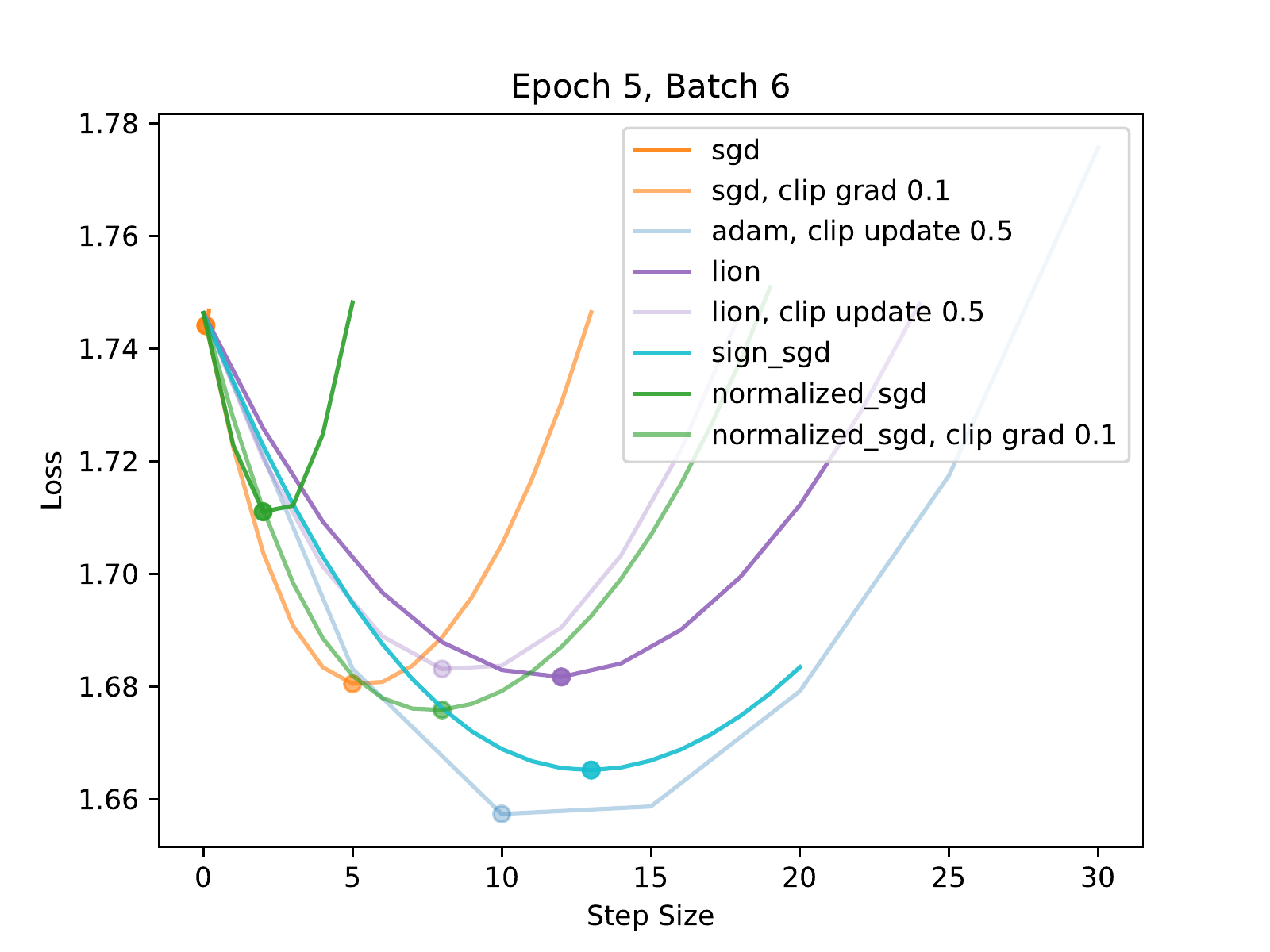}
        \includegraphics[width=0.32\textwidth]{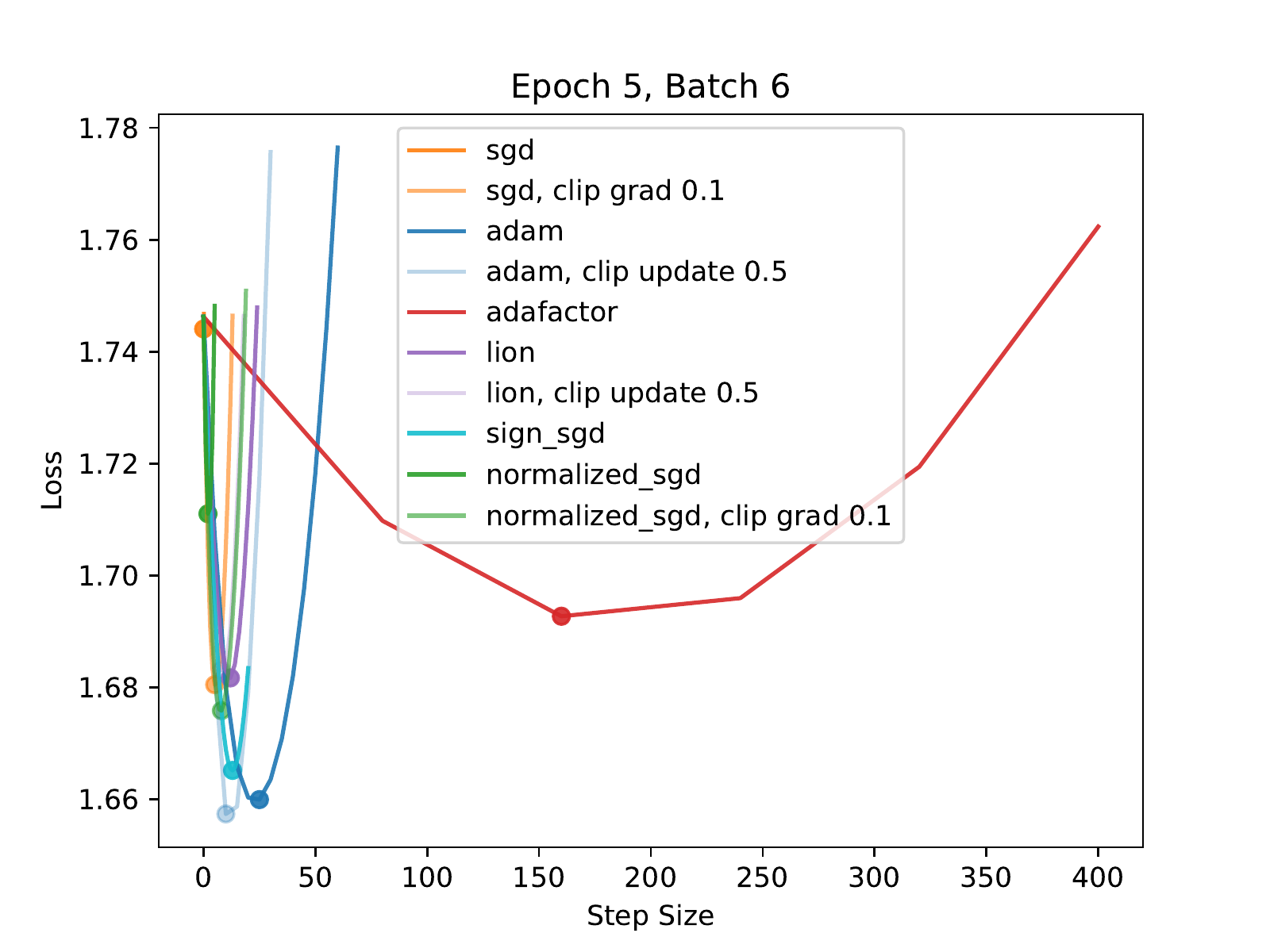}
        \includegraphics[width=0.32\textwidth]{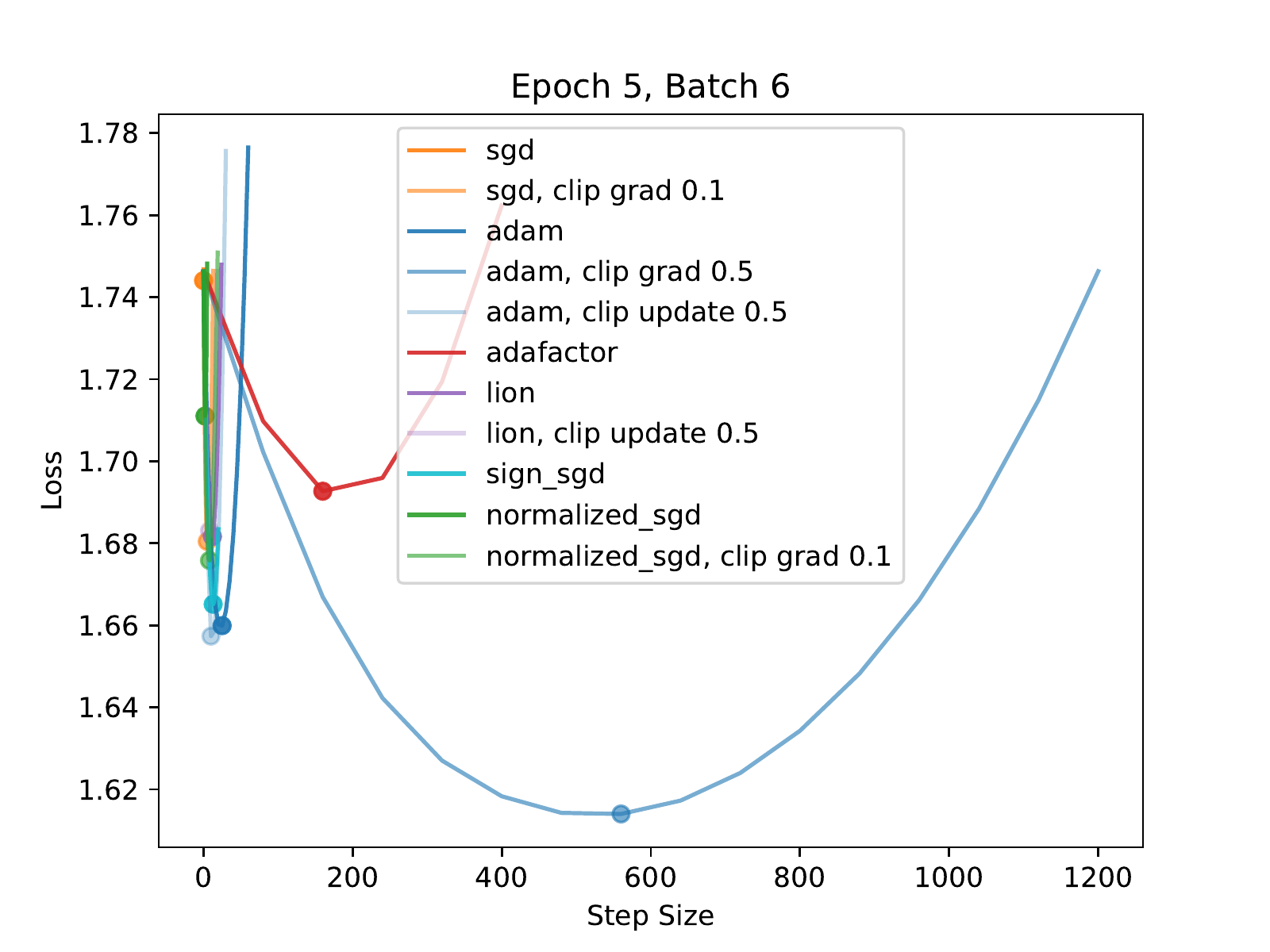}
        \caption{Experiment 2}
    \end{subfigure}
    \begin{subfigure}{\textwidth}\centering
        \includegraphics[width=0.32\textwidth]{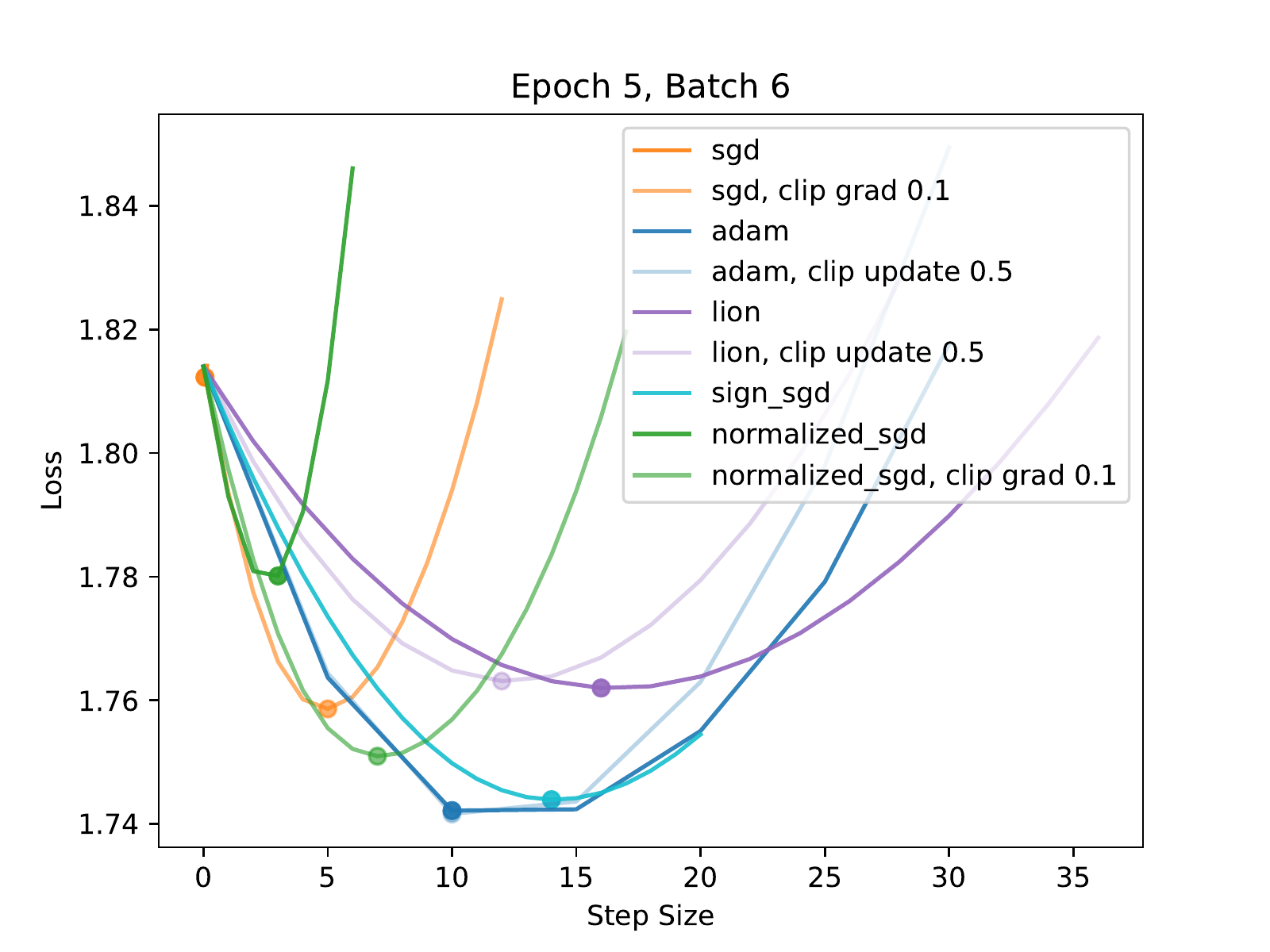}
        \includegraphics[width=0.32\textwidth]{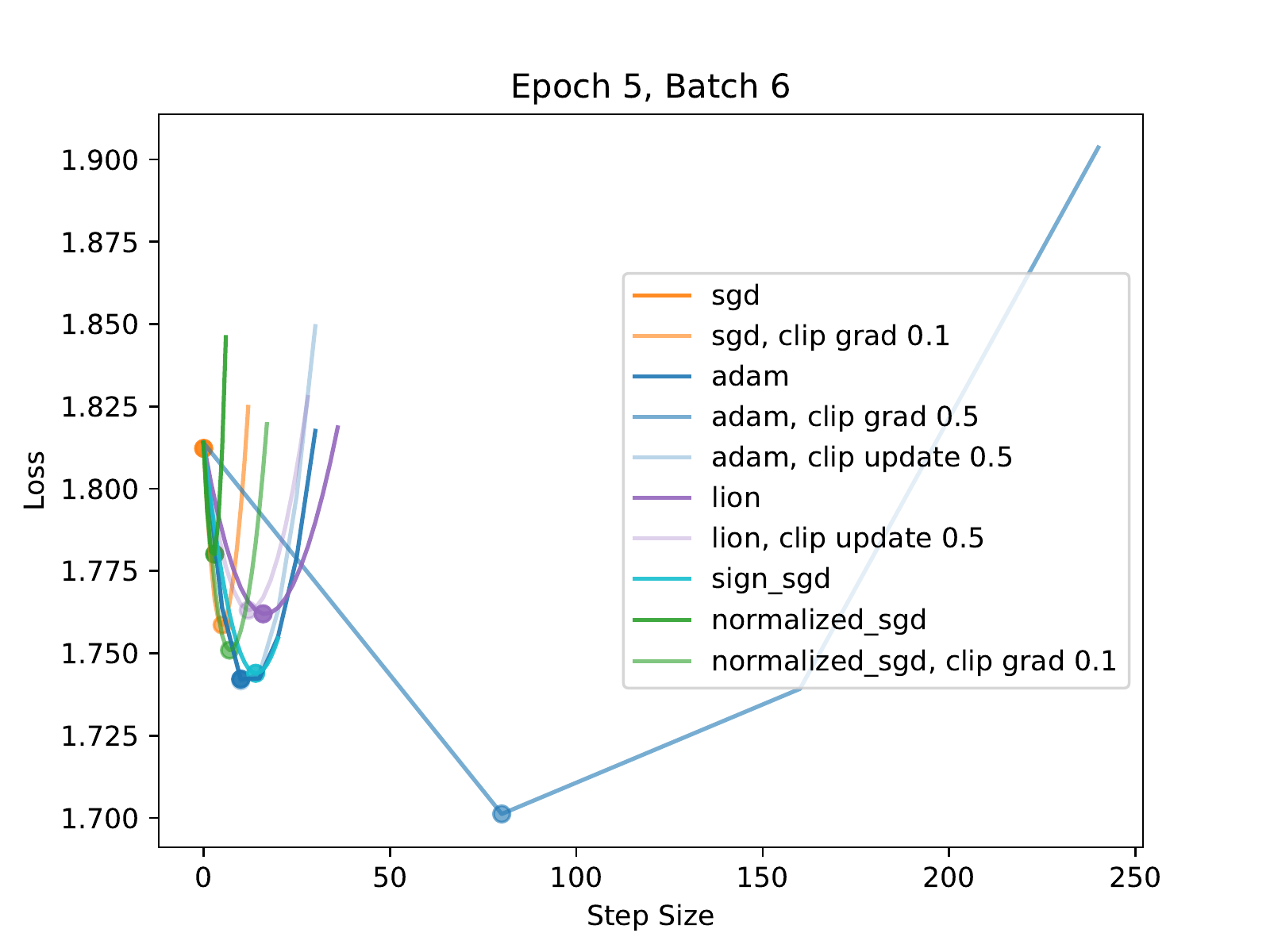}
        \includegraphics[width=0.32\textwidth]{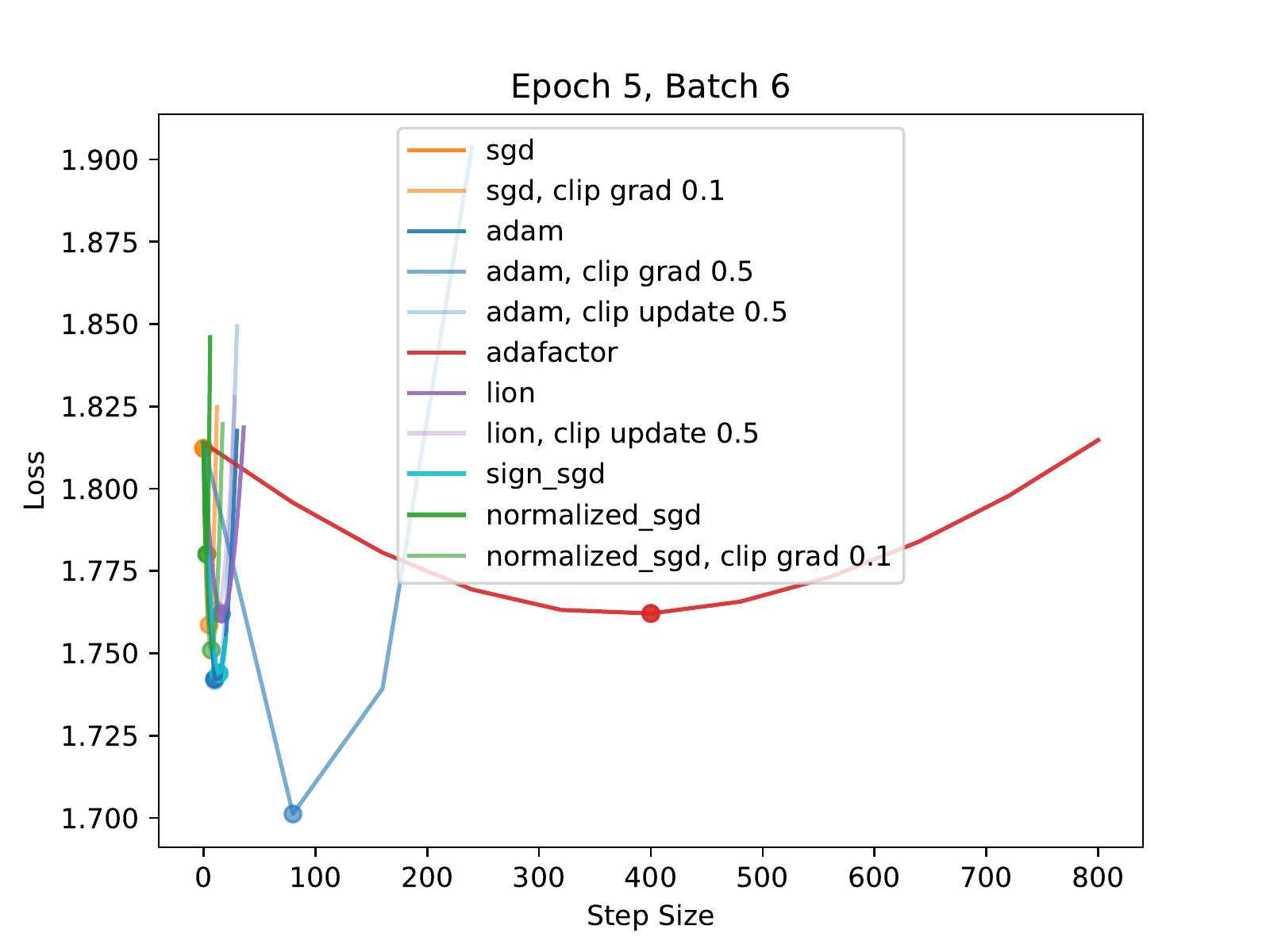}
        \caption{Experiment 3}
    \end{subfigure}
    \caption{Landscape visualization of machine translation in Adam trajectory at Epoch 5.}
\end{figure}

\begin{figure}[h]
    \centering
    \begin{subfigure}{\textwidth}\centering
        \includegraphics[width=0.32\textwidth]{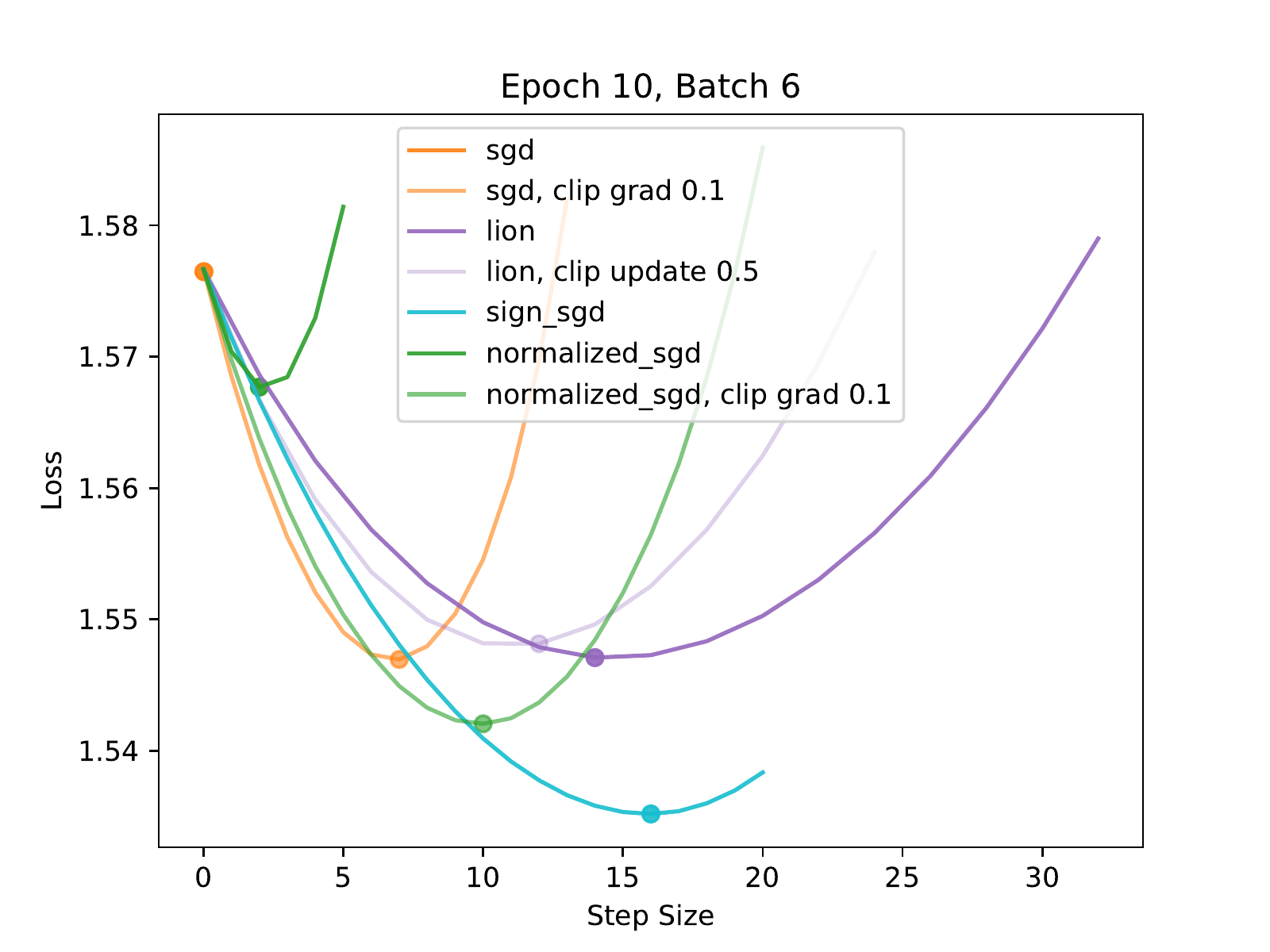}
        \includegraphics[width=0.32\textwidth]{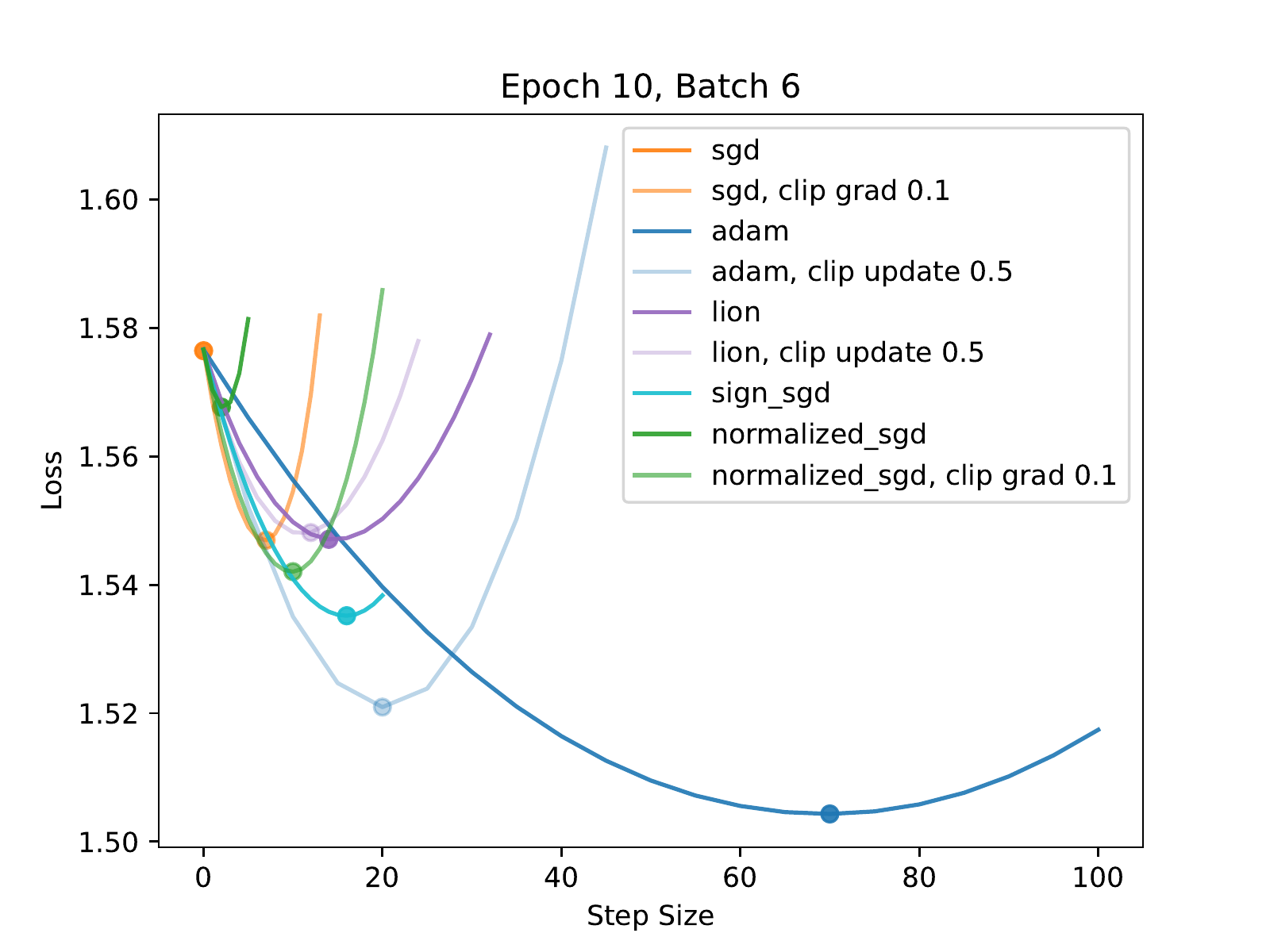}
        \includegraphics[width=0.32\textwidth]{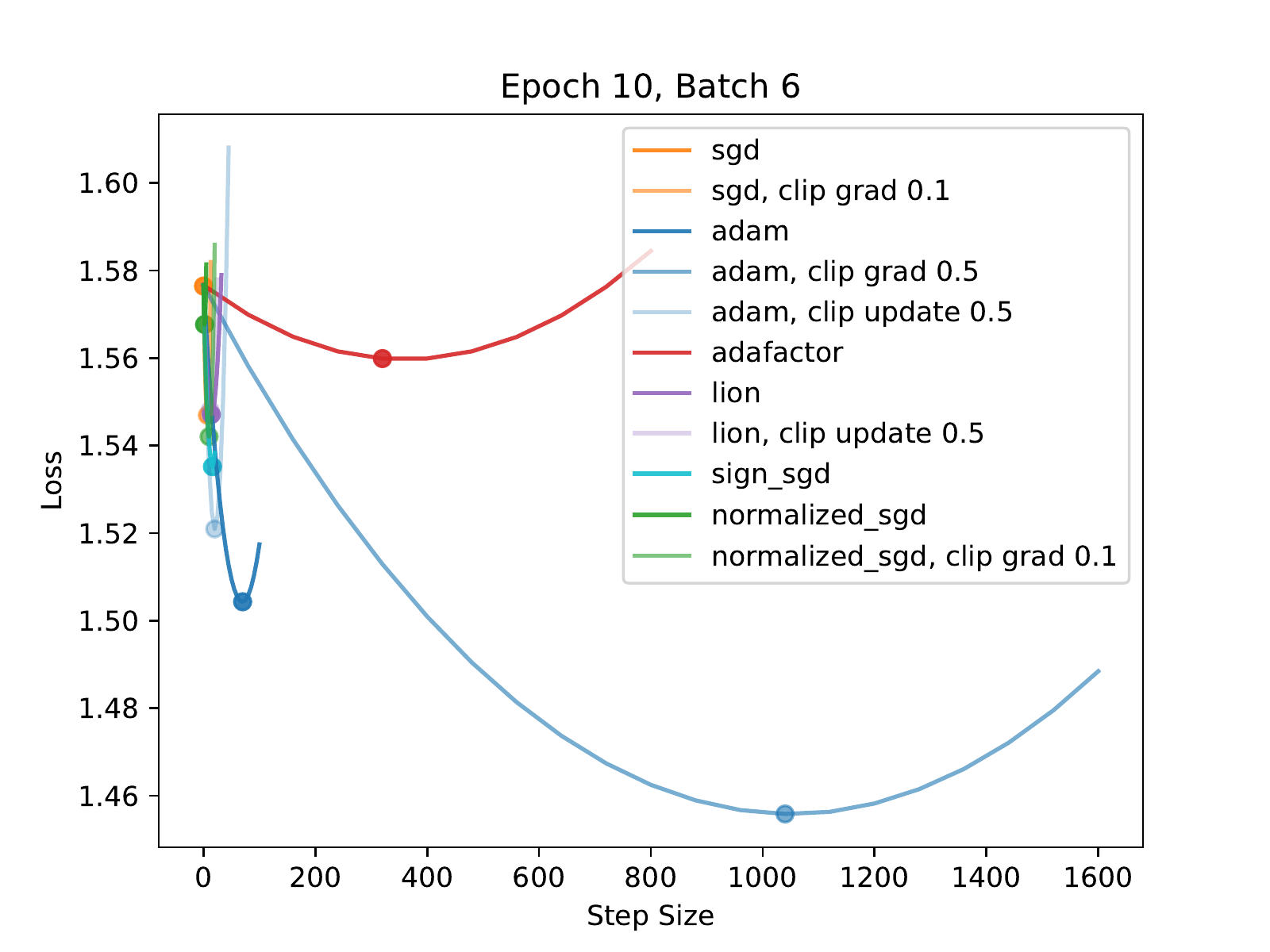}
        \caption{Experiment 1}
    \end{subfigure}
    \begin{subfigure}{\textwidth}\centering
        \includegraphics[width=0.32\textwidth]{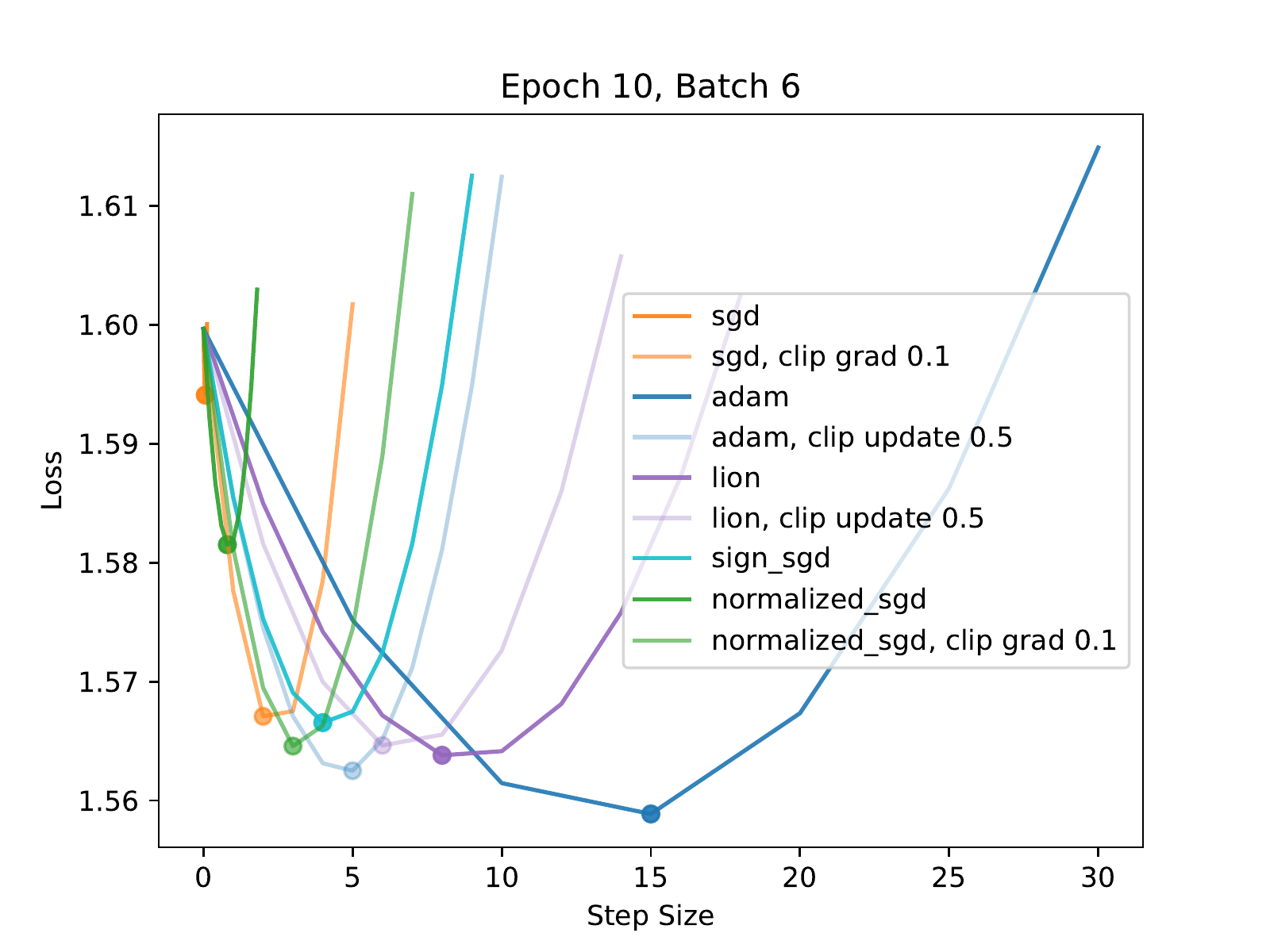}
        \includegraphics[width=0.32\textwidth]{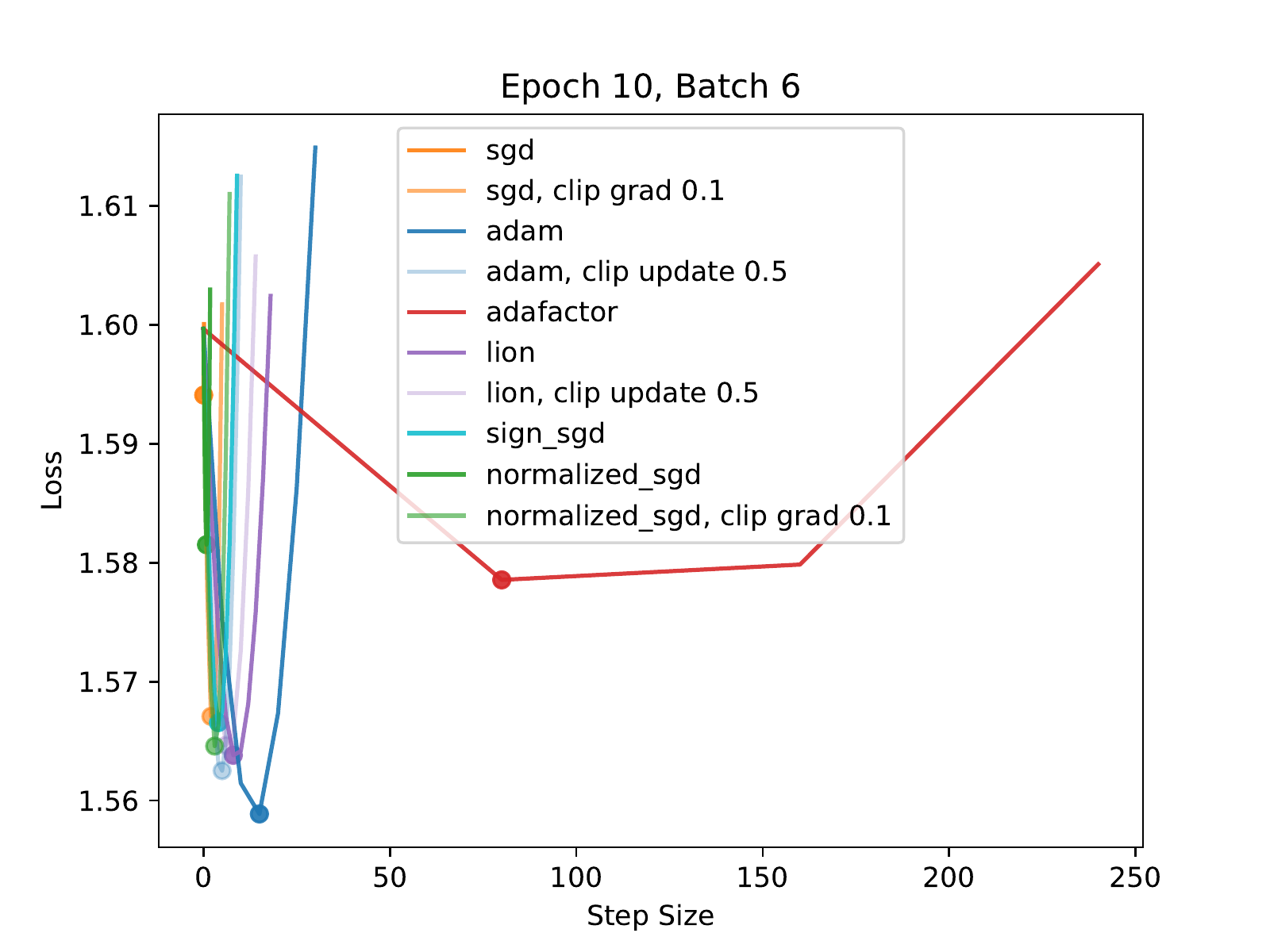}
        \includegraphics[width=0.32\textwidth]{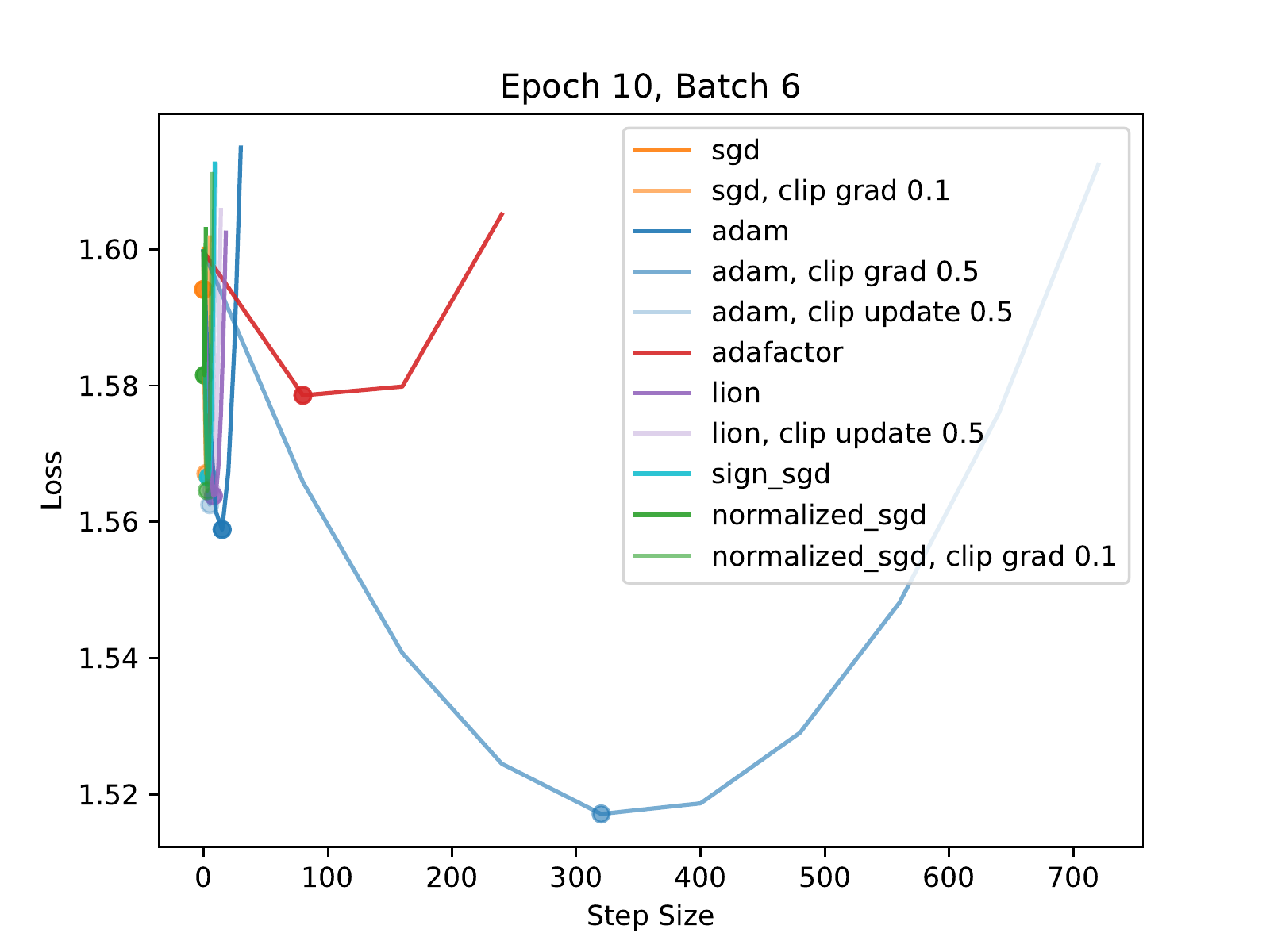}
        \caption{Experiment 2}
    \end{subfigure}
    \begin{subfigure}{\textwidth}\centering
        \includegraphics[width=0.32\textwidth]{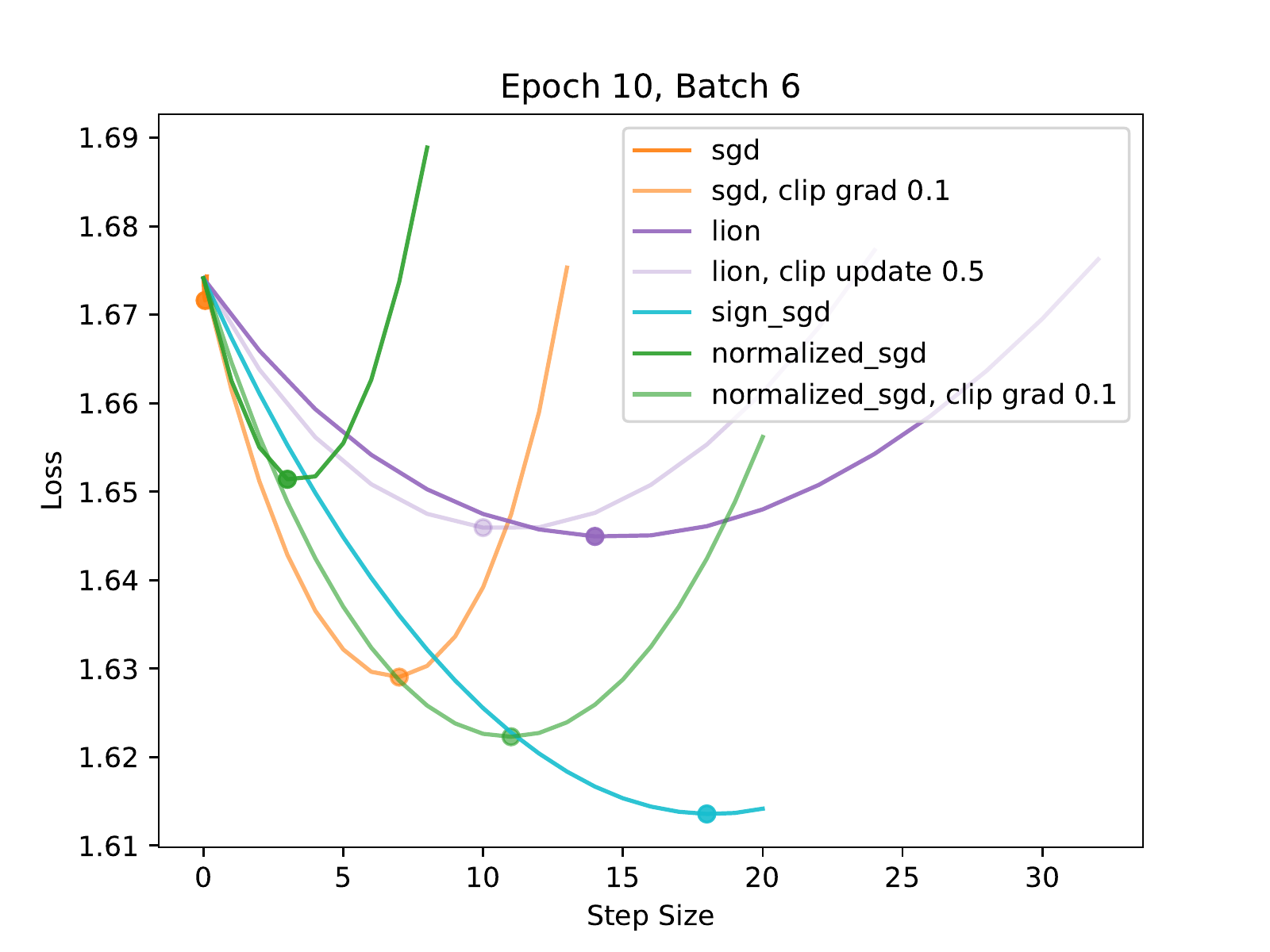}
        \includegraphics[width=0.32\textwidth]{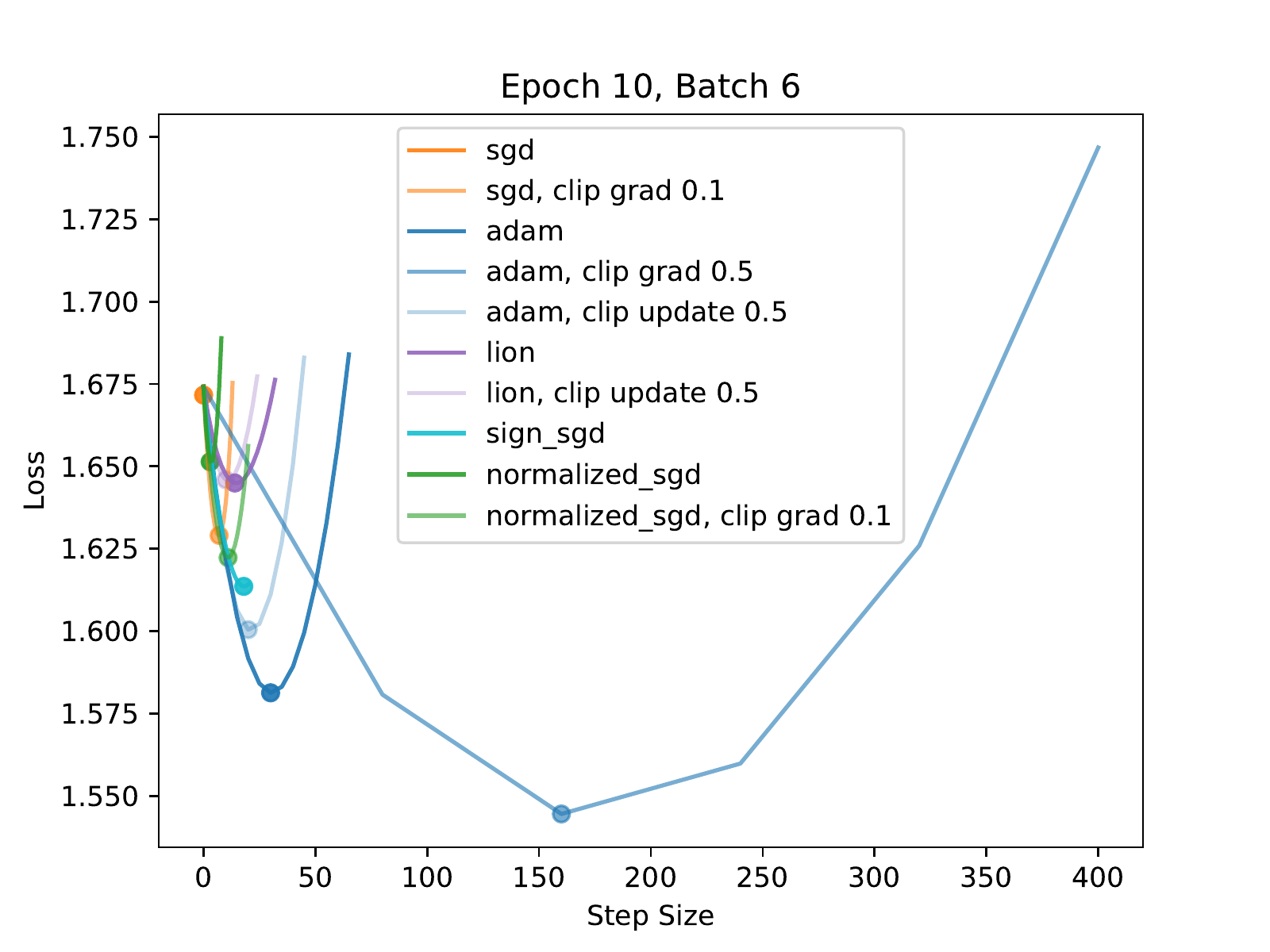}
        \includegraphics[width=0.32\textwidth]{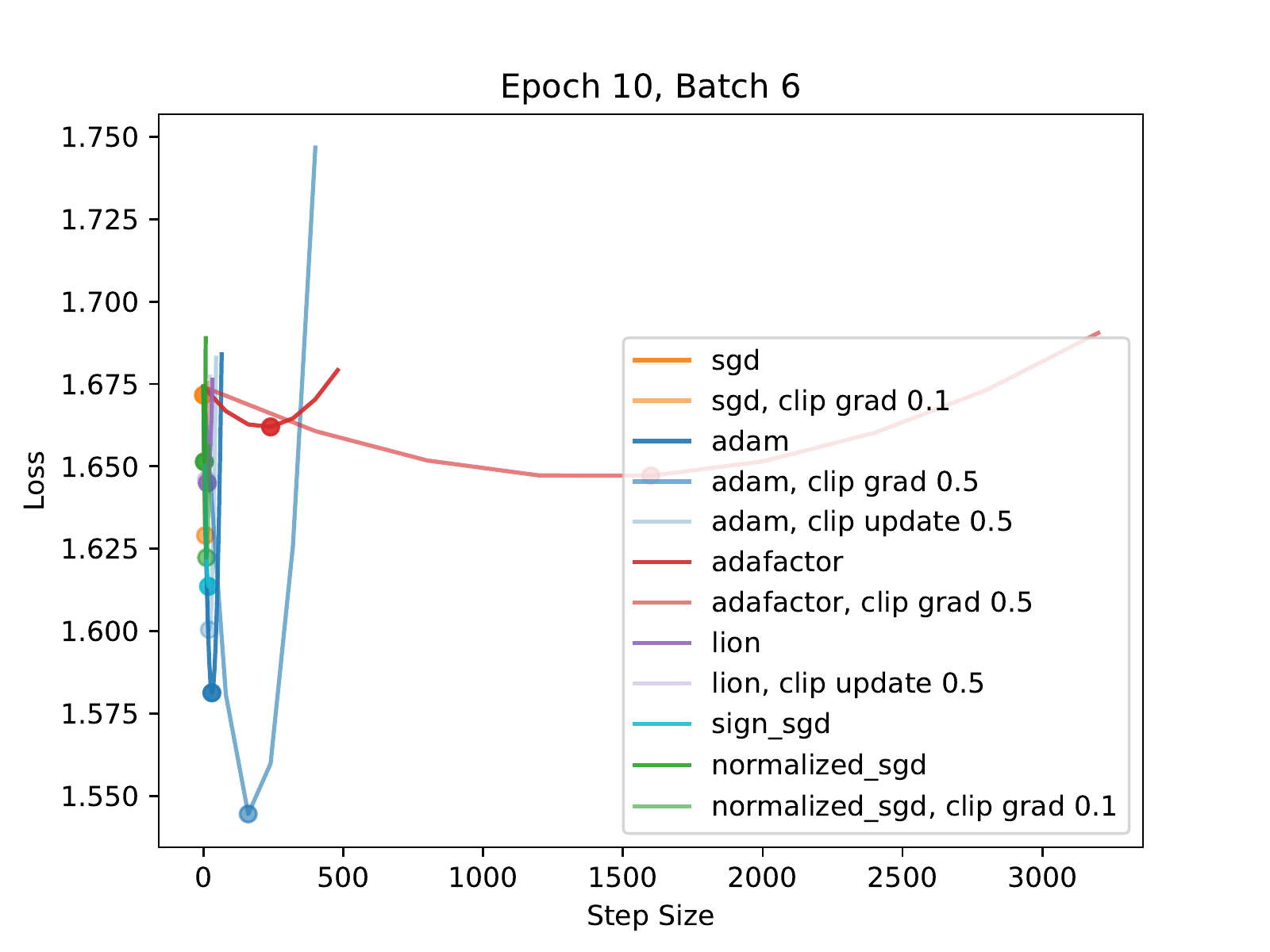}
        \caption{Experiment 3}
    \end{subfigure}
    \caption{Landscape visualization of machine translation in Adam trajectory at Epoch 10.}
\end{figure}

\begin{figure}[h]
    \centering
    \begin{subfigure}{\textwidth}\centering
        \includegraphics[width=0.32\textwidth]{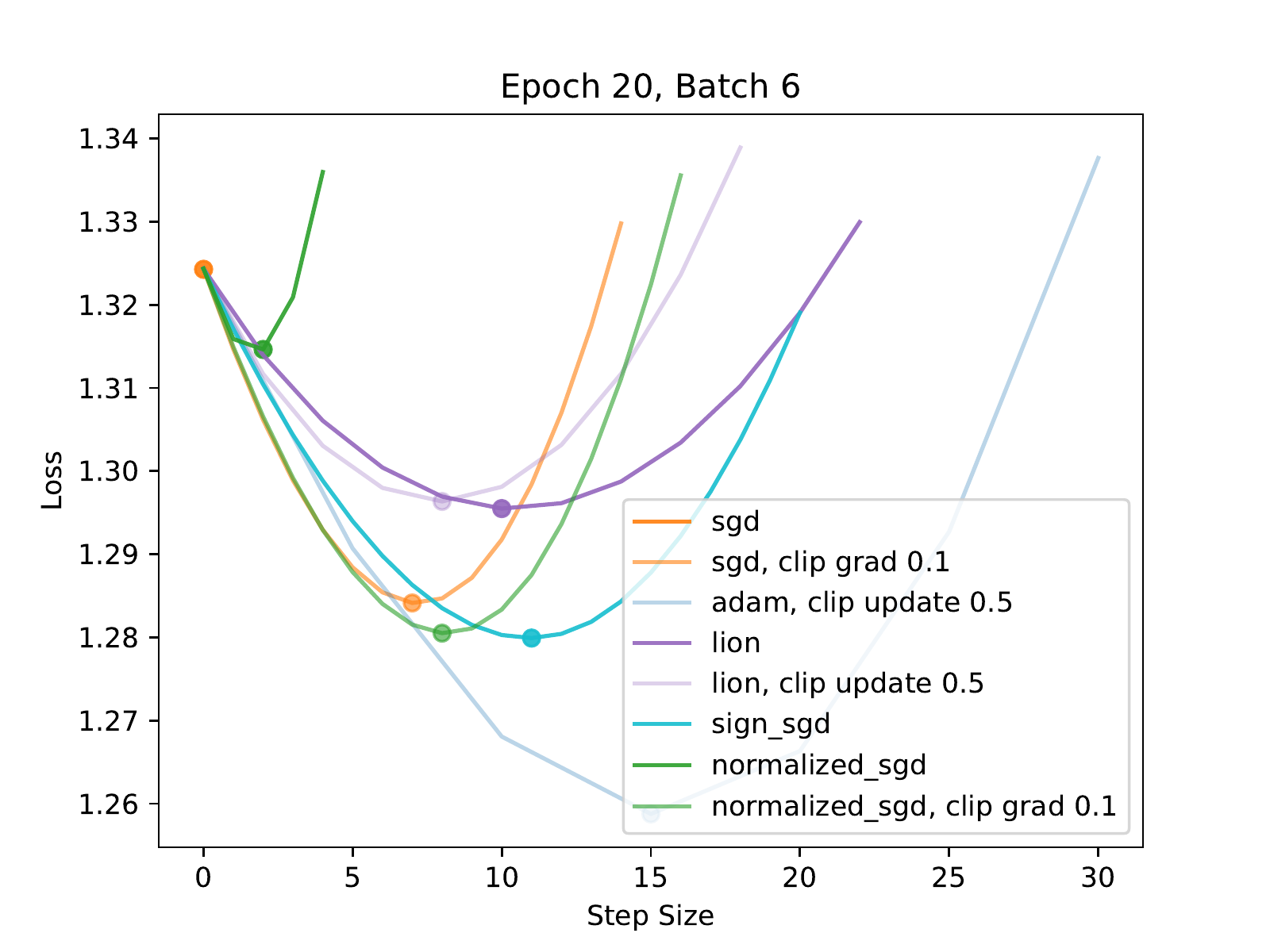}
        \includegraphics[width=0.32\textwidth]{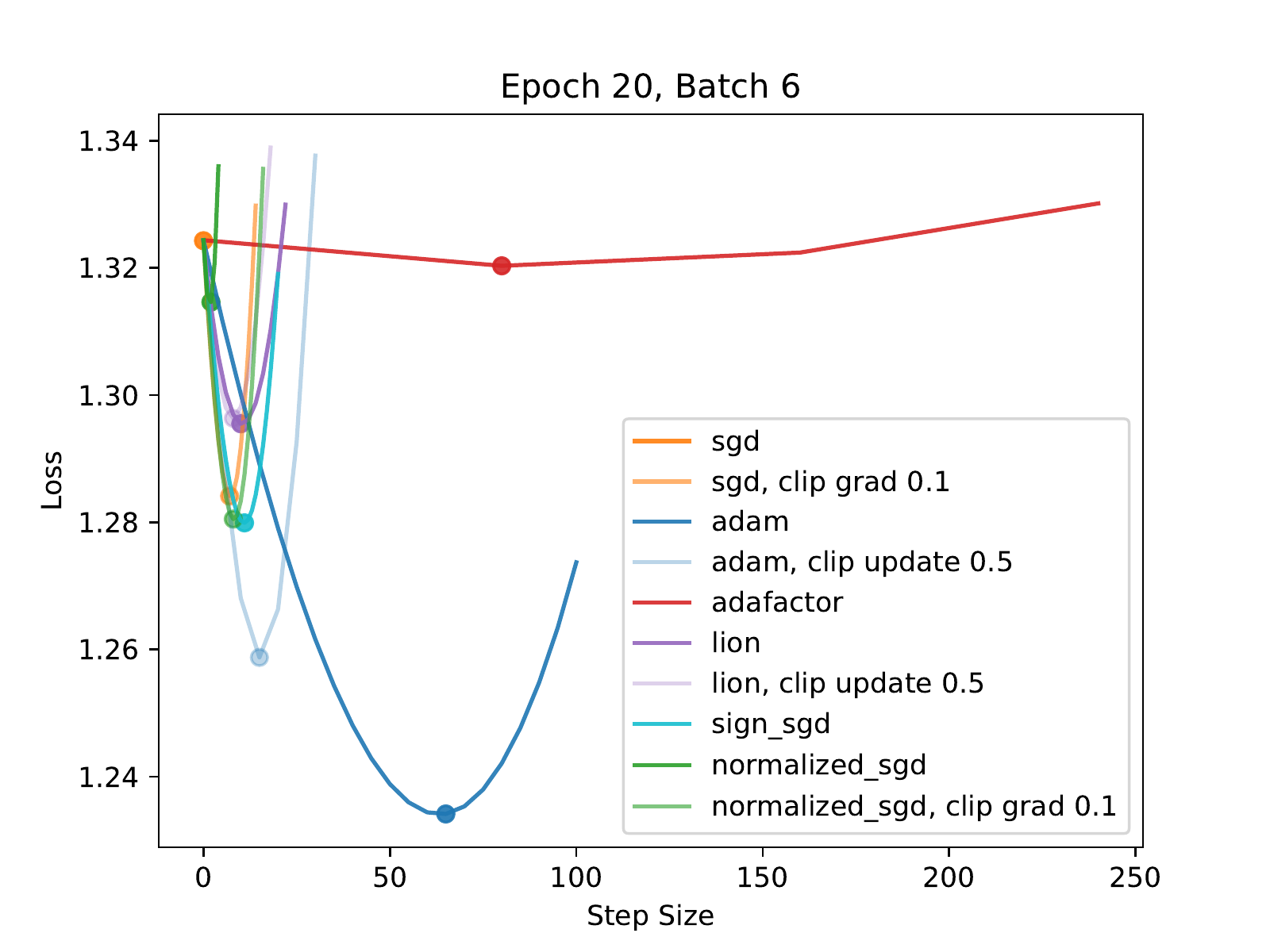}
        \includegraphics[width=0.32\textwidth]{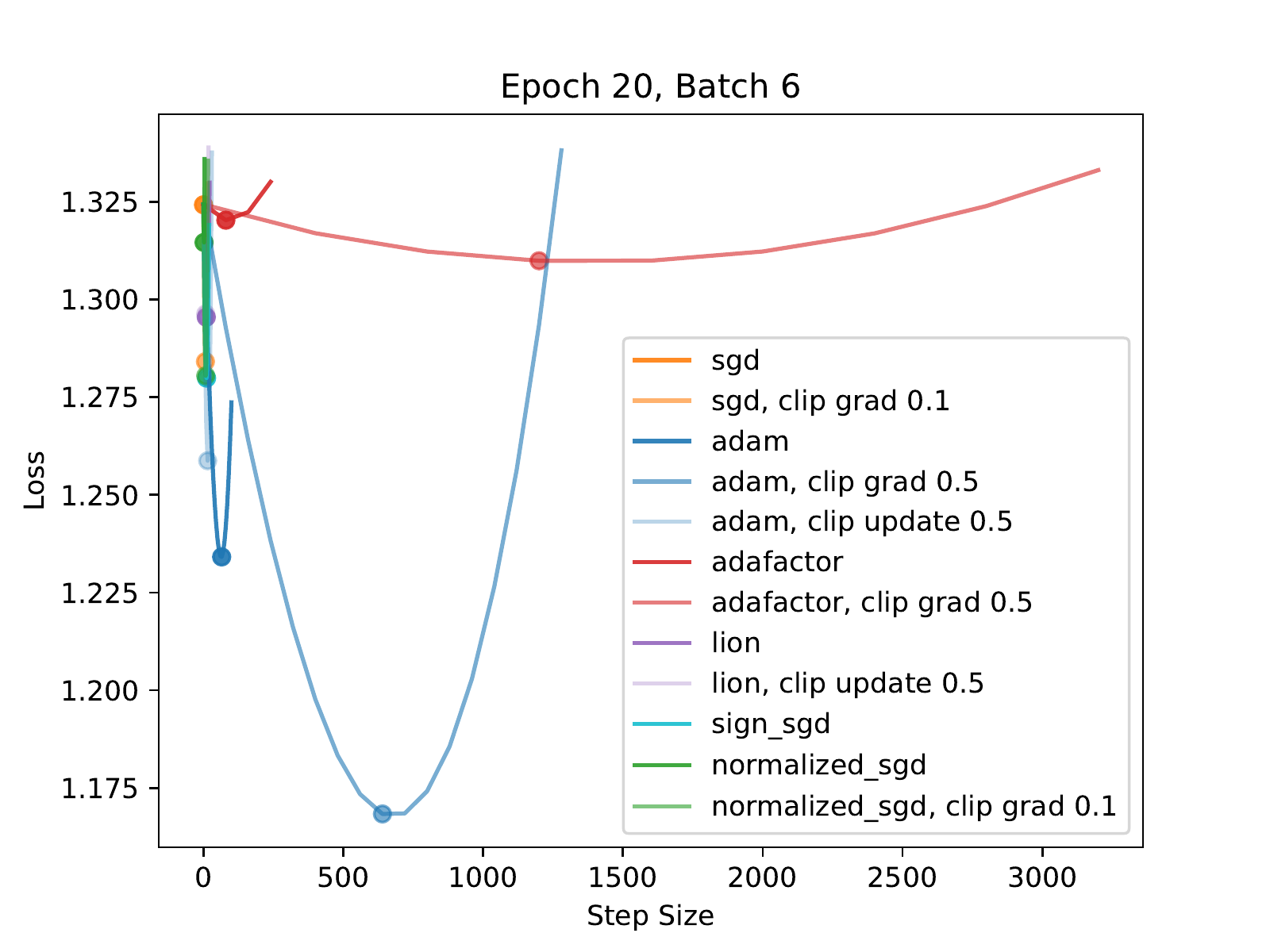}
        \caption{Experiment 1}
    \end{subfigure}
    \begin{subfigure}{\textwidth}\centering
        \includegraphics[width=0.32\textwidth]{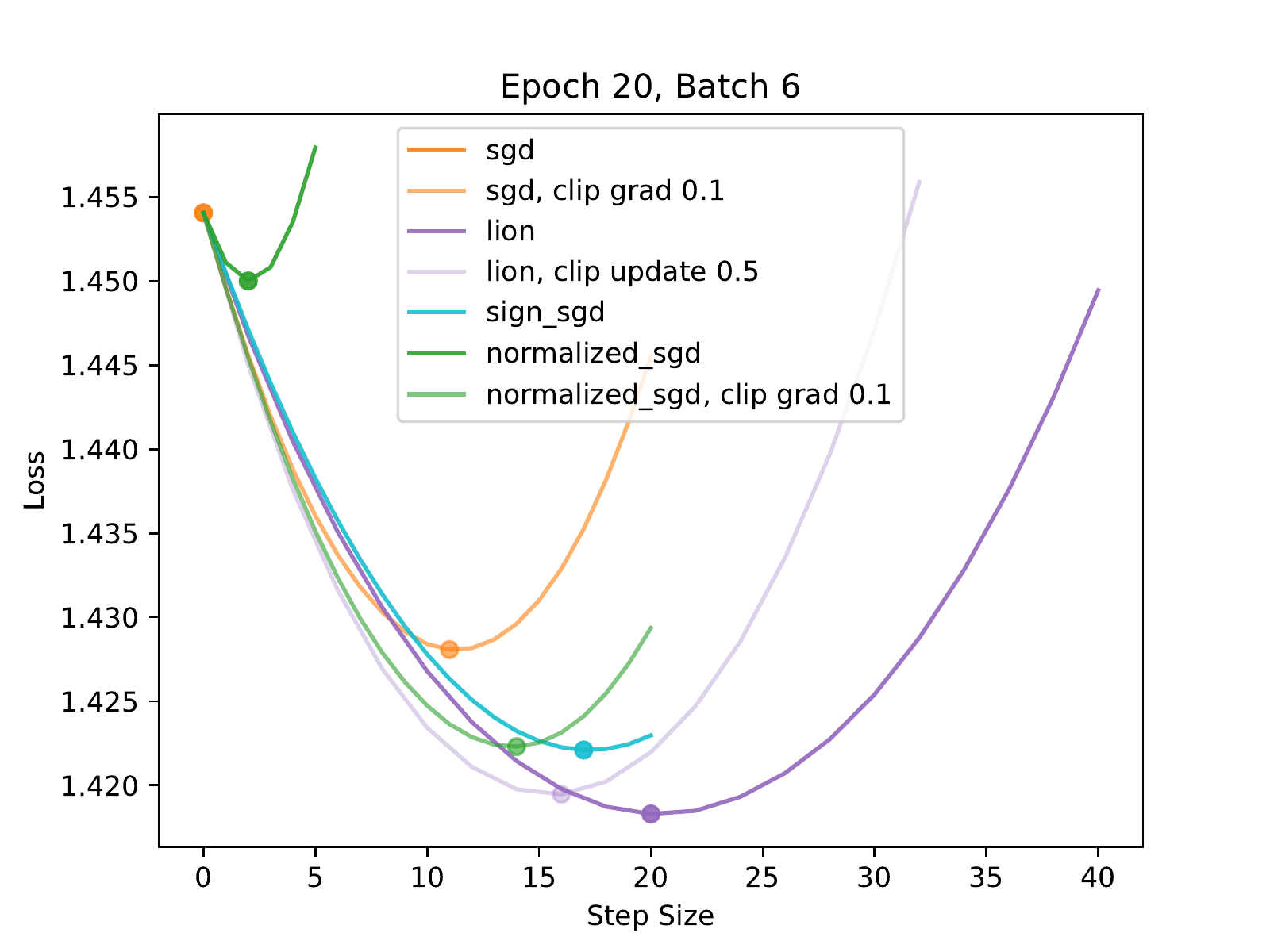}
        \includegraphics[width=0.32\textwidth]{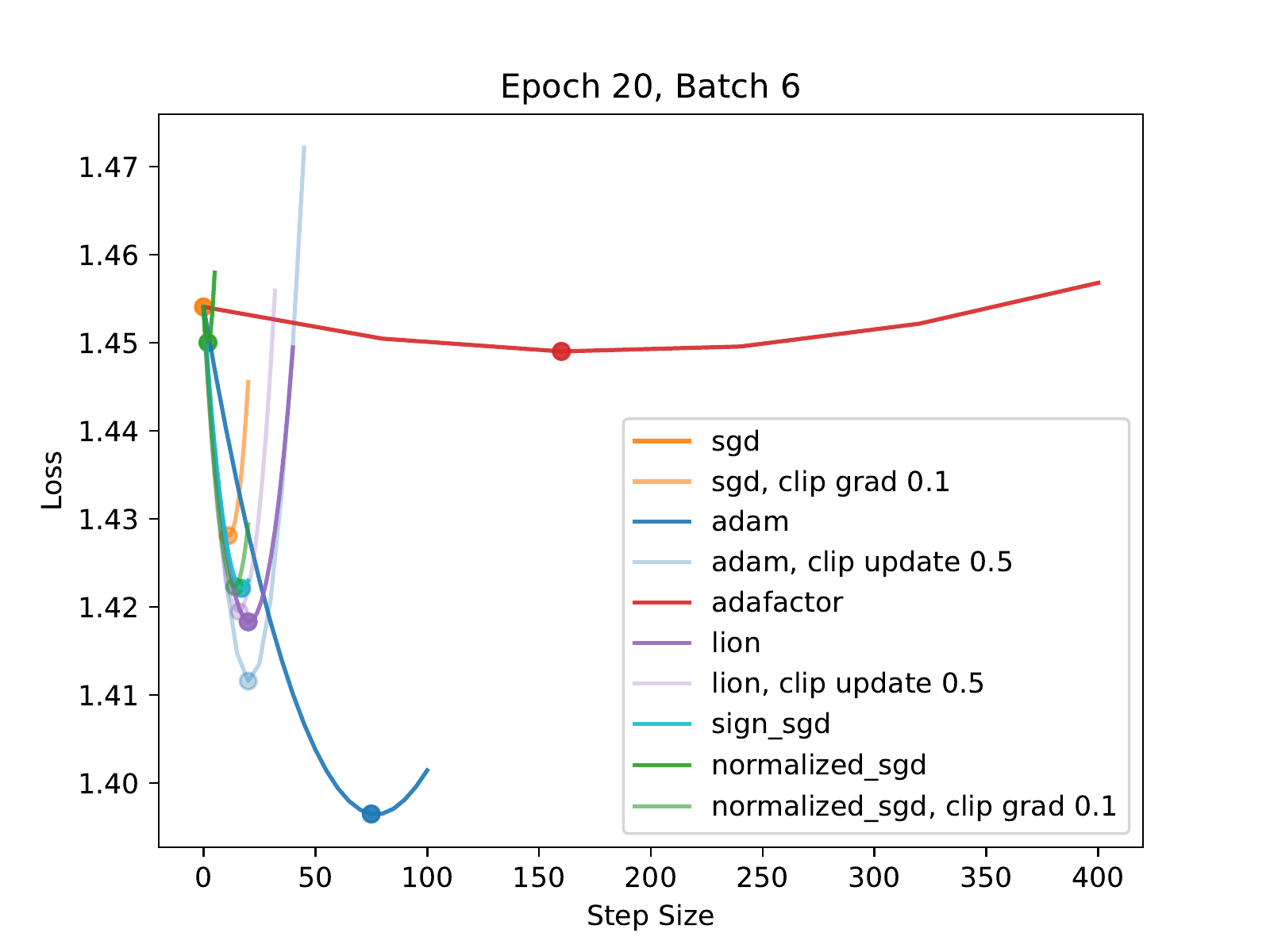}
        \includegraphics[width=0.32\textwidth]{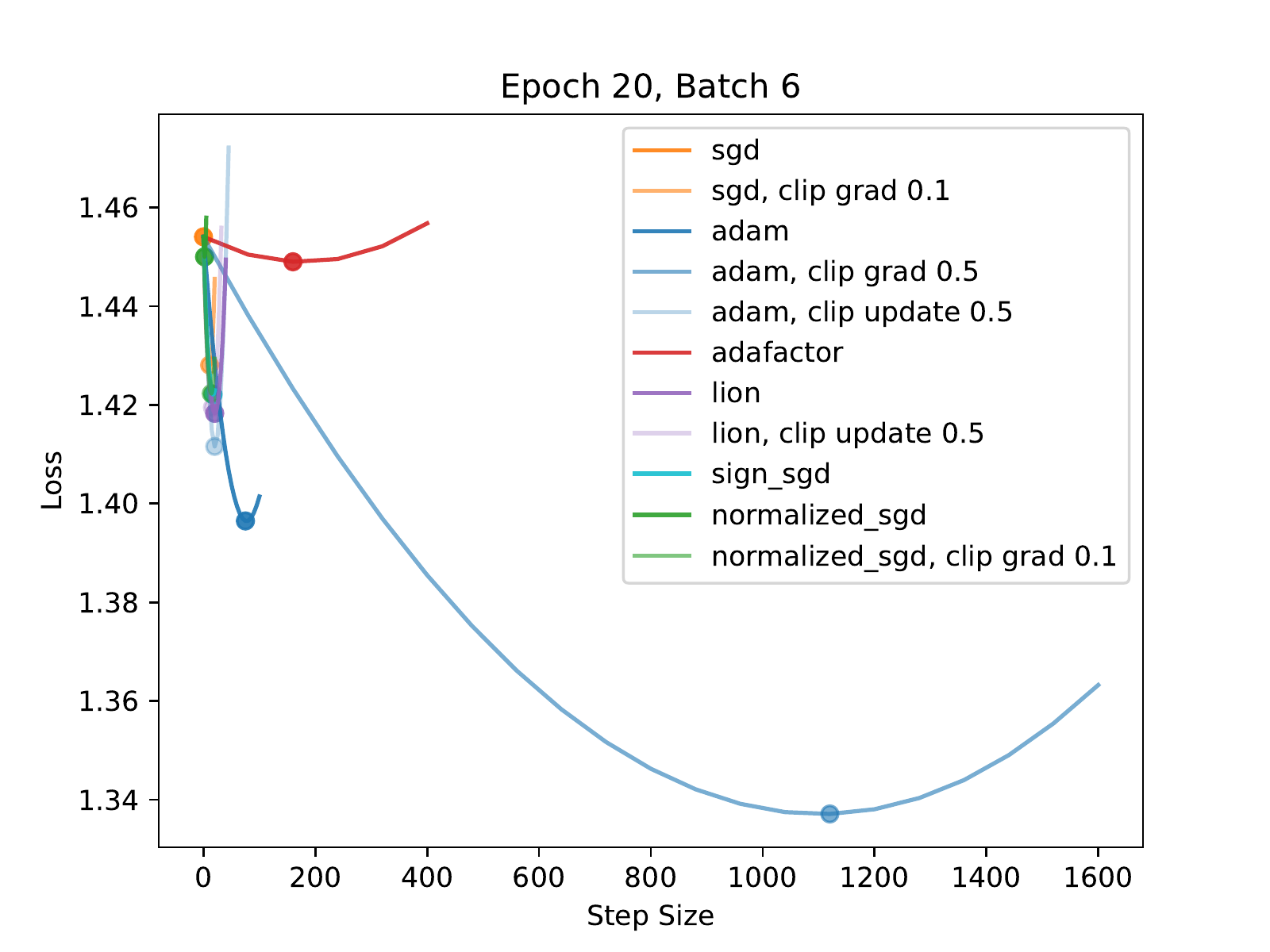}
        \caption{Experiment 2}
    \end{subfigure}
    \begin{subfigure}{\textwidth}\centering
        \includegraphics[width=0.32\textwidth]{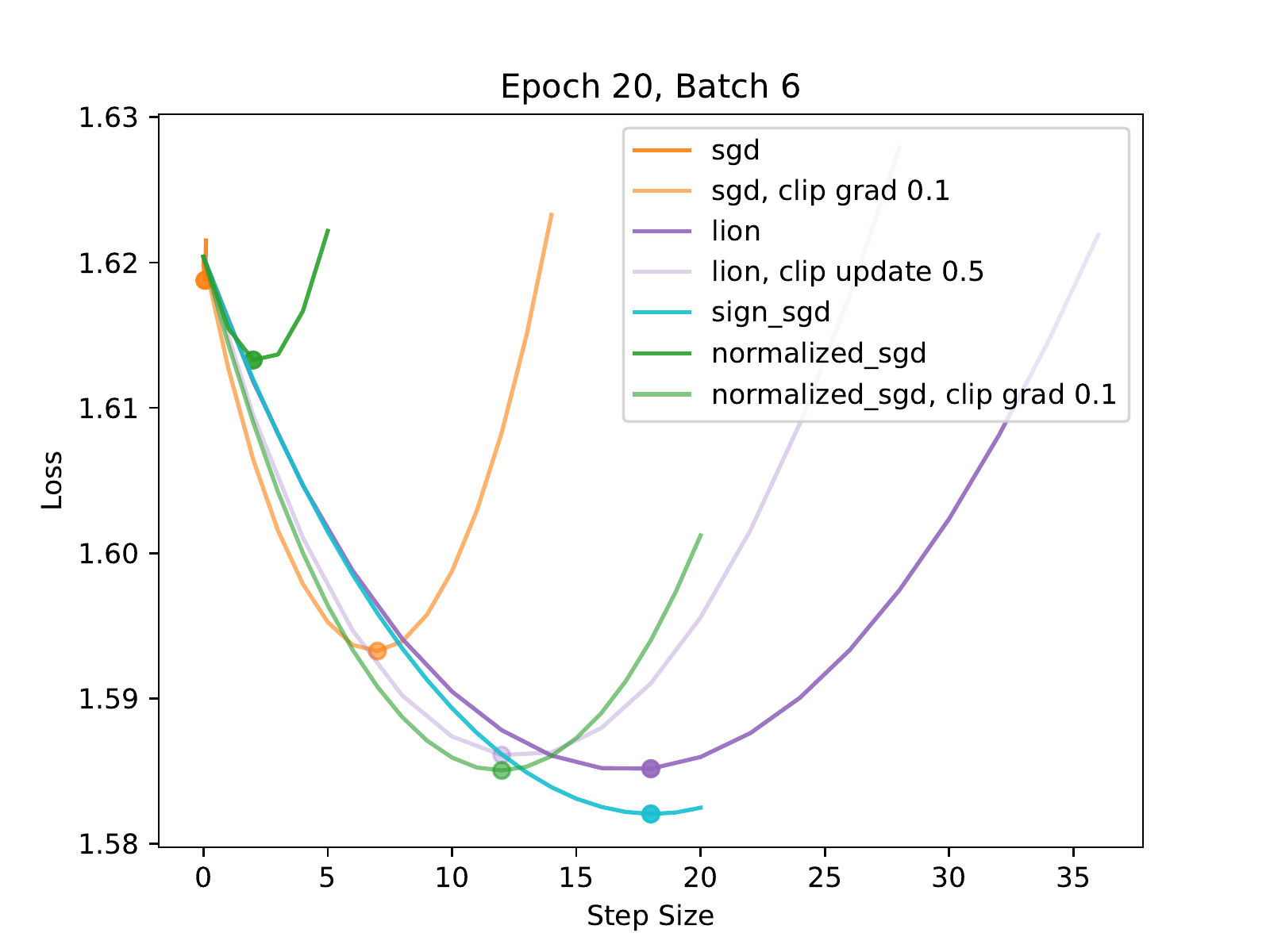}
        \includegraphics[width=0.32\textwidth]{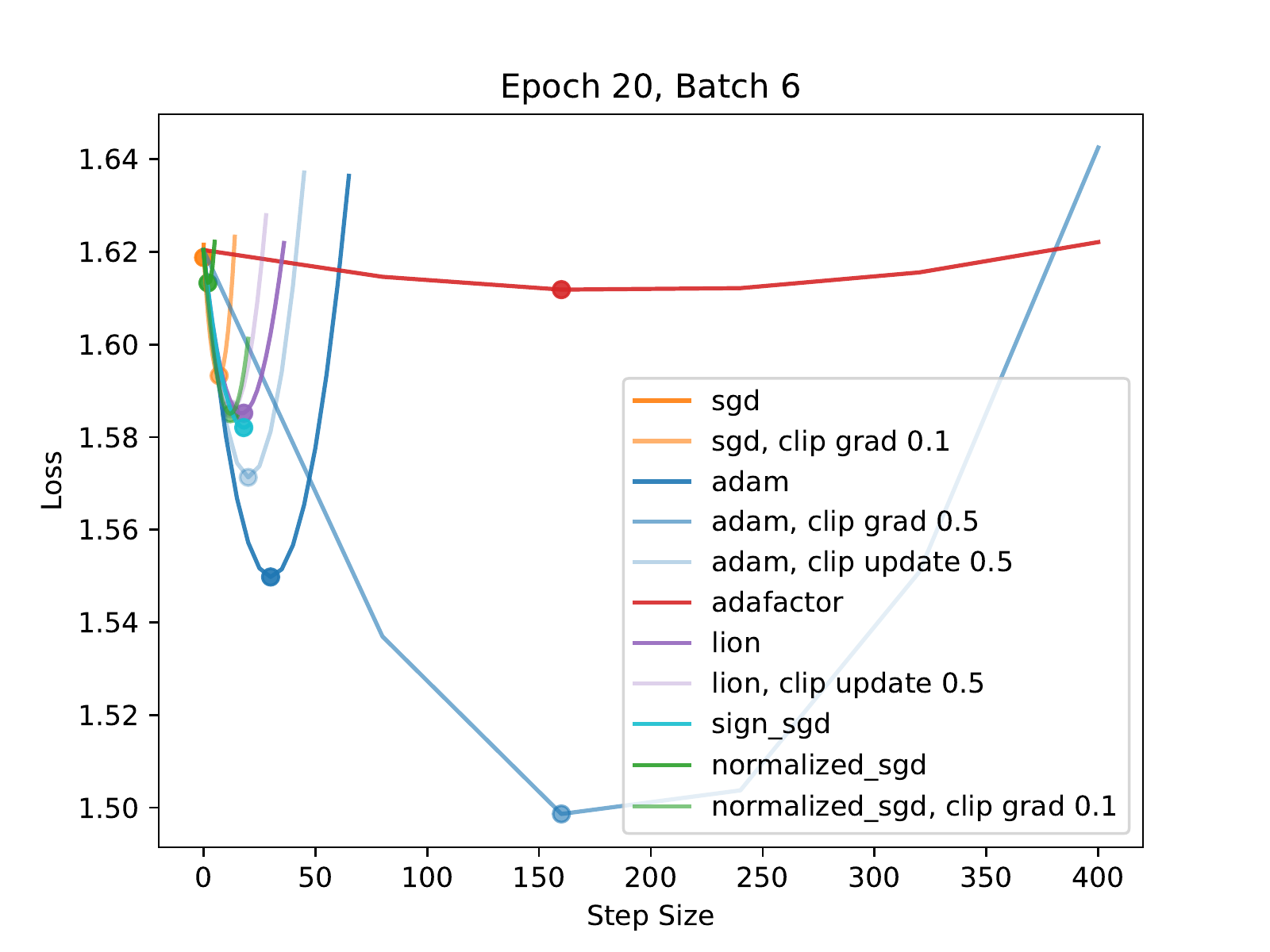}
        \includegraphics[width=0.32\textwidth]{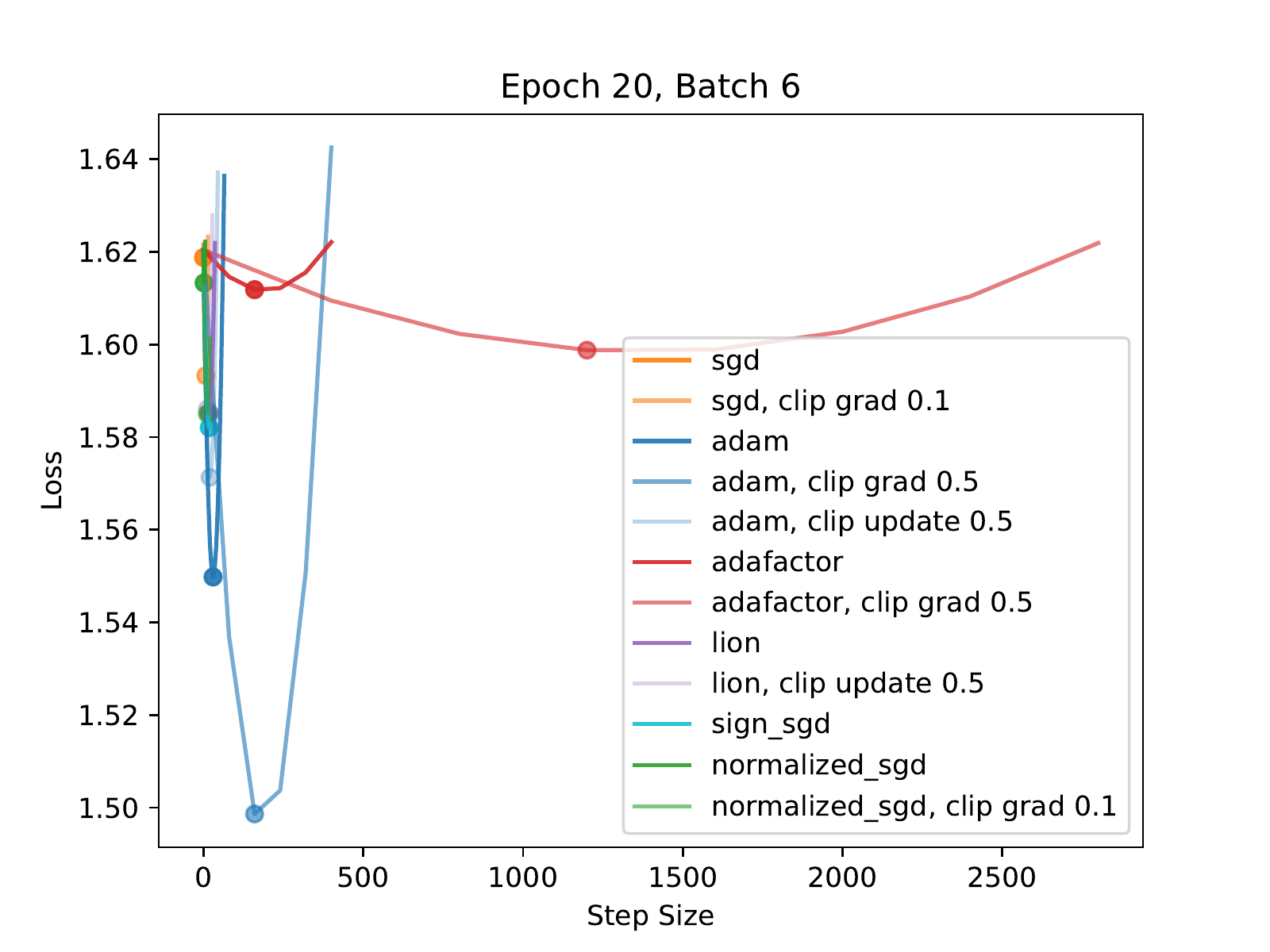}
        \caption{Experiment 3}
    \end{subfigure}
    \caption{Landscape visualization of machine translation in Adam trajectory at Epoch 20.}
\end{figure}

\begin{table}[h]
    \centering
    \begin{tabular}[h]{|l|l|l|l|l|l|}
        \hline
        \textbf{Epoch} & \textbf{Algorithm} & \textbf{Ratio 1} & \textbf{Ratio 2} & \textbf{Ratio 3} & \textbf{Mean}\\\hline
        \multirow{12}{*}{2} & sgd & $1.0$ & $1.0$ & $1.0$ & $1.0$ \\\cline{2-6}
        & sgd, clip grad 0.1 & $0.002899$ & $0.003474$ & $0.003715$ & $0.003363$ \\\cline{2-6}
        & adam & $0.003248$ & $0.003711$ & $0.00396$ & $0.00364$ \\\cline{2-6}
        & adam, clip grad 0.5 & $6.24\times 10^{-5}$ & $5.58\times 10^{-5}$ & $6.71\times 10^{-5}$ & $6.18\times 10^{-5}$ \\\cline{2-6}
        & adam, clip update 0.5 & $0.00284$ & $0.003167$ & $0.003346$ & $0.003118$ \\\cline{2-6}
        & adafactor & $4.1\times 10^{-7}$ & $3.42\times 10^{-7}$ & $5.83\times 10^{-7}$ & $4.45\times 10^{-7}$ \\\cline{2-6}
        & adafactor, clip grad 0.5 & $1.53\times 10^{-8}$ & $1.14\times 10^{-8}$ & $1.79\times 10^{-8}$ & $1.49\times 10^{-8}$ \\\cline{2-6}
        & lion & $0.001195$ & $0.001467$ & $0.001468$ & $0.001377$ \\\cline{2-6}
        & lion, clip update 0.5 & $0.00212$ & $0.002613$ & $0.002611$ & $0.002448$ \\\cline{2-6}
        & sign sgd & $0.000638$ & $0.000762$ & $0.000765$ & $0.000722$ \\\cline{2-6}
        & normalized sgd & $0.01143$ & $0.009482$ & $0.011515$ & $0.010809$ \\\cline{2-6}
        & normalized sgd, clip grad 0.1 & $0.001342$ & $0.001393$ & $0.00158$ & $0.001438$ \\\cline{1-6}
        \multirow{12}{*}{5} & sgd & $1.0$ & $1.0$ & $1.0$ & $1.0$ \\\cline{2-6}
        & sgd, clip grad 0.1 & $0.003316$ & $0.002051$ & $0.002137$ & $0.002501$ \\\cline{2-6}
        & adam & $0.000294$ & $0.000151$ & $0.000345$ & $0.000263$ \\\cline{2-6}
        & adam, clip grad 0.5 & $8.33\times 10^{-7}$ & $5.14\times 10^{-7}$ & $1.05\times 10^{-5}$ & $3.94\times 10^{-6}$ \\\cline{2-6}
        & adam, clip update 0.5 & $0.001268$ & $0.000636$ & $0.000565$ & $0.000823$ \\\cline{2-6}
        & adafactor & $2.5\times 10^{-7}$ & $1.46\times 10^{-7}$ & $1.16\times 10^{-7}$ & $1.7\times 10^{-7}$ \\\cline{2-6}
        & adafactor, clip grad 0.5 & $2.59\times 10^{-9}$ & $1.83\times 10^{-9}$ & $5.99\times 10^{-9}$ & $3.47\times 10^{-9}$ \\\cline{2-6}
        & lion & $0.000495$ & $0.000338$ & $0.000274$ & $0.000369$ \\\cline{2-6}
        & lion, clip update 0.5 & $0.000862$ & $0.000586$ & $0.000483$ & $0.000644$ \\\cline{2-6}
        & sign sgd & $0.000673$ & $0.000408$ & $0.000373$ & $0.000485$ \\\cline{2-6}
        & normalized sgd & $0.00463$ & $0.003971$ & $0.003808$ & $0.004136$ \\\cline{2-6}
        & normalized sgd, clip grad 0.1 & $0.001301$ & $0.000971$ & $0.000925$ & $0.001066$ \\\cline{1-6}
        \multirow{12}{*}{10} & sgd & $1.0$ & $1.0$ & $1.0$ & $1.0$ \\\cline{2-6}
        & sgd, clip grad 0.1 & $0.000693$ & $0.000805$ & $0.00096$ & $0.000819$ \\\cline{2-6}
        & adam & $1.44\times 10^{-5}$ & $1.94\times 10^{-5}$ & $8.62\times 10^{-5}$ & $4.0\times 10^{-5}$ \\\cline{2-6}
        & adam, clip grad 0.5 & $1.04\times 10^{-7}$ & $1.38\times 10^{-7}$ & $3.71\times 10^{-6}$ & $1.32\times 10^{-6}$ \\\cline{2-6}
        & adam, clip update 0.5 & $0.00013$ & $0.000168$ & $0.000148$ & $0.000149$ \\\cline{2-6}
        & adafactor & $1.13\times 10^{-7}$ & $7.67\times 10^{-8}$ & $1.33\times 10^{-7}$ & $1.08\times 10^{-7}$ \\\cline{2-6}
        & adafactor, clip grad 0.5 & $1.69\times 10^{-9}$ & $1.23\times 10^{-9}$ & $6.34\times 10^{-9}$ & $3.09\times 10^{-9}$ \\\cline{2-6}
        & lion & $0.000195$ & $0.000206$ & $0.000185$ & $0.000196$ \\\cline{2-6}
        & lion, clip update 0.5 & $0.000329$ & $0.000349$ & $0.000325$ & $0.000334$ \\\cline{2-6}
        & sign sgd & $0.000167$ & $0.000216$ & $0.00018$ & $0.000188$ \\\cline{2-6}
        & normalized sgd & $0.00202$ & $0.002424$ & $0.001556$ & $0.002$ \\\cline{2-6}
        & normalized sgd, clip grad 0.1 & $0.000385$ & $0.000471$ & $0.000414$ & $0.000423$ \\\cline{1-6}
        \multirow{12}{*}{20} & sgd & $1.0$ & $1.0$ & $1.0$ & $1.0$ \\\cline{2-6}
        & sgd, clip grad 0.1 & $0.00028$ & $0.000492$ & $0.001019$ & $0.000597$ \\\cline{2-6}
        & adam & $7.68\times 10^{-6}$ & $1.62\times 10^{-5}$ & $0.000118$ & $4.73\times 10^{-5}$ \\\cline{2-6}
        & adam, clip grad 0.5 & $8.58\times 10^{-8}$ & $1.64\times 10^{-7}$ & $4.54\times 10^{-6}$ & $1.6\times 10^{-6}$ \\\cline{2-6}
        & adam, clip update 0.5 & $9.61\times 10^{-5}$ & $0.000178$ & $0.00019$ & $0.000155$ \\\cline{2-6}
        & adafactor & $3.77\times 10^{-7}$ & $3.97\times 10^{-7}$ & $2.06\times 10^{-7}$ & $3.27\times 10^{-7}$ \\\cline{2-6}
        & adafactor, clip grad 0.5 & $4.89\times 10^{-9}$ & $5.12\times 10^{-9}$ & $8.16\times 10^{-9}$ & $6.06\times 10^{-9}$ \\\cline{2-6}
        & lion & $0.000222$ & $0.000296$ & $0.000244$ & $0.000254$ \\\cline{2-6}
        & lion, clip update 0.5 & $0.000369$ & $0.000491$ & $0.000427$ & $0.000429$ \\\cline{2-6}
        & sign sgd & $0.00014$ & $0.000223$ & $0.000197$ & $0.000187$ \\\cline{2-6}
        & normalized sgd & $0.002141$ & $0.002281$ & $0.001938$ & $0.00212$ \\\cline{2-6}
        & normalized sgd, clip grad 0.1 & $0.000247$ & $0.000393$ & $0.000405$ & $0.000348$ \\\cline{1-6}
    \end{tabular}
    \caption{Ratio of directional sharpness of optimization algorithms with respect to SGD on the machine translation task in Adam trajectory in 3 experiments.}
\end{table}

\begin{figure}[h]
    \centering
    \begin{subfigure}{0.45\textwidth}
        \centering
        \includegraphics[width=\textwidth]{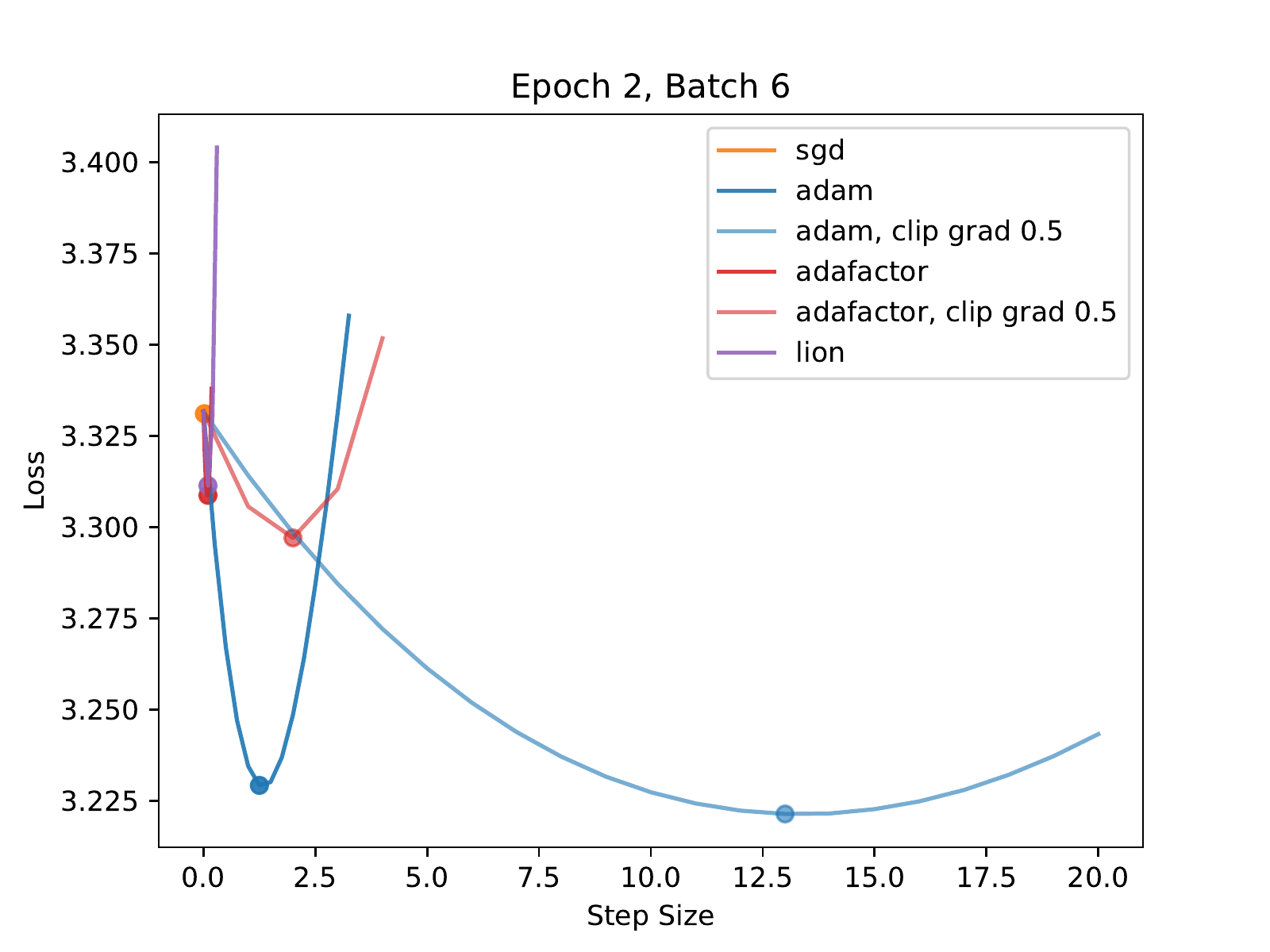}
        \caption{Epoch 2}
    \end{subfigure}
    \begin{subfigure}{0.45\textwidth}
        \centering
        \includegraphics[width=\textwidth]{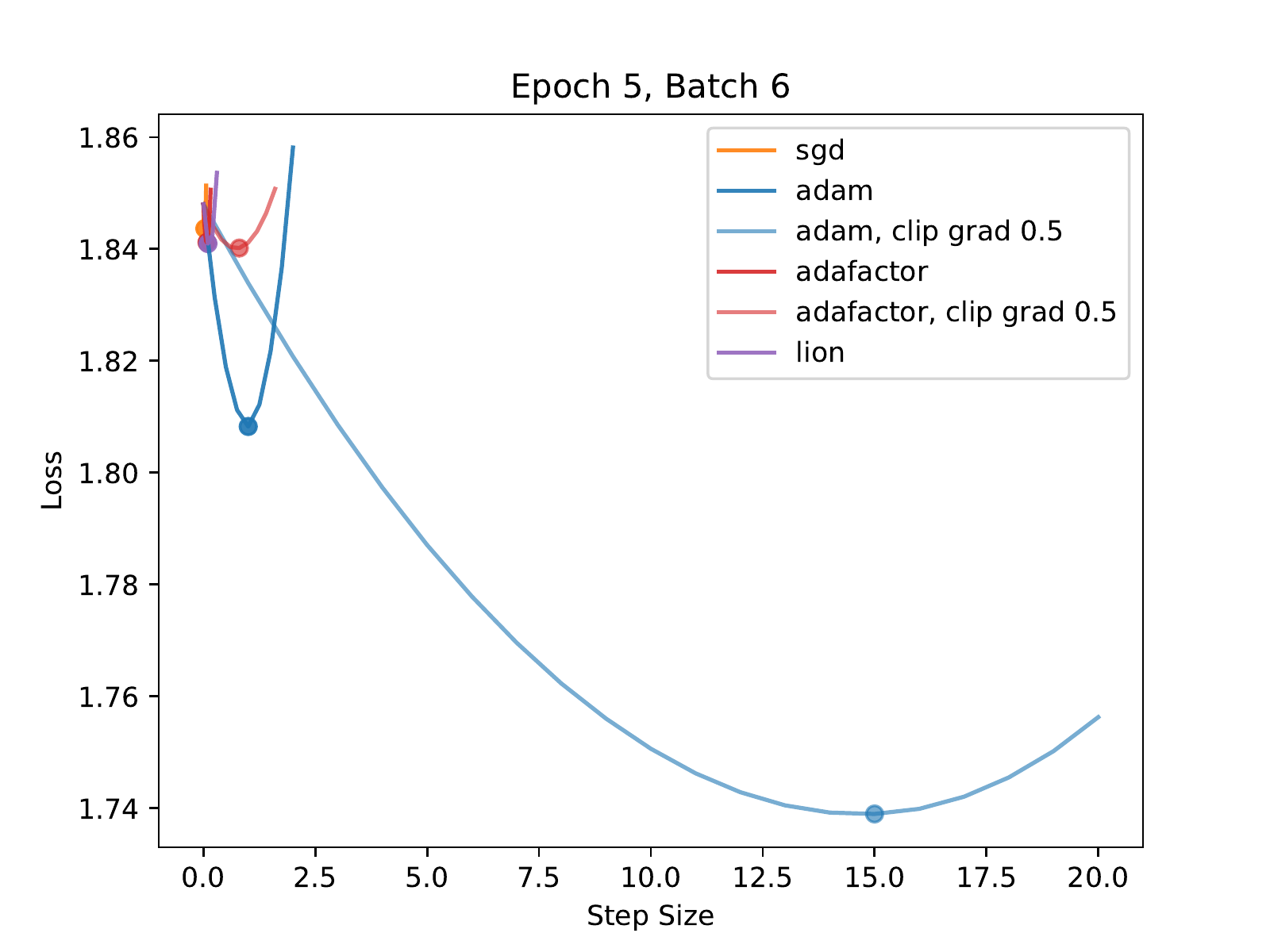}
        \caption{Epoch 5}
    \end{subfigure}
    \begin{subfigure}{0.45\textwidth}
        \centering
        \includegraphics[width=\textwidth]{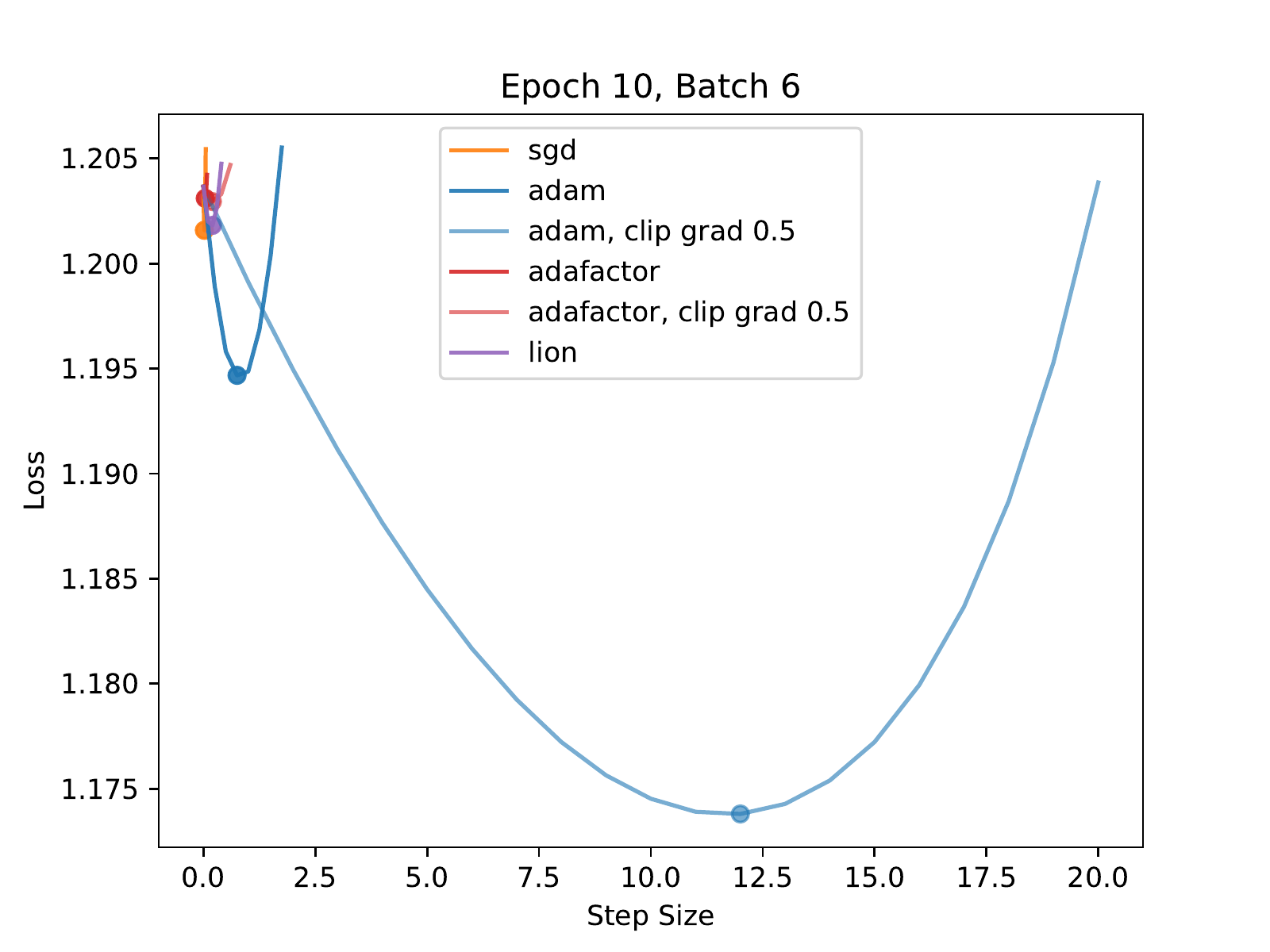}
        \caption{Epoch 10}
    \end{subfigure}
    \begin{subfigure}{0.45\textwidth}
        \centering
        \includegraphics[width=\textwidth]{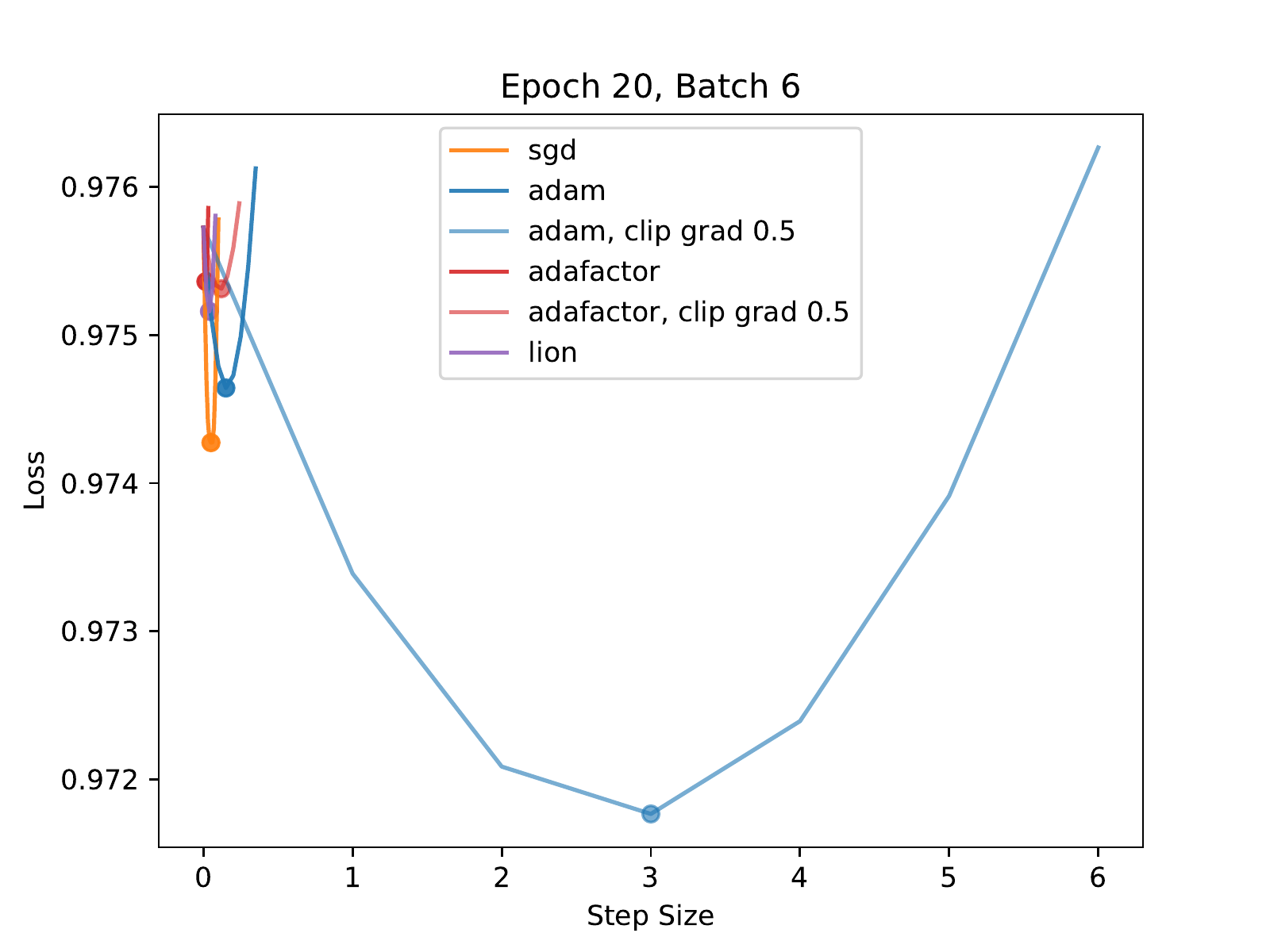}
        \caption{Epoch 20}
    \end{subfigure}
    \caption{Landscape visualization of autoregressive in Adam trajectory.}
\end{figure}

\clearpage
\begin{table}[h]
    \centering
    \begin{tabular}[h]{|l|l|l|}
        \hline
        \textbf{Epoch} & \textbf{Algorithm} & \textbf{Ratio}\\\hline
        \multirow{6}{*}{2} & sgd & $1.0$ \\\cline{2-3}
        & adam & $-0.000237$ \\\cline{2-3}
        & adam, clip grad 0.5 & $-7.8\times 10^{-6}$ \\\cline{2-3}
        & adafactor & $0.002897$ \\\cline{2-3}
        & adafactor, clip grad 0.5 & $4.66\times 10^{-7}$ \\\cline{2-3}
        & lion & $-0.003176$ \\\cline{1-3}
        \multirow{6}{*}{5} & sgd & $1.0$ \\\cline{2-3}
        & adam & $0.000166$ \\\cline{2-3}
        & adam, clip grad 0.5 & $-2.97\times 10^{-8}$ \\\cline{2-3}
        & adafactor & $-8.17\times 10^{-5}$ \\\cline{2-3}
        & adafactor, clip grad 0.5 & $6.31\times 10^{-7}$ \\\cline{2-3}
        & lion & $1.56\times 10^{-5}$ \\\cline{1-3}
        \multirow{6}{*}{10} & sgd & $1.0$ \\\cline{2-3}
        & adam & $-0.041286$ \\\cline{2-3}
        & adam, clip grad 0.5 & $-1.3\times 10^{-7}$ \\\cline{2-3}
        & adafactor & $-0.002071$ \\\cline{2-3}
        & adafactor, clip grad 0.5 & $2.04\times 10^{-6}$ \\\cline{2-3}
        & lion & $-0.001063$ \\\cline{1-3}
        \multirow{6}{*}{20} & sgd & $1.0$ \\\cline{2-3}
        & adam & $-0.000795$ \\\cline{2-3}
        & adam, clip grad 0.5 & $3.02\times 10^{-6}$ \\\cline{2-3}
        & adafactor & $0.010273$ \\\cline{2-3}
        & adafactor, clip grad 0.5 & $0.000208$ \\\cline{2-3}
        & lion & $0.000946$ \\\cline{1-3}
    \end{tabular}
    \caption{Ratio of directional sharpness of optimization algorithms with respect to SGD on the autoregressive task in Adam trajectory.}
\end{table}

\clearpage
\subsection{Discussion\label{sec:discussion}}
As we can observe, our observation is very coherent across different tasks, model architectures, iterations, and local geometry. The directional sharpness is relatively stable for the same task across iterations, and coordinate-wise clipping always improve the sharpness of the direction and find a better direction to optimize.

\textbf{Trade-off Between Directional Sharpness and Gradient Correlation.}
While we want the directional sharpness of our optimization algorithm to be small in order to decrease loss faster, having as small sharpness as possible does not necessarily lead to fast loss decrement.
Adafactor almost always has the lowest directional sharpness across all tasks, iterations, and local geometry, but Adafactor does not always find a good direction to optimize. In many cases, the loss does not decrease significantly even for the optimal step size, and the direction can be even worse than SGD.
This shows that merely minimizing the directional sharpness is not enough for an optimization algorithm to work well.
As discussed in~\Cref{sec:clip}, gradient correlation is also important in the convergence of optimization algorithms.
However, we can conclude that high sharpness will lead to bad performance, as demonstrated by the performance of SGD, even if SGD has good gradient correlation.

\textbf{Effect of Trajectory.}
It is well known that different optimization algorithms can follow different trajectory and converge to different in deep learning.
\cite{jiang2022does} also point out the impact of local geometry of adaptive algorithms such that they implicitly select the trajectory with good smoothness.
For SGD on machine translation, the landscape in the direction found by different optimization algorithms were similar in SGD geometry at all epochs.
For Adam, the geometry is similar to SGD in the first few iterations, but changes significantly after more itereations.
Landscape visualizations show that Adafactor performs well in SGD trajectory but not Adam trajectory on the machine translation task.
This shows that different optimization algorithms has local geometry with different properties. The effect of trajectory is therefore an interesting problem to study.
However, we point out that almost in all cases, Adam has good performance and significantly outperforms SGD, so trajectory is not necessarily related to the explanation for Adam's excellent performance in practice.

\clearpage
\section{Experiment on the Smoothness of Hessian\label{sec:hessian}}
\subsection{Experiment Setup and Hyperparameters}
We use the standard BERT model~\cite{devlin2019bert} for binary text classification task trained with Adam on the IMDb dataset.
The reason we use binary classification is that the output dimension of binary classification is 2, so the Gauss-Newton Hessian approximation of a batch of size $k$ has rank at most $2k$, so we can have a larger batch size as compared to the cases where the logits have high dimension.
Similarly, due to space constraints, we cannot use a large batch.

We use the PyTorch framework for all of our experiments. We use the Huggingface implementation of BERT. We use pretrained \texttt{BertForSequentialClassification} model. We use learning rate $10^{-4}$ for both experiments and batch size $25$.

\subsection{Approximation of Hessian with Gauss-Newton Matrix}
We use the Gauss-Newton matrix~\cite{martens2016second,bottou2018optimization} to approximate the Hessian of the neural network. Let $h : \R^d \times \R^p \to \R^k$ be a neural network that maps input data and parameters to a prediction, $\ell : \R^k \to \R$ be the cross entropy loss function, and $f_i(\theta) = \ell(h(x_i; \theta), y_i)$, then the Gauss-Newton Hessian approximation with respect to one data point is
\[
    \nabla^2 f_i(\theta) \approx \nabla_\theta h(x_i; \theta) \nabla^2 \ell(z_i; y_i) \nabla_\theta h(x_i; \theta)^\top.
\]
where $\nabla_\theta h(x_i; \theta) \in \R^{k \times p}$ is the Jacobian of the logits, and $\nabla^2 \ell(z_i; y_i) = \mathrm{diag}(p_i) - p_ip_i^\top$, where $p_i = \frac{\exp(z_i)}{\sum_j\exp([z_i]_j)}$. See Appendix B of~\cite{cohen2021gradient} for additional details. In the case of $k = 2$ for our experiment, we can then compute the exact formula for $\nabla^2 \ell(z_i; y_i)$ as
\begin{align*}
    \nabla^2 \ell(z_i; y_i)
    &= \frac{\exp([z_i]_1)\exp([z_i]_2)}{(\exp([z_i]_1) + \exp([z_i]_2))^2} \begin{bmatrix}
        1 & -1\\
        -1 & 1
    \end{bmatrix}
\end{align*}
so it is positive semidefinite and one square root of it is
\[
    (\nabla^2 \ell(z_i; y_i))^{1/2} = \frac{\sqrt{\exp([z_i]_1) \exp([z_i]_2)}}{\exp([z_i]_1) + \exp([z_i]_2)}\begin{bmatrix}
        -1 \\ 1
    \end{bmatrix}.
\]
Hence let $g_i = \nabla_\theta h(x_i; \theta) (\nabla^2 \ell(z_i; y_i))^{1/2}$, then $\nabla^2 f_i(\theta) \approx g_i g_i^\top$.
We notice that this is essentially an outer product between the gradient of the neural network with respect to its parameters and the gradient itself, so this matrix has rank 1. Then, for any batch $S \subseteq [n]$, the Hessian is
\[
    \nabla^2 f_{S}(\theta) \approx \frac{1}{|S|}\sum_{i \in S} g_i g_i^\top =: \frac{1}{|S|} G_S G_S^\top.
\]
Then, by $\|G_S G_S^\top\|_2 = \|G_S\|_2^2$, it suffices to compute the spectral norm of a $p \times |S|$ matrix.

\clearpage
\subsection{Results}
The experimental result is shown in \Cref{tab:rs}. We remove all the coordinates that the row norm is at least 4 times the mean of the row norms. This removes $1\%$ to $2.5\%$ of the coordinates, and as a consequence the smoothness of the function is 2 to 3 times better compared to the smoothness of the full function, and this is the case for all batches and epochs. This shows that our definition of robust smoothness is reasonable, that it is indeed possible to optimize most of the coordinates under the robust smoothness setting.
\begin{table}[h]
    \centering
    \begin{tabular}{|l|l|l|l|l|}
        \hline
        \textbf{Epoch} & $L$ & $\ell$ & $\frac{L}{\ell}$ & $\varepsilon$ \\
        \hline
        0 & $3582.06$ & $1229.51$ & $2.91$ & $0.0159$ \\
        1 & $2892.63$ & $899.70$ & $3.21$ & $0.0114$ \\
        2 & $10142.38$ & $4245.65$ & $2.39$ & $0.0235$ \\
        7 & $15094.21$ & $7118.42$ & $2.12$ & $0.0183$\\
        8 & $27881.08$ & $12477.13$ & $2.23$ & $0.0315$ \\
        9 & $8046.26$ & $3471.40$ & $2.32$ & $0.0226$ \\
        \hline
    \end{tabular}
    \caption{Result for smoothness experiment, where $L$ is the spectral norm of the full Hessian, $\varepsilon$ is the fraction removed, and $\ell$ is the spectral norm of the remaining Hessian.}
    \label{tab:rs}
\end{table}

\clearpage
\section{Directional Sharpness of ResNet\label{sec:resnet}}
We do an additional simple experiment with ResNet~\cite{he2016deep} on the CIFAR-10 dataset~\cite{krizhevsky2009learning} that shows the properties of the directional sharpness that we discovered are related to the transformer architecture.
We use the ResNet-152 architecture with batch sizes of 1000.
We compute the directional sharpness in the same setting as~\Cref{sec:exp_appendix}.
The results are shown in~\Cref{tab:resnet}.
As we can see, the directional sharpness of adaptive algorithms can be much worse than SGD.
This shows that the property we discovered is related to the transformer architecture, and does not hold for ResNet.

\begin{table}[h]
    \centering
    \begin{tabular}{|l|l|l|}
        \hline
        \textbf{Epoch} & \textbf{Algorithm} & \textbf{Ratio}\\\hline
        \multirow{10}{*}{2} & sgd & $1.0$ \\\cline{2-3}
        & sgd, clip grad 0.1 & $0.208685$ \\\cline{2-3}
        & adam & $4.812922$ \\\cline{2-3}
        & adam, clip grad 0.1 & $1.215817$ \\\cline{2-3}
        & adafactor & $18.85215$ \\\cline{2-3}
        & adafactor, clip grad 0.1 & $5.098444$ \\\cline{2-3}
        & lion & $2.106887$ \\\cline{2-3}
        & sign sgd & $0.605515$ \\\cline{2-3}
        & normalized sgd & $1.472489$ \\\cline{2-3}
        & normalized sgd, clip grad 0.1 & $0.266836$ \\\cline{1-3}
        \multirow{10}{*}{5} & sgd & $1.0$ \\\cline{2-3}
        & sgd, clip grad 0.1 & $0.097681$ \\\cline{2-3}
        & adam & $1.181251$ \\\cline{2-3}
        & adam, clip grad 0.1 & $0.30124$ \\\cline{2-3}
        & adafactor & $1.919904$ \\\cline{2-3}
        & adafactor, clip grad 0.1 & $0.43885$ \\\cline{2-3}
        & lion & $0.089412$ \\\cline{2-3}
        & sign sgd & $0.133457$ \\\cline{2-3}
        & normalized sgd & $1.323987$ \\\cline{2-3}
        & normalized sgd, clip grad 0.1 & $0.172629$ \\\cline{1-3}
        \multirow{10}{*}{10} & sgd & $1.0$ \\\cline{2-3}
        & sgd, clip grad 0.1 & $0.283633$ \\\cline{2-3}
        & adam & $0.377285$ \\\cline{2-3}
        & adam, clip grad 0.1 & $0.109241$ \\\cline{2-3}
        & adafactor & $0.981839$ \\\cline{2-3}
        & adafactor, clip grad 0.1 & $0.283508$ \\\cline{2-3}
        & lion & $0.069054$ \\\cline{2-3}
        & sign sgd & $0.078127$ \\\cline{2-3}
        & normalized sgd & $0.755517$ \\\cline{2-3}
        & normalized sgd, clip grad 0.1 & $0.317746$ \\\cline{1-3}
    \end{tabular}
    \caption{Directional sharpness of ResNet on the image classification task in SGD trajectory. The directional sharpness of adaptive algorithms can be higher than SGD.}
    \label{tab:resnet}
\end{table}

\end{document}